\documentclass{article}

\usepackage{microtype}
\usepackage{graphicx}
\usepackage{subfigure}
\usepackage{booktabs} 

\usepackage{hyperref}

\usepackage[accepted]{icml2021}

\icmltitlerunning{MARINA: Faster Non-Convex Distributed Learning with Compression}

\usepackage{amssymb, amsmath, amsthm, latexsym}
\usepackage{url}
\usepackage{tabularx}
\usepackage{paralist}
\usepackage{mathtools}
\usepackage{nicefrac}
\usepackage{makecell}

\usepackage[flushleft]{threeparttable} 

\usepackage{caption}
\usepackage{multirow}
\usepackage{colortbl}
\definecolor{bgcolor}{rgb}{0.8,1,1}
\definecolor{bgcolor2}{rgb}{0.8,1,0.8}

\renewcommand{\algorithmicrequire}{\textbf{Input:}}
\renewcommand{\algorithmicensure}{\textbf{Output:}}

\usepackage[dvipsnames]{xcolor}
\usepackage{pifont}

\usepackage{subfigure}

\newcommand{\R}{\mathbb{R}}

\def\<#1,#2>{\left\langle #1,#2\right\rangle}

\newtheorem{lemma}{Lemma}[section]
\newtheorem{theorem}{Theorem}[section]
\newtheorem{definition}{Definition}[section]

\newtheorem{assumption}{Assumption}[section]
\newtheorem{corollary}{Corollary}[section]
\newtheorem{remark}{Remark}[section]

\usepackage{xspace}

\newcommand{\squeeze}{\textstyle} 

\newcommand{\algname}[1]{{\sf \small #1}\xspace}
\newcommand{\algnamex}[1]{{\sf #1}\xspace}

\usepackage[colorinlistoftodos,bordercolor=orange,backgroundcolor=orange!20,linecolor=orange,textsize=scriptsize]{todonotes}

\newcommand{\cD}{{\cal D}}

\newcommand{\cL}{{\cal L}}

\newcommand{\cO}{{\cal O}}

\newcommand{\cQ}{{\cal Q}}

\newcommand{\EE}{\mathbf{E}}

\usepackage{hyperref}
\graphicspath{{plots/}}

\usepackage{makecell}

\usepackage{accents}
\newlength{\dhatheight}

\begin{document}

\twocolumn[
\icmltitle{MARINA: Faster Non-Convex Distributed Learning with Compression}



\icmlsetsymbol{equal}{*}

\begin{icmlauthorlist}
\icmlauthor{Eduard Gorbunov}{mipt,yandex,kaust}
\icmlauthor{Konstantin Burlachenko}{kaust}
\icmlauthor{Zhize Li}{kaust}
\icmlauthor{Peter Richt\'{a}rik}{kaust}
\end{icmlauthorlist}

\icmlaffiliation{mipt}{Moscow Institute of Physics and Technology, Moscow, Russia}
\icmlaffiliation{yandex}{Yandex, Moscow, Russia}
\icmlaffiliation{kaust}{King Abdullah University of Science and Technology, Thuwal, Saudi Arabia}

\icmlcorrespondingauthor{Eduard Gorbunov}{eduard.gorbunov@phystech.edu}
\icmlcorrespondingauthor{Peter Richt\'{a}rik}{peter.richtarik@kaust.edu.sa}

\icmlkeywords{Non-Convex Optimization, Distributed Learning}

\vskip 0.3in
]



\printAffiliationsAndNotice{}  

\begin{abstract}
We develop and analyze \algname{MARINA}: a new communication efficient method for non-convex distributed learning over heterogeneous datasets.   \algname{MARINA} employs a novel communication compression strategy based on the compression of gradient differences that is reminiscent of but different from the strategy employed in the \algname{DIANA} method of Mishchenko et al. (2019). Unlike  virtually all competing distributed first-order methods,  including \algname{DIANA}, ours is based on a carefully designed {\em biased} gradient estimator, which is the key to its superior theoretical and practical performance. The communication complexity bounds we prove for \algname{MARINA} are evidently better than those of all previous first-order methods. Further, we develop and analyze two variants of \algname{MARINA}: \algname{VR-MARINA} and \algname{PP-MARINA}. The first method  is designed for the case when the local loss functions owned by clients are either of a finite sum or of an expectation form, and the second method allows for a partial participation of clients -- a feature important in federated learning. All our methods are superior to previous state-of-the-art methods in terms of oracle/communication complexity. Finally, we provide a convergence analysis of all methods  for problems satisfying the Polyak-{\L}ojasiewicz condition.
\end{abstract}

\section{Introduction}\label{sec:intro}
Non-convex optimization problems appear in various applications of machine learning, such as training deep neural networks \cite{goodfellow2016deep} and matrix completion and recovery \cite{ma2018implicit, bhojanapalli2016dropping}. Because of their practical importance, these problems gained much attention in recent years, which led to a rapid development of new efficient methods for non-convex optimization problems \cite{danilova2020recent}, and especially the training of deep learning models \cite{sun2019optimization}.

Training deep neural networks is notoriously computationally challenging and time-consuming. In the quest to improve the generalization performance of modern deep learning models, practitioners resort to using increasingly larger datasets in the training process, and to support such workloads, it is imperative to use advanced parallel and distributed hardware, systems, and algorithms. Distributed computing is often necessitated by the desire to train models from data naturally distributed across several edge devices, as is the case in federated learning \cite{FEDLEARN, FL2017-AISTATS}. However, even when this is not the case, distributed methods are often very efficient at reducing the training time \cite{goyal2017accurate,You2020Large}. Due to these and other reasons,   distributed optimization has gained immense popularity in recent years.

However, distributed methods almost invariably suffer from the so-called \textit{communication bottleneck}: the communication cost of information necessary for the workers  to jointly solve the problem  at hand is often very high, and depending on the particular compute architecture, workload, and algorithm used, it can be orders of magnitude higher than the computation cost. A popular technique for resolving this issue  is \textit{communication compression} \cite{seide20141, FEDLEARN, Suresh2017}, which is based on applying a lossy transformation/compression to the models, gradients, or tensors to be sent over the network to save on communication.  Since applying a lossy compression  generally decreases the utility of the exchanged messages, such an approach will typically lead to an increase in the number of communications, and the overall usefulness of this technique manifests itself in situations where the communication savings are larger compared to the increased need for the number of communication rounds \cite{Cnat}. 

The optimization and machine learning communities have exerted considerable effort in recent years to design  distributed methods  supporting compressed communication. From many methods proposed, we emphasize \algname{VR-DIANA} \cite{horvath2019stochastic},  \algname{FedCOMGATE} \cite{haddadpour2020federated}, and \algname{FedSTEPH} \cite{das2020improved} because these papers
contain the state-of-the-art results in the setup when the local loss functions can be arbitrary heterogeneous.


\begin{table*}[t]
    \centering
    \scriptsize
	\caption{\small Summary of the state-of-the-art results for finding an \textbf{$\varepsilon$-stationary point} for the problem \eqref{eq:main_problem}, i.e., such a point $\hat x$ that $\EE\left[\|\nabla f(\hat x)\|^2\right] \le \varepsilon^2$. Dependences on the numerical constants, ``quality'' of the starting point, and smoothness constants are omitted in the complexity bounds.  Abbreviations: ``PP'' = partial participation; ``Communication complexity'' = the number of communications rounds needed to find an $\varepsilon$-stationary point; ``Oracle complexity'' = the number of (stochastic) first-order oracle calls needed to find an $\varepsilon$-stationary point. Notation: $\omega$ = the quantization parameter (see Def.~\ref{def:quantization}); $n$ = the number of nodes; $m$ = the size of the local dataset; $r$ = (expected) number of clients sampled at each iteration; $b'$ = the batchsize for \algnamex{VR-MARINA} at the iterations with compressed communication. To simplify the bounds, we assume that the expected density $\zeta_{\cQ}$ of the quantization operator $\cQ$ (see Def.~\ref{def:quantization}) satisfies $\omega+1 = \Theta(\nicefrac{d}{\zeta_{\cQ}})$ (e.g., this holds for RandK and $\ell_2$-quantization, see \cite{beznosikov2020biased}). We notice that \cite{haddadpour2020federated} and \cite{das2020improved} contain also better rates under different assumptions on clients' similarity.}
    \label{tab:comparison}    
    \begin{threeparttable}
    \begin{tabular}{|c|c|c c c|}
         \hline
         Setup & Method & Citation & Communication Complexity & Oracle Complexity\\ 
\hline\hline
    \multirow{5}{0.7cm}{\centering \eqref{eq:main_problem}} &\algnamex{DIANA} &\makecell{\cite{mishchenko2019distributed}\\\cite{horvath2019stochastic}\\ \cite{li2020unified}} & $\frac{1+\left(1 + \omega\right)\sqrt{\nicefrac{\omega}{n}}}{\varepsilon^2}$ &  $\frac{1+\left(1 + \omega\right)\sqrt{\nicefrac{\omega}{n}}}{\varepsilon^2}$ \\
    & \algnamex{FedCOMGATE}\tnote{\color{red} (1)} & \cite{haddadpour2020federated} & $\frac{1 + \omega}{\varepsilon^2}$ & $\frac{1+\omega}{n\varepsilon^4}$\\
    & \algnamex{FedSTEPH}, $r=n$ & \cite{das2020improved} & $\frac{1 + \nicefrac{\omega}{n}}{\varepsilon^4}$ & $\frac{1 + \nicefrac{\omega}{n}}{\varepsilon^4}$ \\
    & \cellcolor{bgcolor2}\algnamex{MARINA} (Alg.~\ref{alg:marina}) &\cellcolor{bgcolor2} Thm.~\ref{thm:main_result_non_cvx} \& Cor.~\ref{cor:main_result_non_cvx} (NEW) &\cellcolor{bgcolor2} $\frac{1 + \nicefrac{\omega}{\sqrt{n}}}{\varepsilon^2}$ &\cellcolor{bgcolor2} $\frac{1 + \nicefrac{\omega}{\sqrt{n}}}{\varepsilon^2}$\\    
    \hline\hline
    \multirow{4}{0.7cm}{\centering\eqref{eq:main_problem}+\eqref{eq:f_i_finite_sum}}& \algnamex{DIANA} &\cite{li2020unified} & $\frac{1+\left(1 + \omega\right)\sqrt{\nicefrac{\omega}{n}}}{\varepsilon^2} + \frac{1+\omega}{n\varepsilon^4}$ &  $\frac{1+\left(1 + \omega\right)\sqrt{\nicefrac{\omega}{n}}}{\varepsilon^2} + \frac{1+\omega}{n\varepsilon^4}$\\
    & \algnamex{VR-DIANA} & \cite{horvath2019stochastic} & $\frac{\left(m^{\nicefrac{2}{3}} + \omega\right)\sqrt{1+\nicefrac{\omega}{n}}}{\varepsilon^2}$ &  $\frac{\left(m^{\nicefrac{2}{3}} + \omega\right)\sqrt{1+\nicefrac{\omega}{n}}}{\varepsilon^2}$\\
    &\cellcolor{bgcolor2} \algnamex{VR-MARINA} (Alg.~\ref{alg:vr_marina}), $b'=1$\tnote{\color{red}(2)} &\cellcolor{bgcolor2} Thm.~\ref{thm:main_result_non_cvx_finite_sums} \& Cor.~\ref{cor:main_result_non_cvx_finite_sums} (NEW) &\cellcolor{bgcolor2} $\frac{1 + \nicefrac{\max\left\{\omega,\sqrt{(1+\omega)m}\right\}}{\sqrt{n}}}{\varepsilon^2}$ &\cellcolor{bgcolor2} $\frac{1 + \nicefrac{\max\left\{\omega,\sqrt{(1+\omega)m}\right\}}{\sqrt{n}}}{\varepsilon^2}$\\
    \hline\hline
    \multirow{5}{0.7cm}{\centering\eqref{eq:main_problem}+\eqref{eq:f_i_expectation}}& \algnamex{DIANA}\tnote{\color{red}(3)} &\makecell{\cite{mishchenko2019distributed}\\\cite{li2020unified}} & $\frac{1+\left(1 + \omega\right)\sqrt{\nicefrac{\omega}{n}}}{\varepsilon^2} + \frac{1+\omega}{n\varepsilon^4}$ &  $\frac{1+\left(1 + \omega\right)\sqrt{\nicefrac{\omega}{n}}}{\varepsilon^2} + \frac{1+\omega}{n\varepsilon^4}$\\
    & \algnamex{FedCOMGATE}\tnote{\color{red}(3)} & \cite{haddadpour2020federated} & $\frac{1 + \omega}{\varepsilon^2}$ & $\frac{1+\omega}{n\varepsilon^4}$\\
    &\cellcolor{bgcolor2} \algnamex{VR-MARINA}  (Alg.~\ref{alg:vr_marina}), $b' = 1$ &\cellcolor{bgcolor2} Thm.~\ref{thm:main_result_non_cvx_online} \& Cor.~\ref{cor:main_result_non_cvx_online} (NEW) &\cellcolor{bgcolor2} $\omega + \frac{1 + \nicefrac{\omega}{\sqrt{n}}}{\varepsilon^2} + \frac{\sqrt{1+\omega}}{n\varepsilon^3}$ &\cellcolor{bgcolor2} $\omega + \frac{1 + \nicefrac{\omega}{\sqrt{n}}}{\varepsilon^2} + \frac{\sqrt{1+\omega}}{n\varepsilon^3}$\\
    &\cellcolor{bgcolor2} \algnamex{VR-MARINA} (Alg.~\ref{alg:vr_marina}), $b' = \Theta\left(\frac{1}{n\varepsilon^2}\right)$ &\cellcolor{bgcolor2} Thm.~\ref{thm:main_result_non_cvx_online} \& Cor.~\ref{cor:main_result_non_cvx_online} (NEW) &\cellcolor{bgcolor2} $\omega + \frac{1 + \nicefrac{\omega}{\sqrt{n}}}{\varepsilon^2}$ &\cellcolor{bgcolor2} $\frac{\omega}{n\varepsilon^2} + \frac{1 + \nicefrac{\omega}{\sqrt{n}}}{n\varepsilon^4}$\\
    \hline\hline
    \multirow{2.5}{0.7cm}{\centering PP, \eqref{eq:main_problem}}& \algnamex{FedSTEPH} & \cite{das2020improved} & $\frac{1+\nicefrac{\omega}{n}}{r\varepsilon^4} + \frac{(1+\omega)(n-r)}{r(n-1)\varepsilon^4}$ & $\frac{1+\nicefrac{\omega}{n}}{r\varepsilon^4} + \frac{(1+\omega)(n-r)}{r(n-1)\varepsilon^4}$\\
    &\cellcolor{bgcolor2} \algnamex{PP-MARINA} (Alg.~\ref{alg:pp_marina}) &\cellcolor{bgcolor2} Thm.~\ref{thm:main_result_non_cvx_pp} \& Cor.~\ref{cor:main_result_non_cvx_pp} (NEW) &\cellcolor{bgcolor2} $\frac{1+ \nicefrac{(1 + \omega)\sqrt{n}}{r}}{\varepsilon^2}$ &\cellcolor{bgcolor2} $\frac{1+ \nicefrac{(1 + \omega)\sqrt{n}}{r}}{\varepsilon^2}$\\
    \hline
    \end{tabular}
    \begin{tablenotes}
      {\scriptsize
        \item [{\color{red}(1)}] The results for \algnamex{FedCOMGATE} are derived under assumption that for all vectors $x_1,\ldots,x_n \in\R^d$ the quantization operator $\cQ$ satisfies $\EE\left[\left\|\frac{1}{n}\sum_{i=1}^n\cQ(x_j)\right\|^2 - \left\|\cQ\left(\frac{1}{n}\sum_{i=1}^n x_j\right)\right\|^2\right] \le G$ for some constant $G \ge 0$. In fact, this assumption does not hold for classical quantization operators like RandK and $\ell_2$-quantization on $\R^d$. The counterexample: $n=2$ and $x_1 = -x_2 = (t,t,\ldots,t)^\top$ with arbitrary large $t > 0$.
        \item [{\color{red}(2)}] One can even further improve the communication complexity by increasing $b'$.
        \item [{\color{red}(3)}] No assumptions on the smoothness of the stochastic realizations $f_{\xi}(x)$ are used.
      }
    \end{tablenotes}
    \end{threeparttable}
\end{table*}

\subsection{Contributions}

We propose several new distributed optimization methods supporting compressed communication, specifically focusing on smooth but  nonconvex problems of the form
\begin{equation}
\squeeze
	\min\limits_{x\in\R^d}\left\{f(x) = \frac{1}{n}\sum\limits_{i=1}^n f_i(x)\right\}, \label{eq:main_problem}
\end{equation}
where $n$ workers/devices/clients/peers are connected in a centralized way with a parameter-server, and client $i$ has an access to the local loss function $f_i$ only.   We establish strong complexity rates for them and show that they are better than previous state-of-the-art results.

$\bullet$ \textbf{MARINA.} The main contribution of our paper is a new distributed method supporting communication compression called \algname{MARINA} (Alg~\ref{alg:marina}). In this algorithm, workers apply an unbiased compression operator to the {\em gradient differences} at each iteration with some probability and send them to the server that performs aggregation by averaging. Unlike all known methods operating with unbiased compression operators, this procedure leads to a {\em biased} gradient estimator.    We prove  convergence guarantees for \algname{MARINA}, which are strictly better than previous state-of-the-art methods (see Table~\ref{tab:comparison}). For example, \algname{MARINA}'s rate $\cO(\frac{1+\omega/\sqrt{n}}{\varepsilon^2})$ is $\cO(\sqrt{\omega})$ times better than that of the state-of-the-art method \algname{DIANA} \citep{mishchenko2019distributed}, where $\omega$ is the variance parameter associated with the deployed compressor. For example, in the case of the Rand1 sparsification compressor, we have $\omega=d-1$, and hence we get an improvement by the factor $\cO(\sqrt{d})$. Since the number $d$ of features can be truly very large when training modern models, this is a substantial improvement that can even amount to  {\em several orders of magnitude.} 

$\bullet$ \textbf{Variance Reduction on Nodes.} We generalize \algname{MARINA} to \algname{VR-MARINA}, which can handle the situation when the local functions $f_i$ have either a finite-sum (each $f_i$ is an average of $m$ functions) or an expectation form, and when it is more efficient to rely on local stochastic gradients rather than on local gradients. When compared with \algname{MARINA},  \algname{VR-MARINA} additionally performs {\em local variance reduction} on all nodes, progressively removing the variance coming from the stochastic approximation, leading to a better oracle complexity than previous state-of-the-art results (see Table~\ref{tab:comparison}). When no compression is used (i.e., $\omega=0$), the rate of \algname{VR-MARINA} is $\cO(\frac{\sqrt{m}}{\sqrt{n} \varepsilon^2})$, while the rate of the state-of-the-art method \algname{VR-DIANA} is $\cO(\frac{m^{2/3}}{\varepsilon^2})$. This is an improvement by the factor $\cO(\sqrt{n}m^{1/6})$. When much compression is applied, and $\omega$ is large, our method is faster by the factor  $\cO(\frac{m^{2/3} + \omega}{m^{1/2} + \omega^{1/2}})$. In the special case, when there is just a single node ($n=1$), and no compression is used, \algname{VR-MARINA} reduces to the \algname{PAGE} method of \citet{li2020page}; this is an optimal first-order algorithm for smooth non-convex finite-sum/online optimization problems.

$\bullet$ \textbf{Partial Participation.} We develop a modification of \algname{MARINA} allowing for {\em partial participation} of the clients, which is a feature critical in federated learning. The resulting method, \algname{PP-MARINA}, has  superior communication complexity to the existing methods developed for this settings (see Table~\ref{tab:comparison}).

$\bullet$ \textbf{Convergence Under the Polyak-{\L}ojasiewicz Condition.} We analyze all proposed methods for problems satisfying the Polyak-{\L}ojasiewicz condition \cite{polyak1963gradient,lojasiewicz1963topological}. Again, the obtained results are strictly better than previous ones (see Table~\ref{tab:comparison_pl}). Statements and proofs of all these results are in the Appendix.

$\bullet$ \textbf{Simple Analysis.} The simplicity and flexibility of our analysis offer several extensions. For example, one can easily generalize our analysis to the case of different quantization operators and different batch sizes used by clients. Moreover, one can combine the ideas of \algname{VR-MARINA} and \algname{PP-MARINA} and obtain a single distributed algorithm with compressed communications, variance reduction on nodes, and clients' sampling. We did not do this to keep the exposition simpler.

\begin{table*}[h]
    \centering
    \scriptsize
	\caption{\small Summary of the state-of-the-art results for finding an $\varepsilon$-solution for the problem \eqref{eq:main_problem} satifying \textbf{Polyak-{\L}ojasiewicz condition} (see As.~\ref{as:pl_condition}), i.e., such a point $\hat x$ that $\EE\left[f(\hat x) - f(x^*)\right] \le \varepsilon$. Dependences on the numerical constants and $\log(\nicefrac{1}{\varepsilon})$ factors are omitted and all smoothness constanst are denoted by $L$ in the complexity bounds.  Abbreviations: ``PP'' = partial participation; ``Communication complexity'' = the number of communications rounds needed to find an $\varepsilon$-stationary point; ``Oracle complexity'' = the number of (stochastic) first-order oracle calls needed to find an $\varepsilon$-stationary point. Notation: $\omega$ = the quantization parameter (see Def.~\ref{def:quantization}); $n$ = the number of nodes; $m$ = the size of the local dataset; $r$ = (expected) number of clients sampled at each iteration; $b'$ = the batchsize for \algnamex{VR-MARINA} at the iterations with compressed communication. To simplify the bounds, we assume that the expected density $\zeta_{\cQ}$ of the quantization operator $\cQ$ (see Def.~\ref{def:quantization}) satisfies $\omega+1 = \Theta(\nicefrac{d}{\zeta_{\cQ}})$ (e.g., this holds for RandK and $\ell_2$-quantization, see \cite{beznosikov2020biased}). We notice that \cite{haddadpour2020federated} and \cite{das2020improved} contain also better rates under different assumptions on clients' similarity.}
    \label{tab:comparison_pl}    
   \begin{threeparttable}
    \begin{tabular}{|c|c|c c c|}
         \hline
         Setup & Method & Citation & Communication Complexity & Oracle Complexity \\ 
\hline\hline
    \multirow{3}{0.7cm}{\centering \eqref{eq:main_problem}} & \algnamex{DIANA} &\cite{li2020unified} & $\frac{L(1+\left(1 + \omega\right)\sqrt{\nicefrac{\omega}{n}})}{\mu}$ &  $\frac{L(1+\left(1 + \omega\right)\sqrt{\nicefrac{\omega}{n}})}{\mu}$ \\
    & \algnamex{FedCOMGATE}\tnote{\color{red} (1)} & \cite{haddadpour2020federated} & $\frac{L(1 + \omega)}{\mu}$ & $\frac{L(1+\omega)}{n\mu\varepsilon}$\\
    & \cellcolor{bgcolor2}\algnamex{MARINA} (Alg.~\ref{alg:marina}) &\cellcolor{bgcolor2} Thm.~\ref{thm:main_result_pl} \& Cor.~\ref{cor:main_result_pl_appendix} (NEW) &\cellcolor{bgcolor2} $\omega+\frac{L(1 + \nicefrac{\omega}{\sqrt{n}})}{\mu}$ &\cellcolor{bgcolor2} $\omega+\frac{L(1 + \nicefrac{\omega}{\sqrt{n}})}{\mu}$\\    
    \hline\hline
    \multirow{6}{0.7cm}{\centering\eqref{eq:main_problem}+\eqref{eq:f_i_finite_sum}}& \algnamex{DIANA} &\cite{li2020unified} & \makecell{$\frac{L(1+\left(1 + \omega\right)\sqrt{\nicefrac{\omega}{n}})}{\mu}+$\quad\quad\quad\\ \quad\quad\quad$ + \frac{L(1+\omega)}{n\mu}\left(\frac{L}{\mu}+\frac{1}{\varepsilon}\right)$} &  \makecell{$\frac{L(1+\left(1 + \omega\right)\sqrt{\nicefrac{\omega}{n}})}{\mu}+$\quad\quad\quad\\ \quad\quad\quad$ + \frac{L(1+\omega)}{n\mu}\left(\frac{L}{\mu}+\frac{1}{\varepsilon}\right)$}\\
    & \algnamex{VR-DIANA}& \cite{li2020unified} & $\frac{L\left(m^{\nicefrac{2}{3}} + \omega\right)\sqrt{1+\nicefrac{\omega}{n}}}{\mu}$ &  $\frac{L\left(m^{\nicefrac{2}{3}} + \omega\right)\sqrt{1+\nicefrac{\omega}{n}}}{\mu}$\\
    &\cellcolor{bgcolor2} \algnamex{VR-MARINA} (Alg.~\ref{alg:vr_marina}), $b'=1$\tnote{\color{red}(2)} &\cellcolor{bgcolor2} Thm.~\ref{thm:main_result_pl_finite_sums_appendix} \& Cor.~\ref{cor:main_result_pl_finite_sums_appendix} (NEW) &\cellcolor{bgcolor2} 
\makecell{$\omega + m +$\quad\quad\quad\quad\quad\quad\quad\quad\quad\quad~~\\$+\frac{L(1 + \nicefrac{\max\left\{\omega,\sqrt{(1+\omega)m}\right\}}{\sqrt{n}})}{\mu}$}&\cellcolor{bgcolor2} 
    \makecell{$\omega + m +$\quad\quad\quad\quad\quad\quad\quad\quad\quad\quad~~\\$+\frac{L(1 + \nicefrac{\max\left\{\omega,\sqrt{(1+\omega)m}\right\}}{\sqrt{n}})}{\mu}$}\\
    \hline\hline
    \multirow{7}{0.7cm}{\centering\eqref{eq:main_problem}+\eqref{eq:f_i_expectation}}& \algnamex{DIANA}\tnote{\color{red}(3)} &\makecell{\cite{mishchenko2019distributed}\\\cite{li2020unified}} & $\frac{1+\left(1 + \omega\right)\sqrt{\nicefrac{\omega}{n}}}{\varepsilon^2} + \frac{1+\omega}{n\varepsilon^4}$ &  $\frac{1+\left(1 + \omega\right)\sqrt{\nicefrac{\omega}{n}}}{\varepsilon^2} + \frac{1+\omega}{n\varepsilon^4}$\\
    & \algnamex{FedCOMGATE}\tnote{\color{red}(3)} & \cite{haddadpour2020federated} & $\frac{L(1 + \omega)}{\mu}$ & $\frac{L(1+\omega)}{n\mu\varepsilon}$\\
    &\cellcolor{bgcolor2} \algnamex{VR-MARINA} (Alg.~\ref{alg:vr_marina}), $b' = 1$ &\cellcolor{bgcolor2} Thm.~\ref{thm:main_result_pl_online_appendix} \& Cor.~\ref{cor:main_result_pl_online_appendix} (NEW) &\cellcolor{bgcolor2} 
 \makecell{    $\omega + \frac{1}{n\mu\varepsilon} + $ \quad\quad\quad\quad\quad\quad\quad\quad\quad~\\$+\frac{L}{\mu}\left(1 + \frac{\omega}{\sqrt{n}} + \sqrt{\frac{\omega}{n^2\mu\varepsilon}}\right)$}&\cellcolor{bgcolor2} 
\makecell{$\omega + \frac{1}{n\mu\varepsilon} + $ \quad\quad\quad\quad\quad\quad\quad\quad\quad~\\$+\frac{L}{\mu}\left(1 + \frac{\omega}{\sqrt{n}} + \sqrt{\frac{\omega}{n^2\mu\varepsilon}}\right)$}\\
    &\cellcolor{bgcolor2} \algnamex{VR-MARINA} (Alg.~\ref{alg:vr_marina}), $b' = \Theta\left(\frac{1}{n\mu\varepsilon}\right)$ &\cellcolor{bgcolor2} Thm.~\ref{thm:main_result_pl_online_appendix} \& Cor.~\ref{cor:main_result_pl_online_appendix} (NEW) &\cellcolor{bgcolor2} $\omega+\frac{L(1 + \nicefrac{\omega}{\sqrt{n}})}{\mu}$ &\cellcolor{bgcolor2} $\frac{\omega}{n\mu\varepsilon}+\frac{L(1 + \nicefrac{\omega}{\sqrt{n}})}{n\mu^2\varepsilon}$\\
    \hline\hline
    \multirow{2.5}{0.7cm}{\centering PP, \eqref{eq:main_problem}}& \algnamex{FedSTEPH}\tnote{\color{red}(4)} & \cite{das2020improved} & $\left(\frac{L}{\mu}\right)^{\nicefrac{3}{2}}$ & $\left(\frac{L}{\mu}\right)^{\nicefrac{3}{2}}$\\
    &\cellcolor{bgcolor2} \algnamex{PP-MARINA} (Alg.~\ref{alg:pp_marina}) &\cellcolor{bgcolor2} Thm.~\ref{thm:main_result_pl_pp_appendix} \& Cor.~\ref{cor:main_result_pl_pp_appendix} (NEW) &\cellcolor{bgcolor2} $\frac{(\omega+1)n}{r}+\frac{L(1+ \nicefrac{(1 + \omega)\sqrt{n}}{r})}{\mu}$ &\cellcolor{bgcolor2} $\frac{(\omega+1)n}{r}+\frac{L(1+ \nicefrac{(1 + \omega)\sqrt{n}}{r})}{\mu}$\\
    \hline
    \end{tabular}
    \begin{tablenotes}
      {\scriptsize
        \item [{\color{red}(1)}] The results for \algnamex{FedCOMGATE} are derived under assumption that for all vectors $x_1,\ldots,x_n \in\R^d$ the quantization operator $\cQ$ satisfies $\EE\left[\left\|\frac{1}{n}\sum_{i=1}^n\cQ(x_j)\right\|^2 - \left\|\cQ\left(\frac{1}{n}\sum_{i=1}^n x_j\right)\right\|^2\right] \le G$ for some constant $G \ge 0$. In fact, this assumption does not hold for classical quantization operators like RandK and $\ell_2$-quantization on $\R^d$. The counterexample: $n=2$ and $x_1 = -x_2 = (t,t,\ldots,t)^\top$ with arbitrary large $t > 0$.
        \item [{\color{red}(2)}] One can even further improve the communication complexity by increasing $b'$.
        \item [{\color{red}(3)}] No assumptions on the smoothness of the stochastic realizations $f_{\xi}(x)$ are used.
        \item [{\color{red}(4)}] The rate is derived under assumption that $r = \Omega((1+\omega)\sqrt{\nicefrac{L}{\mu}}\log(\nicefrac{1}{\varepsilon}))$.
      }
    \end{tablenotes}
    \end{threeparttable}
\end{table*}

\subsection{Related Work}
\noindent\textbf{Non-Convex Optimization.} Since finding a global minimum of a non-convex function is, in general, an NP-hard problem \cite{murty1987some}, many researchers in non-convex optimization focus on relaxed goals such as finding an $\varepsilon$-stationary point. The theory of stochastic first-order methods for finding $\varepsilon$-stationary points is well-developed: it contains lower bounds for expectation minimization without smoothness of stochastic realizations \cite{arjevani2019lower} and for finite-sum/expectation minimization \cite{fang2018near,li2020page} as well as optimal methods matching the lower bounds (see \cite{danilova2020recent,li2020page} for the overview). Recently, distributed variants of such methods were proposed \cite{sun2020improving,sharma2019parallel,khanduri2020distributed}.

\noindent\textbf{Compressed Communications.} Works on  distributed methods supporting communication compression can be roughly split into two large groups: the first group focuses on methods using {\em unbiased} compression operators (which refer to as quantizations in this paper), such as RandK, and the second one studies methods using {\em biased} compressors such as TopK. One can find a detailed summary of the most popular compression operators in \citep{UP2020, beznosikov2020biased}.

\noindent\textbf{Unbiased Compression.} In this line of work, the first convergence result in the non-convex case was obtained by \citet{alistarh2017qsgd} for  \algname{QSGD}, under assumptions that the local loss functions are the same for all workers, and the stochastic gradient has uniformly bounded second moment. After that, \citet{mishchenko2019distributed} proposed \algname{DIANA} (and its momentum version)  and proved its convergence rate for non-convex problems without any assumption on the boundedness of the second moment of the stochastic gradient, but under the assumption that the dissimilarity between local loss functions is bounded. This restriction was later eliminated by \citet{horvath2019stochastic} for the variance reduced version of \algname{DIANA} called  \algname{VR-DIANA}, and the analysis was extended to a large class of unbiased compressors. Finally, the results for  \algname{QSGD} and \algname{DIANA} were recently generalized and tightened by \citet{li2020unified} in a unifying framework that included many other methods as well.

\noindent\textbf{Biased Compression.} Biased compression operators are less ``optimization-friendly'' than unbiased ones. Indeed, one can construct a simple convex quadratic problem for which distributed \algname{SGD} with Top1 compression diverges exponentially fast  \cite{beznosikov2020biased}. However, this issue can be resolved using {\em error compensation} \cite{seide20141}. The first analysis of error-compensated \algname{SGD} (\algname{EC-SGD}) for non-convex problems was obtained by \citet{karimireddy2019error} for homogeneous problems under the assumption that the second moment of the stochastic gradient is uniformly  bounded. The last assumption was recently removed from the analysis of \algname{EC-SGD} by \citet{stich2020error, beznosikov2020biased}, while the first results without the homogeneity assumption were obtained by \citet{KoloskovaLSJ19decentralized} for \algname{Choco-SGD}, but still under the assumption that the second moment of the stochastic gradient is uniformly  bounded. This issue was resolved by \citet{beznosikov2020biased}. In general, the current understanding of optimization methods with biased compressors is far from complete: even in the strongly convex case, the first linearly converging \cite{gorbunov2020linearly} and accelerated \cite{qian2020error} error-compensated stochastic methods were proposed just recently.

\noindent\textbf{Other Approaches.} Besides communication compression, there are also different techniques aiming to reduce the overall communication cost of distributed methods. The most popular ones are based on decentralized communications and multiple local steps between communication rounds, where the second technique is very popular in federated learning \cite{konecny2016federated,kairouz2019advances}. One can find the state-of-the-art distributed optimization methods using these techniques and their combinations in \cite{lian2017can,karimireddy2020scaffold,li2019communication,koloskova2020unified}. Moreover, there exist results based on the combinations of communication compression with either decentralized communication, e.g., \algname{Choco-SGD} \cite{KoloskovaLSJ19decentralized}, or local updates, e.g., \algname{Qsparse-Local-SGD} \cite{basu2019qsparse}, \algname{FedCOMGATE} \cite{haddadpour2020federated}, \algname{FedSTEPH} \cite{das2020improved}, where in \cite{basu2019qsparse} the convergence rates were derived under an assumption that the stochastic gradient has uniformly bounded second moment and the results for \algname{Choco-SGD}, \algname{FedCOMGATE}, \algname{FedSTEPH} were described either earlier in the text, or in Table~\ref{tab:comparison}.


\subsection{Preliminaries}

We will rely on two key assumptions thrughout the text.
\begin{assumption}[Uniform lower bound]\label{as:lower_bound}
	There exists  $f_*\in \R$ such that $f(x) \ge f_*$ for all $x\in\R^d$.
\end{assumption}
\begin{assumption}[$L$-smoothness]\label{as:L_smoothness}
	We assume that $f_i$ is $L_i$-smooth for all $i\in [n] = \{1,2,\dots,n\}$ meaning that the following inequality holds $\forall x,y\in \R^d$, $\forall i\in [n]$:
	\begin{equation}
		\left\|\nabla f_{i}(x) - \nabla f_{i}(y)\right\| \le L_i\|x-y\|.\label{eq:L_smoothness}
	\end{equation}
\end{assumption}
This assumption implies that $f$ is $L_f$-smooth with $L_f^2 \le L^2 = \frac{1}{n}\sum_{i=1}^nL_i^2$.

Finally, we describe a large class of unbiased compression operators satisfying a certain variance bound, which we will refer to, in this paper, by the name {\em quantization}.
\begin{definition}[Quantization]\label{def:quantization}
	We say that a stochastic mapping $\cQ:\R^d \to \R^d$ is a quantization operator/quantization if there exists  $\omega > 0$ such that for any $x\in\R^d$ , we have
	\begin{equation}
		\EE\left[\cQ(x)\right] = x,\quad \EE\left[\|\cQ(x) - x\|^2\right] \le \omega\|x\|^2. \label{eq:quantization_def}
	\end{equation}
	For the given quantization operator $\cQ(x)$, we define the the expected density as $\zeta_{\cQ} = \sup_{x\in\R^d}\EE\left[\left\|\cQ(x)\right\|_0\right],$ 
	where $\|y\|_0$ is the number of non-zero components of $y\in\R^d$.
\end{definition}
Notice that the expected density is well-defined for any quantization operator since $\left\|\cQ(x)\right\|_0 \le d$.

\section{{\sf MARINA}}\label{sec:marina}
In this section, we describe the main algorithm of this work: \algname{MARINA}  (see Algorithm~\ref{alg:marina}). At each iteration of \algname{MARINA}, each  worker $i$ either sends to the server  the dense vector $\nabla f_i(x^{k+1})$ with probability $p$, or it sends the quantized gradient difference $\cQ\left(\nabla f_{i}(x^{k+1}) - \nabla f_{i}(x^k))\right)$ with probability $1-p$. In the first situation, the server just averages the vectors received from workers and gets $g^{k+1} = \nabla f(x^{k+1})$, whereas in the second case, the server averages the quantized differences from all workers and then adds the result to $g^k$ to get $g^{k+1}.$ Moreover, if $\cQ$ is identity quantization, i.e., $\cQ(x) = x$, then \algname{MARINA} reduces to Gradient Descent (\algname{GD}).

\begin{algorithm}[h]
   \caption{\algname{MARINA}}\label{alg:marina}
\begin{algorithmic}[1]
   \STATE {\bfseries Input:} starting point $x^0$, stepsize $\gamma$, probability $p\in(0,1]$, number of iterations $K$
   \STATE Initialize $g^0 = \nabla f(x^0)$
   \FOR{$k=0,1,\ldots,K-1$}
   \STATE Sample $c_k \sim \text{Be}(p)$
   \STATE Broadcast $g^k$ to all workers
   \FOR{$i = 1,\ldots,n$ in parallel} 
   \STATE $x^{k+1} = x^k - \gamma g^k$
   \STATE Set $g_i^{k+1} = \nabla f_i(x^{k+1})$ if $c_k = 1$, and $g_i^{k+1} = g^k + \cQ\left(\nabla f_{i}(x^{k+1}) - \nabla f_{i}(x^k))\right)$ otherwise
   \ENDFOR
   \STATE $g^{k+1} = \frac{1}{n}\sum_{i=1}^ng_i^{k+1}$
   \ENDFOR
   \STATE {\bfseries Return:} $\hat x^K$ chosen uniformly at random from $\{x^k\}_{k=0}^{K-1}$
\end{algorithmic}
\end{algorithm}

However, for non-trivial quantizations, we have $\EE[g^{k+1}\mid x^{k+1}] \neq \nabla f(x^{k+1})$ unlike all other distributed methods using exclusively unbiased compressors we know of. That is, $g^{k+1}$ is a \textit{biased} stochastic estimator of $\nabla f(x^{k+1})$. However, \algname{MARINA} is an example of a rare phenomenon in stochastic optimization  when the {\em bias of the stochastic gradient helps to achieve better complexity.}

\subsection{Convergence Results for Generally Non-Convex Problems}
We start with the following result.
\begin{theorem}\label{thm:main_result_non_cvx}
	Let Assumptions~\ref{as:lower_bound}~and~\ref{as:L_smoothness} be satisfied. Then, after
	\begin{equation}
	\squeeze
		K = \cO\left(\frac{\Delta_0 L}{\varepsilon^2}\left(1 + \sqrt{\frac{(1-p)\omega}{pn}}\right)\right) \notag
	\end{equation}
	iterations with $\Delta_0 = f(x^0)-f_*$, $L^2 = \frac{1}{n}\sum_{i=1}^nL_i^2$ and the stepsize $\gamma \le L^{-1}\left(1 + \sqrt{\nicefrac{(1-p)\omega}{(pn)}}\right)^{-1}$,
	\algname{MARINA} produces  point $\hat x^K$ for which $\EE[\|\nabla f(\hat x^K)\|^2] \le \varepsilon^2$.
\end{theorem}
One can find the full statement of the theorem together with its proof in Section~\ref{sec:proof_of_thm_non_cvx} of the Appendix.

The following corollary provides the bounds on the number of iterations/communication rounds and estimates the total communication cost needed to achieve an $\varepsilon$-stationary point in expectation. Moreover, for simplicity, throughout the paper we assume that the communication cost is proportional to the number of non-zero components of transmitted vectors from workers to the server.
\begin{corollary}\label{cor:main_result_non_cvx}
	Let the assumptions of Theorem~\ref{thm:main_result_non_cvx} hold and $p = \nicefrac{\zeta_{\cQ}}{d}$. If $\gamma \le L^{-1}\left(1 + \sqrt{\nicefrac{\omega(d-\zeta_{\cQ})}{(n\zeta_{\cQ})}}\right)^{-1}$,
	then \algname{MARINA} requires 
	\begin{equation*}
	\squeeze
		\cO\left(\frac{\Delta_0 L}{\varepsilon^2}\left(1 + \sqrt{\frac{\omega}{n}\left(\frac{d}{\zeta_{\cQ}}-1\right)}\right)\right)
	\end{equation*}
	iterations/communication rounds in order to achieve $\EE[\|\nabla f(\hat x^K)\|^2] \le \varepsilon^2$, and the expected total communication cost per worker is $\cO(d + \zeta_{\cQ}K)$.
\end{corollary}

Let us clarify the obtained result. First of all, if $\omega = 0$ (no quantization), then $\zeta_{\cQ} = 0$ and the rate coincides with the rate of Gradient Descent (\algname{GD}). Since \algname{GD} is optimal among first-order methods in terms of reducing the norm of the gradient \cite{carmon2019lower}, the dependence on $\varepsilon$ in our bound cannot be improved in general. Next, if $n$ is large enough, i.e., $n \ge \omega(\nicefrac{d}{\zeta_{\cQ}}-1)$, then\footnote{For $\ell_2$-quantization this requirement is satisfied when $n \ge d$.}  the iteration complexity of \algname{MARINA} (method with compressed communications) and \algname{GD} (method with dense communications) coincide. This means that in this regime,  \algname{MARINA} is able to reach a provably better communication complexity than \algname{GD}!

\subsection{Convergence Results Under Polyak-{\L}ojasiewicz condition}
In this section, we provide a complexity bounds for \algname{MARINA} under the Polyak-{\L}ojasiewicz (P{\L}) condition.

\begin{assumption}[P{\L} condition]\label{as:pl_condition}
	Function $f$ satisfies Polyak-{\L}ojasiewicz (P{\L}) condition with parameter $\mu$, i.e., 	\begin{equation}
		\|\nabla f(x)\|^2 \ge 2\mu\left(f(x) - f(x^*)\right). \label{eq:pl_condition}
	\end{equation}
	holds for $x^*= \arg \min_{x\in\R^d} f(x)$ and for all $x\in\R^d$.

\end{assumption}

Under this and previously introduced assumptions, we derive the following result.
\begin{theorem}\label{thm:main_result_pl}
	Let Assumptions~\ref{as:lower_bound},~\ref{as:L_smoothness}~and~\ref{as:pl_condition} be satisfied. Then, after
	\begin{equation}
	\squeeze
		K = \cO\left(\max\left\{\frac{1}{p},\frac{L}{\mu}\left(1 + \sqrt{\frac{(1-p)\omega}{pn}}\right)\right\}\log\frac{\Delta_0}{\varepsilon}\right) \notag
	\end{equation}
	iterations with $\Delta_0 = f(x^0)-f(x^*)$, $L^2 = \frac{1}{n}\sum_{i=1}^nL_i^2$ and the stepsize $\gamma \le \min\left\{L^{-1}\left(1 + \sqrt{\nicefrac{2(1-p)\omega}{(pn)}}\right)^{-1}, p(2\mu)^{-1}\right\}$, 
	\algname{MARINA} produces a point $x^K$ for which $\EE[f(x^K) - f(x^*)] \le \varepsilon$.
\end{theorem}
One can find the full statement of the theorem together with its proof in Section~\ref{sec:proof_of_thm_pl} of the Appendix.


\section{Variance Reduction}\label{sec:vr}
Throughout this section, we assume that the local loss on each node has either a finite-sum form (finite sum case), 
\begin{equation}
\squeeze	f_i(x) = \frac{1}{m}\sum\limits_{j=1}^mf_{ij}(x), \label{eq:f_i_finite_sum}
\end{equation}
or an expectation form (online case),
\begin{equation}
\squeeze	f_i(x) = \EE_{\xi_i\sim\cD_i}[f_{\xi_i}(x)]. \label{eq:f_i_expectation}
\end{equation}

\subsection{Finite Sum Case}
In this section, we generalize \algname{MARINA} to problems of the form \eqref{eq:main_problem}+\eqref{eq:f_i_finite_sum}, obtaining  \algname{VR-MARINA} (see Algorithm~\ref{alg:vr_marina}).
\begin{algorithm}[h]
   \caption{\algname{VR-MARINA}: finite sum case}\label{alg:vr_marina}
\begin{algorithmic}[1]
   \STATE {\bfseries Input:} starting point $x^0$, stepsize $\gamma$, minibatch size $b'$, probability $p\in(0,1]$, number of iterations $K$
   \STATE Initialize $g^0 = \nabla f(x^0)$ 
   \FOR{$k=0,1,\ldots,K-1$}
   \STATE Sample $c_k \sim \text{Be}(p)$
   \STATE Broadcast $g^k$ to all workers
   \FOR{$i = 1,\ldots,n$ in parallel} 
   \STATE $x^{k+1} = x^k - \gamma g^k$
   \STATE Set $g_i^{k+1} = \nabla f_i(x^{k+1})$ if $c_k = 1$, and $g_i^{k+1} = g^k + \cQ\left(\frac{1}{b'}\sum_{j\in I_{i,k}'}(\nabla f_{ij}(x^{k+1}) - \nabla f_{ij}(x^k))\right)$ otherwise, where $I_{i,k}'$ is the set of the indices in the minibatch, $|I_{i,k}'| = b'$
   \ENDFOR
   \STATE $g^{k+1} = \frac{1}{n}\sum_{i=1}^ng_i^{k+1}$
   \ENDFOR
   \STATE {\bfseries Return:} $\hat x^K$ chosen uniformly at random from $\{x^k\}_{k=0}^{K-1}$
\end{algorithmic}
\end{algorithm}
At each iteration of \algname{VR-MARINA}, devices are to compute the full gradients $\nabla f_i(x^{k+1})$ and send them to the server with probability $p$. Typically, $p \le \nicefrac{1}{m}$ and $m$ is large, meaning that workers compute full gradients rarely (once per $\ge m$ iterations in expectation). At other iterations, workers compute minibatch stochastic gradients evaluated at the current and previous points, compress them using an unbiased compression operator, i.e., quantization/quantization operator, and send the resulting vectors $g_i^{k+1} - g^{k}$ to the server. Moreover, if $\cQ$ is the identity quantization, i.e., $\cQ(x) = x$, and $n=1$, then \algname{MARINA} reduces to the optimal method \algname{PAGE} \cite{li2020page}.

In this part, we will rely on the following average smoothness assumption.
\begin{assumption}[Average $\cL$-smoothness]\label{as:avg_smoothness}
	For all $k\ge 0$ and $i\in[n]$ the minibatch stochastic gradients difference $\widetilde{\Delta}_i^k = \frac{1}{b'}\sum_{j\in I_{i,k}'}(\nabla f_{ij}(x^{k+1}) - \nabla f_{ij}(x^k))$ computed on the $i$-th worker satisfies $\EE\left[\widetilde{\Delta}_i^k\mid x^k,x^{k+1}\right] = \Delta_i^k$ and
	\begin{equation}
	\squeeze	\EE\left[\left\|\widetilde{\Delta}_i^k - \Delta_i^k\right\|^2\mid x^k,x^{k+1}\right] \le \frac{\cL_i^2}{b'}\|x^{k+1}-x^k\|^2\label{eq:avg_L_smoothness}
	\end{equation}
	with some $\cL_i \ge 0$, where $\Delta_i^k = \nabla f_i(x^{k+1}) - \nabla f_i(x^k)$.
\end{assumption}
This assumption is satisfied in many standard minibatch regimes. In particular, if $I_{i,k}' = \{1,\ldots,m\}$, then $\cL_i = 0$, and if $I_{i,k}'$ consists of $b'$ i.i.d.\ samples from the uniform distributions on $\{1,\ldots,m\}$ and $f_{ij}$ are $L_{ij}$-smooth, then $\cL_i \le \max_{j\in [m]}L_{ij}$.

Under this and the previously introduced assumptions, we derive the following result.
\begin{theorem}\label{thm:main_result_non_cvx_finite_sums}
	Consider the finite sum case \eqref{eq:main_problem}+\eqref{eq:f_i_finite_sum}. Let Assumptions~\ref{as:lower_bound},~\ref{as:L_smoothness}~and~\ref{as:avg_smoothness} be satisfied. Then, after
	\begin{equation}
\squeeze		K = \cO\left(\frac{\Delta_0}{\varepsilon^2}\left(L + \sqrt{\frac{1-p}{pn}\left(\omega L^2 + \frac{(1+\omega)\cL^2}{b'}\right)}\right)\right) \notag
	\end{equation}
	iterations with $\Delta_0 = f(x^0)-f_*$, $L^2 = \frac{1}{n}\sum_{i=1}^nL_i^2$, $\cL^2 = \frac{1}{n}\sum_{i=1}^n\cL_i^2$ and the stepsize $\gamma \le \left(L + \sqrt{\left(\omega L^2 + \nicefrac{(1+\omega)\cL^2}{b'}\right)\nicefrac{(1-p)}{(pn)}}\right)^{-1}$, 
	\algname{VR-MARINA} produces such a point $\hat x^K$ that $\EE[\|\nabla f(\hat x^K)\|^2] \le \varepsilon^2$.
\end{theorem}
One can find the full statement of the theorem together with its proof in Section~\ref{sec:proof_of_thm_non_cvx_fin_sums} of the Appendix.
\begin{corollary}\label{cor:main_result_non_cvx_finite_sums}
	Let the assumptions of Theorem~\ref{thm:main_result_non_cvx_finite_sums} hold and $p = \min\left\{\nicefrac{\zeta_{\cQ}}{d},\nicefrac{b'}{(m+b')}\right\}$, where $b'\le m$. If $\gamma \le \left(L + \sqrt{\left(\omega L^2 + \nicefrac{(1+\omega)\cL^2}{b'}\right)\nicefrac{\max\left\{\nicefrac{d}{\zeta_{\cQ}} - 1,\nicefrac{m}{b'}\right\}}{n}}\right)^{-1}$
	then \algname{VR-MARINA} requires 
	\begin{eqnarray*}
	\squeeze	\cO\Bigg(\frac{\Delta_0}{\varepsilon^2}\Bigg(L\Bigg(1 + \sqrt{\frac{\omega\max\left\{\nicefrac{d}{\zeta_{\cQ}} - 1,\nicefrac{m}{b'}\right\}}{n}}\Bigg)&\\
		&\squeeze\hspace{-3.5cm} + \cL\sqrt{\frac{(1+\omega)\max\left\{\nicefrac{d}{\zeta_{\cQ}} - 1,\nicefrac{m}{b'}\right\}}{nb'}}\Bigg)\Bigg)
	\end{eqnarray*}
	iterations/communication rounds and $\cO\left(m + b'K\right)$
	stochastic oracle calls per node in expectation in order to achieve $\EE[\|\nabla f(\hat x^K)\|^2] \le \varepsilon^2$, and the expected total communication cost per worker is $\cO(d + \zeta_{\cQ}K)$.
\end{corollary}

First of all, when workers quatize differences of the full gradients, then $I_{i,k}' = \{1,\ldots,m\}$ for all $i\in[n]$ and $k\ge 0$,  implying $\cL = 0$. In this case, the complexity bounds for \algname{VR-MARINA} recover the ones for \algname{MARINA}. Next, when $\omega = 0$ (no quantization) and $n = 1$, our bounds for iteration and oracle complexities for \algname{VR-MARINA} recover the bounds for \algname{PAGE} \cite{li2020unified}, which is optimal for finite-sum smooth non-convex optimization. This observation implies that the dependence on $\varepsilon$ and $m$ in the complexity bounds for \algname{VR-MARINA} cannot be improved in the class of first-order stochastic methods. Next, we notice that up to the differences in smoothness constants, the iteration and oracle complexities for \algname{VR-MARINA} benefit from the number of workers $n$. Finally, as Table~\ref{tab:comparison} shows, the rates for \algname{VR-MARINA} are strictly better than ones for the previous state-of-the-art method \algname{VR-DIANA}\cite{horvath2019stochastic}.

We provide the convergence results for \algname{VR-MARINA} in the finite-sum case under the Polyak-{\L}ojasiewicz condition,  together with complete proofs, in Section~\ref{sec:proof_of_thm_pl_fin_sums} of the Appendix.

\subsection{Online Case}
In this section, we focus on  problems of type \eqref{eq:main_problem}+\eqref{eq:f_i_expectation}. For this type of problems, we consider a  slightly modified version of \algname{VR-MARINA}. That is, we replace line 8 in Algorithm~\ref{alg:vr_marina} with the following update rule: $g_i^{k+1} = \frac{1}{b}\sum_{j\in I_{i,k}}\nabla f_{\xi_{ij}^k}(x^{k+1})$ if $c_k = 1$, and $g_i^{k+1} = g^k + \cQ\left(\frac{1}{b'}\sum_{j\in I_{i,k}'}(\nabla f_{\xi_{ij}^k}(x^{k+1}) - \nabla f_{\xi_{ij}^k}(x^k))\right)$ otherwise, where $I_{i,k}, I_{i,k}'$ are the sets of the indices in the minibatches, $|I_{i,k}| = b$, $|I_{i,k}'| = b'$,  and $\xi_{ij}^k$ is independently sampled from $\cD_i$ for $i\in[n]$, $j\in[m]$ (see Algorithm~\ref{alg:vr_marina_online} in the Appendix).
Before we provide our convergence results in this setup, we reformulate Assumption~\ref{as:avg_smoothness} for the online case.
\begin{assumption}[Average $\cL$-smoothness]\label{as:avg_smoothness_online}
	For all $k\ge 0$ and $i\in[n]$ the minibatch stochastic gradients difference $\widetilde{\Delta}_i^k = \frac{1}{b'}\sum_{j\in I_{i,k}'}(\nabla f_{\xi_{ij}^k}(x^{k+1}) - \nabla f_{\xi_{ij}^k}(x^k))$ computed on the $i$-th worker satisfies $\EE\left[\widetilde{\Delta}_i^k\mid x^k,x^{k+1}\right] = \Delta_i^k$ and
	\begin{equation}
	\squeeze	\EE\left[\left\|\widetilde{\Delta}_i^k - \Delta_i^k\right\|^2\mid x^k,x^{k+1}\right] \le \frac{\cL_i^2}{b'}\|x^{k+1}-x^k\|^2\label{eq:avg_L_smoothness_online}
	\end{equation}
	with some $\cL_i \ge 0$, where $\Delta_i^k = \nabla f_i(x^{k+1}) - \nabla f_i(x^k)$.
\end{assumption}
Moreover, we assume that the variance of the stochastic gradients on all nodes is uniformly upper bounded.
\begin{assumption}\label{as:bounded_var}
	We assume that for all $i \in [n]$ there exists such constant $\sigma_i \in [0,+\infty)$ that for all $x\in\R^d$
	\begin{eqnarray}
		\EE_{\xi_i \sim\cD_i}\left[\nabla f_{\xi_i}(x)\right]	&=& \nabla f_i(x), \label{eq:unbiasedness}\\
		\EE_{\xi_i \sim\cD_i}\left[\left\|\nabla f_{\xi_i}(x) - \nabla f_i(x)\right\|^2\right] &\le& \sigma_i^2. \label{eq:bounded_var}
	\end{eqnarray}
\end{assumption}

Under these and previously introduced assumptions, we derive the following result.
\begin{theorem}\label{thm:main_result_non_cvx_online}
	Consider the online case \eqref{eq:main_problem}+\eqref{eq:f_i_expectation}. Let Assumptions~\ref{as:lower_bound},~\ref{as:L_smoothness},~\ref{as:avg_smoothness_online}~and~\ref{as:bounded_var} be satisfied. Then, after
	\begin{equation}
	\squeeze	K = \cO\left(\frac{\Delta_0}{\varepsilon^2}\left(L + \sqrt{\frac{1-p}{pn}\left(\omega L^2 + \frac{(1+\omega)\cL^2}{b'}\right)}\right)\right) \notag
	\end{equation}
	iterations with $\Delta_0 = f(x^0)-f_*$, $L^2 = \frac{1}{n}\sum_{i=1}^nL_i^2$, $\cL^2 = \frac{1}{n}\sum_{i=1}^n\cL_i^2$, the stepsize $\gamma \le \left(L + \sqrt{\left(\omega L^2 + \nicefrac{(1+\omega)\cL^2}{b'}\right)\nicefrac{(1-p)}{(pn)}}\right)^{-1},$
	and $b = \Theta\left(\nicefrac{\sigma^2}{(n\varepsilon^2)}\right),$ $\sigma^2 = \frac{1}{n}\sum_{i=1}^n\sigma_i^2$,  
	\algname{VR-MARINA} produces a point $\hat x^K$ for which $\EE[\|\nabla f(\hat x^K)\|^2] \le \varepsilon^2$.
\end{theorem}
One can find the full statement of the theorem, together with its proof, in Section~\ref{sec:proof_of_thm_non_cvx_online} of the Appendix.
\begin{corollary}\label{cor:main_result_non_cvx_online}
	Let the assumptions of Theorem~\ref{thm:main_result_non_cvx_online} hold and choose $p = \min\left\{\nicefrac{\zeta_{\cQ}}{d},\nicefrac{b'}{(b+b')}\right\}$, where $b'\le b$, $b = \Theta\left(\nicefrac{\sigma^2}{(n\varepsilon^2)}\right)$. If $\gamma \le \left(L + \sqrt{\left(\omega L^2 + \nicefrac{(1+\omega)\cL^2}{b'}\right)\nicefrac{\max\left\{\nicefrac{d}{\zeta_{\cQ}} - 1,\nicefrac{b}{b'}\right\}}{n}}\right)^{-1},$
	then \algname{VR-MARINA} requires 
	\begin{eqnarray*}
	\squeeze	\cO\Bigg(\frac{\Delta_0}{\varepsilon^2}\Bigg(L\Bigg(1 + \sqrt{\frac{\omega}{n}\max\left\{\frac{d}{\zeta_{\cQ}} - 1,\frac{\sigma^2}{nb'\varepsilon^2}\right\}}\Bigg)&\\
		&\squeeze\hspace{-4.5cm} + \cL\sqrt{\frac{(1+\omega)}{nb'}\max\left\{\frac{d}{\zeta_{\cQ}} - 1,\frac{\sigma^2}{nb'\varepsilon^2}\right\}}\Bigg)\Bigg)
	\end{eqnarray*}
	iterations/communication rounds and $\cO(\zeta_{\cQ}K+ \nicefrac{\sigma^2}{(n\varepsilon^2)})$  
	stochastic oracle calls per node in expectation to achieve $\EE[\|\nabla f(\hat x^K)\|^2] \le \varepsilon^2$, and the expected total communication cost per worker is $\cO(d + \zeta_{\cQ}K)$.
\end{corollary}

Similarly to the finite-sum case, when $\omega = 0$ (no quantization) and $n = 1$, our bounds for iteration and oracle complexities for \algname{VR-MARINA} recover the bounds for \algname{PAGE} \cite{li2020unified}, which is optimal for online smooth non-convex optimization as well. That is, the dependence on $\varepsilon$ in the complexity bound for \algname{VR-MARINA} cannot be improved in the class of first-order stochastic methods. As previously, up to the differences in smoothness constants, the iteration and oracle complexities for \algname{VR-MARINA} benefit from an increase in the  number of workers $n$.

We provide the convergence results for \algname{VR-MARINA} in the online case under the Polyak-{\L}ojasiewicz condition, together with complete proofs, in Section~\ref{sec:proof_of_thm_pl_online} of the Appendix.

\begin{figure*}[t]
\centering
\includegraphics[width=0.24\textwidth]{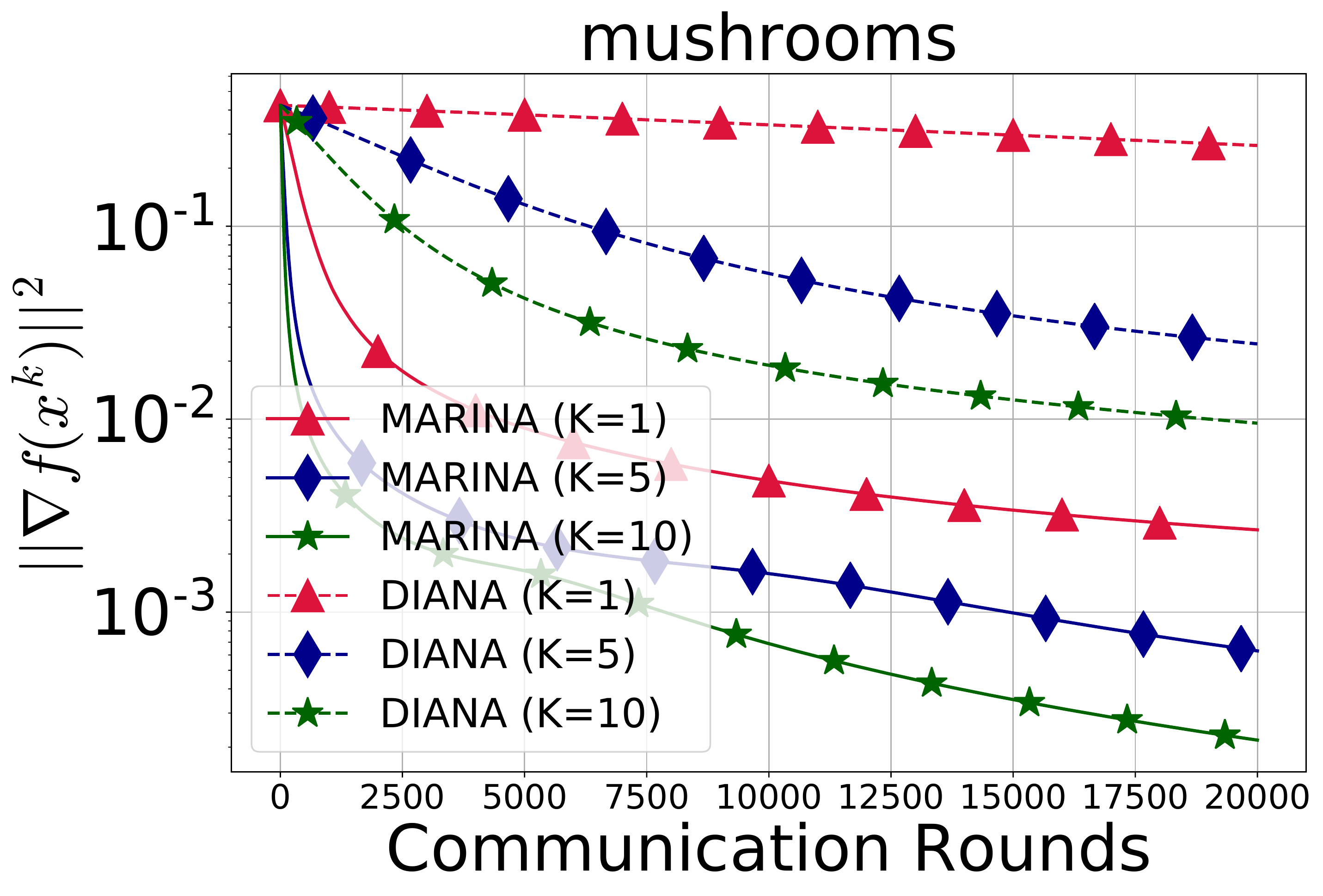}
\includegraphics[width=0.24\textwidth]{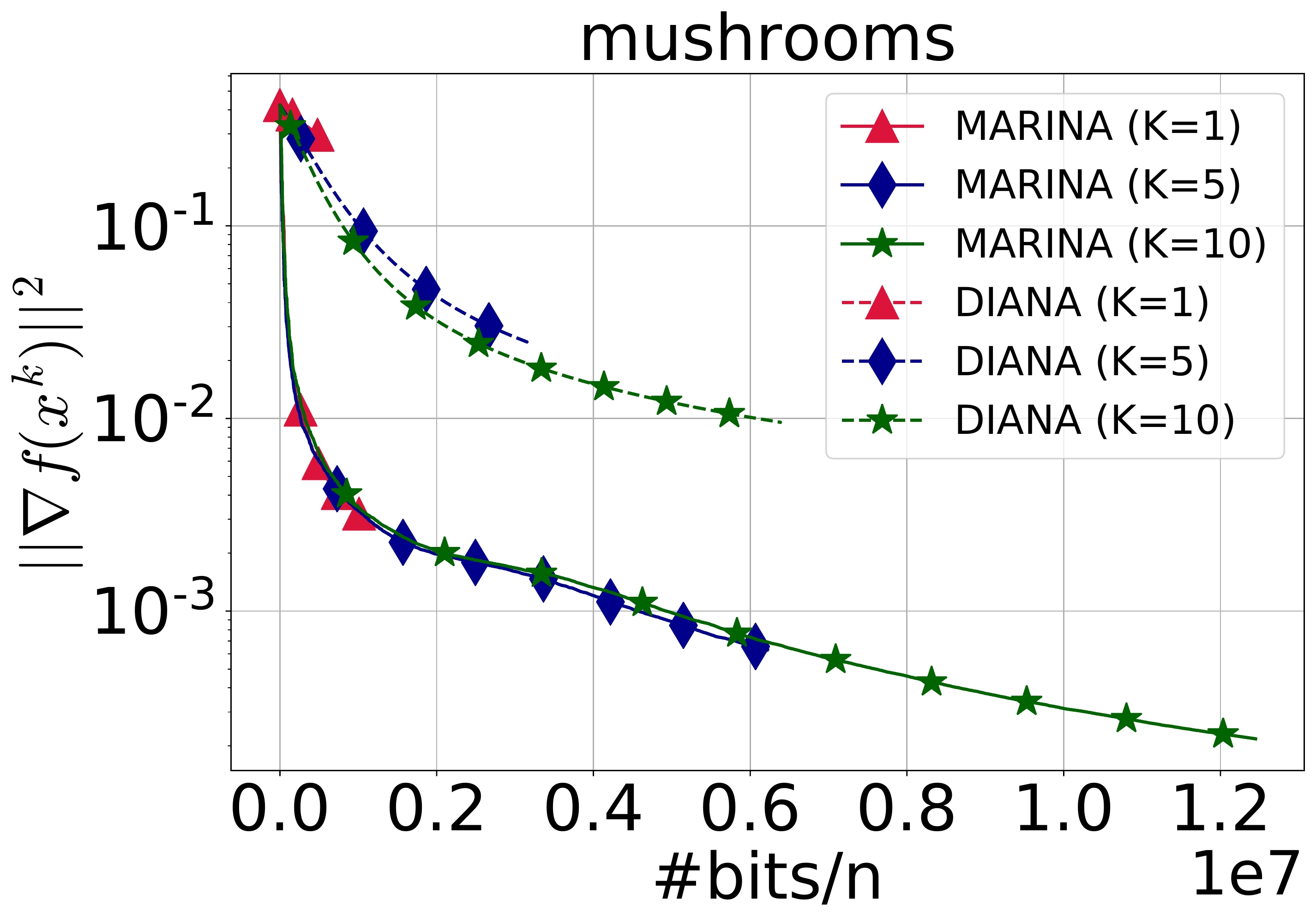}
\includegraphics[width=0.24\textwidth]{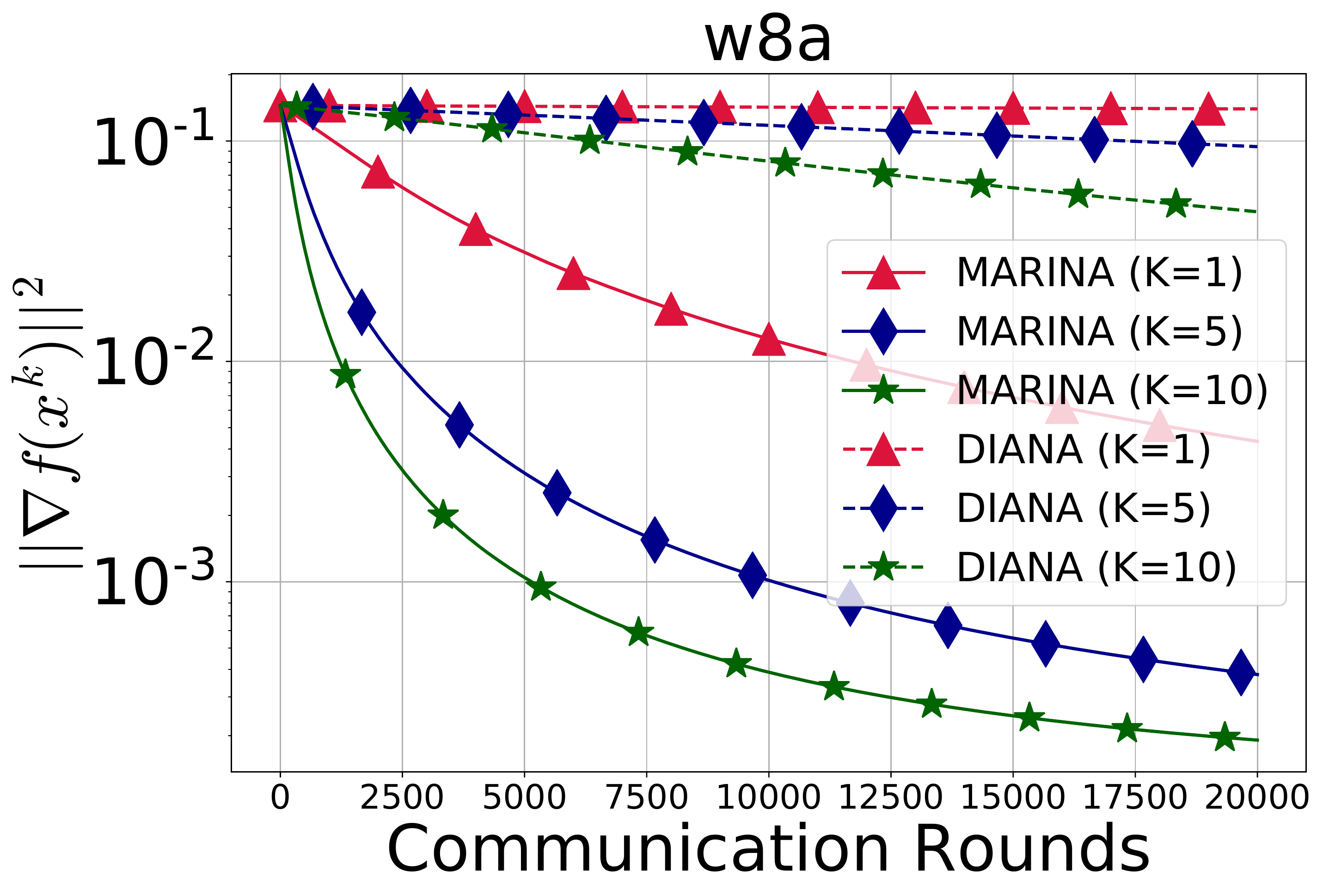}
\includegraphics[width=0.24\textwidth]{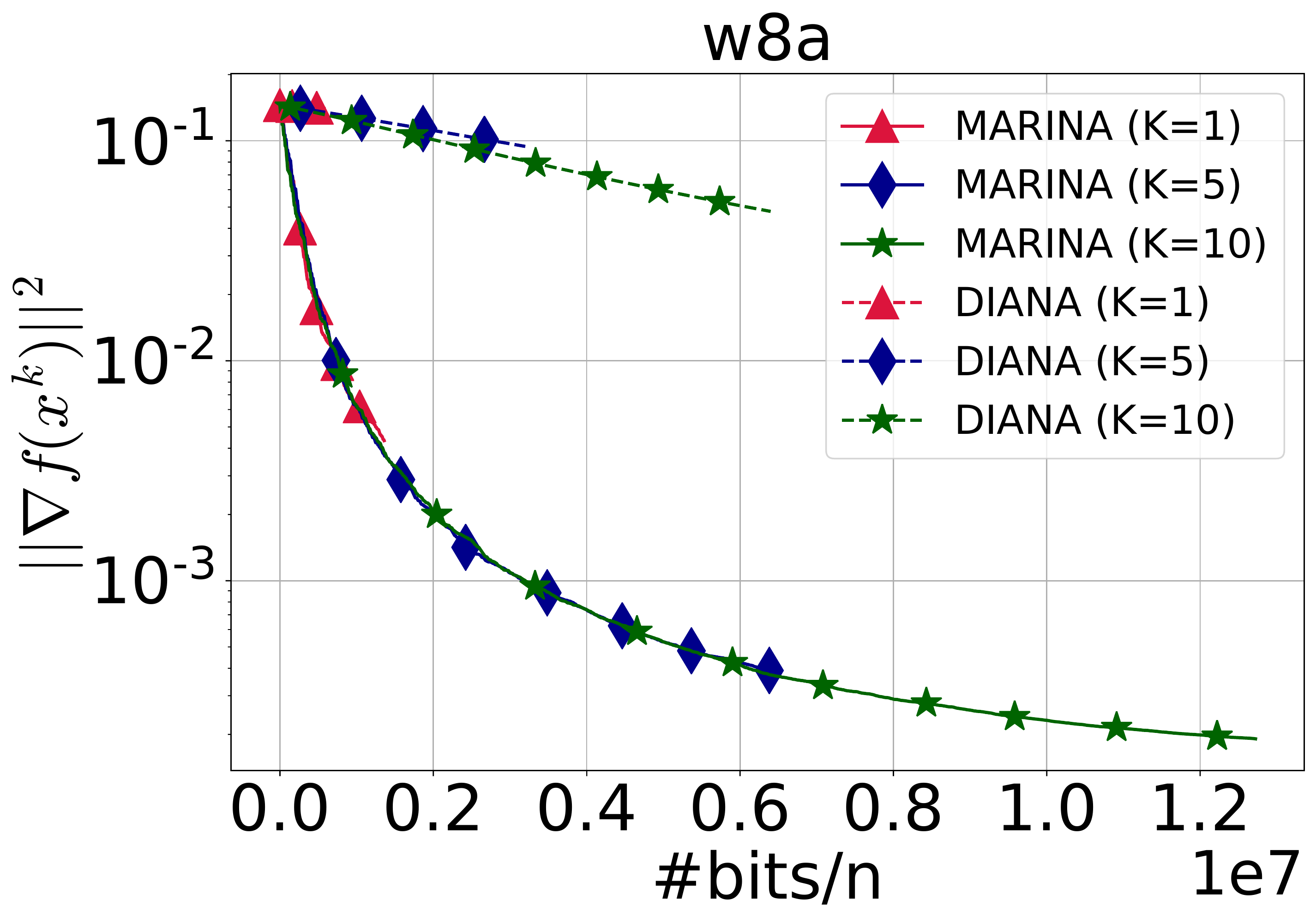}
\includegraphics[width=0.35\textwidth]{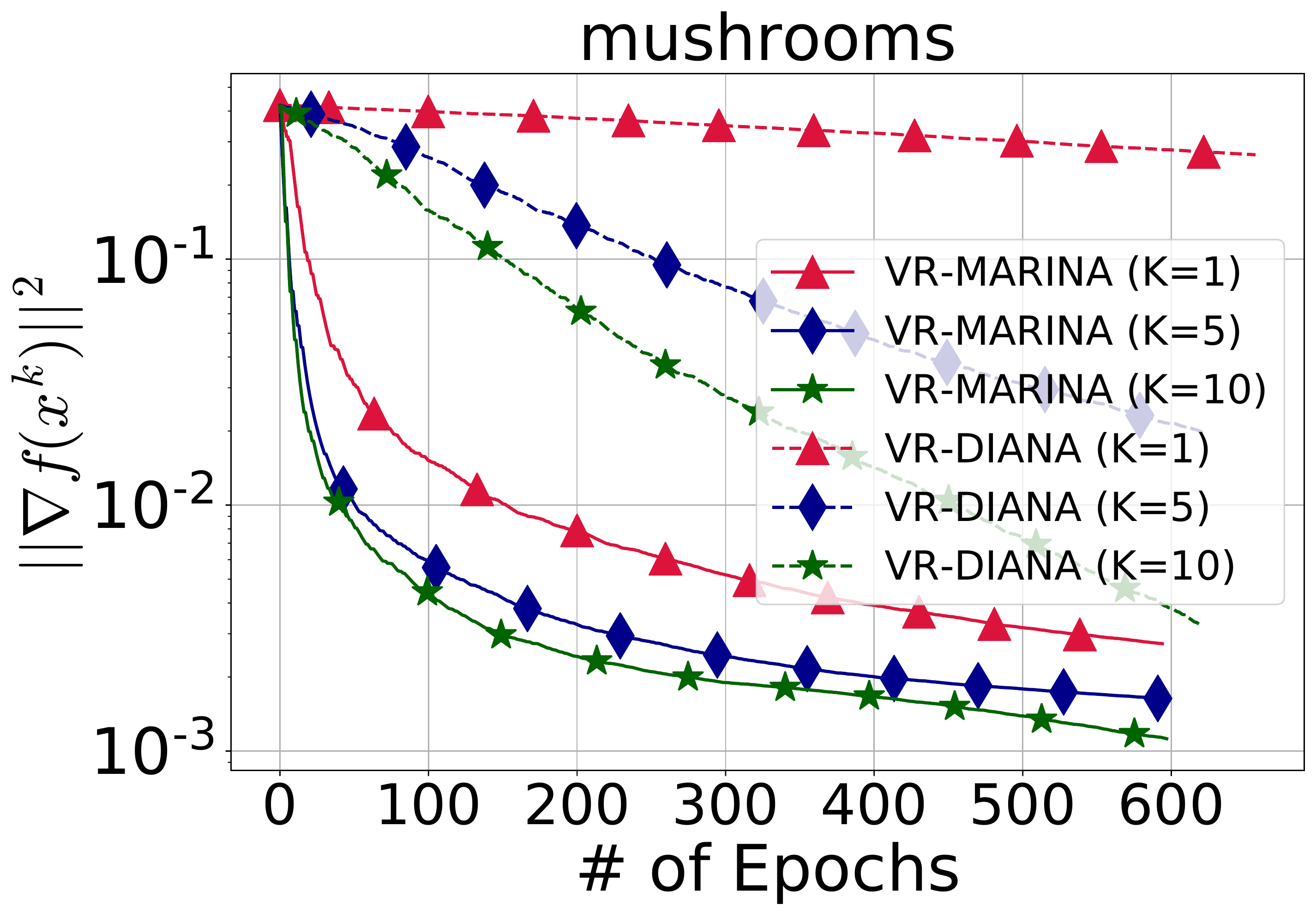}
\includegraphics[width=0.35\textwidth]{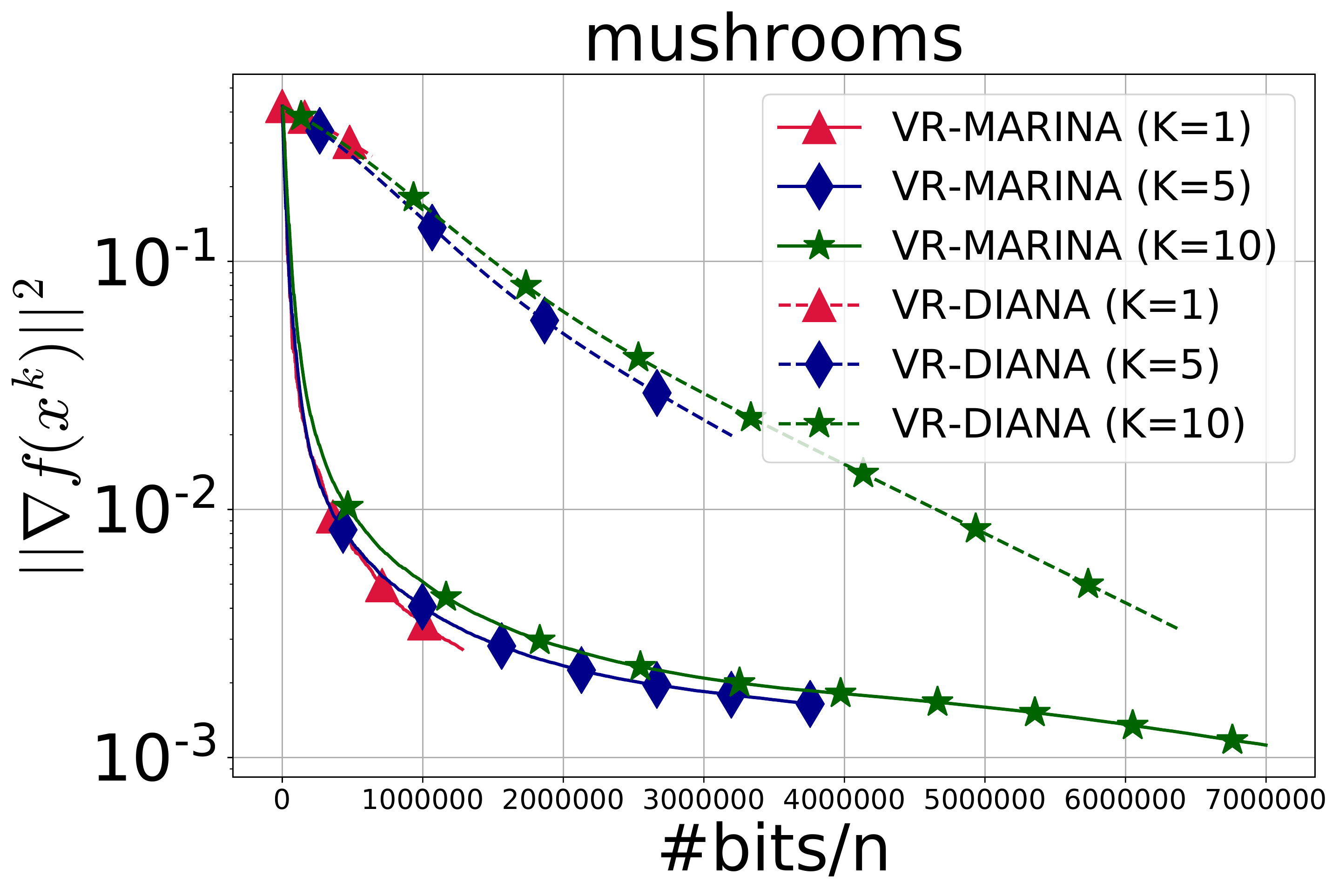}
\caption{Comparison of \algname{MARINA} with  \algname{DIANA}, and of \algname{VR-MARINA} with \algname{VR-DIANA}, on binary classification problem involving non-convex loss \eqref{eq:experiment_problem} with LibSVM data \cite{chang2011libsvm}. Parameter $n$ is chosen as per Tbl.~\ref{tbl:ns} in the Appendix. Stepsizes for the methods are chosen according to the theory and the batchsizes for \algname{VR-MARINA} and \algname{VR-DIANA} are $\sim \nicefrac{m}{100}$. In all cases, we used the RandK sparsification operator with K $\in \{1,5,10\}$.}
\label{fig:mushrooms_w8a_main}
\end{figure*}

\section{Partial Participation}\label{sec:pp}
Finally, we propose another modification of \algname{MARINA}. In particular, we prove an option for {\em partial participation} of the clients - a feature important in federated learning. The resulting method is called \algname{PP-MARINA} (see Algorithm~\ref{alg:pp_marina} in the Appendix). At each iteration of \algname{PP-MARINA}, the server receives the quantized gradient differences from $r$ clients with probability $1-p$, and aggregates full gradients from all clients with probability $p$, i.e., \algname{PP-MARINA} coincides with \algname{MARINA} up to the following difference: $g_i^{k+1} = \nabla f_i(x^{k+1})$, $g^{k+1} = \frac{1}{n}\sum_{i=1}^ng_i^{k+1}$ if $c_k = 1$, and $g_i^{k+1} = g^k + \cQ\left(\nabla f_{i}(x^{k+1}) - \nabla f_{i}(x^k))\right)$, $g^{k+1} = \frac{1}{r}\sum_{i_k\in I_k'}g_{i_k}^{k+1}$ otherwise, where $I_k'$ is the set of $r$ i.i.d.\ samples from the uniform distribution over $\{1,\ldots,n\}$. That is, if the probability $p$ is chosen to be small enough, then with high probability the server receives only quantized vectors from a subset of clients at each iteration.

Below, we provide a convergence result for \algname{PP-MARINA} for smooth  non-convex problems.
\begin{theorem}\label{thm:main_result_non_cvx_pp}
	Let Assumptions~\ref{as:lower_bound}~and~\ref{as:L_smoothness} be satisfied. Then, after
	\begin{equation}
	\squeeze	K = \cO\left(\frac{\Delta_0 L}{\varepsilon^2}\left(1 + \sqrt{\frac{(1-p)(1+\omega)}{pr}}\right)\right) \notag
	\end{equation}
	iterations with $\Delta_0 = f(x^0)-f_*$, $L^2 = \frac{1}{n}\sum_{i=1}^nL_i^2$ and the stepsize $\gamma \le L^{-1}\left(1 + \sqrt{\nicefrac{(1-p)(1+\omega)}{(pr)}}\right)^{-1}$, 
	\algname{PP-MARINA} produces  a point $\hat x^K$ for which  $\EE[\|\nabla f(\hat x^K)\|^2] \le \varepsilon^2$.
\end{theorem}
One can find the full statement of the theorem together with its proof in Section~\ref{sec:proof_of_thm_non_cvx_pp} of the appendix.
\begin{corollary}\label{cor:main_result_non_cvx_pp}
	Let the assumptions of Theorem~\ref{thm:main_result_non_cvx_pp} hold and choose $p = \nicefrac{\zeta_{\cQ}r}{(dn)}$, where $r\le n$. If $\gamma \le L^{-1}\left(1 + \sqrt{\nicefrac{(1+\omega)(dn-\zeta_{\cQ}r)}{(b'\zeta_{\cQ}r)}}\right)^{-1}$, 
	then \algname{PP-MARINA} requires 
	\begin{equation*}
	\squeeze	\cO\left(\frac{\Delta_0 L}{\varepsilon^2}\left(1 + \sqrt{\frac{1+\omega}{r}\left(\frac{dn}{\zeta_{\cQ}r}-1\right)}\right)\right)
	\end{equation*}
	iterations/communication rounds to achieve $\EE[\|\nabla f(\hat x^K)\|^2] \le \varepsilon^2$, and the expected total communication cost is $\cO\left(dn + \zeta_{\cQ}rK\right)$.
\end{corollary}

When $r=n$, i.e., all clients participate in communication with the server at each iteration, the rate for \algname{PP-MARINA} recovers the rate for \algname{MARINA} under the assumption that $(1+\omega)(\nicefrac{d}{\zeta_{\cQ}}-1) = \cO(\omega(\nicefrac{d}{\zeta_{\cQ}}-1))$, which holds for a wide class of quantization operators, e.g., for identical quantization, RandK, and $\ell_p$-quantization. In general, the derived complexity is strictly better than previous state-of-the-art one (see Table~\ref{tab:comparison}).

We provide the convergence results for \algname{PP-MARINA} under the Polyak-{\L}ojasiewicz condition, together with complete proofs, in Section~\ref{sec:proof_of_thm_pl_pp} of the Appendix.

\section{Numerical Experiments}
\subsection{Binary Classification with Non-Convex Loss}\label{sec:experiments}
We conduct several numerical experiments\footnote{Our code is available at \url{https://github.com/burlachenkok/marina}.} on binary classification problem involving non-convex loss \cite{zhao2010convex} (used for two-layer neural networks) with LibSVM data \cite{chang2011libsvm} to justify the theoretical claims of the paper. That is, we consider the following optimization problem:
\begin{equation}
	\min\limits_{x\in\R^d}\left\{f(x) = \frac{1}{N}\sum\limits_{t=1}^N \ell(a_t^\top x, y_i)\right\},\label{eq:experiment_problem}
\end{equation}
where $\{a_t\} \in \R^d$, $y_i\in\{-1,1\}$ for all $t=1,\ldots,N$, and the function $\ell:\R^d \to \R$ is defined as
\begin{equation*}
	\ell(b,c) = \left(1 - \frac{1}{1+\exp(-bc)}\right)^2.
\end{equation*}
The distributed environment is simulated in Python 3.8 using \textsc{mpi4py} and other standard libraries. Additional details about the experimental setup together with extra experiments are deferred to Section~\ref{sec:extra_experiments} of the Appendix.

In our experiments, we compare \algname{MARINA} with the full-batch version of \algname{DIANA}, and then \algname{VR-MARINA} with \algname{VR-DIANA}. We exclude \algname{FedCOMGATE} and \algname{FedPATH} from this comparison since they have significantly worse oracle complexities (see Table~\ref{tab:comparison}). The results are presented in Fig.~\ref{fig:mushrooms_w8a_main}. As our theory predicts, the first row shows the superiority of \algname{MARINA} to \algname{DIANA} both in terms of iteration/communication complexity and the total number of transmitted bits to achieve the given accuracy. Next, to study the oracle complexity as well, we consider non-full-batched methods -- \algname{VR-MARINA} and \algname{VR-DIANA} -- since they have better oracle complexity than the full-batched methods in the finite-sum case. Again, the results presented in the second row justify that \algname{VR-MARINA} outperforms \algname{VR-DIANA} in terms of oracle complexity and the total number of transmitted bits to achieve the given accuracy.

\subsection{Image Classification}\label{sec:NN_experiments}
We also compared the performance of \algname{VR-MARINA} and \algname{VR-DIANA} on the training {\tt ResNet-18} \cite{he2016deep} at {\tt CIFAR100} \cite{krizhevsky2009learning} dataset. Formally, the optimization problem is
\begin{equation}
	\min\limits_{x\in \R^d}\left\{f(x) = \frac{1}{N}\sum\limits_{i=1}^N\ell(p(f(a_i, x)), y_i)\right\}, \label{eq:NN_problem}
\end{equation} 
where $\{(a_i,y_i)\}_{i=1}^N$ encode images and labels from {\tt CIFAR100} dataset, $f(a_i,x)$ is the output of {\tt ResNet-18} on image $a_i$ with weights $x$, $p$ is softmax function, and $\ell(\cdot,\cdot)$ is cross-entropy loss. The code is wrtitten in Python 3.9 using {\tt PyTorch 1.7}, and the distributed environment is simulated.

The results are presented in Fig.~\ref{fig:resnet_at_cifar100}. Again, \algname{VR-MARINA} converges significantly faster than \algname{VR-DIANA} both in terms of the oracle complexity and the total number of transmitted bits to achieve the given accuracy. See other details and observations in Section~\ref{sec:extra_experiments} of the Appendix.
\begin{figure}[h]
\centering
\includegraphics[width=0.23\textwidth]{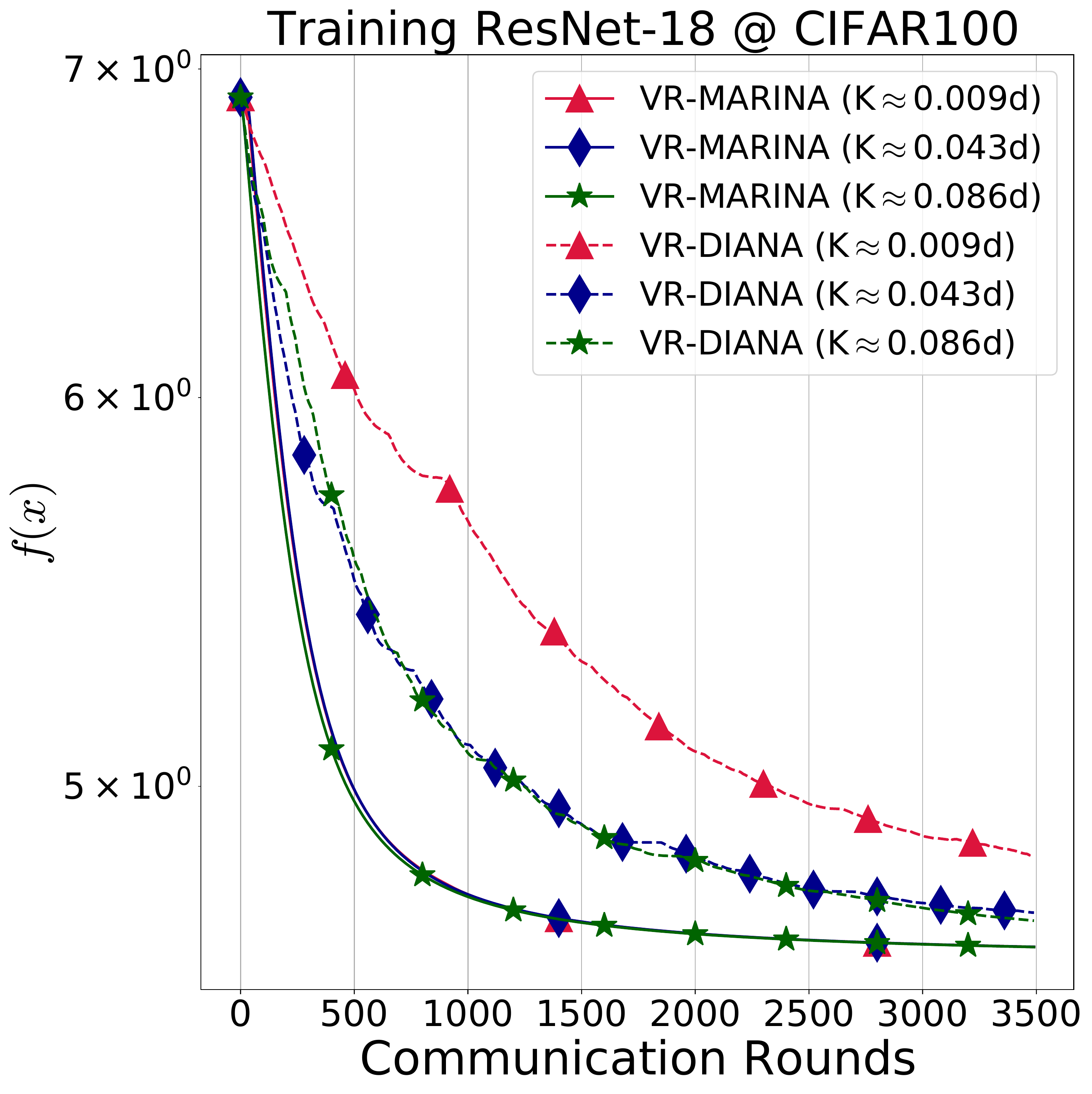}
\includegraphics[width=0.23\textwidth]{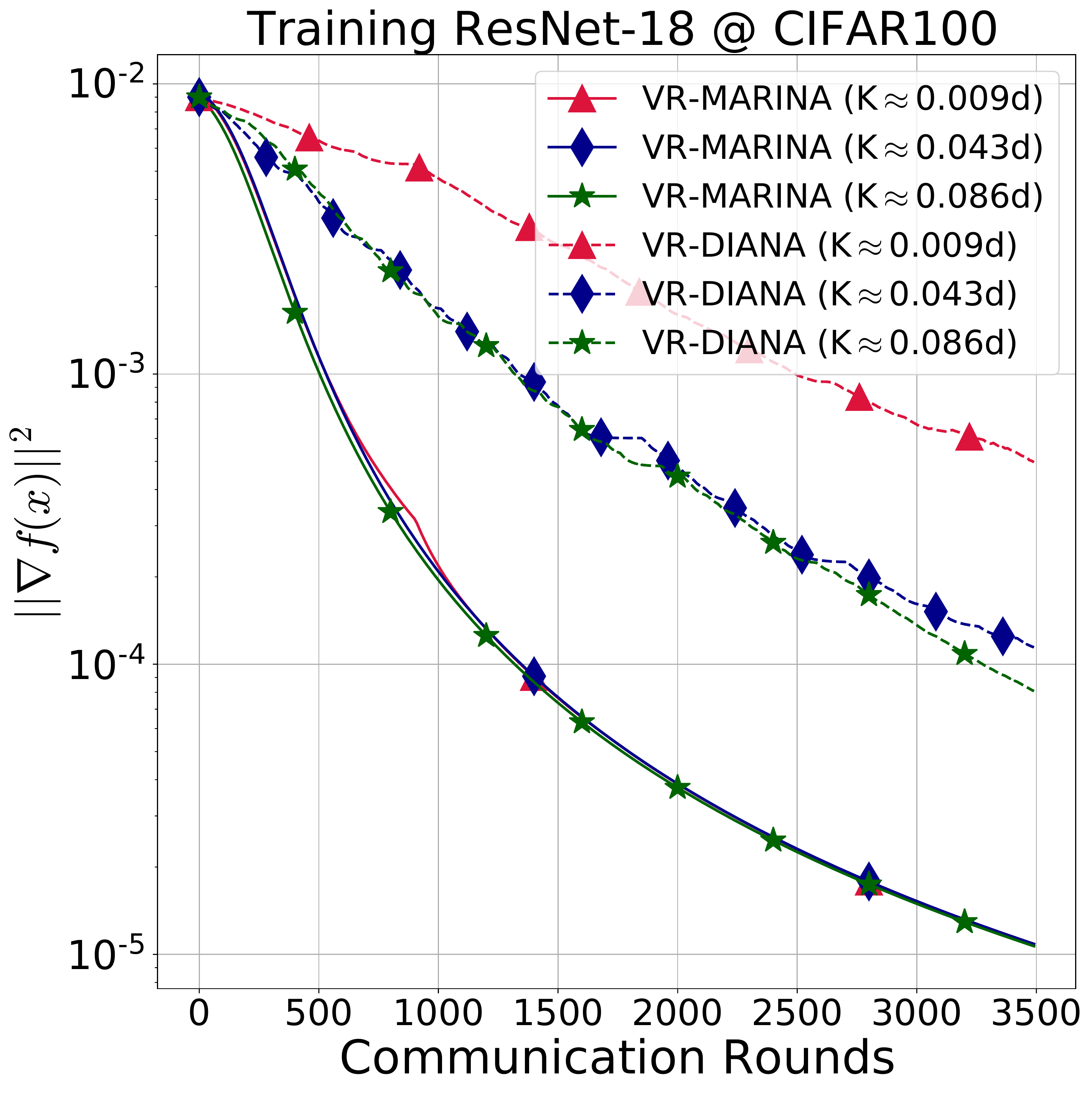}
\includegraphics[width=0.23\textwidth]{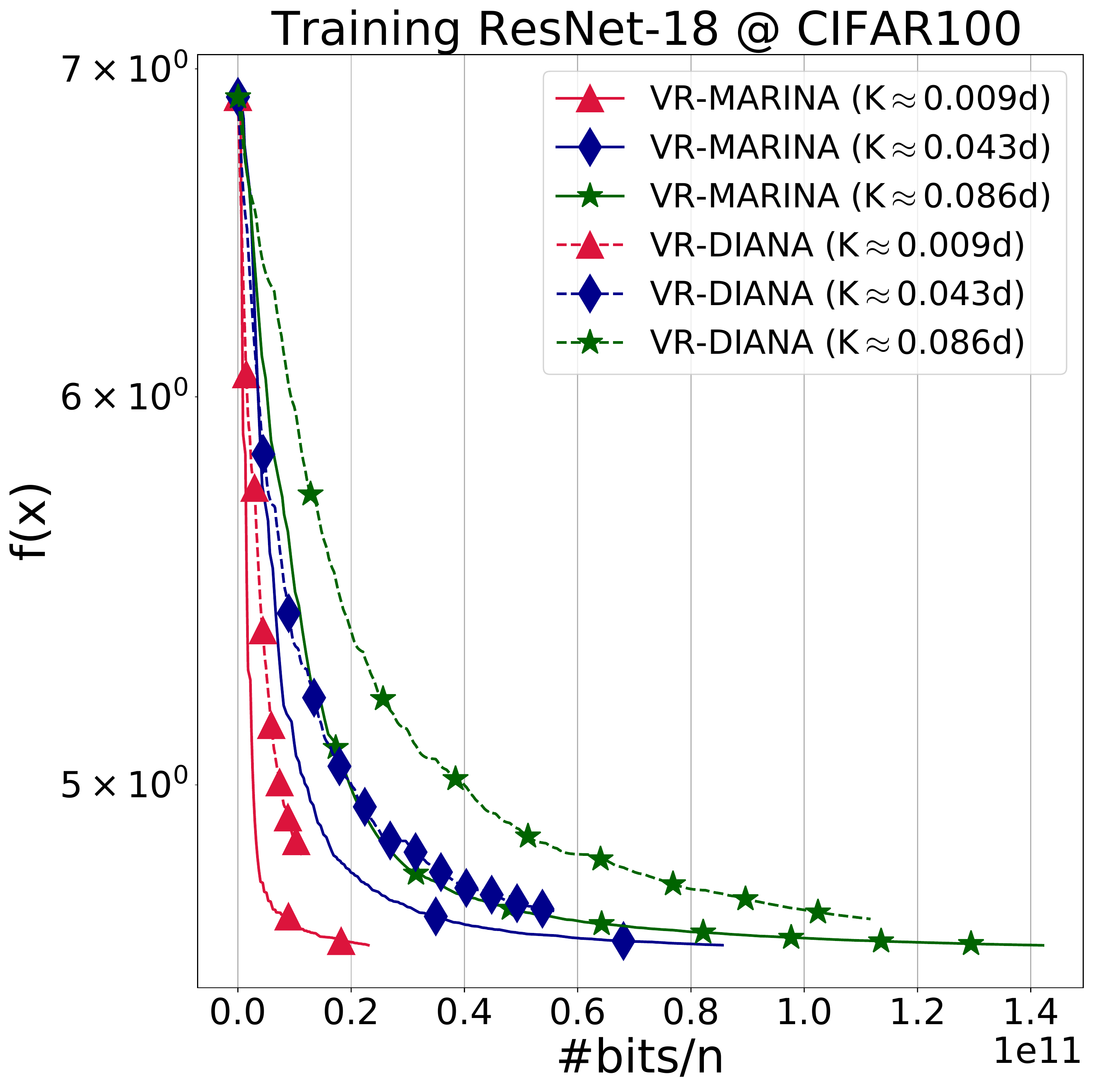}
\includegraphics[width=0.23\textwidth]{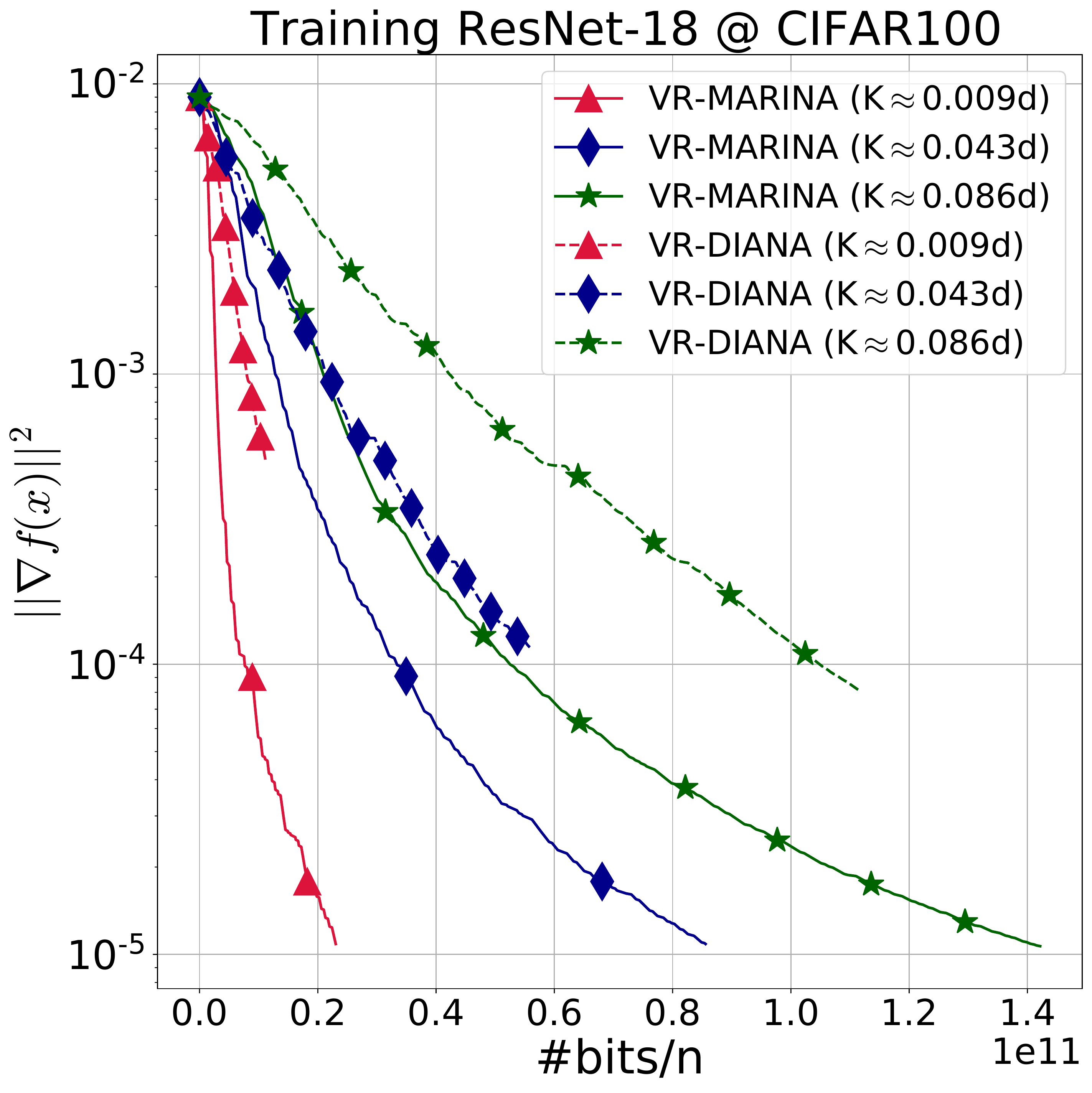}
\includegraphics[width=0.23\textwidth]{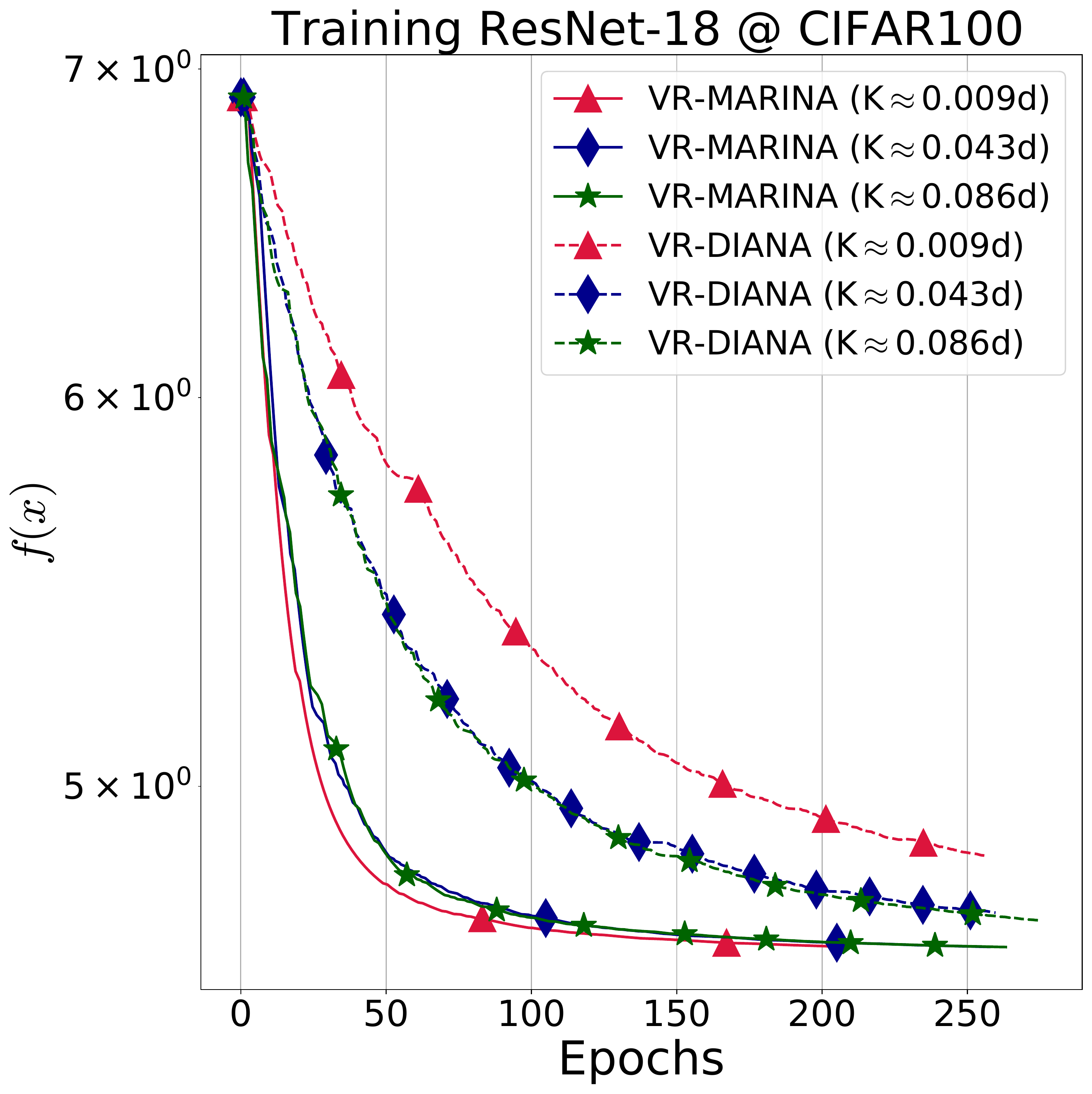}
\includegraphics[width=0.23\textwidth]{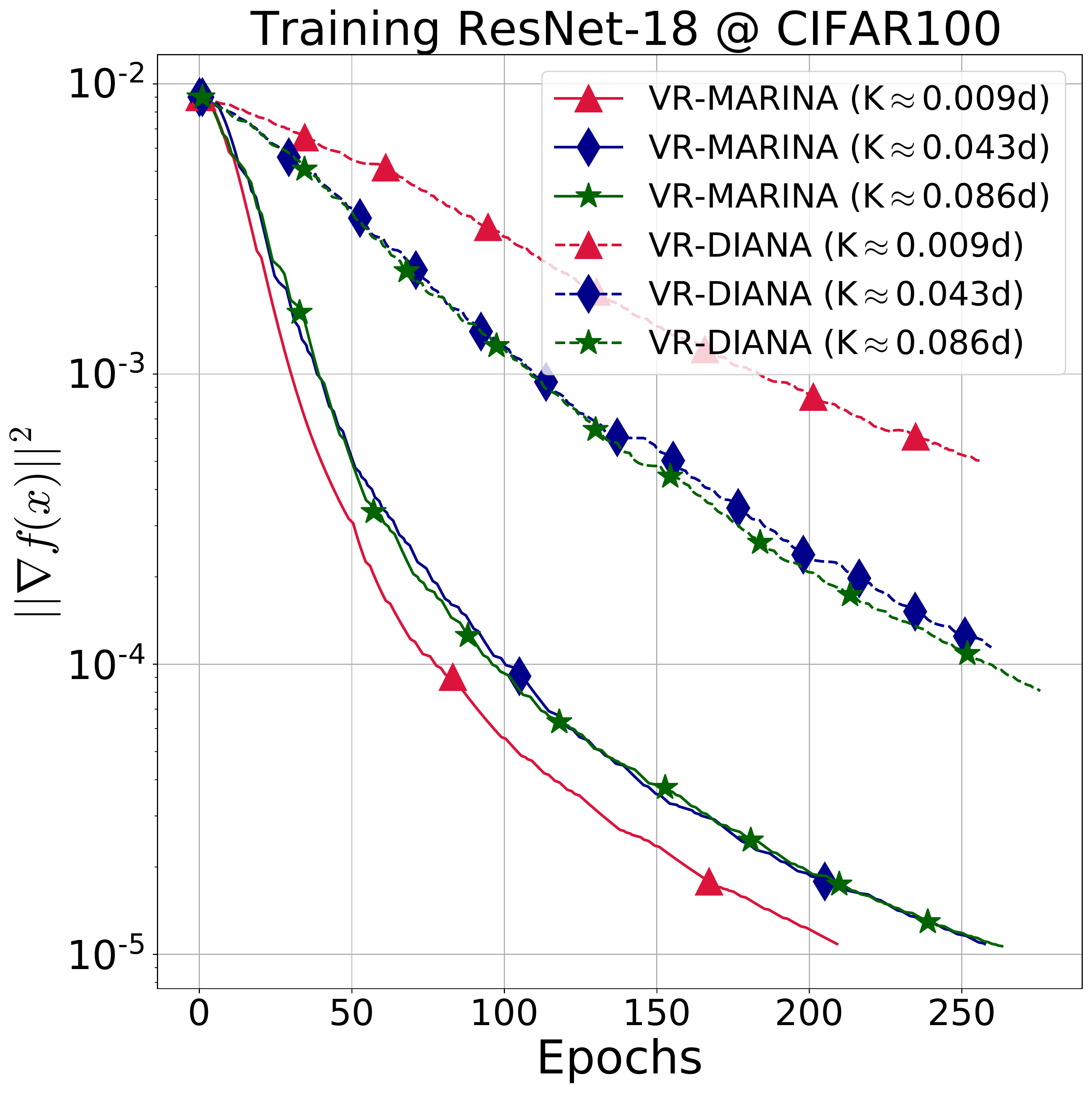}
\caption{Comparison of \algname{VR-MARINA} with \algname{VR-DIANA} on training {\tt ResNet-18} at {\tt CIFAR100} dataset. Number of workers equals $5$. Stepsizes for the methods were tuned and the batchsizes are $\sim \nicefrac{m}{50}$. In all cases, we used the RandK sparsification operator, the approximate values of $K$ are given in the legends ($d$ is dimension of the problem).}
\label{fig:resnet_at_cifar100}
\end{figure}

\section*{Acknowledgements}
The work of Peter Richt\'{a}rik, Eduard Gorbunov, Konstantin Burlachenko and Zhize Li was supported by KAUST Baseline Research Fund. The paper was written while E.~Gorbunov was a research intern at KAUST. The work of E.~Gorbunov in Sections~\ref{sec:intro},~\ref{sec:marina},~and~\ref{sec:marina_proofs} was also partially supported by the Ministry of Science and Higher Education of the Russian Federation (Goszadaniye) 075-00337-20-03, project No. 0714-2020-0005, and in Sections~\ref{sec:vr},~\ref{sec:pp},~\ref{sec:missing_proofs},~\ref{sec:pp_marina_proofs} -- by RFBR, project number 19-31-51001.
We thank Konstantin Mishchenko (KAUST) for a suggestion related to the experiments, Elena Bazanova (MIPT) for the suggestions about improving the text, and Slavom{\'{i}}r Hanzely (KAUST), Egor Shulgin (KAUST), and Alexander Tyurin (KAUST) for spotting the typos.



\bibliography{refs}
\bibliographystyle{icml2021}

\clearpage
\onecolumn

\part*{Appendix}

\appendix


%
%
\section{Extra Experiments}\label{sec:extra_experiments}
\subsection{Binary Classification with Non-Convex Loss}
\subsubsection{Setup}

In Section~\ref{sec:experiments}, we present the behavior of \algname{MARINA}, \algname{VR-MARINA}, \algname{DIANA}, and \algname{VR-DIANA} on the binary classification problem involving non-convex loss \cite{zhao2010convex}. The datasets were taken from LibSVM \cite{chang2011libsvm} and split into five equal parts among five clients (we excluded $N - 5\cdot\lfloor\nicefrac{N}{5}\rfloor$ last datapoints from each dataset), see the summary in Table~\ref{tbl:ns}.
\begin{table}[!h]
 \caption{Summary of the datasets and splitting of the data among clients (Figure~\ref{fig:mushrooms_w8a_main}).}
\label{tbl:ns}
\begin{center}
\begin{tabular}{|c|c|c|c|}
\hline
Dataset  & $n$ & $N$ (\# of datapoints) & $d$ (\# of features)   \\
 \hline
  \hline
\texttt{mushrooms} & 5 & 8 120 & 112   \\ \hline
\texttt{w8a} & 5  &49 745 & 300  \\ \hline
\texttt{phishing} & 5  &11 055 & 69  \\ \hline
\texttt{a9a} & 5  &32 560 & 124  \\ \hline
\end{tabular}
\end{center}
\end{table}

The code was written in Python 3.8 using \textsc{mpi4py} to emulate the distributed environment and then was executed on a machine with 48 cores, each is Intel(R) Xeon(R) Gold 6246 CPU 3.30GHz.

\subsubsection{Extra Experiments}
In this section, we provide additional numerical results on the comparison of \algname{MARINA}, \algname{VR-MARINA}, \algname{DIANA}, and \algname{VR-DIANA} on the problem \eqref{eq:experiment_problem}. Since one of the main goals of our experiments is to justify the theoretical findings of the paper, in the experiments, we used the stepsizes from the corresponding theoretical results for the methods (for \algname{DIANA} and \algname{VR-DIANA} the stepsizes were chosen according to \cite{horvath2019stochastic,li2020unified}). Next, to compute the stochastic gradients, we use batchsizes $= \max\{1, \nicefrac{m}{100}\}$ for \algname{VR-MARINA} and \algname{VR-DIANA}.

The results for the full-batched methods are reported in Figure~\ref{fig:full_batched_methods}, and the comparison of \algname{VR-MARINA} and \algname{VR-DIANA} is given in Figure~\ref{fig:vr_methods}. Clearly, in both cases, \algname{MARINA} and \algname{VR-MARINA} show faster convergence than the previous state-of-the-art methods, \algname{DIANA} and \algname{VR-DIANA}, for distributed non-convex optimization with compression in terms of $\|\nabla f(x^k)\|^2$ and $f(x^k)$ decrease w.r.t.\ the number of communication rounds, oracle calls per node and the total number of transferred bits from workers to the master.

\begin{figure*}[h!]
\centering
\includegraphics[width=0.24\textwidth]{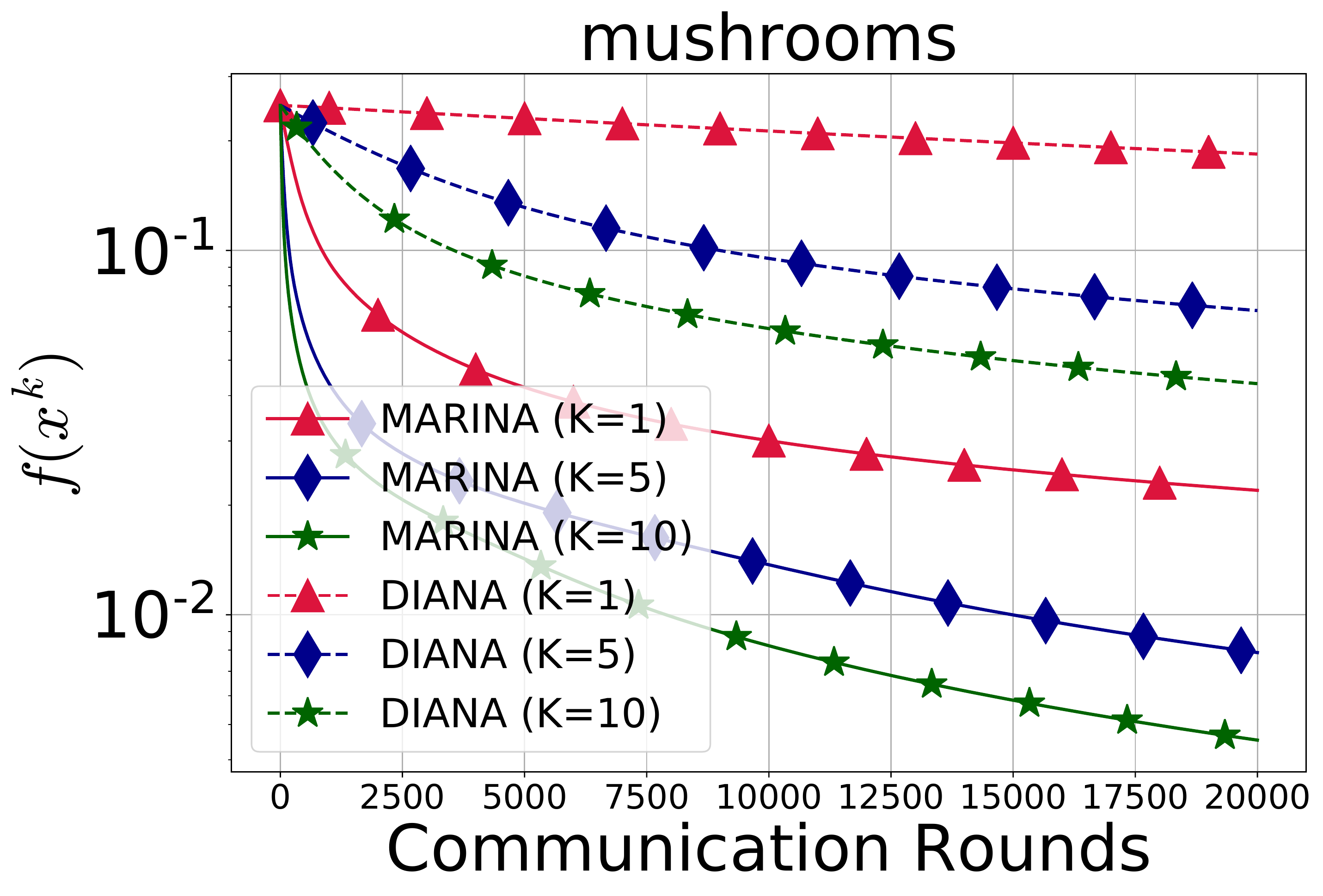}
\includegraphics[width=0.24\textwidth]{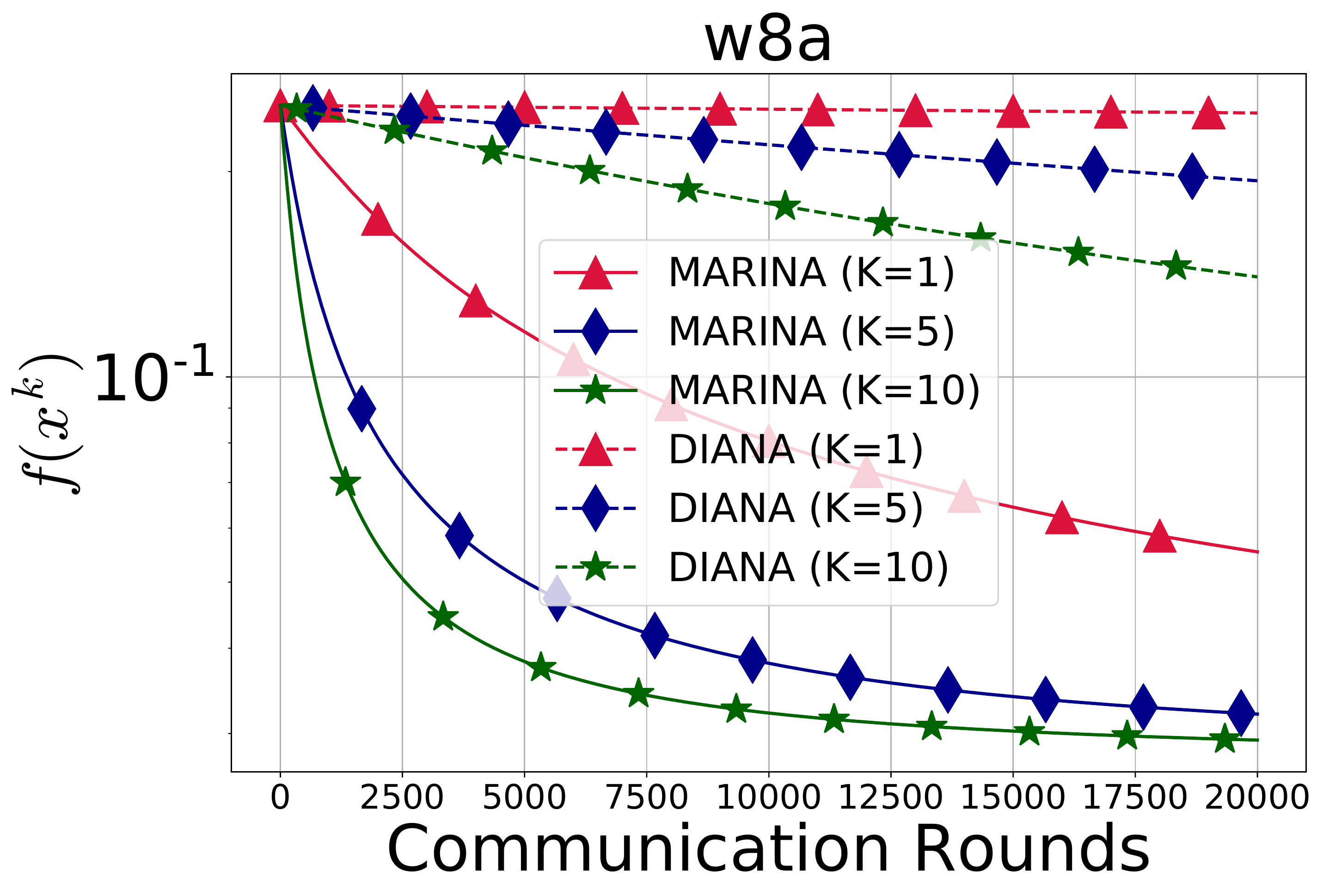}
\includegraphics[width=0.24\textwidth]{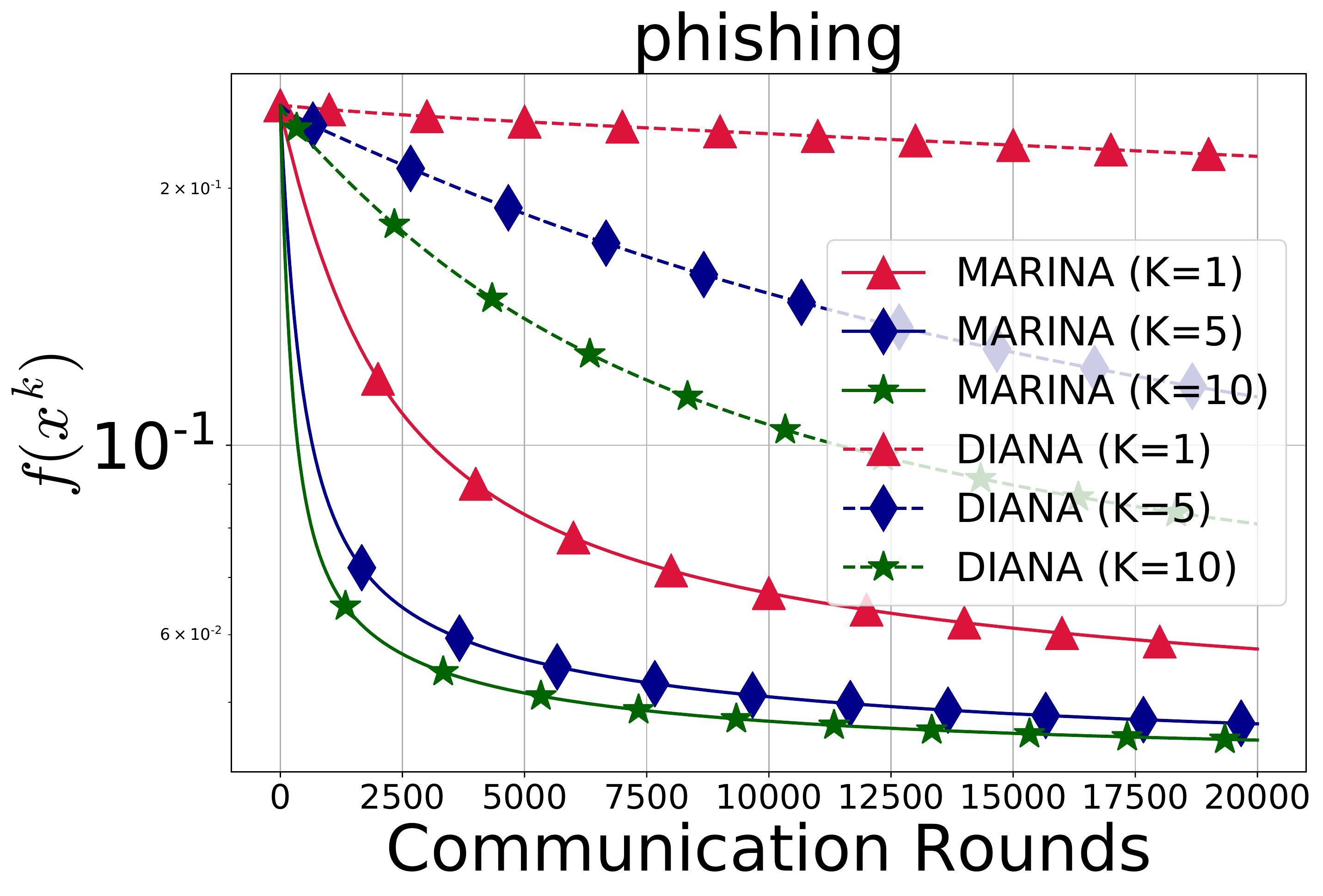}
\includegraphics[width=0.24\textwidth]{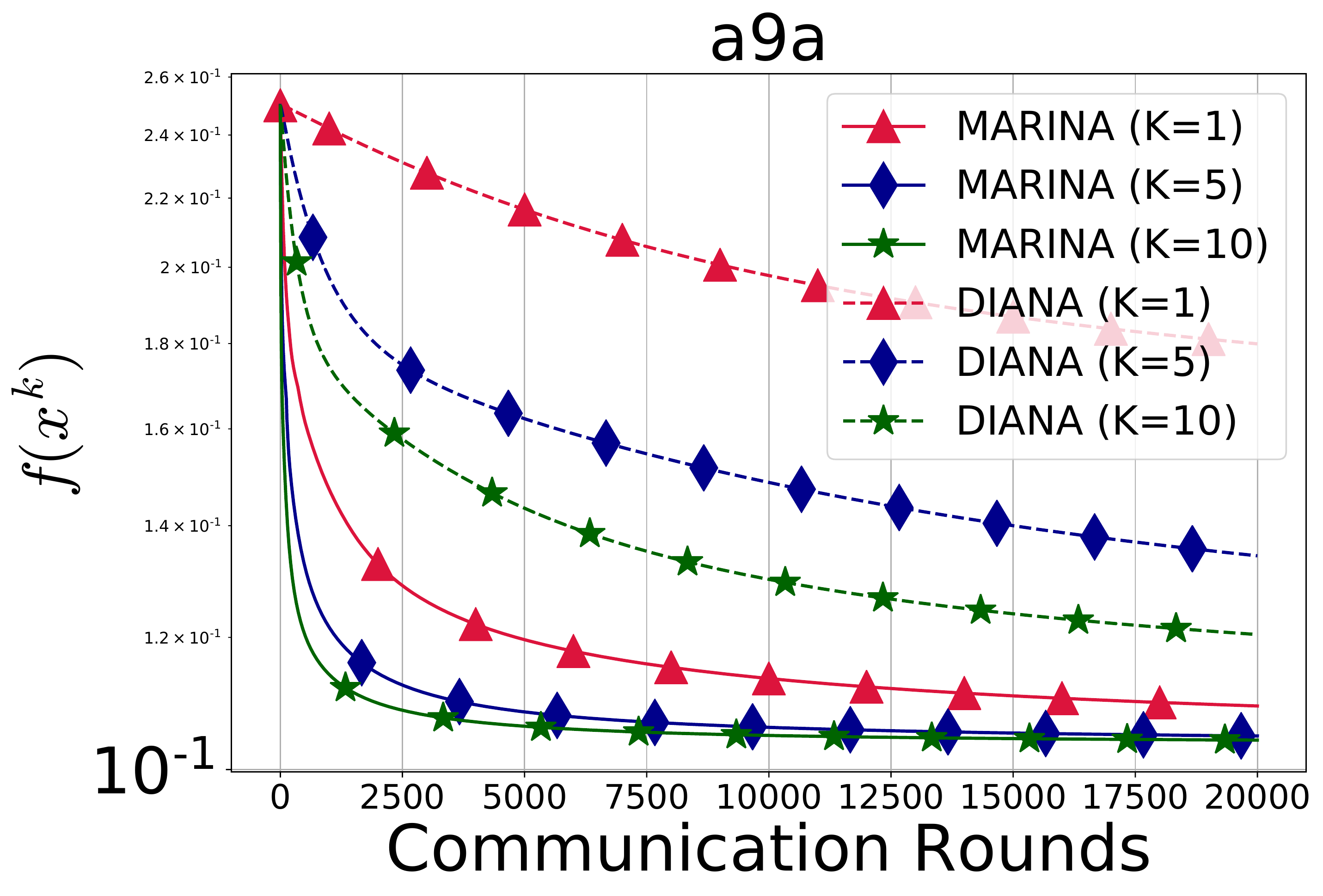}
\includegraphics[width=0.24\textwidth]{mushrooms_grad_norm_iters_marina_diana.pdf}
\includegraphics[width=0.24\textwidth]{w8a_grad_norm_iters_marina_diana.pdf}
\includegraphics[width=0.24\textwidth]{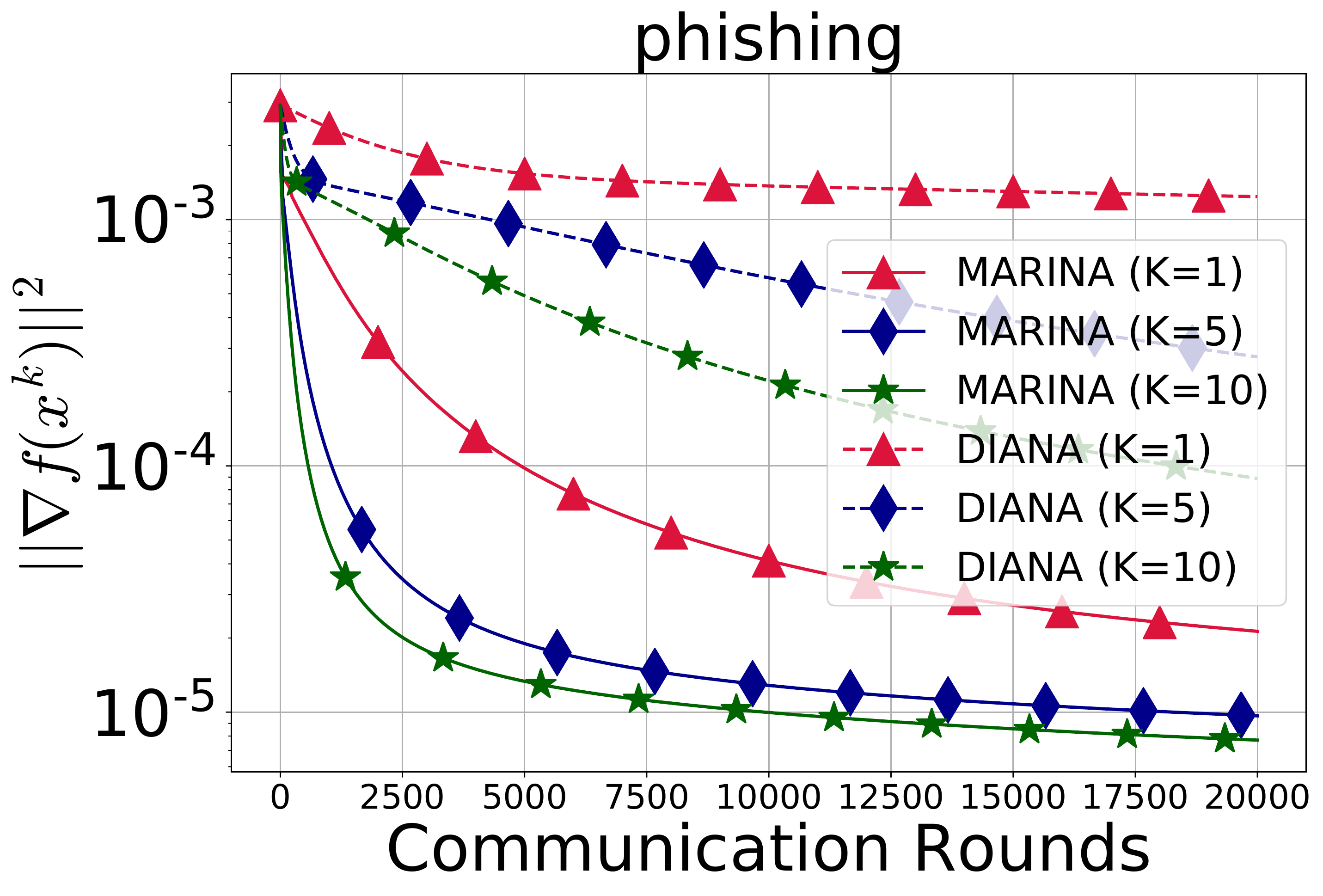}
\includegraphics[width=0.24\textwidth]{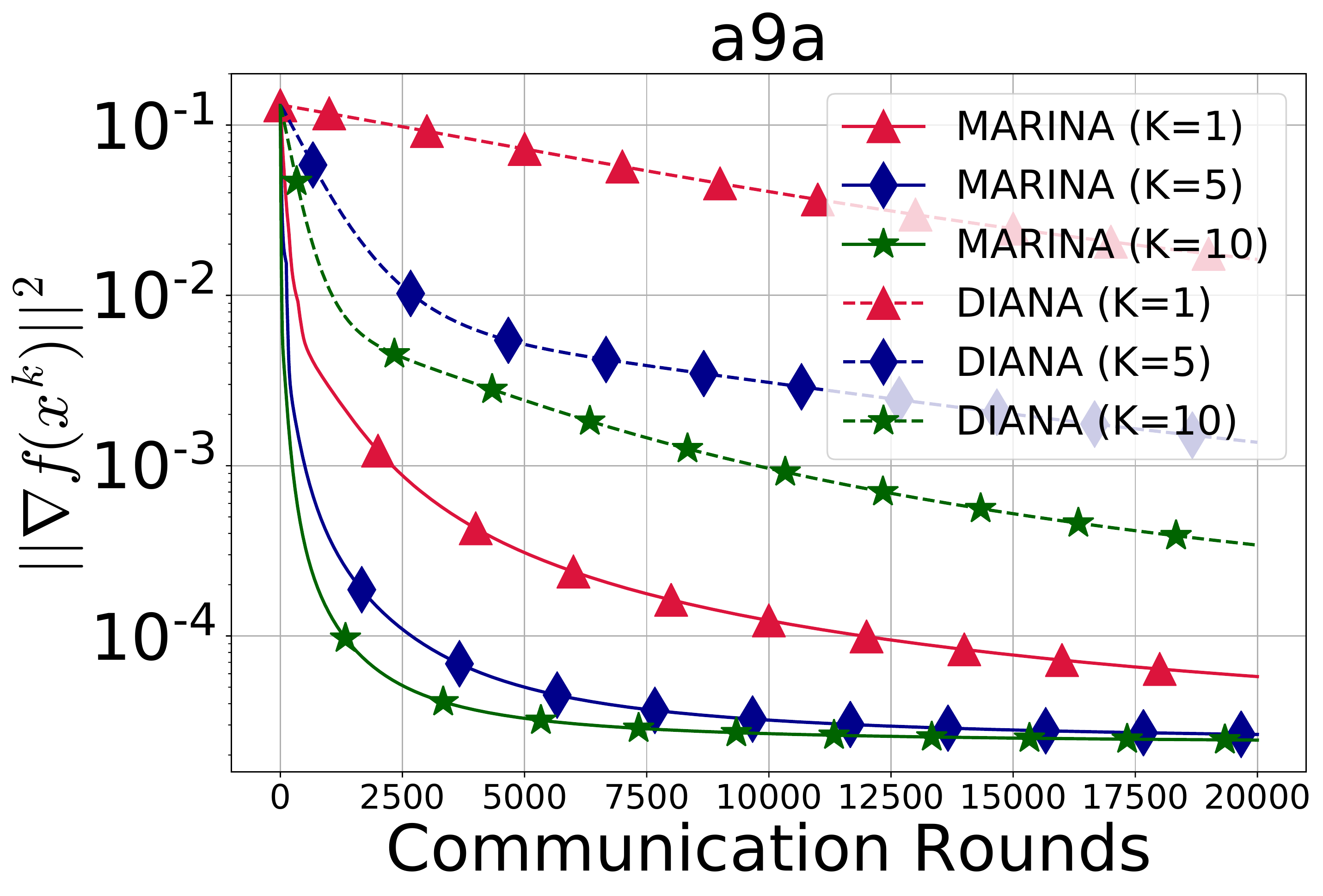}
\includegraphics[width=0.24\textwidth]{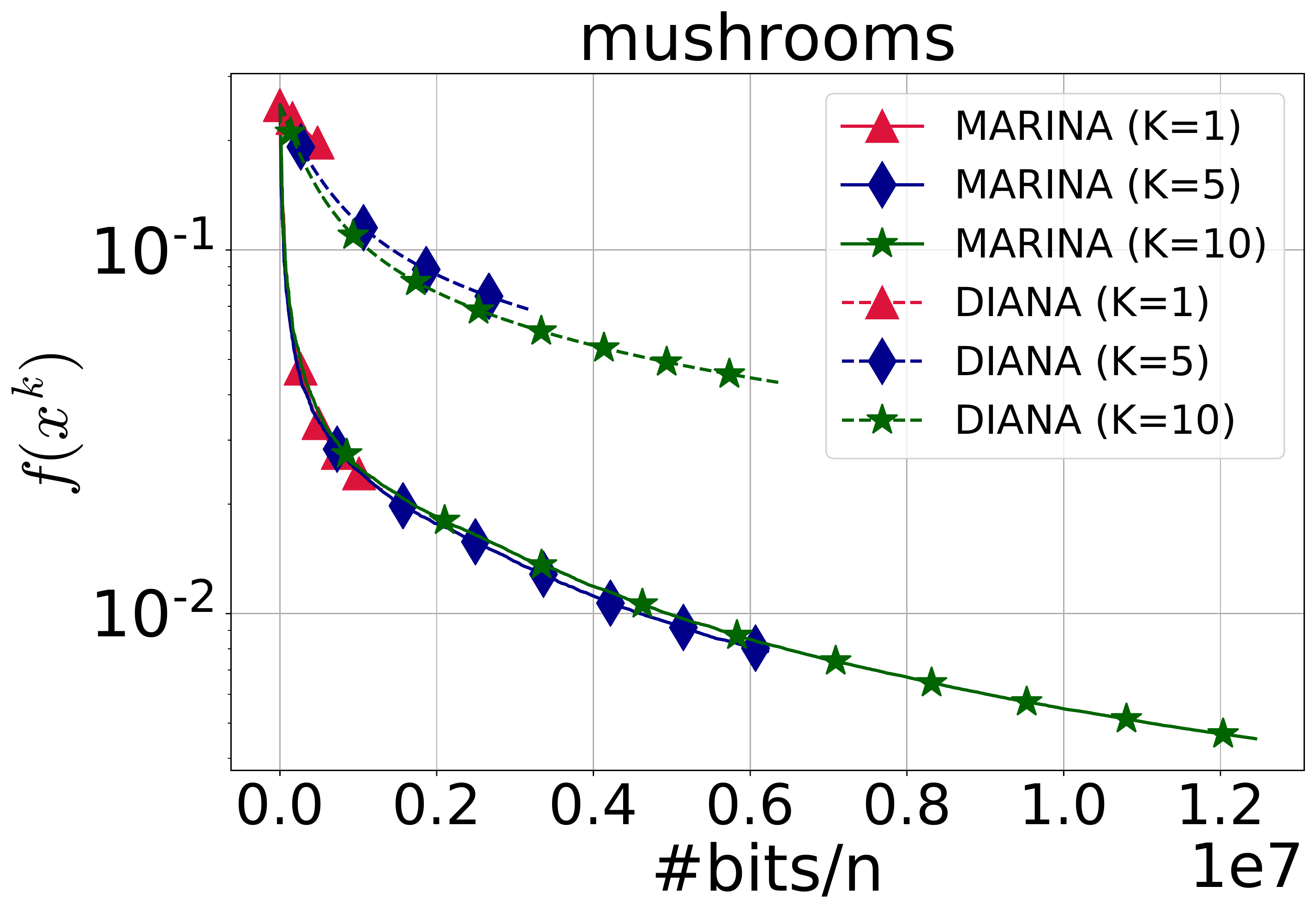}
\includegraphics[width=0.24\textwidth]{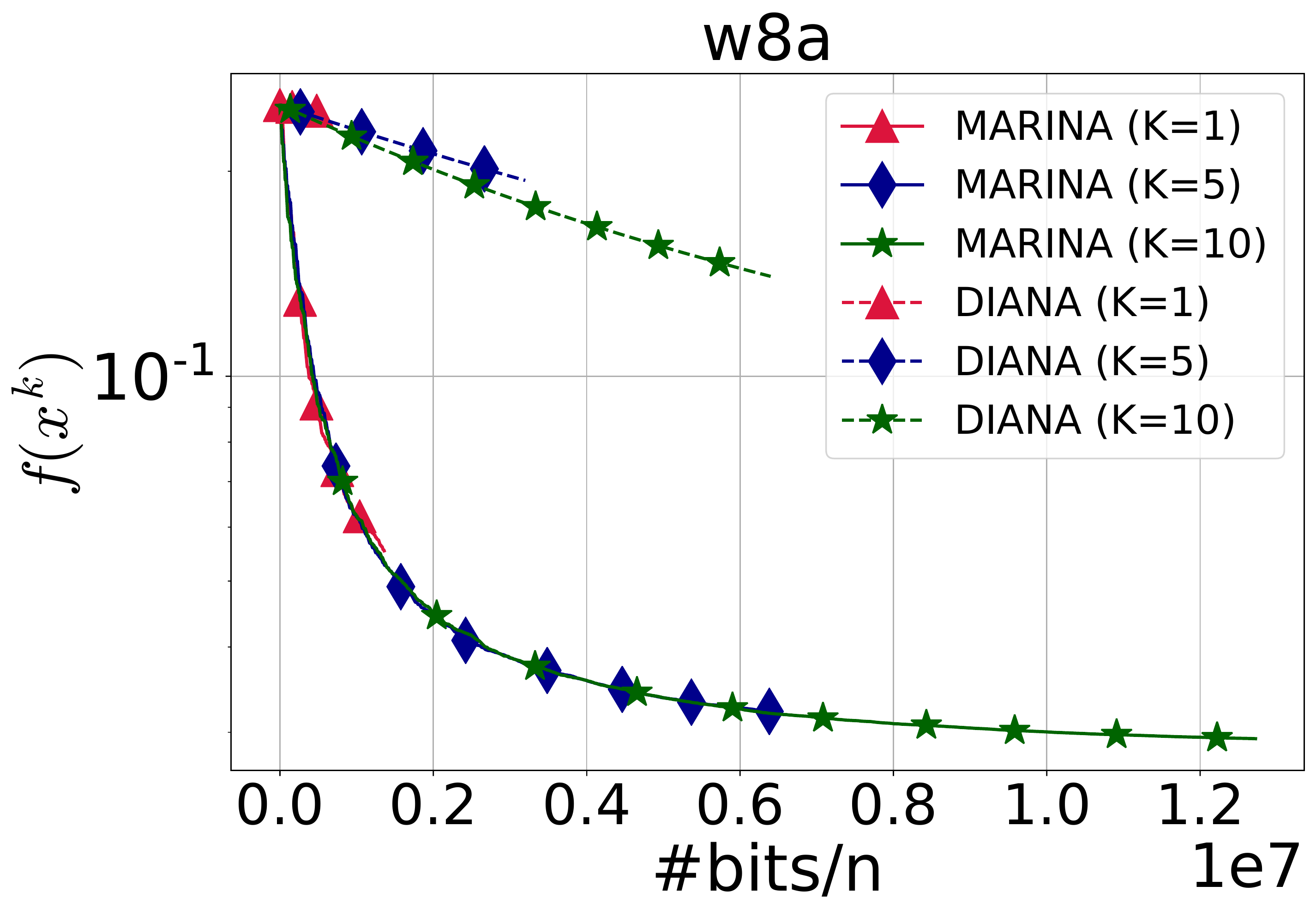}
\includegraphics[width=0.24\textwidth]{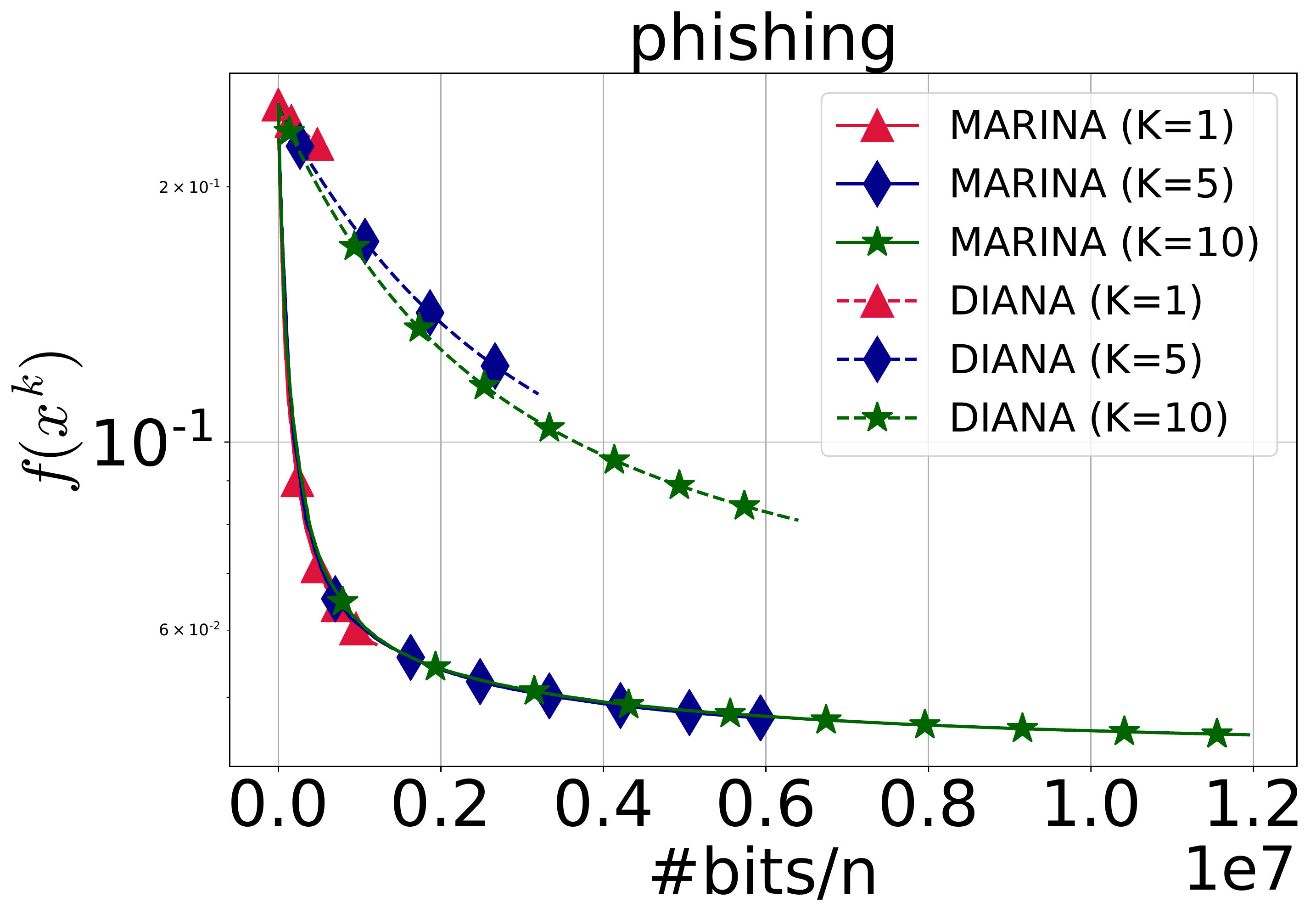}
\includegraphics[width=0.24\textwidth]{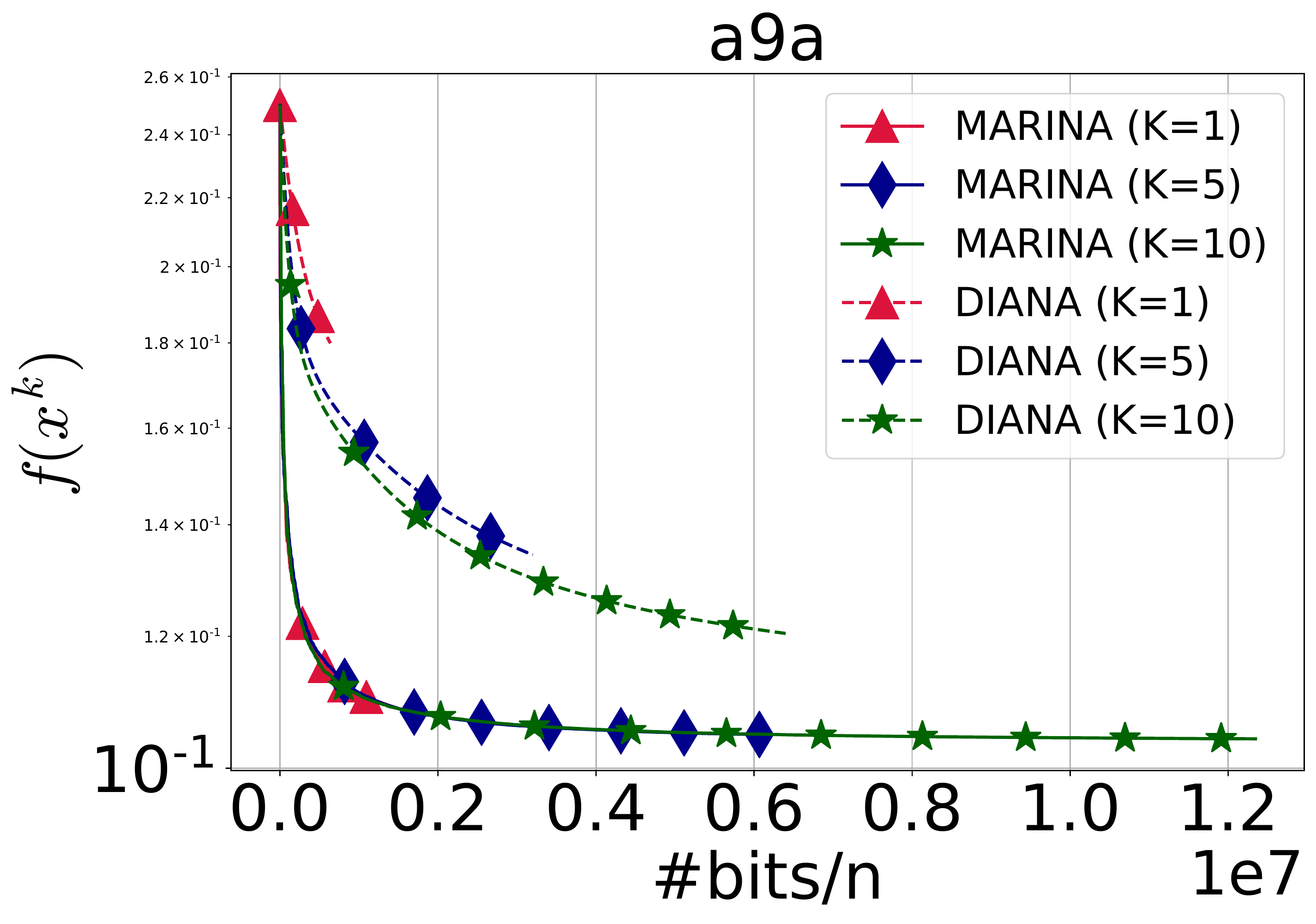}
\includegraphics[width=0.24\textwidth]{mushrooms_grad_norm_bits_marina_diana.pdf}
\includegraphics[width=0.24\textwidth]{w8a_grad_norm_bits_marina_diana.pdf}
\includegraphics[width=0.24\textwidth]{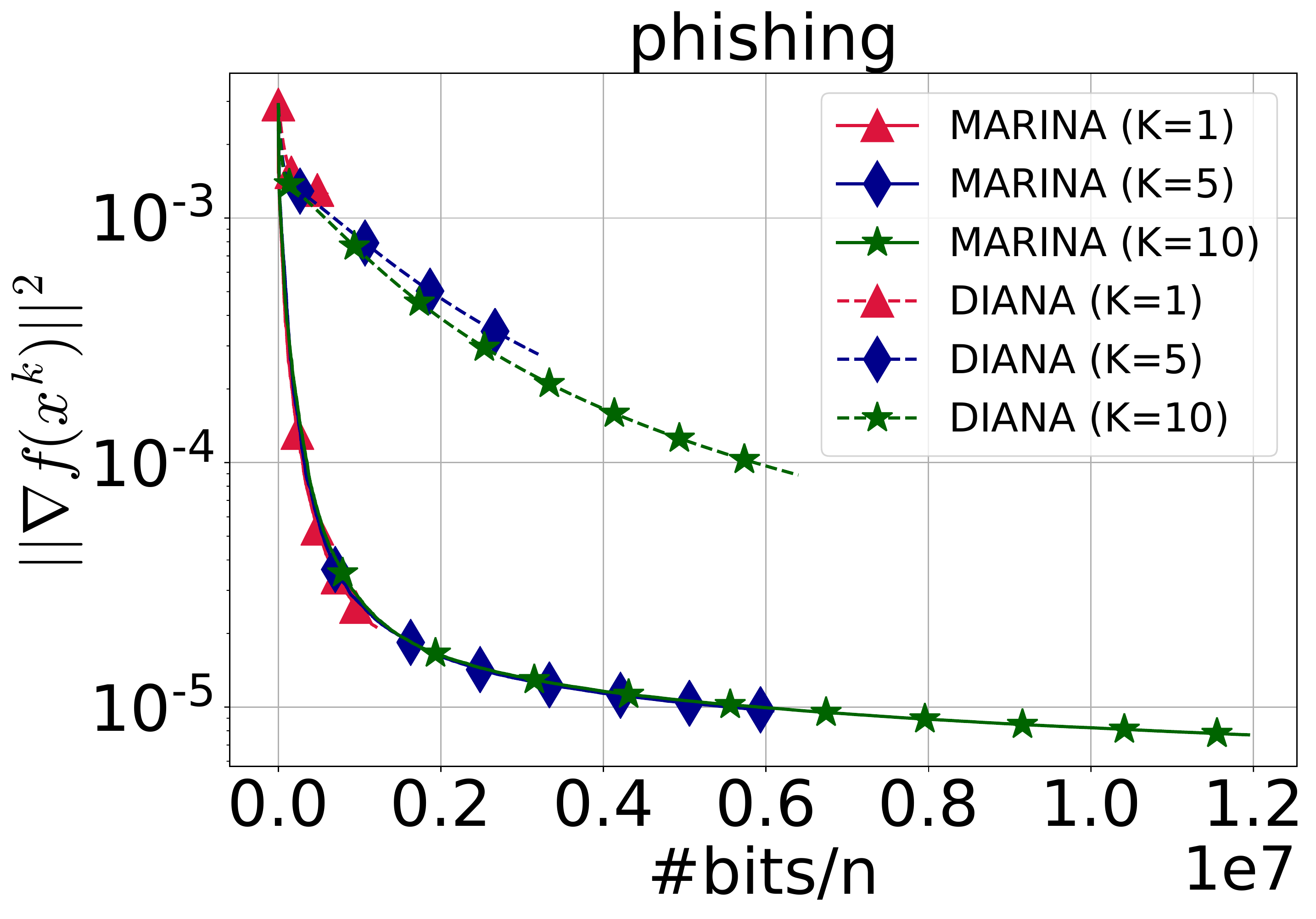}
\includegraphics[width=0.24\textwidth]{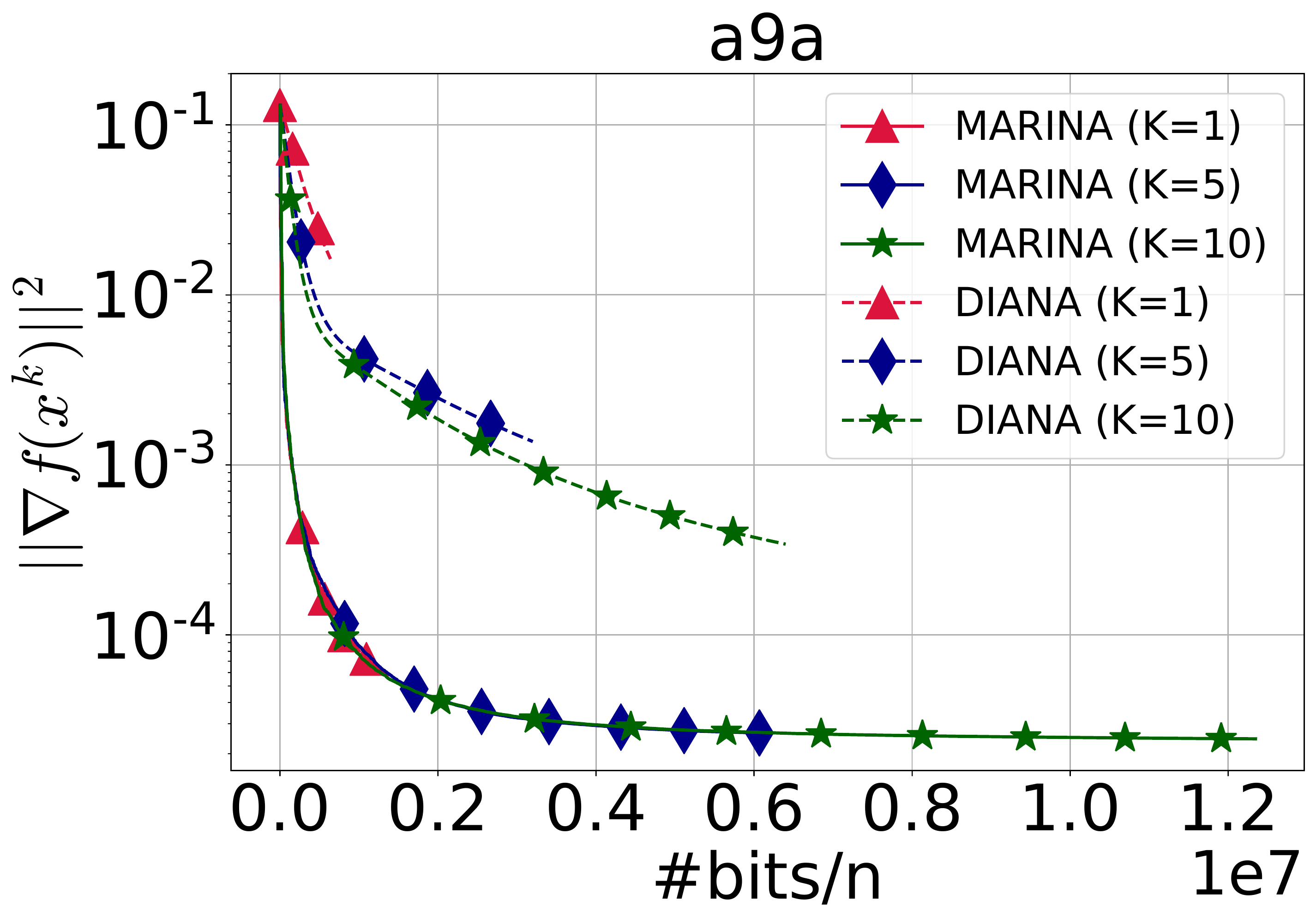}
\caption{Comparison of \algname{MARINA} with  \algname{DIANA} on binary classification problem involving non-convex loss \eqref{eq:experiment_problem} with LibSVM data \cite{chang2011libsvm}. Parameter $n$ is chosen as per Tbl.~\ref{tbl:ns} ($n = 5$). Stepsizes for the methods are chosen according to the theory. In all cases, we used the RandK sparsification operator with K $\in \{1,5,10\}$.}
\label{fig:full_batched_methods}
\end{figure*}

\begin{figure*}[t!]
\centering
\includegraphics[width=0.24\textwidth]{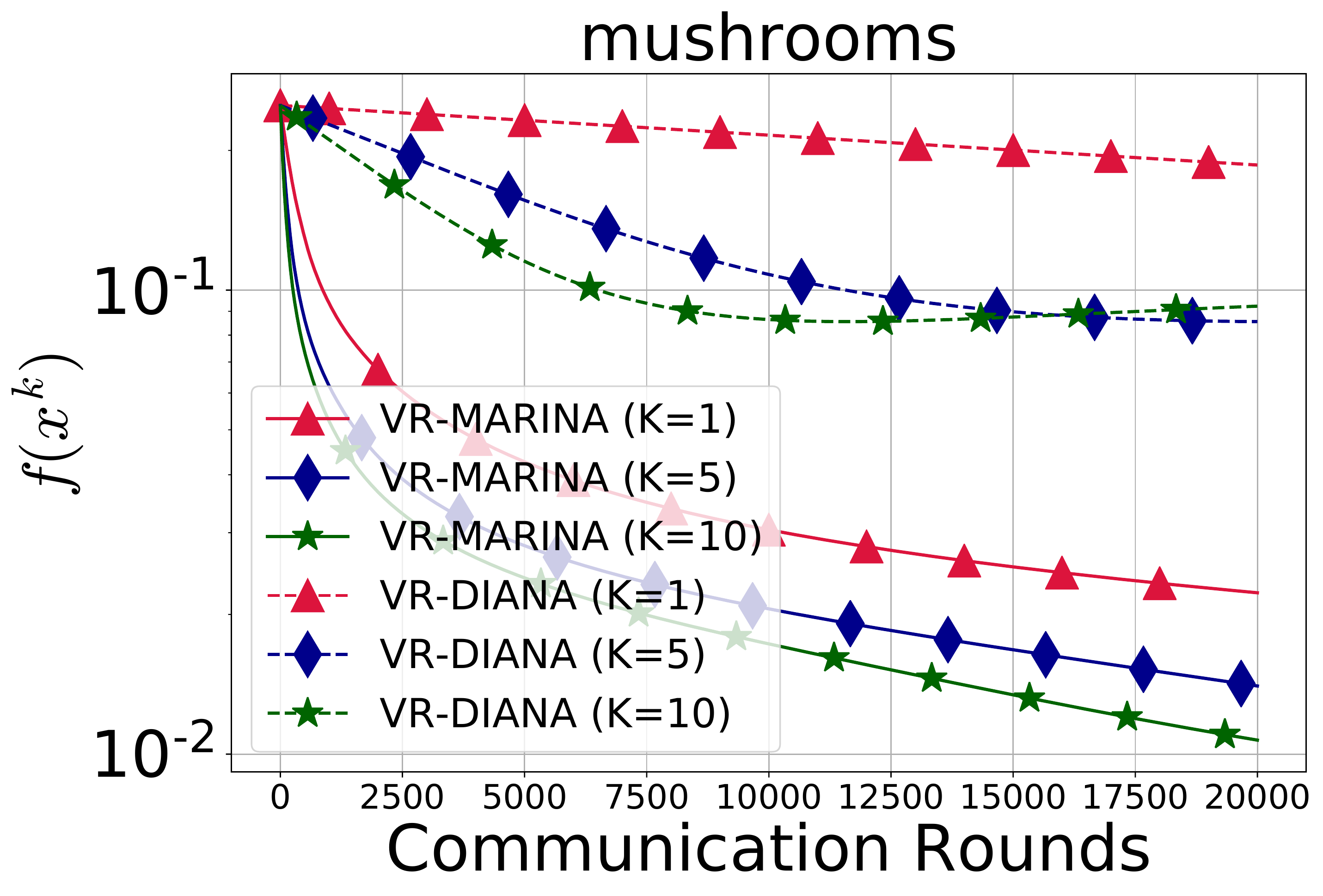}
\includegraphics[width=0.24\textwidth]{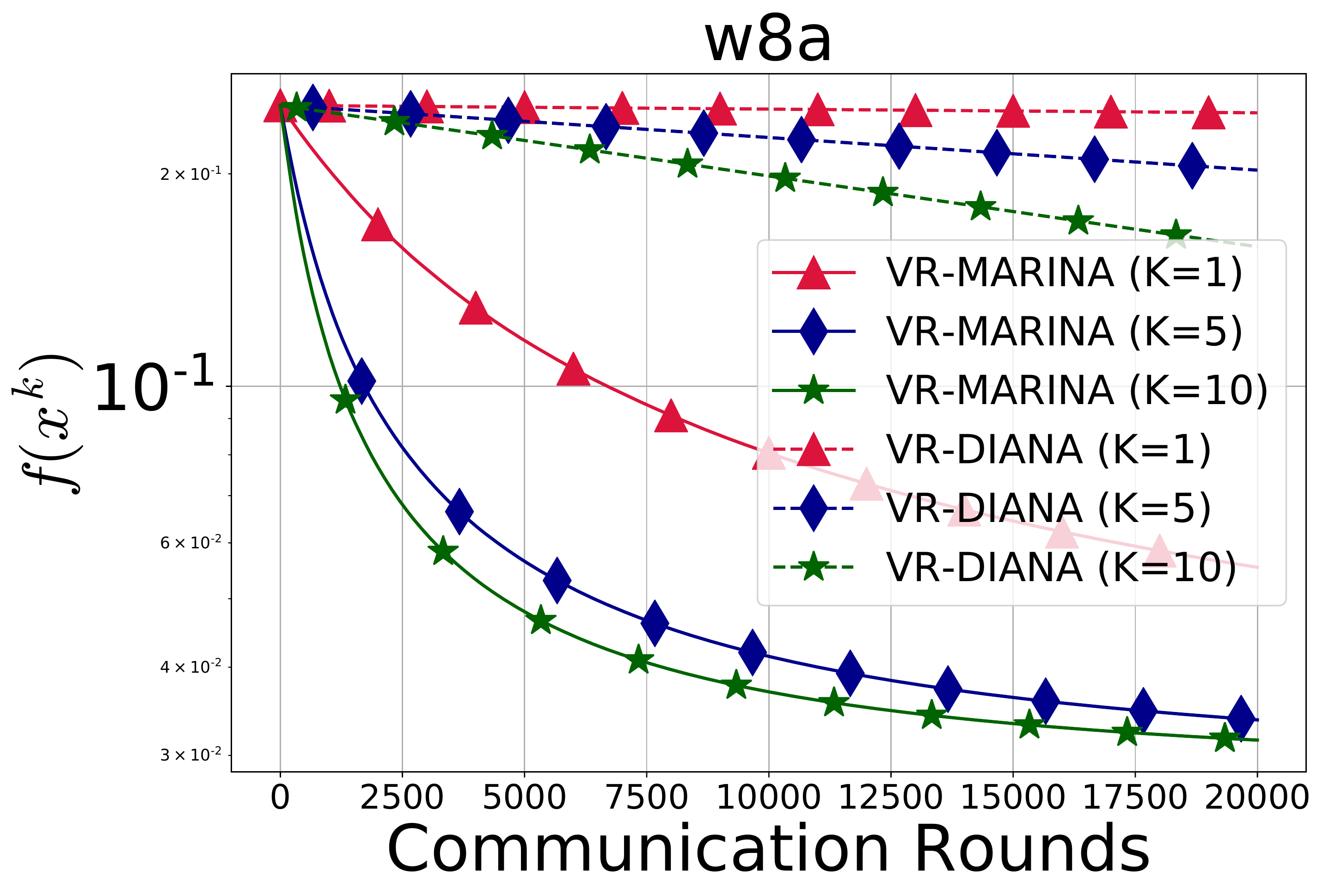}
\includegraphics[width=0.24\textwidth]{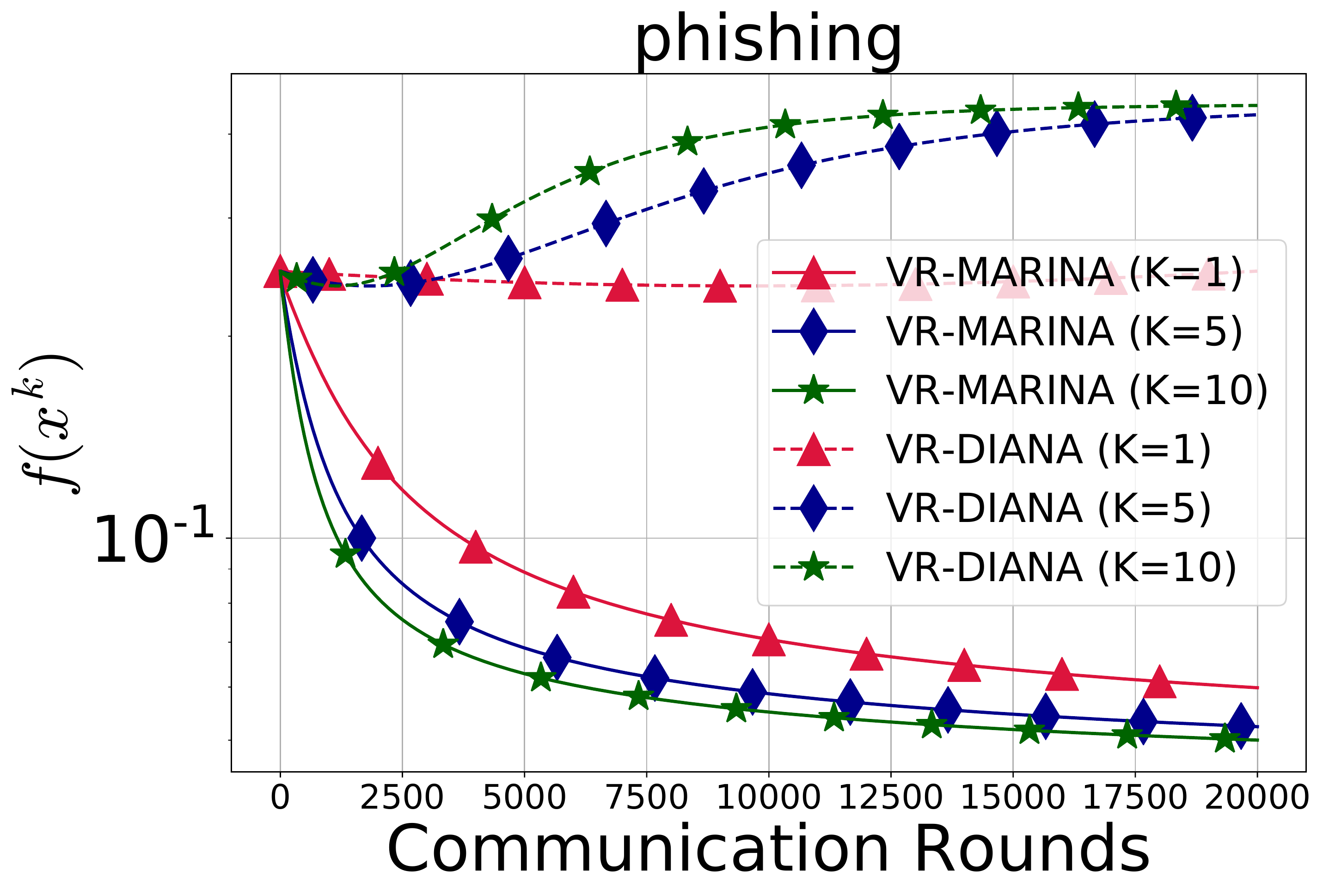}
\includegraphics[width=0.24\textwidth]{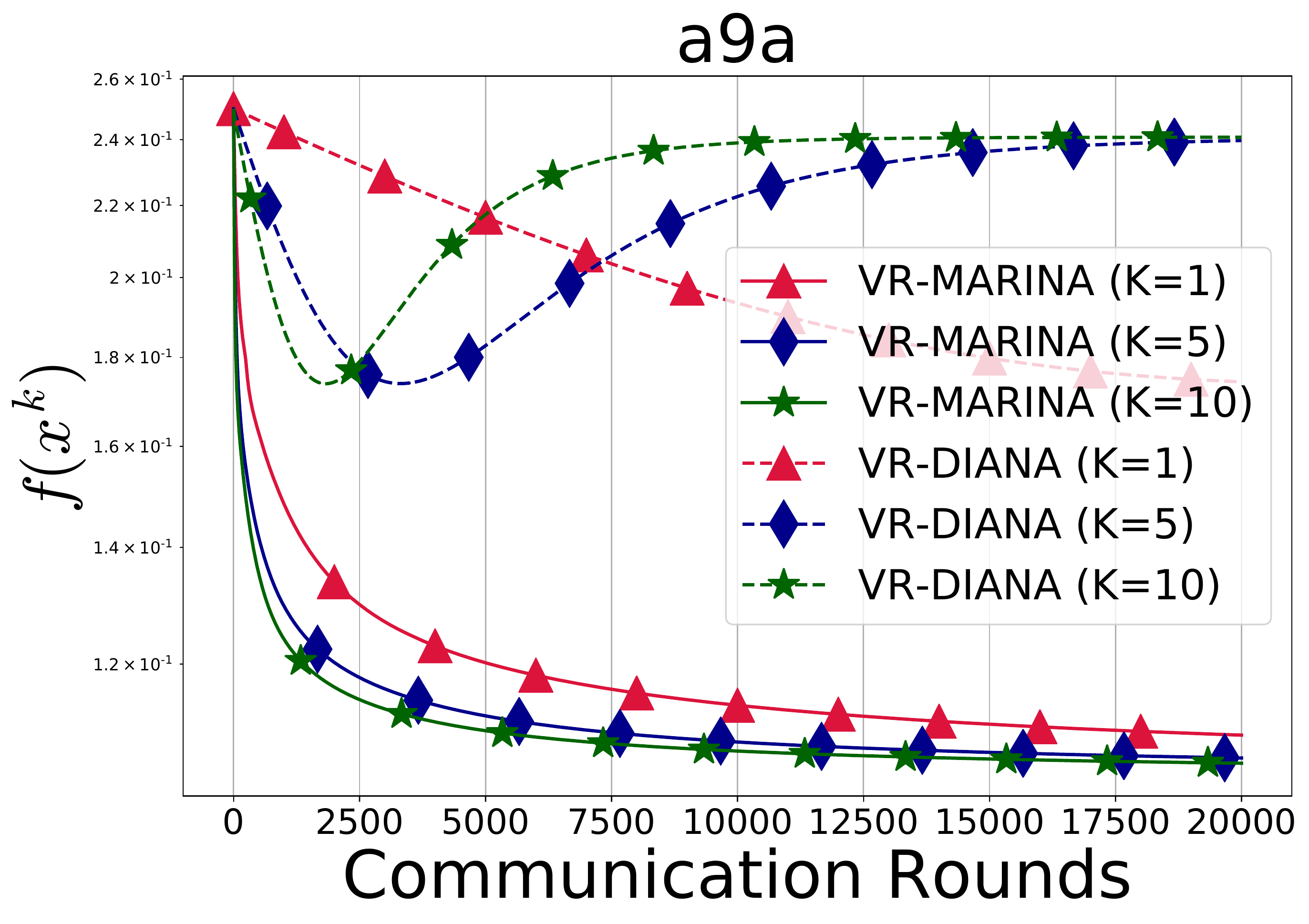}
\includegraphics[width=0.24\textwidth]{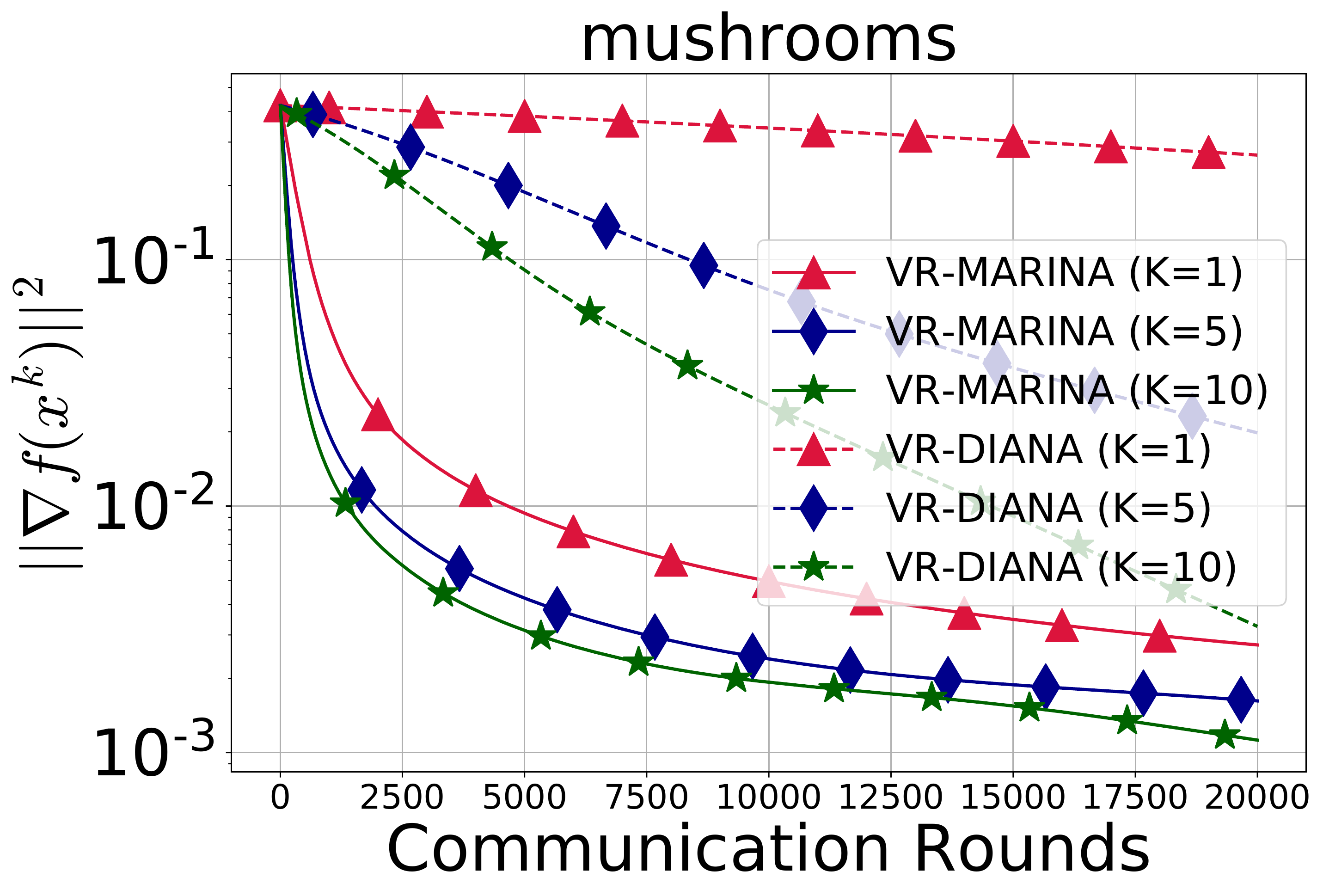}
\includegraphics[width=0.24\textwidth]{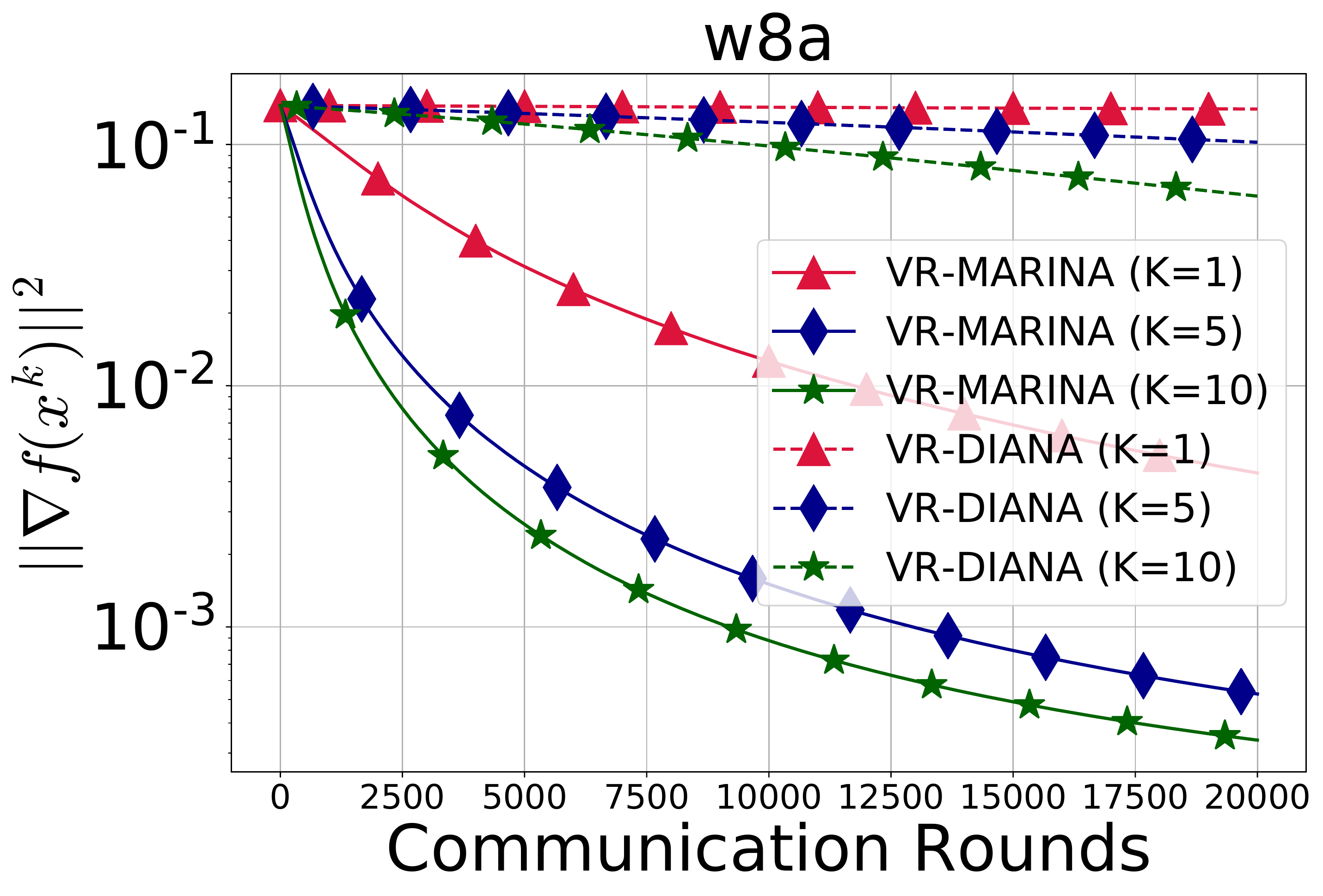}
\includegraphics[width=0.24\textwidth]{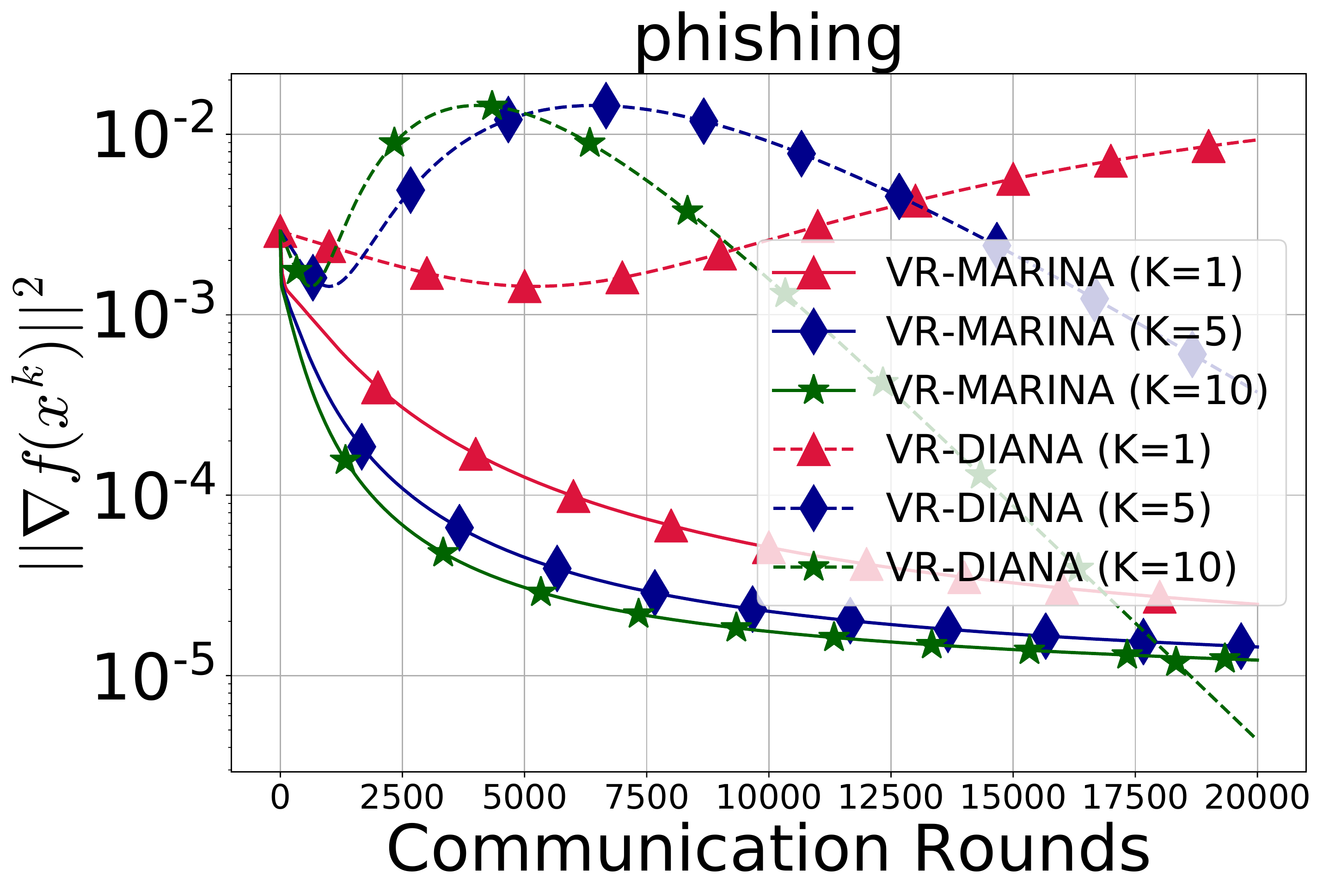}
\includegraphics[width=0.24\textwidth]{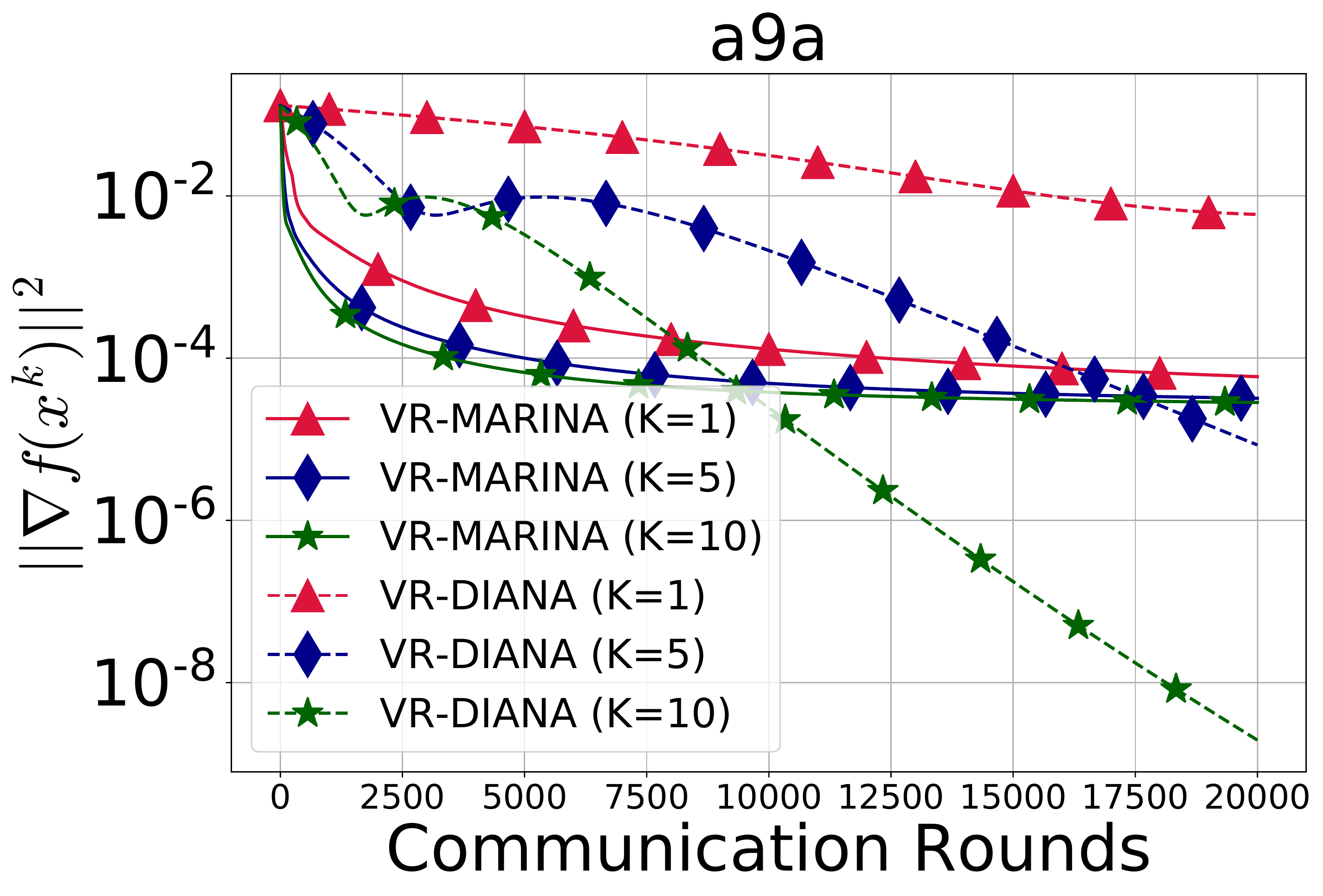}
\includegraphics[width=0.24\textwidth]{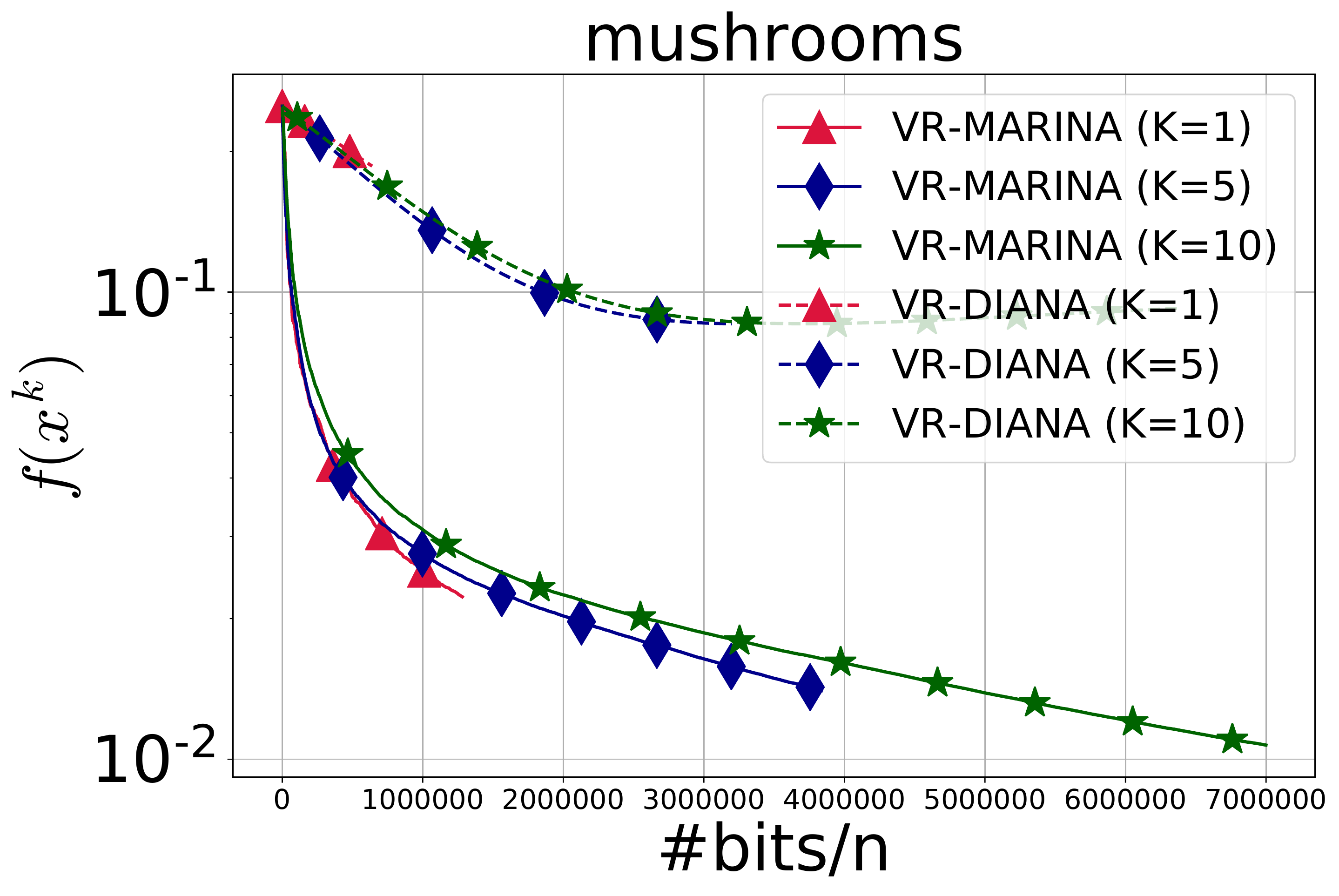}
\includegraphics[width=0.24\textwidth]{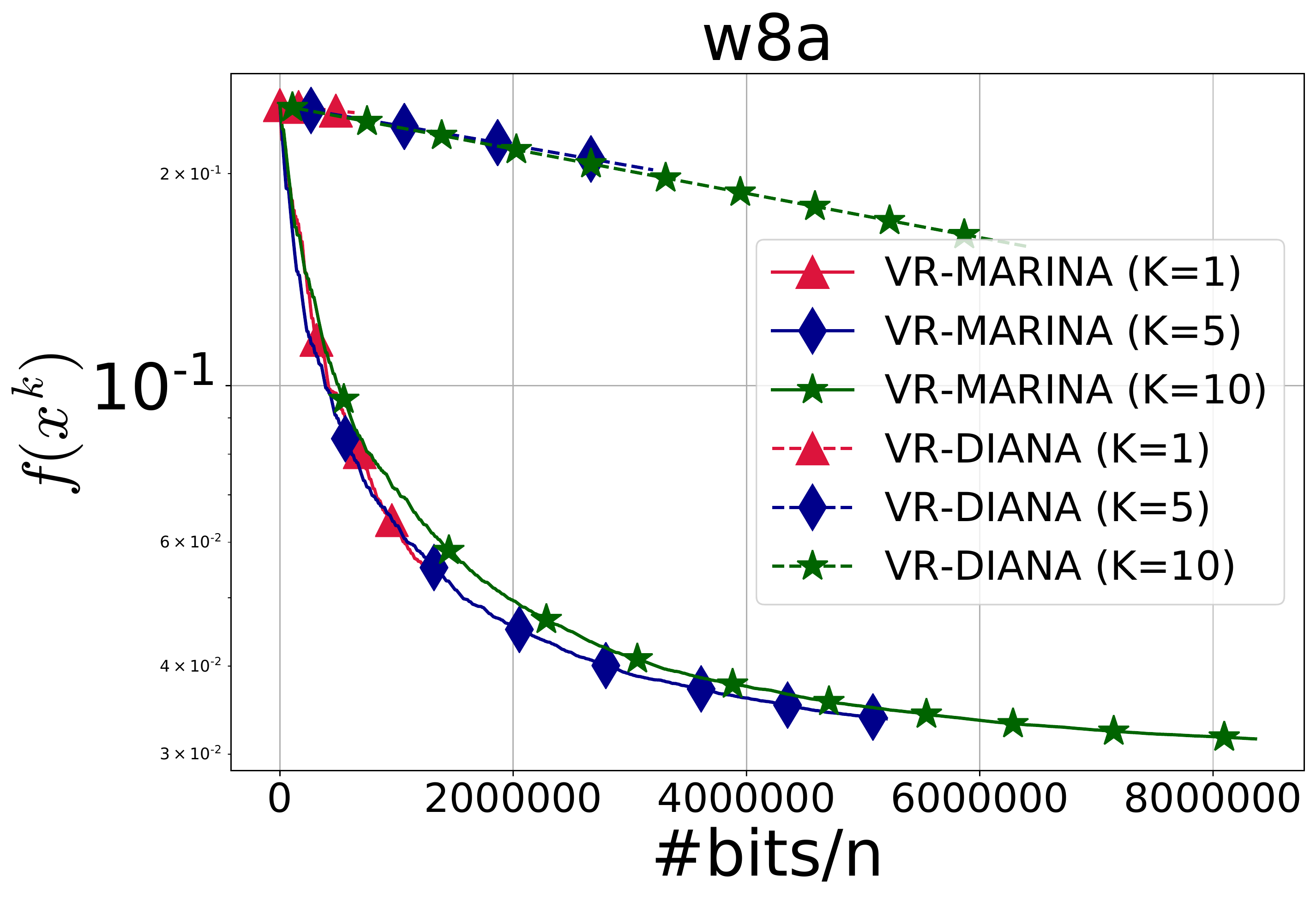}
\includegraphics[width=0.24\textwidth]{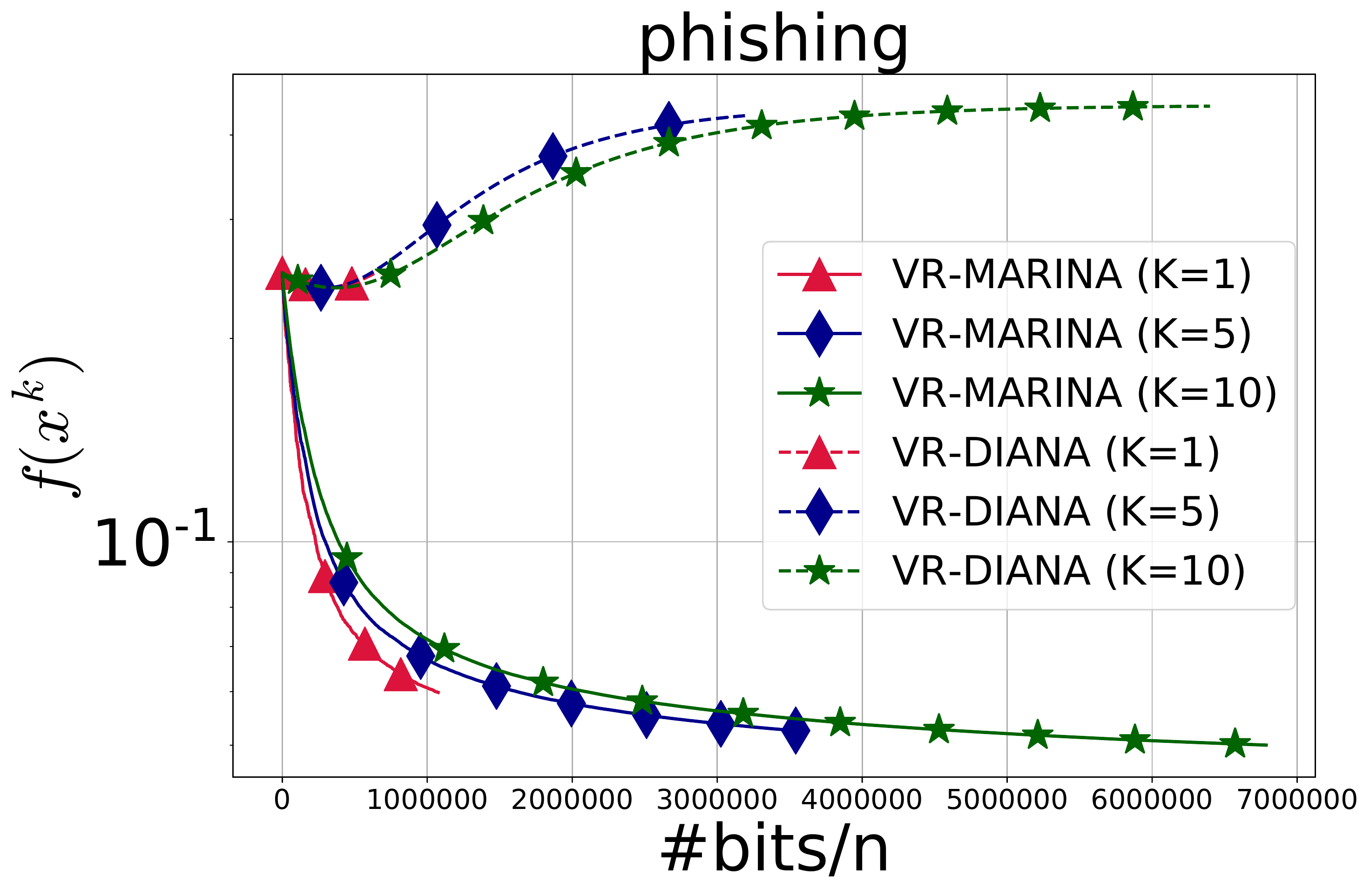}
\includegraphics[width=0.24\textwidth]{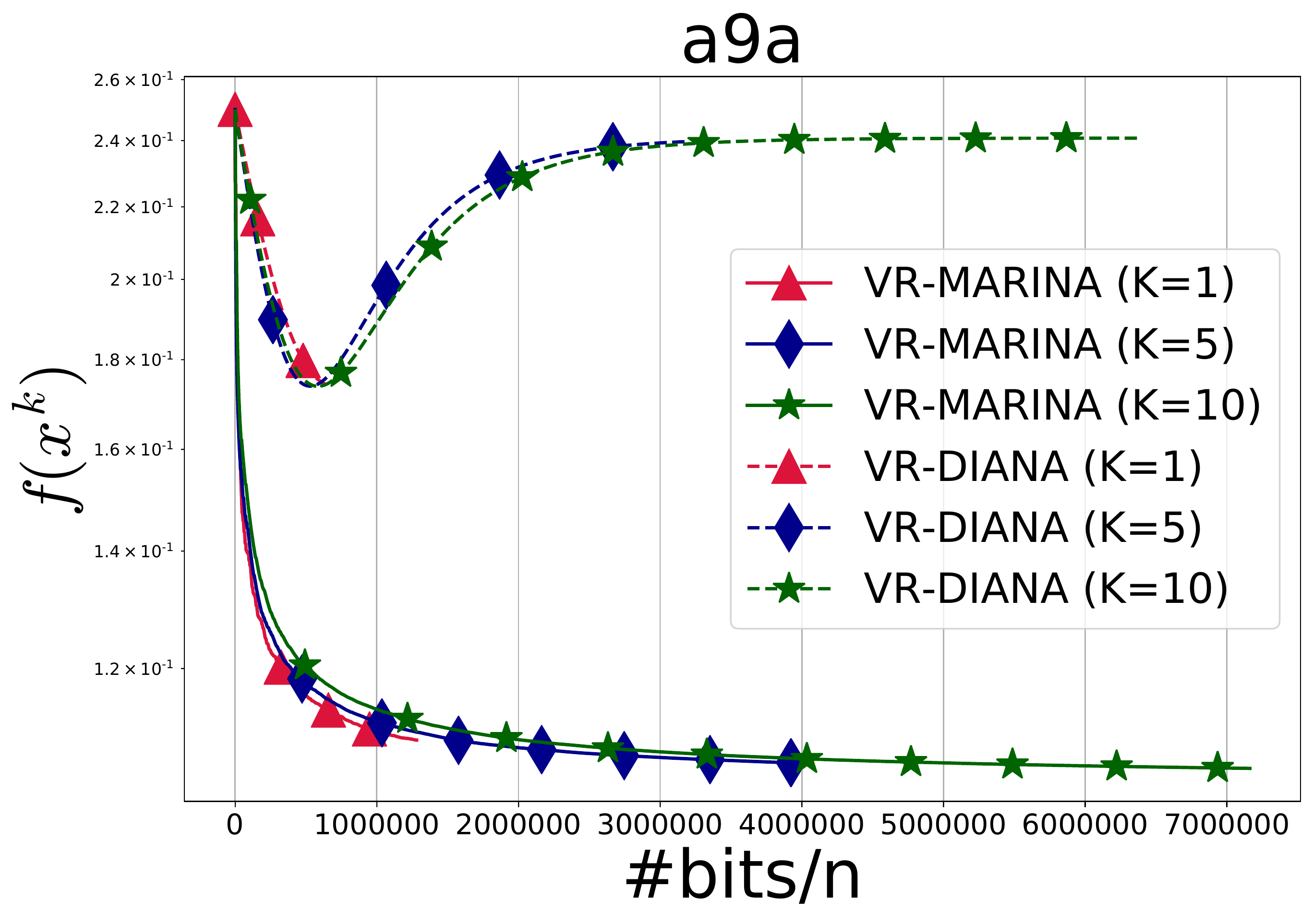}
\includegraphics[width=0.24\textwidth]{mushrooms_grad_norm_bits_vr_marina_vr_diana.pdf}
\includegraphics[width=0.24\textwidth]{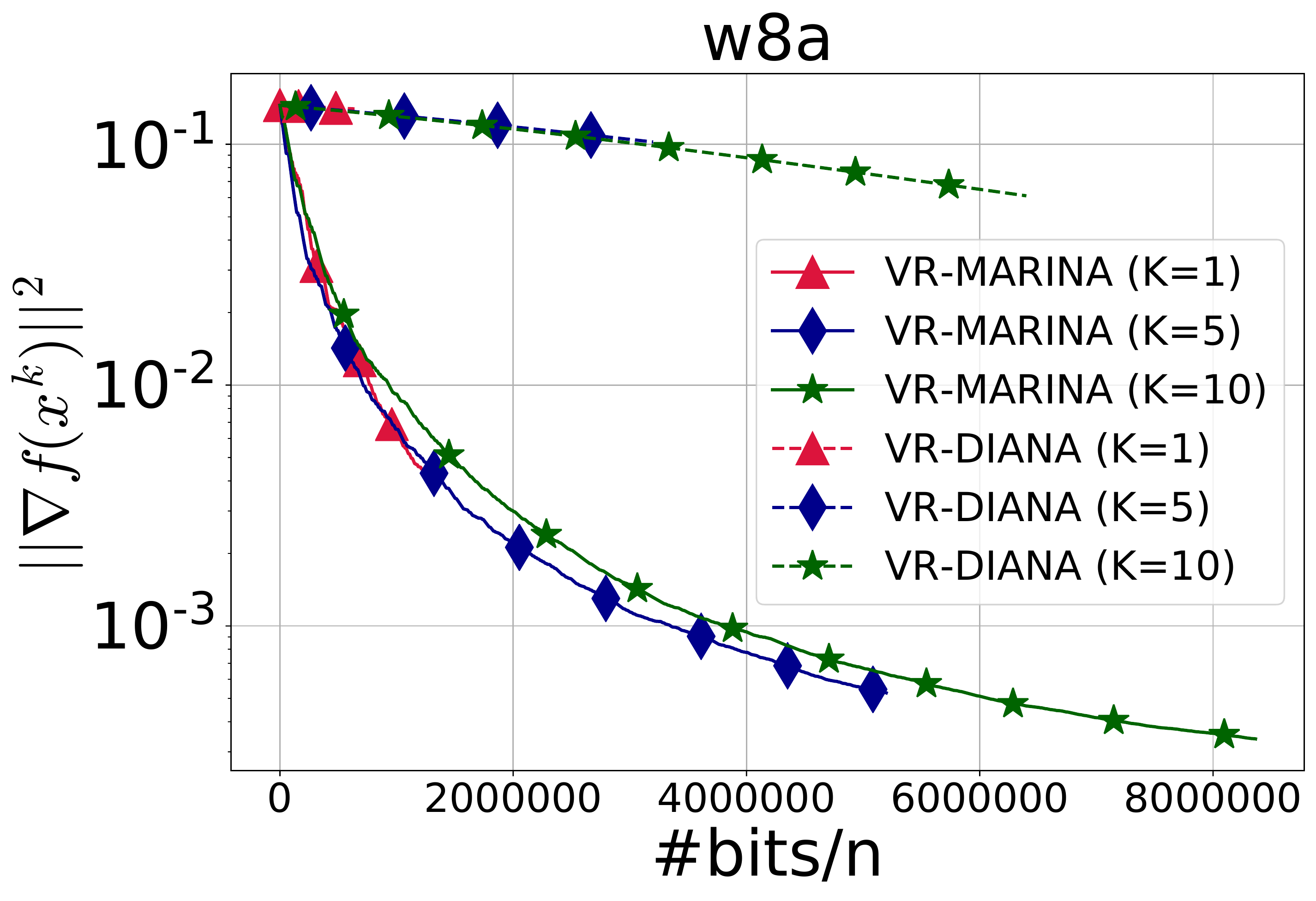}
\includegraphics[width=0.24\textwidth]{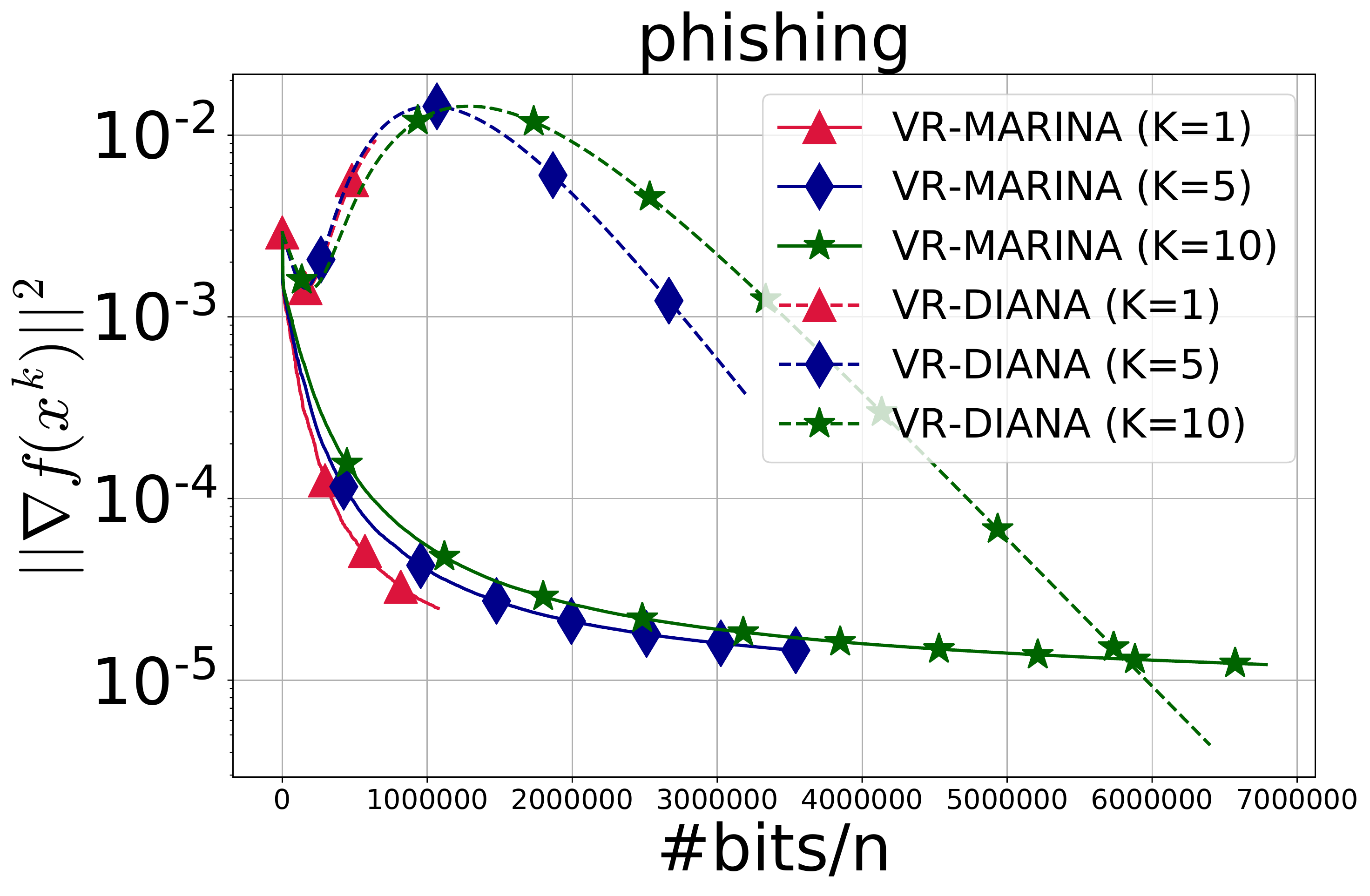}
\includegraphics[width=0.24\textwidth]{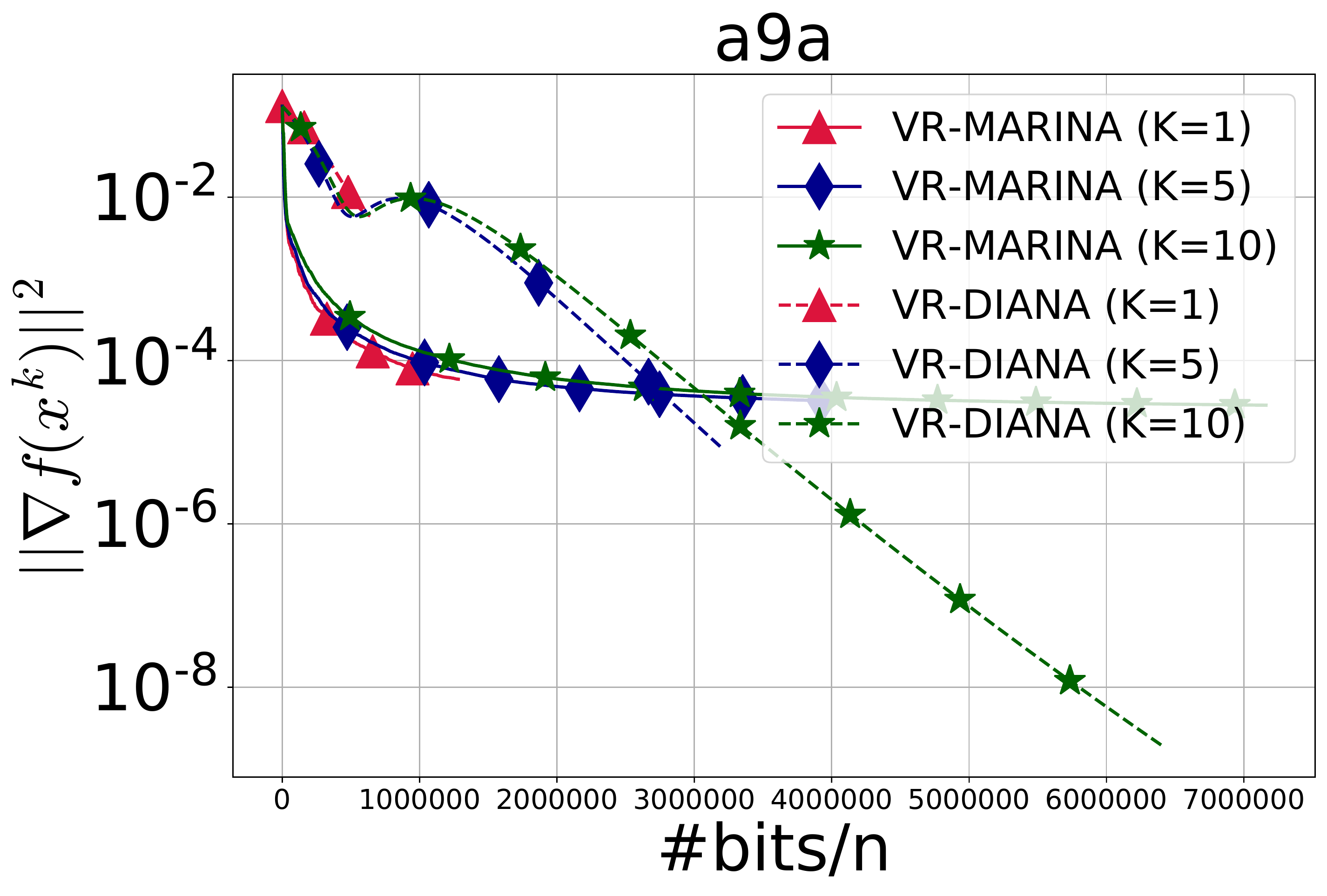}
\includegraphics[width=0.24\textwidth]{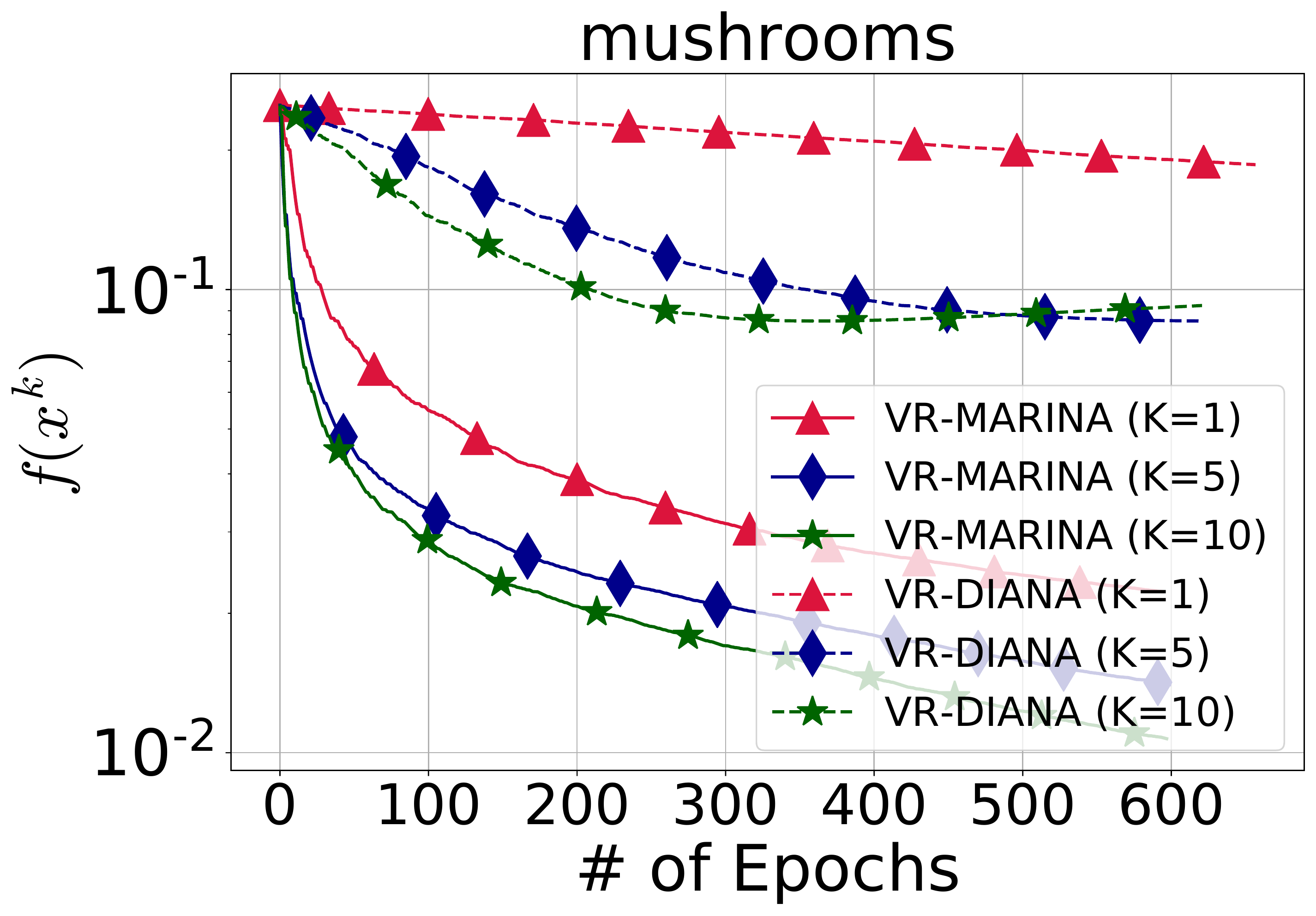}
\includegraphics[width=0.24\textwidth]{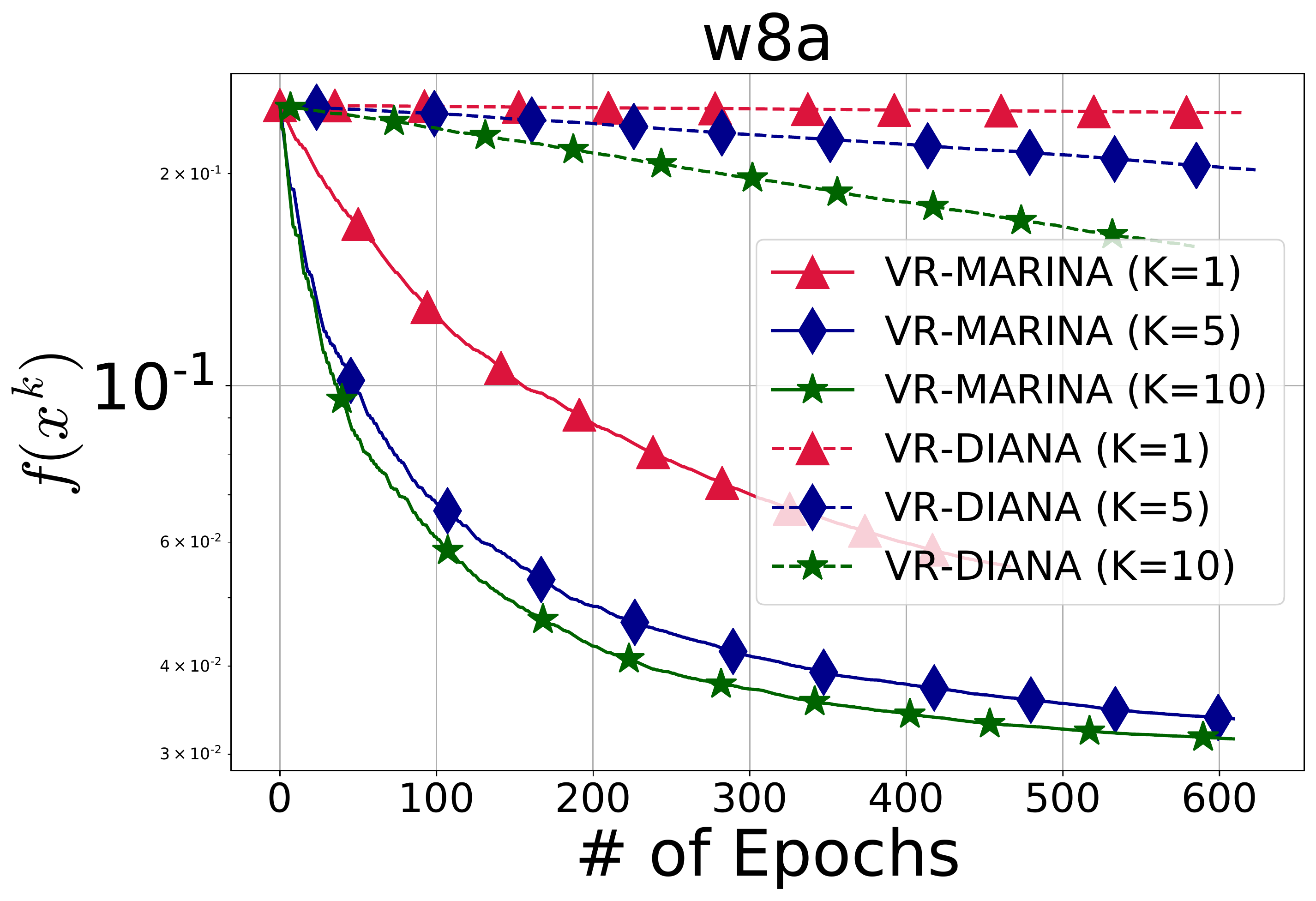}
\includegraphics[width=0.24\textwidth]{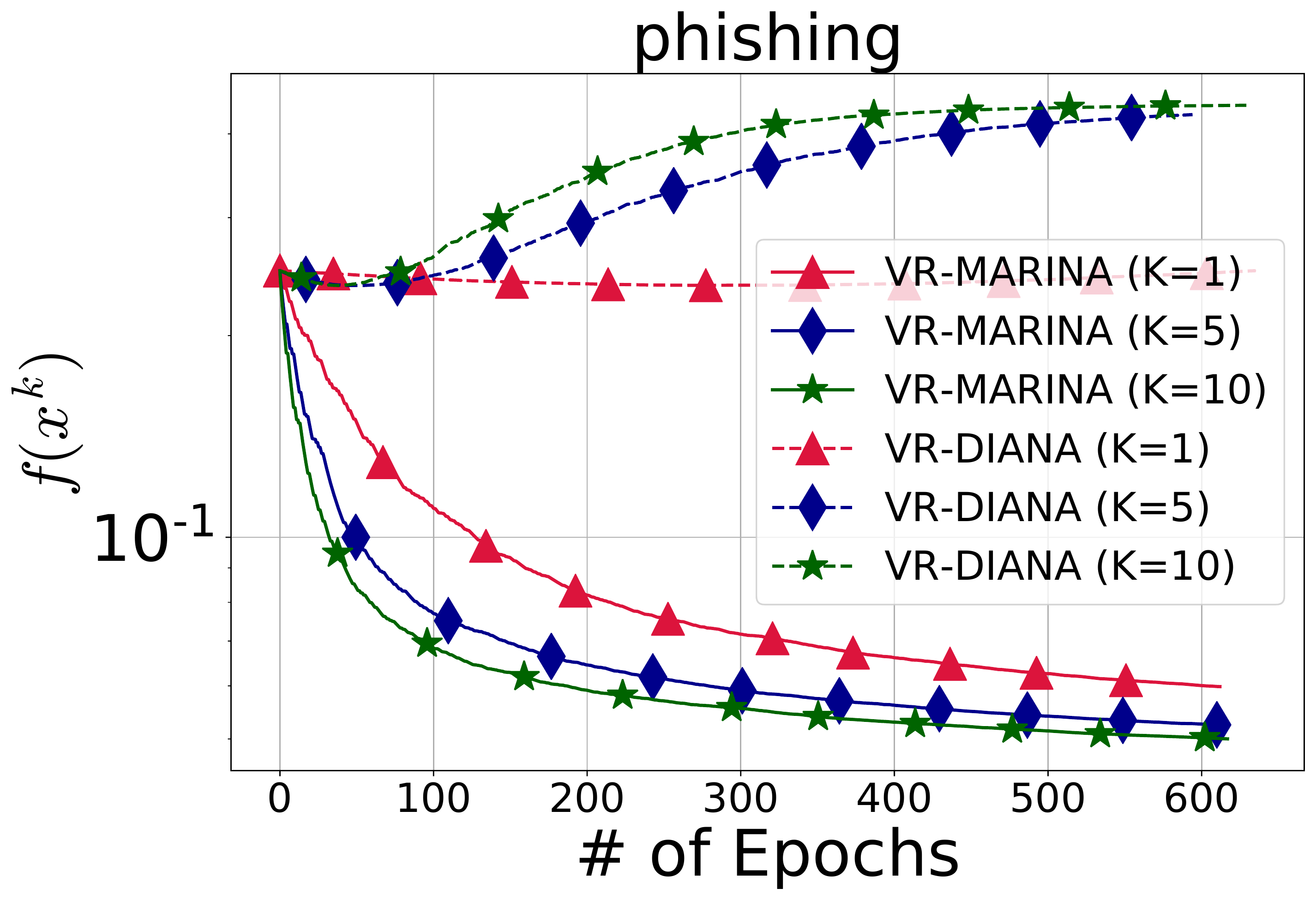}
\includegraphics[width=0.24\textwidth]{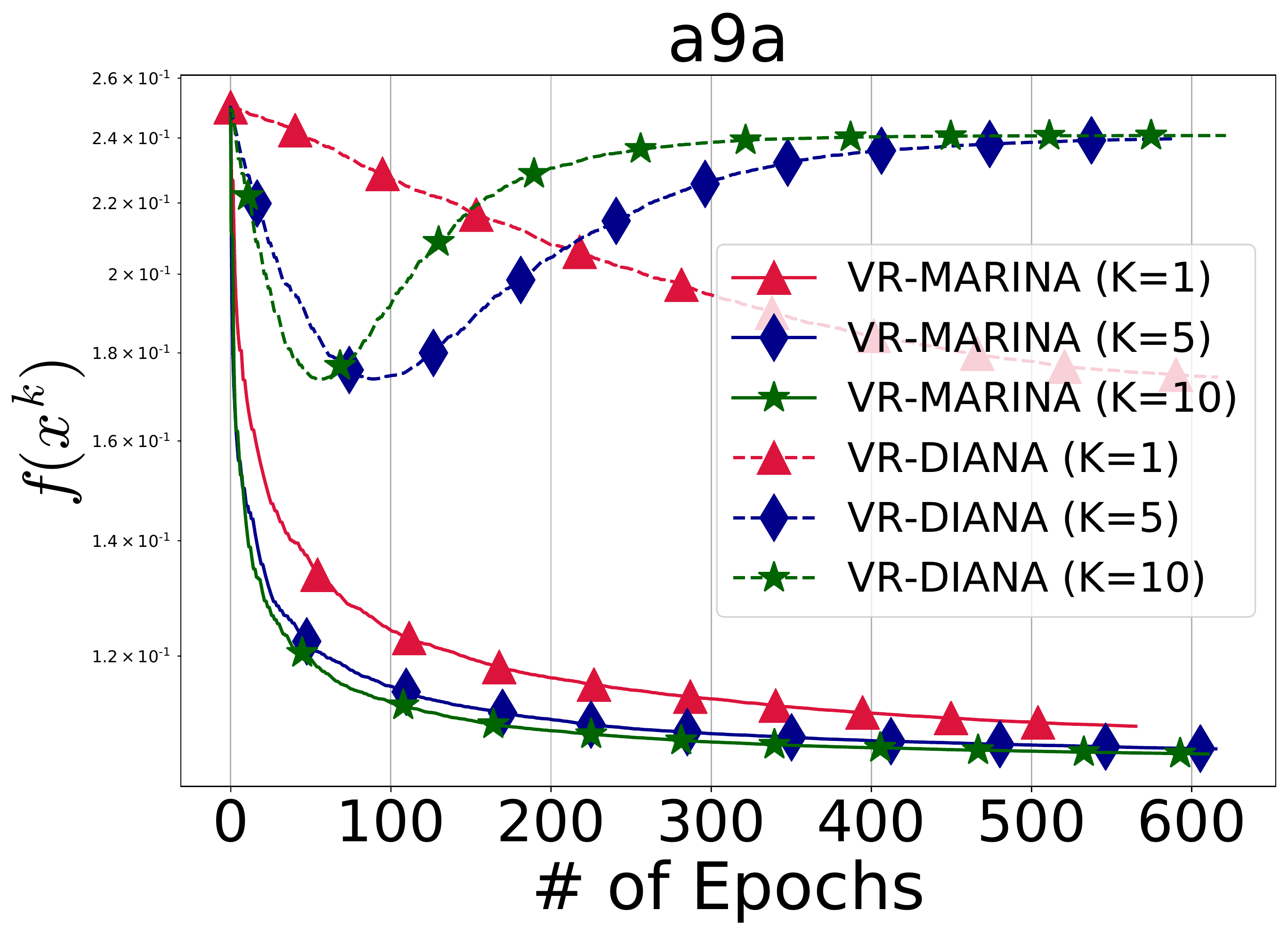}
\includegraphics[width=0.24\textwidth]{mushrooms_grad_norm_data_passes_vr_marina_vr_diana.pdf}
\includegraphics[width=0.24\textwidth]{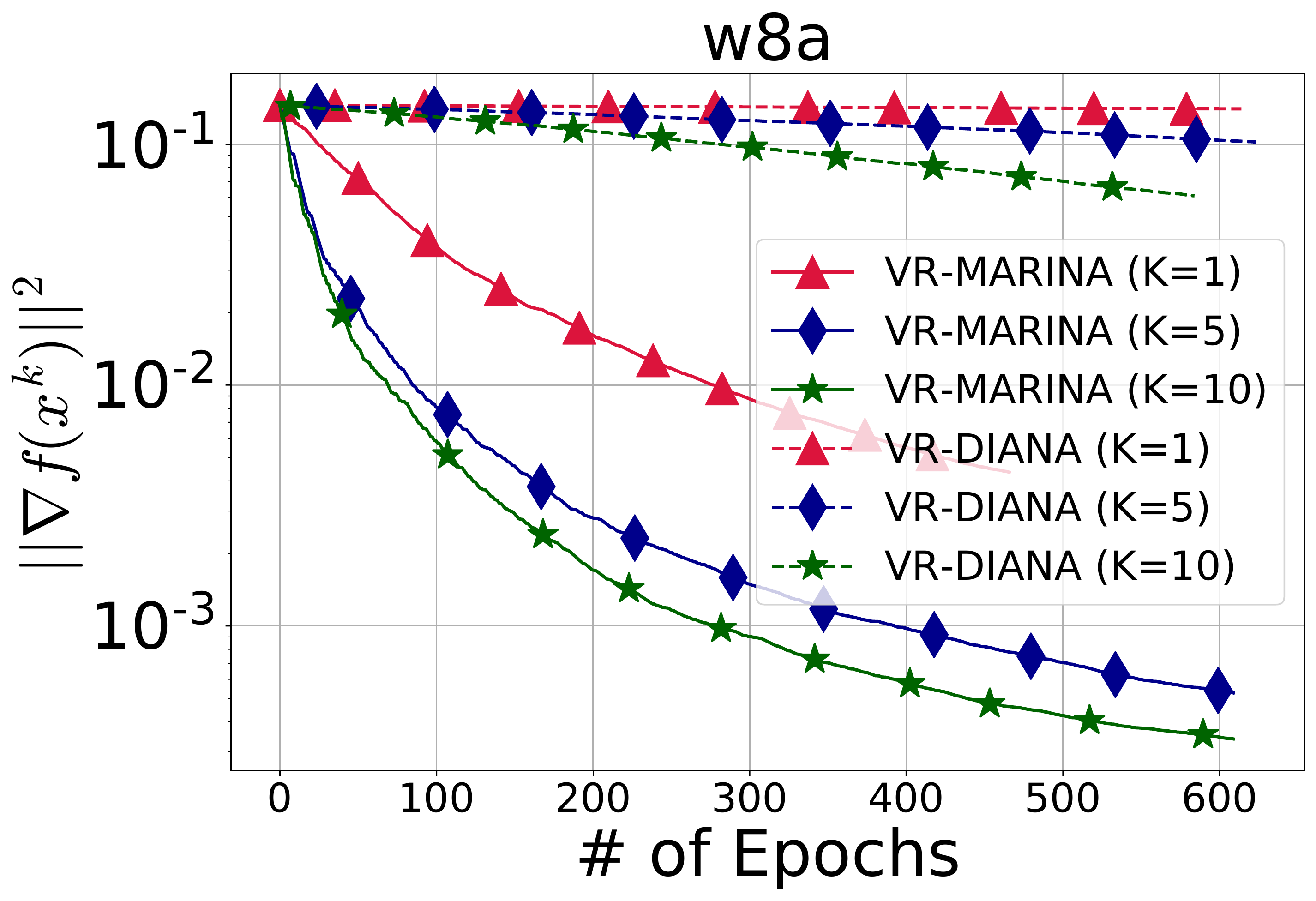}
\includegraphics[width=0.24\textwidth]{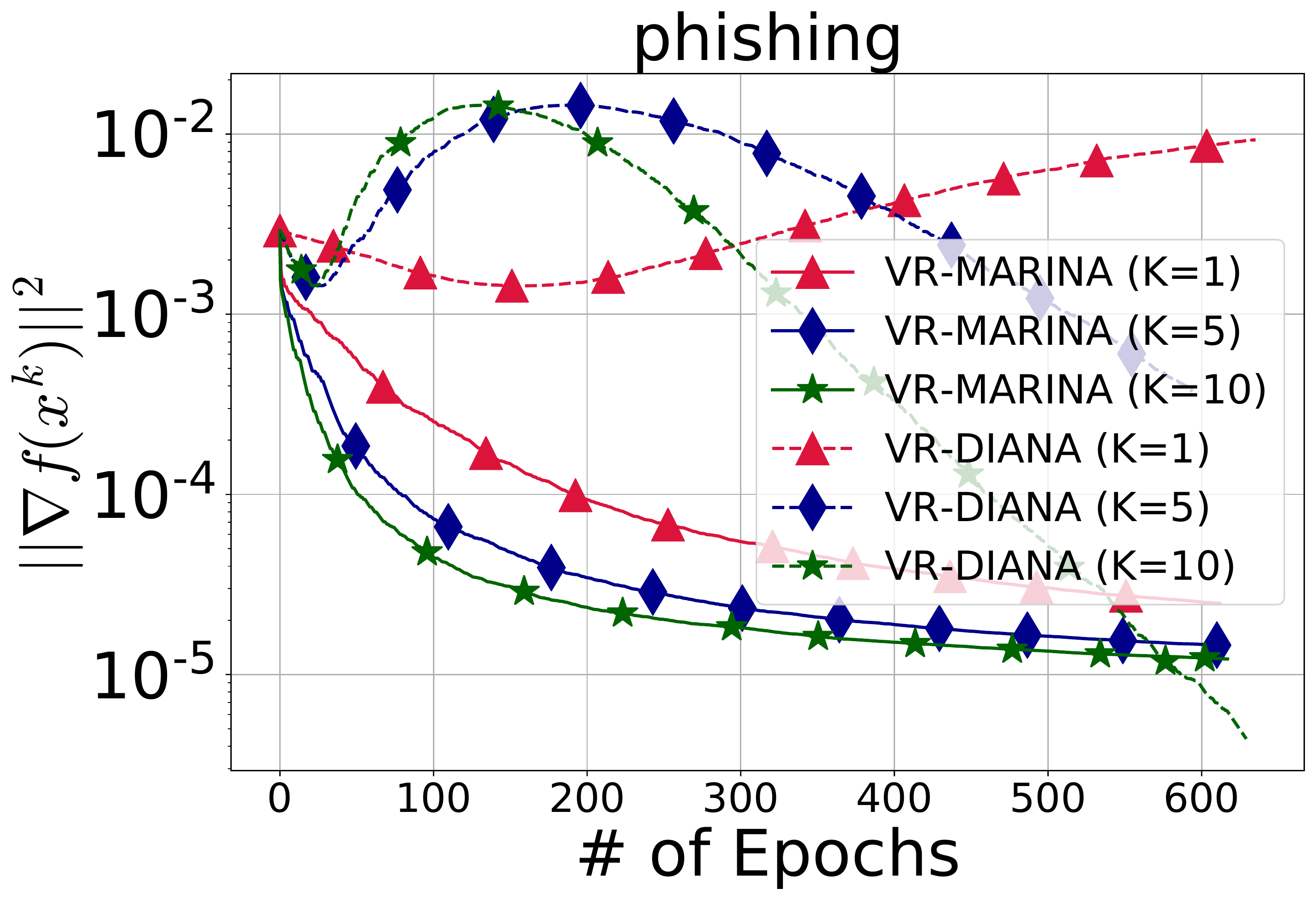}
\includegraphics[width=0.24\textwidth]{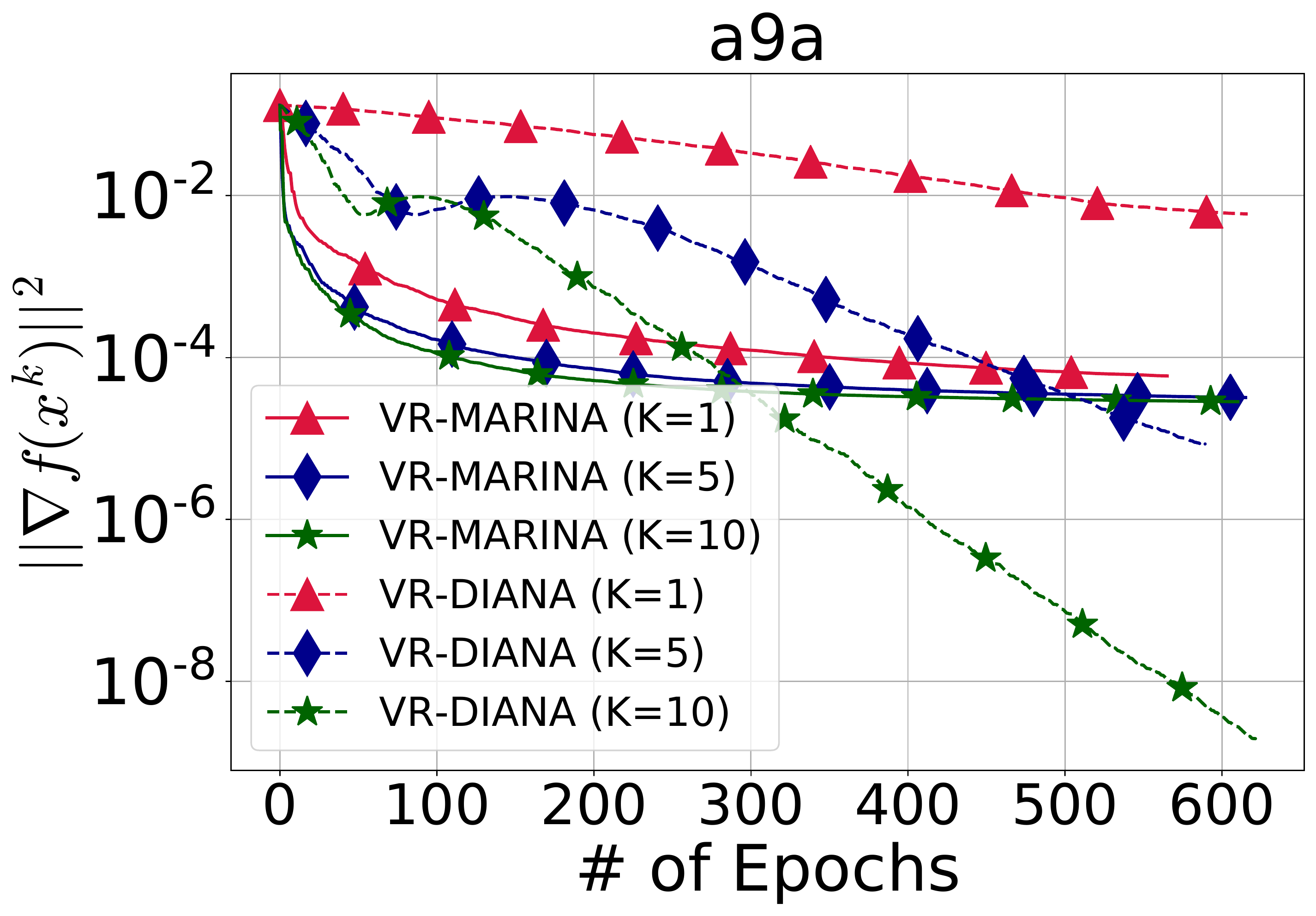}
\caption{Comparison of \algname{VR-MARINA} with  \algname{VR-DIANA} on binary classification problem involving non-convex loss \eqref{eq:experiment_problem} with LibSVM data \cite{chang2011libsvm}. Parameter $n$ is chosen as per Tbl.~\ref{tbl:ns} ($n = 5$). Stepsizes for the methods are chosen according to the theory and the batchsizes are $\sim \nicefrac{m}{100}$. In all cases, we used the RandK sparsification operator with K $\in \{1,5,10\}$.}
\label{fig:vr_methods}
\end{figure*}

We also tested \algname{MARINA} and \algname{DIANA} on \texttt{mushrooms} dataset with a bigger number of workers ($n=20$). The results are reported in Figure~\ref{fig:full_batched_methods_more_workers}. Similarly to the previous numerical tests, \algname{MARINA} shows its superiority to \algname{DIANA} with $n=20$ as well.

\begin{figure*}[t!]
\centering
\includegraphics[width=0.3\textwidth]{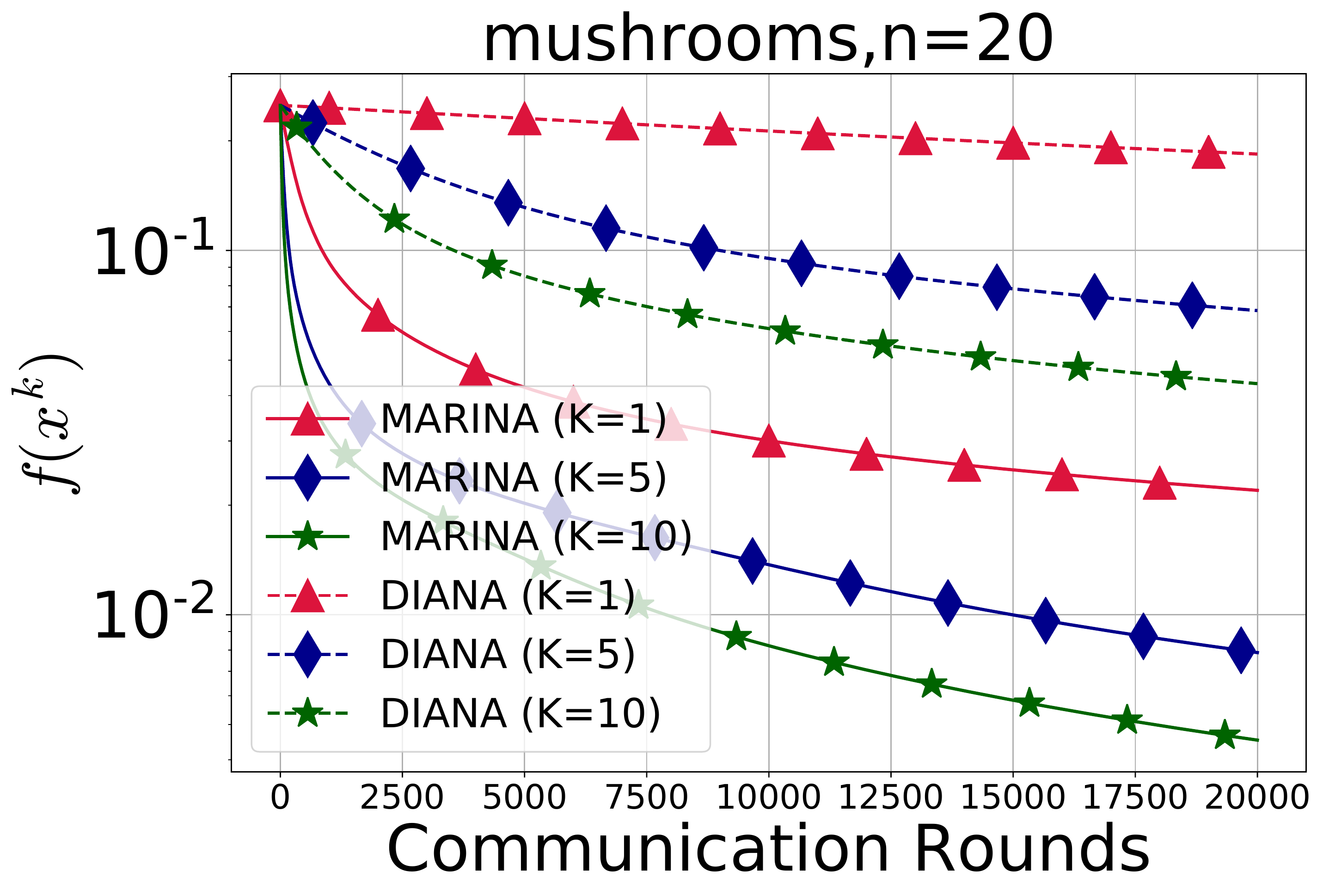}
\includegraphics[width=0.3\textwidth]{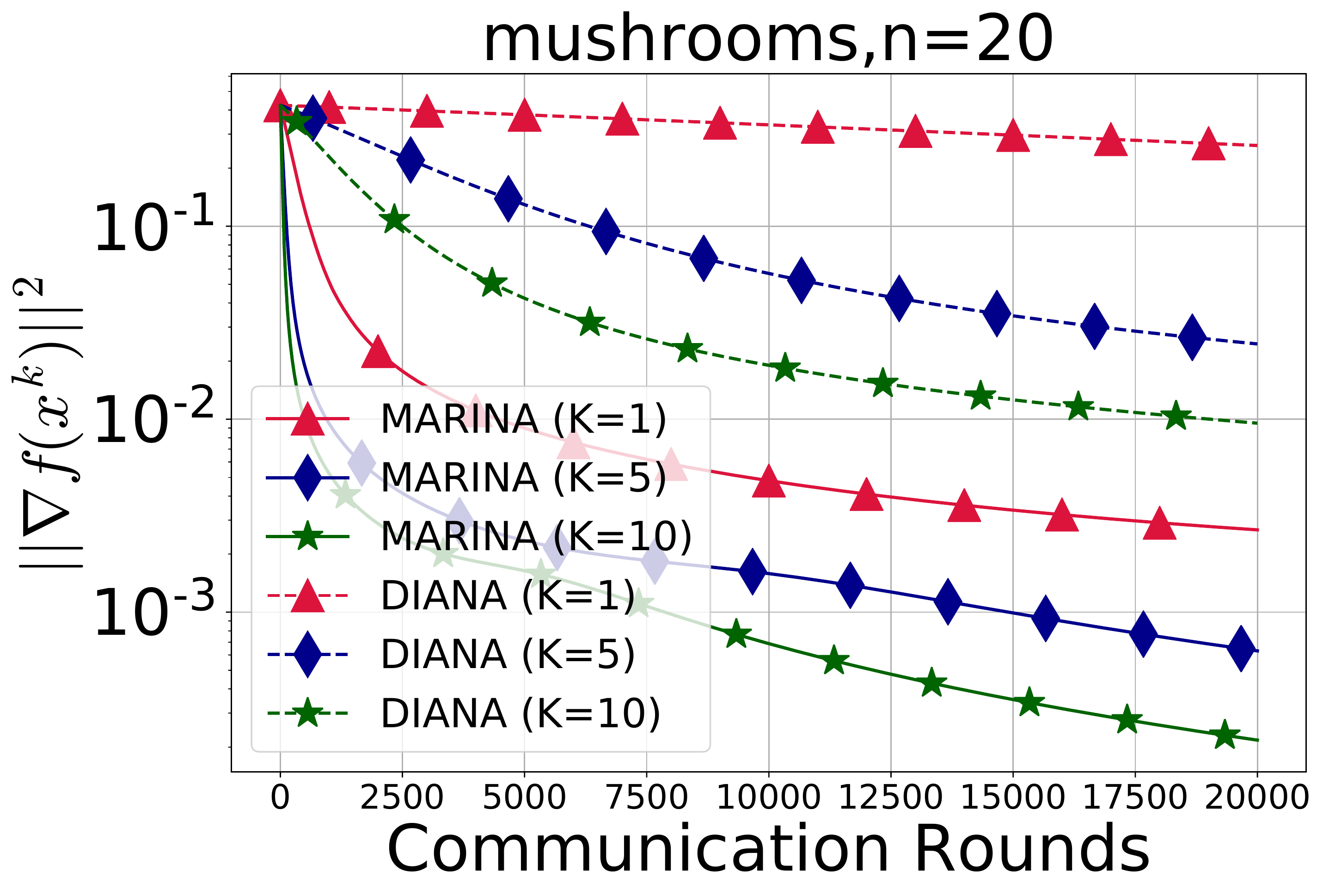}\\
\includegraphics[width=0.3\textwidth]{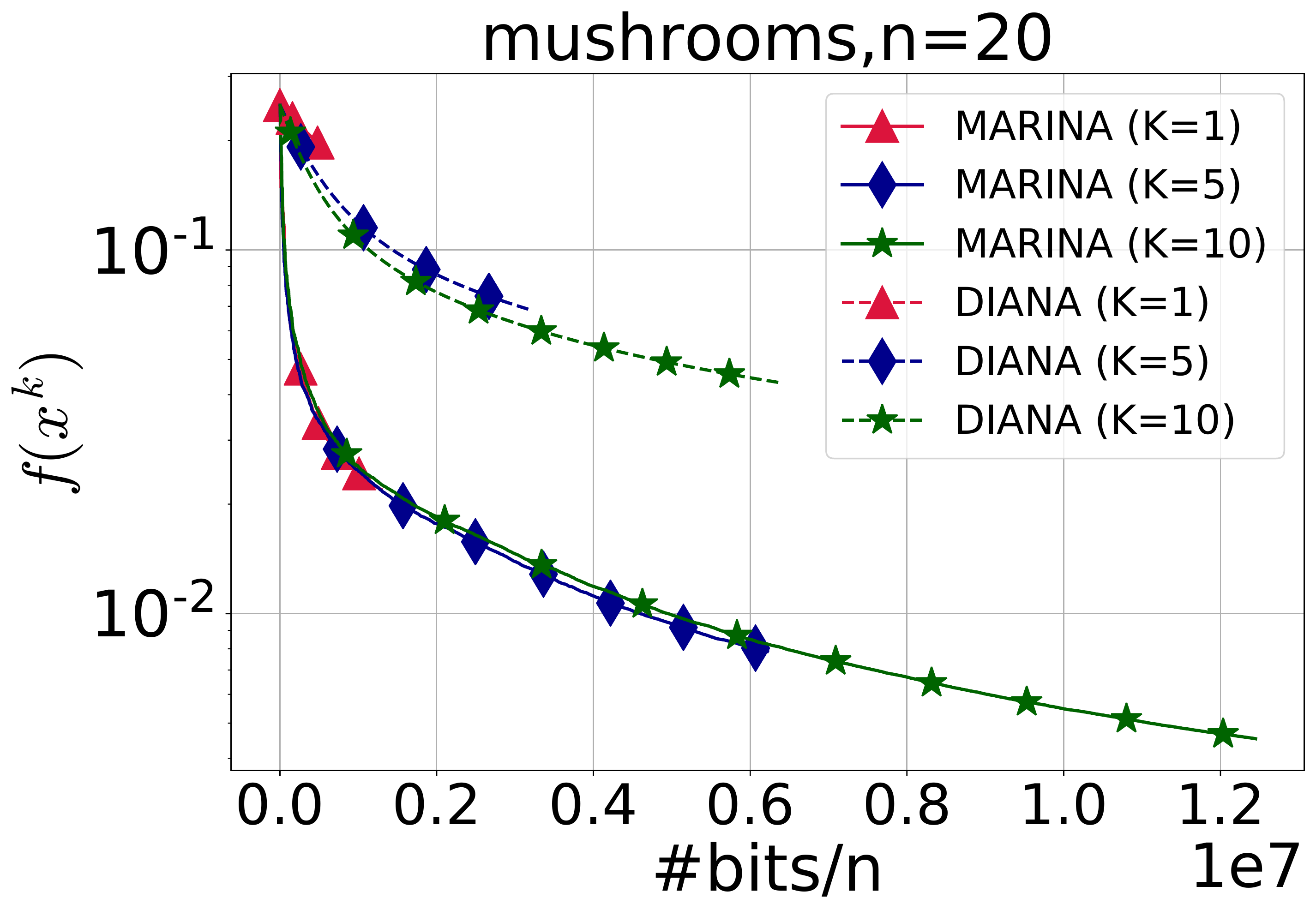}
\includegraphics[width=0.3\textwidth]{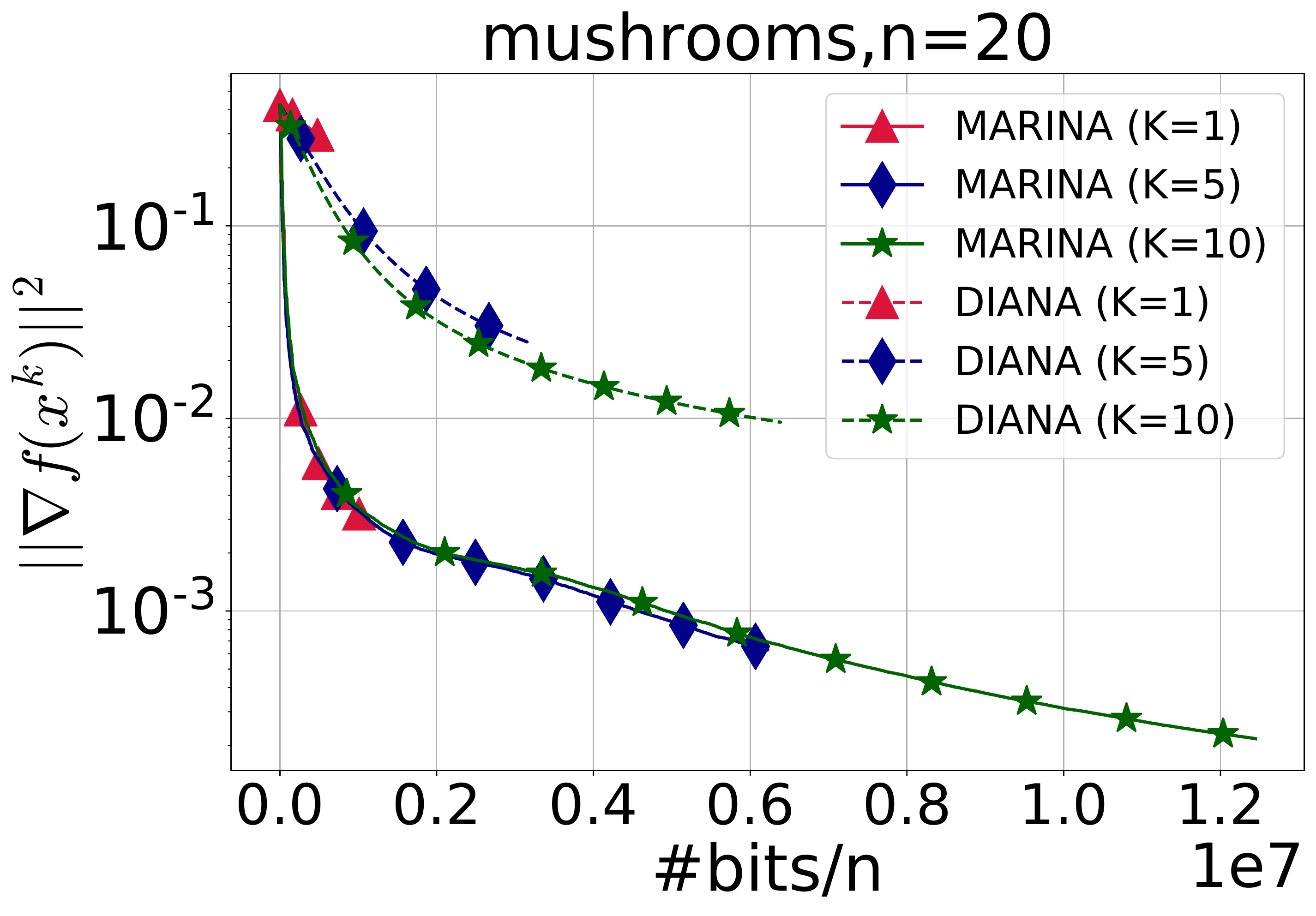}
\caption{Comparison of \algname{MARINA} with  \algname{DIANA} on binary classification problem involving non-convex loss \eqref{eq:experiment_problem} with \texttt{mushrooms} dataset and $n=20$ workers. Stepsizes for the methods are chosen according to the theory. In all cases, we used the RandK sparsification operator with K $\in \{1,5,10\}$.}
\label{fig:full_batched_methods_more_workers}
\end{figure*}

\subsection{Image Classification}

\subsubsection{Setup}
In Section~\ref{sec:NN_experiments}, we demonstrate the performance of \algname{VR-MARINA} and \algname{VR-DIANA} on training {\tt ResNet-18} at {\tt CIFAR100} dataset. {\tt ResNet-18} has $d =$ 11~689~512 parameters to train and {\tt CIFAR100} contains $N =$ 50~000 colored images. The dataset is split into $5$ parts among $5$ workers in such a way that the first four workers get 10~112 samples and the fifth one get 9~552 samples. The code was written in Python 3.9 using \textsc{PyTorch 1.7} and then was executed on a machine with NVIDIA GPU Geforce RTX 2080 Ti with 11 GByte onboard global GPU memory.

In all experiments, we use batchsize $= 256$ on each worker and tune the stepsizes for each method separately. That is, for each method  and for each choice of $K$ for RandK operator we run the method with stepsize $\gamma \in \{10^{-6}, 0.1, 0.2, 0.5, 1.0, 5.0\}$ to find the interval containing the best stepsize. Next, the obtained interal is split into $\sim 10$ equal parts and the method is run with corresponding stepsizes. Other parameters of the methods are chosen according to the theory. The summary of used parameters is given in Table~\ref{tbl:params_resnet}.

\begin{table}[!h]
 \caption{Summary of the parameters used in the experiments presented in Fig.~\ref{fig:resnet_at_cifar100} and Fig.~\ref{fig:resnet_at_cifar100_no_compr}. Stepsizes were tuned, batchsize $= 256$ on each worker, other parameters were picked according to the theory, except the last line, where $p$ for \algname{VR-MARINA} without compression was picked as for \algname{VR-MARINA} with RandK, $K = $ 100 000 compression operator.}
\label{tbl:params_resnet}
\begin{center}
\begin{tabular}{|c|c|c|c|}
\hline
Method  & RandK, $K = $ & $\gamma$ & $p$  \\
 \hline
  \hline
\algname{VR-MARINA} & 100 000 & $0.95$ & $0.008554$    \\ \hline
\algname{VR-MARINA} & 500 000 & $0.95$ & $0.024691$    \\ \hline
\algname{VR-MARINA} & 1 000 000 & $0.95$ & $0.024691$    \\ \hline
\algname{VR-DIANA} & 100 000 & $0.15$ & $0.025316$    \\ \hline
\algname{VR-DIANA} & 500 000 & $0.35$ & $0.025316$    \\ \hline
\algname{VR-DIANA} & 1 000 000 & $0.35$ & $0.025316$    \\ \hline
\algname{VR-MARINA} & 11 689 512 ($K = d$) & $3.5$ & $0.024691$    \\ \hline
\algname{VR-DIANA} & 11 689 512 ($K = d$) & $2.5$ & $0.025316$    \\ \hline
\algname{VR-MARINA} & 11 689 512 ($K = d$) & $3.5$ & $0.008554$    \\ \hline
\end{tabular}
\end{center}
\end{table}

\subsubsection{Extra Experiments}
Results presented in Fig.~\ref{fig:resnet_at_cifar100} show the superiority of \algname{VR-MARINA} to \algname{VR-DIANA} in training {\tt ResNet-18} at {\tt CIFAR100}. To emphasize the effect of compression we also run \algname{VR-MARINA} and \algname{VR-DIANA} without compression, see the results in Fig.~\ref{fig:resnet_at_cifar100_no_compr}. First of all, one con notice that the methods do benefit from compression: \algname{VR-MARINA} and \algname{VR-DIANA} with compression converge much faster than their non-comressed versions in terms of the total number of transmitted bits to achieve given accuracy.

Moreover, as Fig.~\ref{fig:resnet_at_cifar100} shows, \algname{VR-MARINA} with $K = $ 100 000 converges faster than \algname{VR-MARINA} with larger $K$ \textit{in terms of the epochs}. That is, the method with more aggresive compression requires less oracle calls to achieve the same accuracy. The reason of such an unusual behavior is the choice of $p$: when $K = $ 100 000 the theoretical choice of $p$ is much smaller than for $K = $ 500 000 and $K = $ 1 000 000. Therefore, in \algname{VR-MARINA} with $K = $ 100 000, the workers compute the full gradients more rarely than in the case of larger $K$. As the result, it turns out, that the total number of oracle calls needed to achieve given accuracy also smaller for $K =$ 100 000 than for larger $K$. Moreover, we see this phenomenon even without applying compression: \algname{VR-MARINA} without compression and with $p$ as in the experiment with \algname{VR-MARINA} with $K = $ 100 000 converges faster than \algname{VR-MARINA} without compression and with theoretical choice of $p$, which is the same as in the case when $K =$ 500 000, 1 000 000, see Table~\ref{tbl:params_resnet}.

\begin{figure}[h]
\centering
\includegraphics[width=0.33\textwidth]{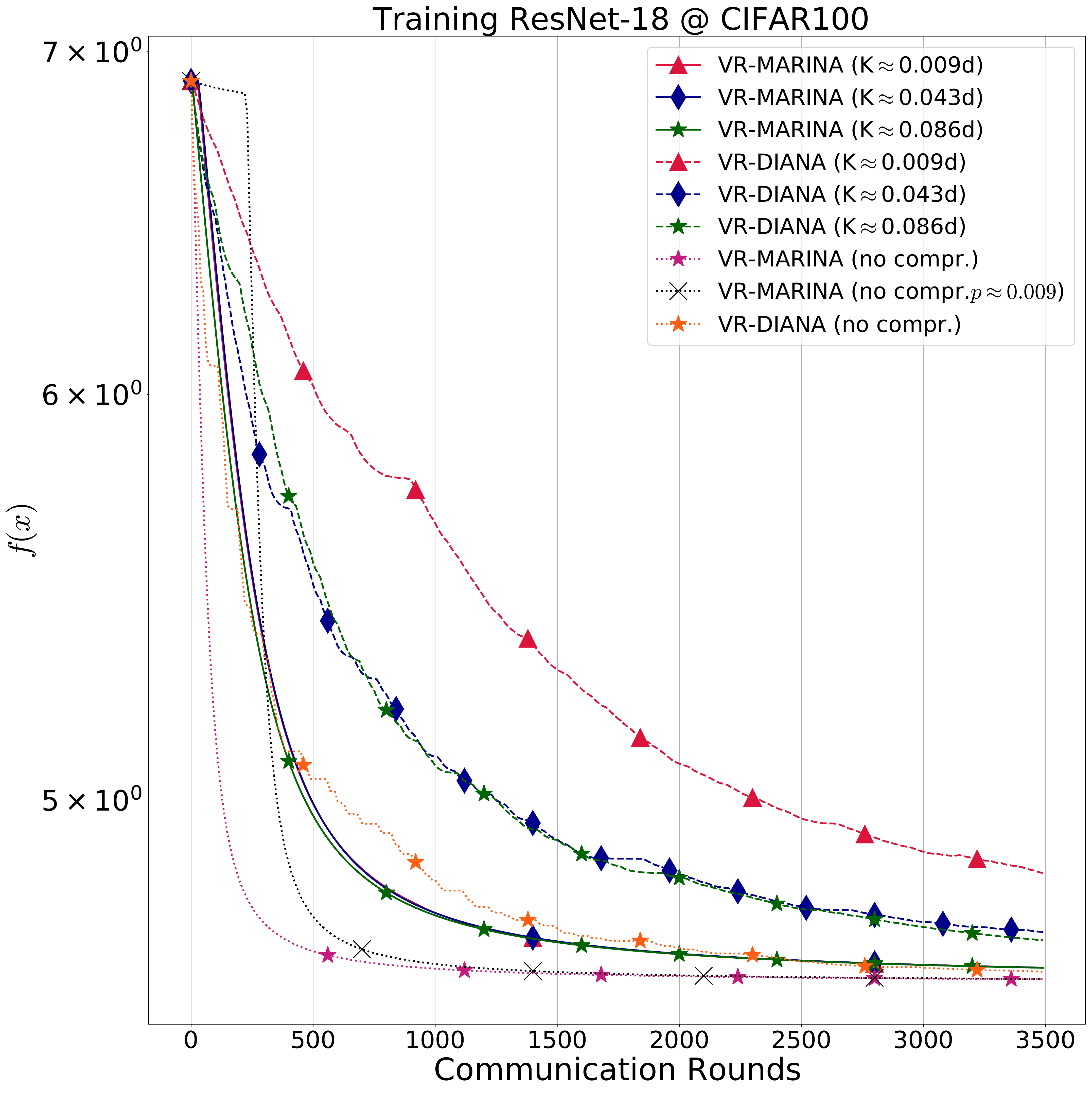}
\includegraphics[width=0.33\textwidth]{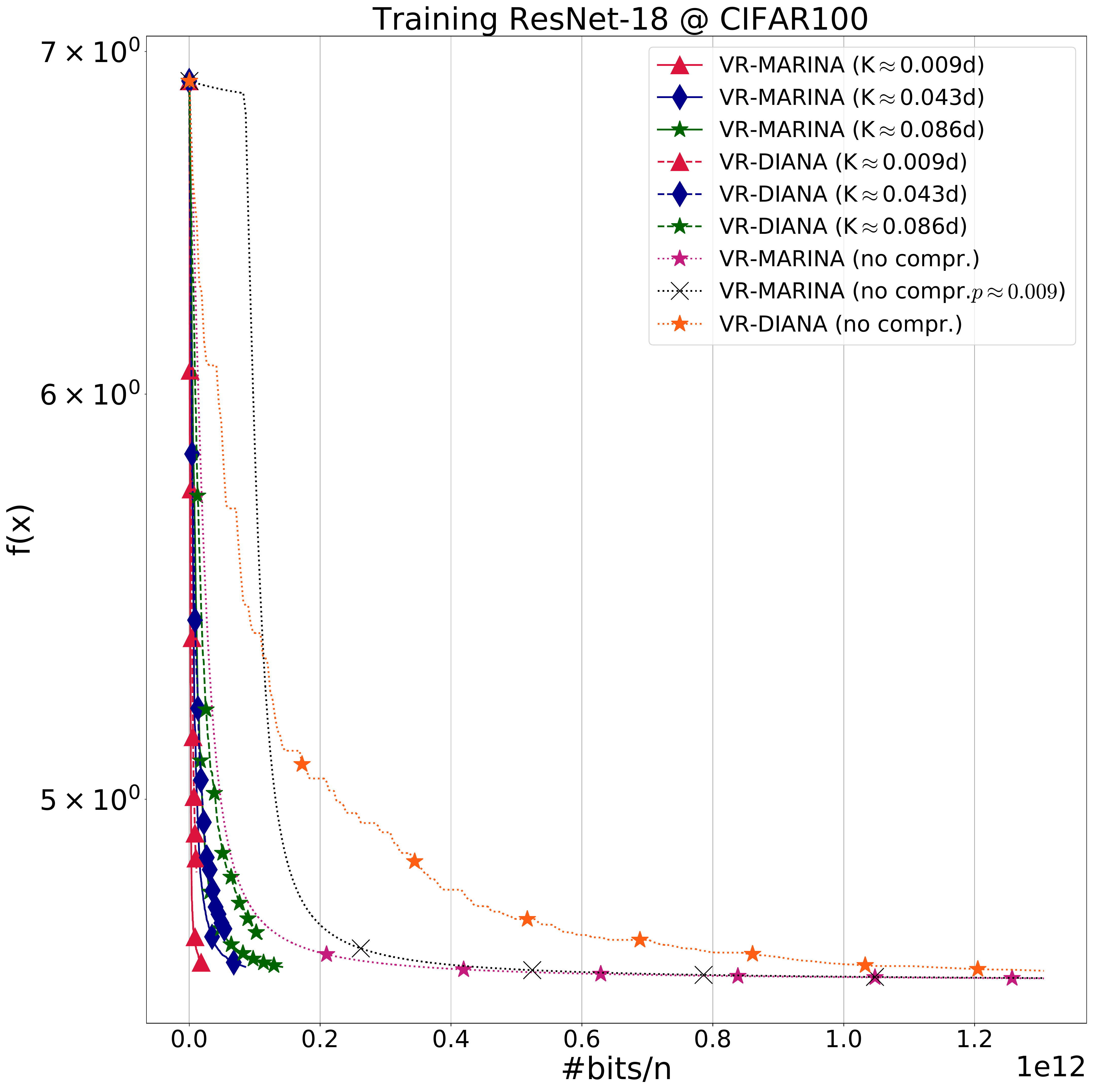}
\includegraphics[width=0.33\textwidth]{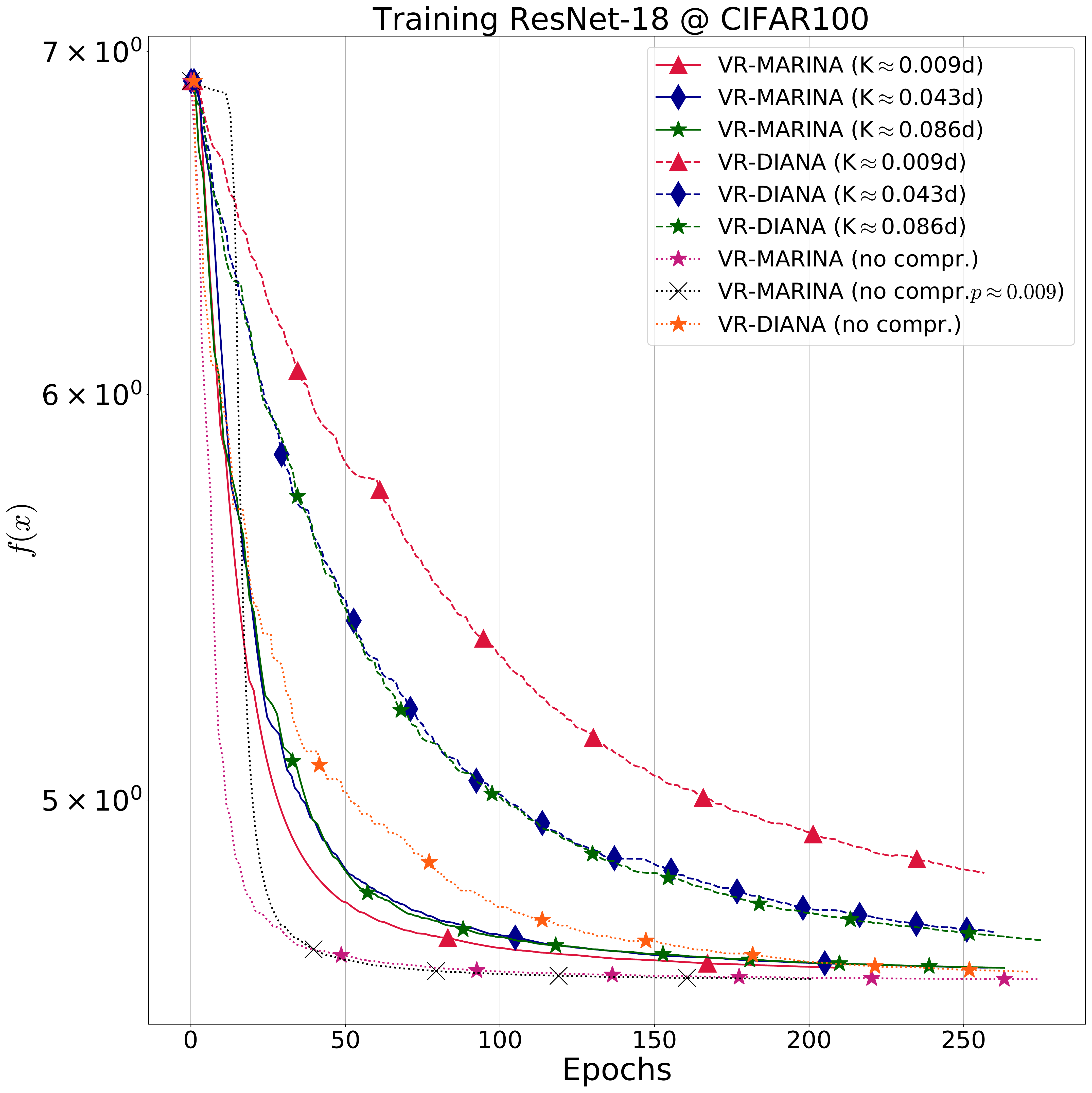}
\includegraphics[width=0.33\textwidth]{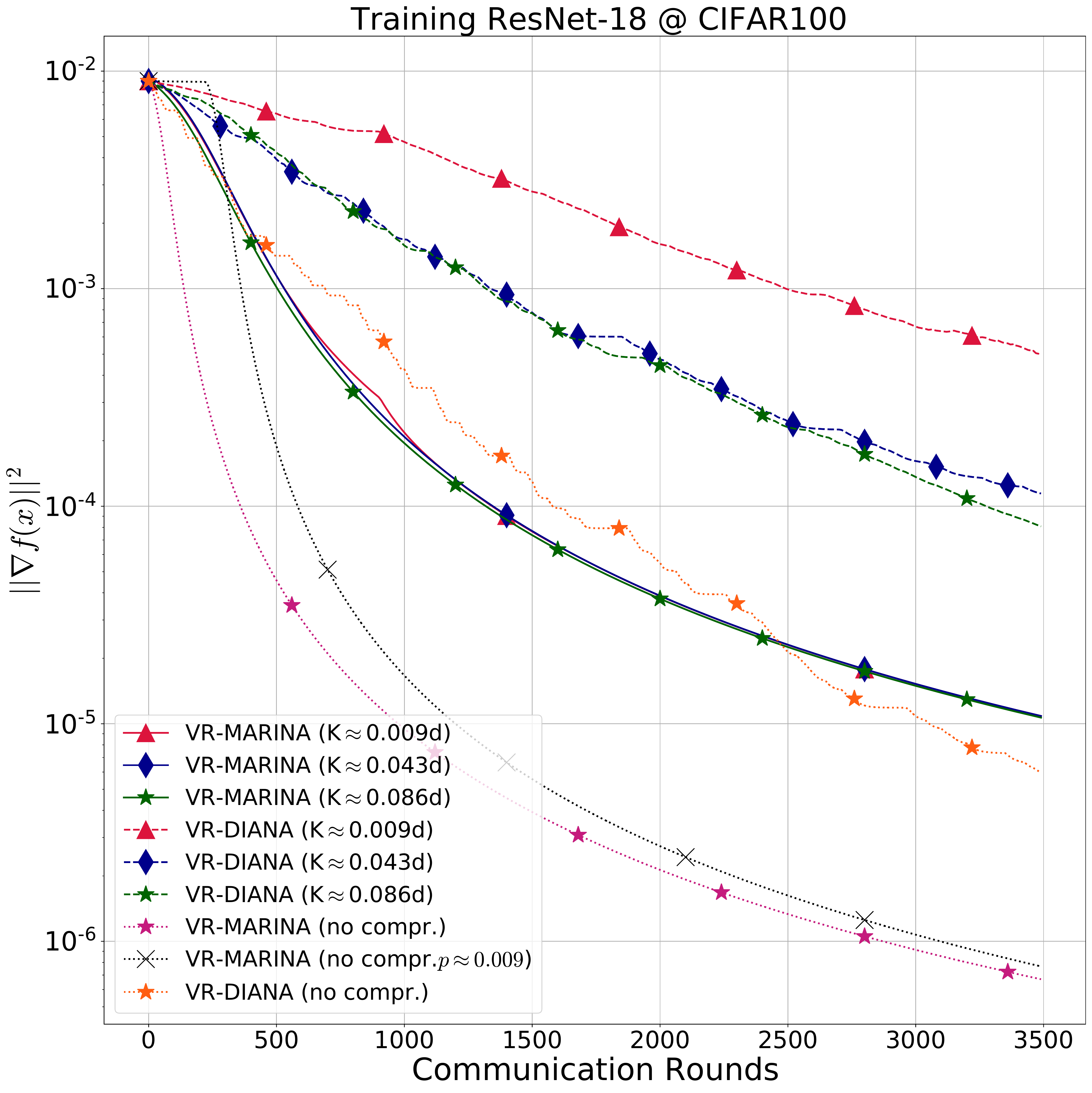}
\includegraphics[width=0.33\textwidth]{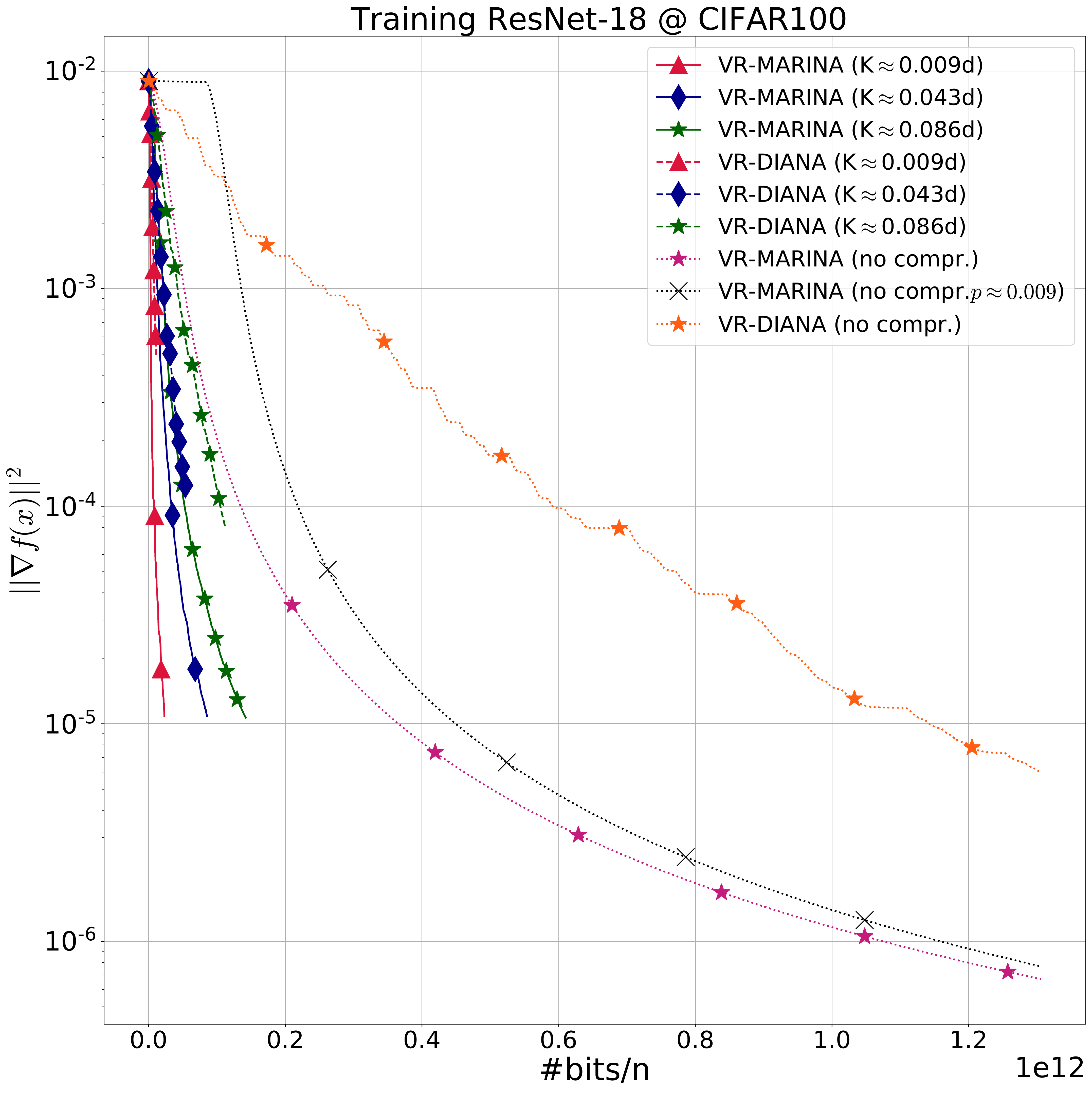}
\includegraphics[width=0.33\textwidth]{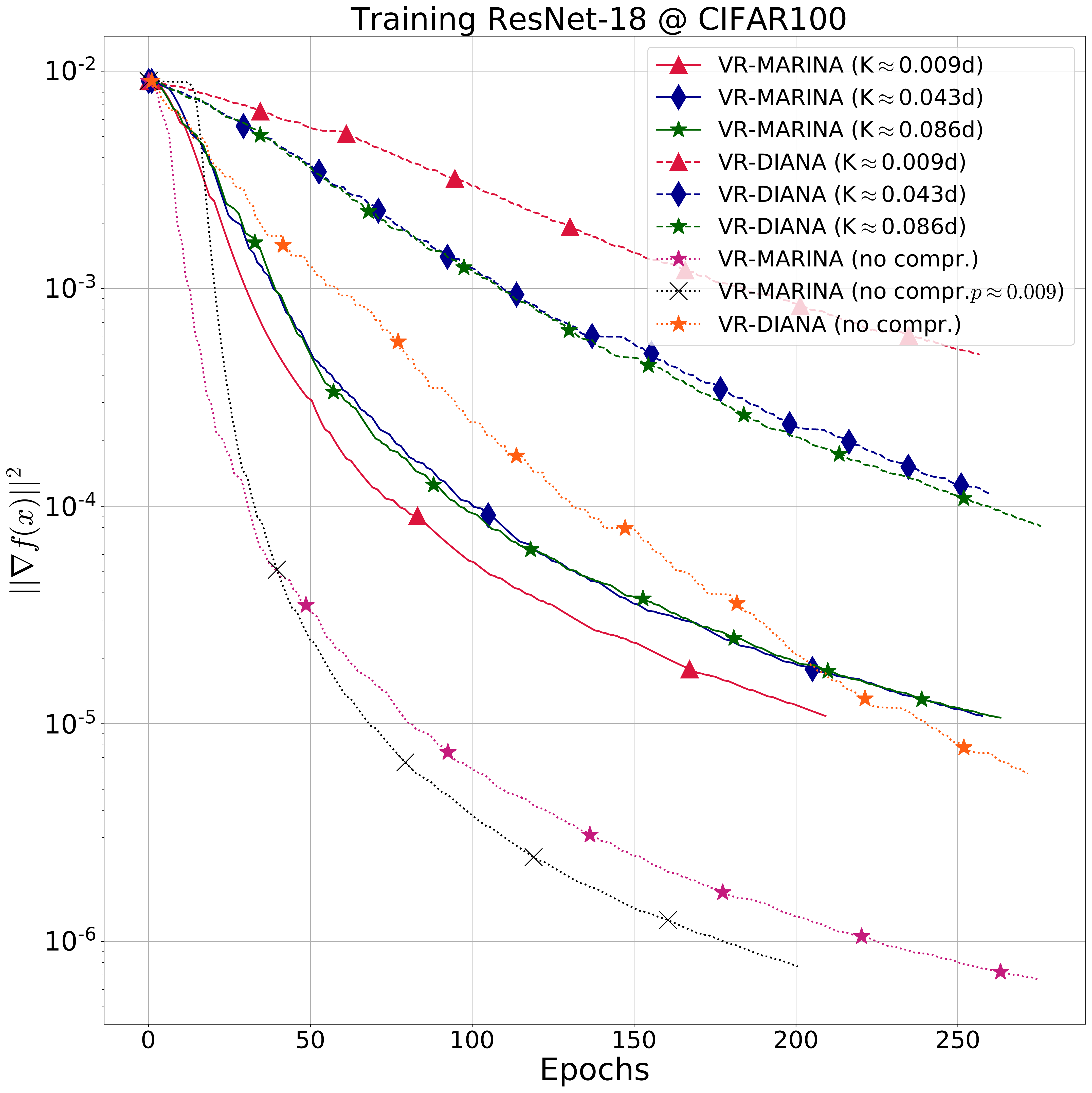}
\caption{Comparison of \algname{VR-MARINA} with \algname{VR-DIANA} on training {\tt ResNet-18} at {\tt CIFAR100} dataset. Number of workers equals $5$. Stepsizes for the methods were tuned and the batchsizes are $\sim \nicefrac{m}{50}$. We used the RandK sparsification operator, the approximate values of $K$ are given in the legends ($d$ is dimension of the problem). We also show the performance of \algname{VR-MARINA} and \algname{VR-DIANA} without compression. }
\label{fig:resnet_at_cifar100_no_compr}
\end{figure}

\clearpage
\section{Basic Facts and Auxiliary Results}\label{sec:basic_facts}

\subsection{Useful Properties of Expectations}
\textbf{Variance decomposition.} For a random vector $\xi \in \R^d$ and any deterministic vector $x \in \R^d$, the variance can be decomposed as
\begin{equation}\label{eq:variance_decomposition}
	\EE\left[\left\|\xi - \EE\xi\right\|^2\right] = \EE\left[\|\xi-x\|^2\right] - \left\|\EE\xi - x\right\|^2
\end{equation}

\textbf{Tower property of mathematical expectation.} For random variables $\xi,\eta\in \R^d$, we have
\begin{equation}
	\EE\left[\xi\right] = \EE\left[\EE\left[\xi\mid \eta\right]\right]\label{eq:tower_property}
\end{equation}
under an assumption that all expectations in the expression above are well-defined.

\subsection{One Lemma}
In this section, we formulate a lemma from \cite{li2020page}, which holds in our settings as well. We omit the proof of this lemmas since it is identical to the one from \cite{li2020page}.
\begin{lemma}[Lemma~2 from \cite{li2020page}]\label{lem:lemma_2_page}
	Assume that function $f$ is $L$-smooth and $x^{k+1} = x^k - \gamma g^k$. Then 
	\begin{equation}
		f(x^{k+1}) \le f(x^k) - \frac{\gamma}{2}\|\nabla f(x^k)\|^2 - \left(\frac{1}{2\gamma} - \frac{L}{2}\right)\|x^{k+1}-x^k\|^2 + \frac{\gamma}{2}\|g^k - \nabla f(x^k)\|^2. \label{eq:key_inequality}
	\end{equation}
\end{lemma}

\clearpage

\section{Missing Proofs for \algname{MARINA}}\label{sec:marina_proofs}
\subsection{Generally Non-Convex Problems}\label{sec:proof_of_thm_non_cvx}
In this section, we provide the full statement of Theorem~\ref{thm:main_result_non_cvx} together with the proof of this result.
\begin{theorem}[Theorem~\ref{thm:main_result_non_cvx}]\label{thm:main_result_non_cvx_appendix}
	Let Assumptions~\ref{as:lower_bound}~and~\ref{as:L_smoothness} be satisfied and 
	\begin{equation}
		\gamma \le \frac{1}{L\left(1 + \sqrt{\frac{(1-p)\omega}{pn}}\right)},\label{eq:gamma_bound_non_cvx_appendix}
	\end{equation}
	where $L^2 = \frac{1}{n}\sum_{i=1}^nL_i^2$. Then after $K$ iterations of \algname{MARINA} we have
	\begin{equation}
		\EE\left[\left\|\nabla f(\hat x^K)\right\|^2\right] \le \frac{2\Delta_0}{\gamma K}, \label{eq:main_res_non_cvx_appendix}
	\end{equation}
	where $\hat{x}^K$ is chosen uniformly at random from $x^0,\ldots,x^{K-1}$ and $\Delta_0 = f(x^0)-f_*$. That is, after
	\begin{equation}
		K = \cO\left(\frac{\Delta_0 L}{\varepsilon^2}\left(1 + \sqrt{\frac{(1-p)\omega}{pn}}\right)\right) \label{eq:main_res_2_non_cvx_appendix}
	\end{equation}
	iterations \algname{MARINA} produces such a point $\hat x^K$ that $\EE[\|\nabla f(\hat x^K)\|^2] \le \varepsilon^2$.
	Moreover, under an assumption that the communication cost is proportional to the number of non-zero components of transmitted vectors from workers to the server, we have that the expected total communication cost per worker equals
	\begin{equation}
		d + K(pd + (1-p)\zeta_{\cQ}) =  \cO\left(d+\frac{\Delta_0 L}{\varepsilon^2}\left(1 + \sqrt{\frac{(1-p)\omega}{pn}}\right)(pd + (1-p)\zeta_{\cQ})\right),\label{eq:main_res_4_non_cvx_appendix}
	\end{equation}
	where $\zeta_{\cQ}$ is the expected density of the quantization (see Def.~\ref{def:quantization}).
\end{theorem}
\begin{proof}[Proof of Theorem~\ref{thm:main_result_non_cvx}]
	The scheme of the proof is similar to the proof of Theorem~1 from \cite{li2020page}. From Lemma~\ref{lem:lemma_2_page}, we have
	\begin{equation}
		\EE[f(x^{k+1})] \le \EE[f(x^k)] - \frac{\gamma}{2}\EE\left[\|\nabla f(x^k)\|^2\right] - \left(\frac{1}{2\gamma} - \frac{L}{2}\right)\EE\left[\|x^{k+1}-x^k\|^2\right] + \frac{\gamma}{2}\EE\left[\|g^k - \nabla f(x^k)\|^2\right]. \label{eq:non_cvx_technical_1}
	\end{equation}
	Next, we need to derive an upper bound for $\EE\left[\|g^{k+1}-\nabla f(x^{k+1})\|^2\right]$. By definition of $g^{k+1}$, we have
	\begin{equation}
		g^{k+1} = \begin{cases}\nabla f(x^{k+1})& \text{with probability } p,\\ g^k + \frac{1}{n}\sum\limits_{i = 1}^n\cQ\left(\nabla f_{i}(x^{k+1}) - \nabla f_{i}(x^k)\right)& \text{with probability } 1-p. \end{cases}\notag
	\end{equation}
	Using this, variance decomposition \eqref{eq:variance_decomposition} and tower property \eqref{eq:tower_property}, we derive:
	\begin{eqnarray}
		\EE\left[\|g^{k+1}-\nabla f(x^{k+1})\|^2\right] &\overset{\eqref{eq:tower_property}}{=}& (1-p)\EE\left[\left\|g^k + \frac{1}{n}\sum\limits_{i=1}^n \cQ\left(\nabla f_{i}(x^{k+1}) - \nabla f_{i}(x^k)\right) - \nabla f(x^{k+1})\right\|^2\right]\notag\\
		&\overset{\eqref{eq:tower_property},\eqref{eq:variance_decomposition}}{=}& (1-p)\EE\left[\left\|\frac{1}{n}\sum\limits_{i=1}^n \cQ\left(\nabla f_{i}(x^{k+1}) - \nabla f_{i}(x^k)\right) - \nabla f(x^{k+1}) + \nabla f(x^k)\right\|^2\right]\notag\\
		&&\quad + (1-p)\EE\left[\left\|g^k - \nabla f(x^k)\right\|^2\right].\notag
	\end{eqnarray}
	Since $\cQ\left(\nabla f_{1}(x^{k+1}) - \nabla f_{1}(x^k)\right),\ldots,\cQ\left(\nabla f_{n}(x^{k+1}) - \nabla f_{n}(x^k)\right)$ are independent random vectors for fixed $x^k$ and $x^{k+1}$ we have
	\begin{eqnarray*}
		\EE\left[\|g^{k+1}-\nabla f(x^{k+1})\|^2\right] &=& (1-p)\EE\left[\left\|\frac{1}{n}\sum\limits_{i=1}^n \left(\cQ\left(\nabla f_{i}(x^{k+1}) - \nabla f_{i}(x^k)\right) - \nabla f_i(x^{k+1}) + \nabla f_i(x^k)\right)\right\|^2\right]\\
		&&\quad + (1-p)\EE\left[\left\|g^k - \nabla f(x^k)\right\|^2\right]\\
		&=& \frac{1-p}{n^2}\sum\limits_{i=1}^n\EE\left[\left\|\cQ\left(\nabla f_{i}(x^{k+1}) - \nabla f_{i}(x^k)\right) - \nabla f_i(x^{k+1}) + \nabla f_i(x^k)\right\|^2\right]\\
		&&\quad + (1-p)\EE\left[\left\|g^k - \nabla f(x^k)\right\|^2\right]\\
		&\overset{\eqref{eq:quantization_def}}{\le}& \frac{(1-p)\omega}{n^2}\sum\limits_{i=1}^n\EE\left[\left\|\nabla f_i(x^{k+1}) - \nabla f_i(x^k)\right\|^2\right] + (1-p)\EE\left[\left\|g^k - \nabla f(x^k)\right\|^2\right].
	\end{eqnarray*}
	Using $L$-smoothness \eqref{eq:L_smoothness} of $f_i$ together with the tower property \eqref{eq:tower_property}, we obtain
	\begin{eqnarray}
		\EE\left[\|g^{k+1}-\nabla f(x^{k+1})\|^2\right] &\le& \frac{(1-p)\omega}{n^2}\sum\limits_{i=1}^nL_i^2\EE\left[\|x^{k+1} - x^k\|^2\right] + (1-p)\EE\left[\left\|g^k - \nabla f(x^k)\right\|^2\right]\notag\\
		&=&\frac{(1-p)\omega L^2}{n}\EE\left[\|x^{k+1}-x^k\|^2\right] + (1-p)\EE\left[\left\|g^k - \nabla f(x^k)\right\|^2\right].\label{eq:non_cvx_technical_2}
	\end{eqnarray}
	Next, we introduce a new notation: $\Phi_k = f(x^k) - f_* + \frac{\gamma}{2p}\|g^k - \nabla f(x^k)\|^2$. Using this and inequalities \eqref{eq:non_cvx_technical_1} and \eqref{eq:non_cvx_technical_2}, we establish the following inequality:
	\begin{eqnarray}
		\EE\left[\Phi_{k+1}\right] &\le& \EE\left[f(x^k) - f_* - \frac{\gamma}{2}\|\nabla f(x^k)\|^2 - \left(\frac{1}{2\gamma} - \frac{L}{2}\right)\|x^{k+1}-x^k\|^2 + \frac{\gamma}{2}\|g^k - \nabla f(x^k)\|^2\right]\notag\\
		&&\quad + \frac{\gamma}{2p}\EE\left[\frac{(1-p)\omega L^2}{n}\|x^{k+1}-x^k\|^2 + (1-p)\left\|g^k - \nabla f(x^k)\right\|^2\right] \notag\\
		&=& \EE\left[\Phi_k\right] - \frac{\gamma}{2}\EE\left[\|\nabla f(x^k)\|^2\right] + \left(\frac{\gamma(1-p)\omega L^2}{2pn} - \frac{1}{2\gamma} + \frac{L}{2}\right)\EE\left[\|x^{k+1}-x^k\|^2\right]\notag\\
		&\overset{\eqref{eq:gamma_bound_non_cvx_appendix}}{\le}& \EE\left[\Phi_k\right] - \frac{\gamma}{2}\EE\left[\|\nabla f(x^k)\|^2\right],\label{eq:non_cvx_technical_3}
	\end{eqnarray}
	where in the last inequality, we use $\frac{\gamma(1-p)\omega L^2}{2pn} - \frac{1}{2\gamma} + \frac{L}{2} \le 0$ following from \eqref{eq:gamma_bound_non_cvx_appendix}. Summing up inequalities \eqref{eq:non_cvx_technical_3} for $k=0,1,\ldots,K-1$ and rearranging the terms, we derive
	\begin{eqnarray}
		\frac{1}{K}\sum\limits_{k=0}^{K-1}\EE\left[\|\nabla f(x^k)\|^2\right] &\le& \frac{2}{\gamma K}\sum\limits_{k=0}^{K-1}\left(\EE[\Phi_k]-\EE[\Phi_{k+1}]\right) = \frac{2\left(\EE[\Phi_0]-\EE[\Phi_{K}]\right)}{\gamma K} = \frac{2\Delta_0}{\gamma K},\notag
	\end{eqnarray}
	since $g^0 = \nabla f(x^0)$ and $\Phi_{k+1} \ge 0$. Finally, using the tower property \eqref{eq:tower_property} and the definition of $\hat x^K$, we obtain \eqref{eq:main_res_non_cvx_appendix} that implies \eqref{eq:main_res_2_non_cvx_appendix} and \eqref{eq:main_res_4_non_cvx_appendix}.
\end{proof}

\begin{corollary}[Corollary~\ref{cor:main_result_non_cvx}]\label{cor:main_result_non_cvx_appendix}
	Let the assumptions of Theorem~\ref{thm:main_result_non_cvx} hold and $p = \frac{\zeta_{\cQ}}{d}$, where $\zeta_{\cQ}$ is the expected density of the quantization (see Def.~\ref{def:quantization}). If 
	\begin{equation*}
		\gamma \le \frac{1}{L\left(1 + \sqrt{\frac{\omega}{n}\left(\frac{d}{\zeta_{\cQ}}-1\right)}\right)},
	\end{equation*}
	then \algname{MARINA} requires 
	\begin{equation*}
		K = \cO\left(\frac{\Delta_0 L}{\varepsilon^2}\left(1 + \sqrt{\frac{\omega}{n}\left(\frac{d}{\zeta_{\cQ}}-1\right)}\right)\right)
	\end{equation*}
	iterations/communication rounds to achieve $\EE[\|\nabla f(\hat x^K)\|^2] \le \varepsilon^2$, and the expected total communication cost per worker is
	\begin{equation*}
		\cO\left(d+\frac{\Delta_0 L}{\varepsilon^2}\left(\zeta_{\cQ} + \sqrt{\frac{\omega\zeta_{\cQ}}{n}\left(d-\zeta_{\cQ}\right)}\right)\right)
	\end{equation*}
	 under an assumption that the communication cost is proportional to the number of non-zero components of transmitted vectors from workers to the server.
\end{corollary}
\begin{proof}[Proof of Corollary~\ref{cor:main_result_non_cvx}]
	The choice of $p = \frac{\zeta_{\cQ}}{d}$ implies
	\begin{eqnarray*}
		\frac{1-p}{p} &=& \frac{d}{\zeta_{\cQ}}-1,\\
		pd + (1-p)\zeta_{\cQ} &\le& \zeta_{\cQ} + \left(1 - \frac{\zeta_{\cQ}}{d}\right)\cdot\zeta_{\cQ} \le 2\zeta_{\cQ}.
	\end{eqnarray*}	
	Plugging these relations in \eqref{eq:gamma_bound_non_cvx_appendix}, \eqref{eq:main_res_2_non_cvx_appendix}, and \eqref{eq:main_res_4_non_cvx_appendix}, we get that if
	\begin{equation*}
		\gamma \le \frac{1}{L\left(1 + \sqrt{\frac{\omega}{n}\left(\frac{d}{\zeta_{\cQ}}-1\right)}\right)},
	\end{equation*}
	then \algname{MARINA} requires 
	\begin{eqnarray*}
		K &=& \cO\left(\frac{\Delta_0 L}{\varepsilon^2}\left(1 + \sqrt{\frac{(1-p)\omega}{pn}}\right)\right)\\
		&=& \cO\left(\frac{\Delta_0 L}{\varepsilon^2}\left(1 + \sqrt{\frac{\omega}{n}\left(\frac{d}{\zeta_{\cQ}}-1\right)}\right)\right)
	\end{eqnarray*}
	iterations/communication rounds in order to achieve $\EE[\|\nabla f(\hat x^K)\|^2] \le \varepsilon^2$, and the expected total communication cost per worker is
	\begin{eqnarray*}
		d + K(pd + (1-p)\zeta_{\cQ}) &=&  \cO\left(d+\frac{\Delta_0 L}{\varepsilon^2}\left(1 + \sqrt{\frac{(1-p)\omega}{pn}}\right)(pd + (1-p)\zeta_{\cQ})\right)\\
		&=&\cO\left(d+\frac{\Delta_0 L}{\varepsilon^2}\left(\zeta_{\cQ} + \sqrt{\frac{\omega\zeta_{\cQ}}{n}\left(d-\zeta_{\cQ}\right)}\right)\right)
	\end{eqnarray*}
	 under an assumption that the communication cost is proportional to the number of non-zero components of transmitted vectors from workers to the server.
\end{proof}

\subsection{Convergence Results Under Polyak-{\L}ojasiewicz condition}\label{sec:proof_of_thm_pl}
In this section, we provide the full statement of Theorem~\ref{thm:main_result_pl} together with the proof of this result.
\begin{theorem}[Theorem~\ref{thm:main_result_pl}]\label{thm:main_result_pl_appendix}
	Let Assumptions~\ref{as:lower_bound},~\ref{as:L_smoothness}~and~\ref{as:pl_condition} be satisfied and 
	\begin{equation}
		\gamma \le \min\left\{\frac{1}{L\left(1 + \sqrt{\frac{2(1-p)\omega}{pn}}\right)}, \frac{p}{2\mu}\right\},\label{eq:gamma_bound_pl_appendix}
	\end{equation}
	where $L^2 = \frac{1}{n}\sum_{i=1}^nL_i^2$. Then after $K$ iterations of \algname{MARINA} we have
	\begin{equation}
		\EE\left[f(x^K) - f(x^*)\right] \le (1-\gamma\mu)^K\Delta_0, \label{eq:main_res_pl_appendix}
	\end{equation}
	where $\Delta_0 = f(x^0)-f(x^*)$. That is, after
	\begin{equation}
		K = \cO\left(\max\left\{\frac{1}{p},\frac{L}{\mu}\left(1 + \sqrt{\frac{(1-p)\omega}{pn}}\right)\right\}\log\frac{\Delta_0}{\varepsilon}\right) \label{eq:main_res_2_pl_appendix}
	\end{equation}
	iterations \algname{MARINA} produces such a point $x^K$ that $\EE[f(x^K) - f(x^*)] \le \varepsilon$.
	Moreover, under an assumption that the communication cost is proportional to the number of non-zero components of transmitted vectors from workers to the server, we have that the expected total communication cost per worker equals
	\begin{equation}
		d + K(pd + (1-p)\zeta_{\cQ}) =  \cO\left(d+\max\left\{\frac{1}{p},\frac{L}{\mu}\left(1 + \sqrt{\frac{(1-p)\omega}{pn}}\right)\right\}(pd + (1-p)\zeta_{\cQ})\log\frac{\Delta_0}{\varepsilon}\right),\label{eq:main_res_4_pl_appendix}
	\end{equation}
	where $\zeta_{\cQ}$ is the expected density of the quantization (see Def.~\ref{def:quantization}).
\end{theorem}
\begin{proof}[Proof of Theorem~\ref{thm:main_result_pl}]
	The proof is very similar to the proof of Theorem~\ref{thm:main_result_non_cvx}. From Lemma~\ref{lem:lemma_2_page} and P{\L} condition, we have
	\begin{eqnarray}
		\EE[f(x^{k+1}) - f(x^*)] &\le& \EE[f(x^k) - f(x^*)] - \frac{\gamma}{2}\EE\left[\|\nabla f(x^k)\|^2\right] - \left(\frac{1}{2\gamma} - \frac{L}{2}\right)\EE\left[\|x^{k+1}-x^k\|^2\right]\notag\\
		&&\quad + \frac{\gamma}{2}\EE\left[\|g^k - \nabla f(x^k)\|^2\right]\notag\\
		&\overset{\eqref{eq:pl_condition}}{\le}& (1-\gamma\mu)\EE\left[f(x^k) - f(x^*)\right] - \left(\frac{1}{2\gamma} - \frac{L}{2}\right)\EE\left[\|x^{k+1}-x^k\|^2\right] + \frac{\gamma}{2}\EE\left[\|g^k - \nabla f(x^k)\|^2\right]. \notag
	\end{eqnarray}
	Using the same arguments as in the proof of \eqref{eq:non_cvx_technical_2}, we obtain
	\begin{eqnarray}
		\EE\left[\|g^{k+1}-\nabla f(x^{k+1})\|^2\right] &\le& \frac{(1-p)\omega L^2}{n}\EE\left[\|x^{k+1}-x^k\|^2\right] + (1-p)\EE\left[\left\|g^k - \nabla f(x^k)\right\|^2\right].\notag
	\end{eqnarray}
	Putting all together, we derive that the sequence $\Phi_k = f(x^k) - f(x^*) + \frac{\gamma}{p}\|g^k - \nabla f(x^k)\|^2$ satisfies
	\begin{eqnarray}
		\EE\left[\Phi_{k+1}\right] &\le& \EE\left[(1-\gamma\mu)(f(x^k) - f(x^*)) - \left(\frac{1}{2\gamma} - \frac{L}{2}\right)\|x^{k+1}-x^k\|^2 + \frac{\gamma}{2}\|g^k - \nabla f(x^k)\|^2\right]\notag\\
		&&\quad + \frac{\gamma}{p}\EE\left[\frac{(1-p)\omega L^2}{n}\|x^{k+1}-x^k\|^2 + (1-p)\left\|g^k - \nabla f(x^k)\right\|^2 \right] \notag\\
		&=& \EE\left[(1-\gamma\mu)(f(x^k) - f(x^*)) + \left(\frac{\gamma}{2} + \frac{\gamma}{p}(1-p)\right)\left\|g^k - \nabla f(x^k)\right\|^2\right]\notag\\
		&&\quad + \left(\frac{\gamma(1-p)\omega L^2}{pn} - \frac{1}{2\gamma} + \frac{L}{2}\right)\EE\left[\|x^{k+1}-x^k\|^2\right]\notag\\
		&\overset{\eqref{eq:gamma_bound_pl_appendix}}{\le}& (1-\gamma\mu)\EE[\Phi_k],\notag
	\end{eqnarray}
	where in the last inequality, we use $\frac{\gamma(1-p)\omega L^2}{pn} - \frac{1}{2\gamma} + \frac{L}{2} \le 0$ and $\frac{\gamma}{2} + \frac{\gamma}{p}(1-p) \le (1-\gamma\mu)\frac{\gamma}{p}$ following from \eqref{eq:gamma_bound_pl_appendix}. Unrolling the recurrence and using $g^0 = \nabla f(x^0)$, we obtain 
	\begin{eqnarray*}
		\EE\left[f(x^{K}) - f(x^*)\right] \le \EE[\Phi_{K}] &\le& (1-\gamma\mu)^{K}\Phi_0 = (1-\gamma\mu)^{K}(f(x^0) - f(x^*))
	\end{eqnarray*}
	that implies \eqref{eq:main_res_2_pl_appendix} and \eqref{eq:main_res_4_pl_appendix}.
\end{proof}

\begin{corollary}\label{cor:main_result_pl_appendix}
	Let the assumptions of Theorem~\ref{thm:main_result_pl} hold and $p = \frac{\zeta_{\cQ}}{d}$, where $\zeta_{\cQ}$ is the expected density of the quantization (see Def.~\ref{def:quantization}). If 
	\begin{equation*}
		\gamma \le \min\left\{\frac{1}{L\left(1 + \sqrt{\frac{2\omega}{n}\left(\frac{d}{\zeta_{\cQ}}-1\right)}\right)}, \frac{p}{2\mu}\right\},
	\end{equation*}
	then \algname{MARINA} requires 
	\begin{equation*}
		K = \cO\left(\max\left\{\frac{d}{\zeta_{\cQ}},\frac{L}{\mu}\left(1 + \sqrt{\frac{\omega}{n}\left(\frac{d}{\zeta_{\cQ}}-1\right)}\right)\right\}\log\frac{\Delta_0}{\varepsilon}\right)
	\end{equation*}
	iterations/communication rounds to achieve $\EE[f(x^K) - f(x^*)] \le \varepsilon$, and the expected total communication cost per worker is
	\begin{equation*}
		\cO\left(d+\max\left\{d,\frac{L}{\mu}\left(\zeta_{\cQ} + \sqrt{\frac{\omega\zeta_{\cQ}}{n}\left(d-\zeta_{\cQ}\right)}\right)\right\}\log\frac{\Delta_0}{\varepsilon}\right)
	\end{equation*}
	 under an assumption that the communication cost is proportional to the number of non-zero components of transmitted vectors from workers to the server.
\end{corollary}
\begin{proof}
	The choice of $p = \frac{\zeta_{\cQ}}{d}$ implies
	\begin{eqnarray*}
		\frac{1-p}{p} &=& \frac{d}{\zeta_{\cQ}}-1,\\
		pd + (1-p)\zeta_{\cQ} &\le& \zeta_{\cQ} + \left(1 - \frac{\zeta_{\cQ}}{d}\right)\cdot\zeta_{\cQ} \le 2\zeta_{\cQ}.
	\end{eqnarray*}	
	Plugging these relations in \eqref{eq:gamma_bound_pl_appendix}, \eqref{eq:main_res_2_pl_appendix}, and \eqref{eq:main_res_4_pl_appendix}, we get that if
	\begin{equation*}
		\gamma \le \min\left\{\frac{1}{L\left(1 + \sqrt{\frac{2\omega}{n}\left(\frac{d}{\zeta_{\cQ}}-1\right)}\right)}, \frac{p}{2\mu}\right\},
	\end{equation*}
	then \algname{MARINA} requires 
	\begin{eqnarray*}
		K &=& \cO\left(\max\left\{\frac{1}{p},\frac{L}{\mu}\left(1 + \sqrt{\frac{(1-p)\omega}{pn}}\right)\right\}\log\frac{\Delta_0}{\varepsilon}\right)\\
		&=& \cO\left(\max\left\{\frac{d}{\zeta_{\cQ}},\frac{L}{\mu}\left(1 + \sqrt{\frac{\omega}{n}\left(\frac{d}{\zeta_{\cQ}}-1\right)}\right)\right\}\log\frac{\Delta_0}{\varepsilon}\right)
	\end{eqnarray*}
	iterations/communication rounds in order to achieve $\EE[f(x^K)-f(x^*)] \le \varepsilon$, and the expected total communication cost per worker is
	\begin{eqnarray*}
		d + K(pd + (1-p)\zeta_{\cQ}) &=&  \cO\left(d+\max\left\{\frac{1}{p},\frac{L}{\mu}\left(1 + \sqrt{\frac{(1-p)\omega}{pn}}\right)\right\}(pd + (1-p)\zeta_{\cQ})\log\frac{\Delta_0}{\varepsilon}\right)\\
		&=&\cO\left(d+\max\left\{d,\frac{L}{\mu}\left(\zeta_{\cQ} + \sqrt{\frac{\omega\zeta_{\cQ}}{n}\left(d-\zeta_{\cQ}\right)}\right)\right\}\log\frac{\Delta_0}{\varepsilon}\right)
	\end{eqnarray*}
	 under an assumption that the communication cost is proportional to the number of non-zero components of transmitted vectors from workers to the server.
\end{proof}

\clearpage
\section{Missing Proofs for \algname{VR-MARINA}}\label{sec:missing_proofs}

\subsection{Finite Sum Case}

\subsubsection{Generally Non-Convex Problems}\label{sec:proof_of_thm_non_cvx_fin_sums}
In this section, we provide the full statement of Theorem~\ref{thm:main_result_non_cvx_finite_sums} together with the proof of this result.
\begin{theorem}[Theorem~\ref{thm:main_result_non_cvx_finite_sums}]\label{thm:main_result_non_cvx_finite_sums_appendix}
	Consider the finite sum case \eqref{eq:main_problem}+\eqref{eq:f_i_finite_sum}. Let Assumptions~\ref{as:lower_bound},~\ref{as:L_smoothness}~and~\ref{as:avg_smoothness} be satisfied and 
	\begin{equation}
		\gamma \le \frac{1}{L + \sqrt{\frac{1-p}{pn}\left(\omega L^2 + \frac{(1+\omega)\cL^2}{b'}\right)}},\label{eq:gamma_bound_non_cvx_finite_sums_appendix}
	\end{equation}
	where $L^2 = \frac{1}{n}\sum_{i=1}^nL_i^2$ and $\cL^2 = \frac{1}{n}\sum_{i=1}^n\cL_i^2$. Then after $K$ iterations of \algname{VR-MARINA} we have
	\begin{equation}
		\EE\left[\left\|\nabla f(\hat x^K)\right\|^2\right] \le \frac{2\Delta_0}{\gamma K}, \label{eq:main_res_non_cvx_finite_sums_appendix}
	\end{equation}
	where $\hat{x}^K$ is chosen uniformly at random from $x^0,\ldots,x^{K-1}$ and $\Delta_0 = f(x^0)-f_*$. That is, after
	\begin{equation}
		K = \cO\left(\frac{\Delta_0}{\varepsilon^2}\left(L + \sqrt{\frac{1-p}{pn}\left(\omega L^2 + \frac{(1+\omega)\cL^2}{b'}\right)}\right)\right) \label{eq:main_res_2_non_cvx_finite_sums_appendix}
	\end{equation}
	iterations \algname{VR-MARINA} produces such a point $\hat x^K$ that $\EE[\|\nabla f(\hat x^K)\|^2] \le \varepsilon^2$, and the expected total number of stochastic oracle calls per node equals
	\begin{equation}
		m + K(pm + 2(1-p)b') = \cO\left(m + \frac{\Delta_0}{\varepsilon^2}\left(L + \sqrt{\frac{1-p}{pn}\left(\omega L^2 + \frac{(1+\omega)\cL^2}{b'}\right)}\right)(pm + (1-p)b')\right). \label{eq:main_res_3_non_cvx_finite_sums_appendix}
	\end{equation}
	Moreover, under an assumption that the communication cost is proportional to the number of non-zero components of transmitted vectors from workers to the server, we have that the expected total communication cost per worker equals
	\begin{equation}
		d + K(pd + (1-p)\zeta_{\cQ}) =  \cO\left(d+\frac{\Delta_0}{\varepsilon^2}\left(L + \sqrt{\frac{1-p}{pn}\left(\omega L^2 + \frac{(1+\omega)\cL^2}{b'}\right)}\right)(pd + (1-p)\zeta_{\cQ})\right),\label{eq:main_res_4_non_cvx_finite_sums_appendix}
	\end{equation}
	where $\zeta_{\cQ}$ is the expected density of the quantization (see Def.~\ref{def:quantization}).
\end{theorem}
\begin{proof}[Proof of Theorem~\ref{thm:main_result_non_cvx_finite_sums}]
	The proof of this theorem is a generalization of the proof of Theorem~\ref{thm:main_result_non_cvx}. From Lemma~\ref{lem:lemma_2_page}, we have
	\begin{equation}
		\EE[f(x^{k+1})] \le \EE[f(x^k)] - \frac{\gamma}{2}\EE\left[\|\nabla f(x^k)\|^2\right] - \left(\frac{1}{2\gamma} - \frac{L}{2}\right)\EE\left[\|x^{k+1}-x^k\|^2\right] + \frac{\gamma}{2}\EE\left[\|g^k - \nabla f(x^k)\|^2\right]. \label{eq:non_cvx_finite_sums_technical_1}
	\end{equation}
	Next, we need to derive an upper bound for $\EE\left[\|g^{k+1}-\nabla f(x^{k+1})\|^2\right]$. Since $g^{k+1} = \frac{1}{n}\sum\limits_{i=1}^ng_i^{k+1}$, we get the following representation of $g^{k+1}$:
	\begin{equation}
		g^{k+1} = \begin{cases}\nabla f(x^{k+1})& \text{with probability } p,\\ g^k + \frac{1}{n}\sum\limits_{i=1}^n \cQ\left(\frac{1}{b'}\sum\limits_{j\in I'_{i,k}}(\nabla f_{ij}(x^{k+1}) - \nabla f_{ij}(x^k))\right)& \text{with probability } 1-p. \end{cases}\notag
	\end{equation}
	Using this, variance decomposition \eqref{eq:variance_decomposition} and tower property \eqref{eq:tower_property}, we derive:
	\begin{eqnarray}
		\EE\left[\|g^{k+1}-\nabla f(x^{k+1})\|^2\right] &\overset{\eqref{eq:tower_property}}{=}& (1-p)\EE\left[\left\|g^k + \frac{1}{n}\sum\limits_{i=1}^n \cQ\left(\frac{1}{b'}\sum\limits_{j\in I'_{i,k}}(\nabla f_{ij}(x^{k+1}) - \nabla f_{ij}(x^k))\right) - \nabla f(x^{k+1})\right\|^2\right]\notag\\
		&\overset{\eqref{eq:tower_property},\eqref{eq:variance_decomposition}}{=}& (1-p)\EE\left[\left\|\frac{1}{n}\sum\limits_{i=1}^n \cQ\left(\frac{1}{b'}\sum\limits_{j\in I'_{i,k}}(\nabla f_{ij}(x^{k+1}) - \nabla f_{ij}(x^k))\right) - \nabla f(x^{k+1}) + \nabla f(x^k)\right\|^2\right]\notag\\
		&&\quad + (1-p)\EE\left[\left\|g^k - \nabla f(x^k)\right\|^2\right].\notag
	\end{eqnarray}
	Next, we use the notation: $\widetilde{\Delta}_i^k = \frac{1}{b'}\sum\limits_{j\in I'_{i,k}}(\nabla f_{ij}(x^{k+1}) - \nabla f_{ij}(x^k))$ and $\Delta_i^k = \nabla f_i(x^{k+1}) - \nabla f_i(x^k)$. These vectors satisfy $\EE\left[\widetilde{\Delta}_i^k \mid x^k,x^{k+1}\right] = \Delta_i^k$ for all $i\in [n]$. Moreover, $\cQ(\widetilde{\Delta}_1^k),\ldots,\cQ(\widetilde{\Delta}_n^k)$ are independent random vectors for fixed $x^k$ and $x^{k+1}$. These observations imply
	\begin{eqnarray}
		\EE\left[\|g^{k+1}-\nabla f(x^{k+1})\|^2\right] &=& (1-p)\EE\left[\left\|\frac{1}{n}\sum\limits_{i=1}^n \left(\cQ(\widetilde{\Delta}_i^k) - \Delta_i^k\right)\right\|^2\right]+(1-p)\EE\left[\left\|g^k - \nabla f(x^k)\right\|^2\right]\notag\\
		&=& \frac{1-p}{n^2}\sum\limits_{i=1}^n\EE\left[\left\|\cQ(\widetilde{\Delta}_i^k) - \widetilde{\Delta}_i^k + \widetilde{\Delta}_i^k - \Delta_i^k\right\|^2\right] + (1-p)\EE\left[\left\|g^k - \nabla f(x^k)\right\|^2\right]\notag\\
		&\overset{\eqref{eq:tower_property},\eqref{eq:variance_decomposition}}{=}& \frac{1-p}{n^2}\sum\limits_{i=1}^n\left(\EE\left[\left\|\cQ(\widetilde{\Delta}_i^k) - \widetilde{\Delta}_i^k\right\|^2\right] + \EE\left[\left\|\widetilde{\Delta}_i^k - \Delta_i^k\right\|^2\right]\right)\notag\\
		&&\quad + (1-p)\EE\left[\left\|g^k - \nabla f(x^k)\right\|^2\right]\notag\\
		&\overset{\eqref{eq:tower_property},\eqref{eq:quantization_def}}{=}& \frac{1-p}{n^2}\sum\limits_{i=1}^n\left(\omega\EE\left[\left\|\widetilde{\Delta}_i^k\right\|^2\right] + \EE\left[\left\|\widetilde{\Delta}_i^k - \Delta_i^k\right\|^2\right]\right) + (1-p)\EE\left[\left\|g^k - \nabla f(x^k)\right\|^2\right]\notag\\
		&\overset{\eqref{eq:tower_property},\eqref{eq:variance_decomposition}}{=}& \frac{1-p}{n^2}\sum\limits_{i=1}^n\left(\omega\EE\left[\left\|\Delta_i^k\right\|^2\right] + (1+\omega)\EE\left[\left\|\widetilde{\Delta}_i^k - \Delta_i^k\right\|^2\right]\right)\notag\\
		&&\quad + (1-p)\EE\left[\left\|g^k - \nabla f(x^k)\right\|^2\right].\notag
	\end{eqnarray}
	Using $L$-smoothness \eqref{eq:L_smoothness} and average $\cL$-smoothness \eqref{eq:avg_L_smoothness} of $f_i$ together with the tower property \eqref{eq:tower_property}, we get
	\begin{eqnarray}
		\EE\left[\|g^{k+1}-\nabla f(x^{k+1})\|^2\right] &\le& \frac{1-p}{n^2}\sum\limits_{i=1}^n\left(\omega L_i^2 + \frac{(1+\omega)\cL_i^2}{b'}\right)\EE\left[\|x^{k+1} - x^k\|^2\right]\notag\\
		&&\quad + (1-p)\EE\left[\left\|g^k - \nabla f(x^k)\right\|^2\right]\notag\\
		&=&\frac{1-p}{n}\left(\omega L^2 + \frac{(1+\omega)\cL^2}{b'}\right)\EE\left[\|x^{k+1}-x^k\|^2\right]\notag\\
		&&\quad + (1-p)\EE\left[\left\|g^k - \nabla f(x^k)\right\|^2\right].\label{eq:non_cvx_finite_sums_technical_2}
	\end{eqnarray}
	Next, we introduce new notation: $\Phi_k = f(x^k) - f_* + \frac{\gamma}{2p}\|g^k - \nabla f(x^k)\|^2$. Using this and inequalities \eqref{eq:non_cvx_finite_sums_technical_1} and \eqref{eq:non_cvx_finite_sums_technical_2}, we establish the following inequality:
	\begin{eqnarray}
		\EE\left[\Phi_{k+1}\right] &\le& \EE\left[f(x^k) - f_* - \frac{\gamma}{2}\|\nabla f(x^k)\|^2 - \left(\frac{1}{2\gamma} - \frac{L}{2}\right)\|x^{k+1}-x^k\|^2 + \frac{\gamma}{2}\|g^k - \nabla f(x^k)\|^2\right]\notag\\
		&&\quad + \frac{\gamma}{2p}\EE\left[\frac{1-p}{n}\left(\omega L^2 + \frac{(1+\omega)\cL^2}{b'}\right)\|x^{k+1}-x^k\|^2 + (1-p)\left\|g^k - \nabla f(x^k)\right\|^2\right] \notag\\
		&=& \EE\left[\Phi_k\right] - \frac{\gamma}{2}\EE\left[\|\nabla f(x^k)\|^2\right] + \left(\frac{\gamma(1-p)}{2pn}\left(\omega L^2 + \frac{(1+\omega)\cL^2}{b'}\right) - \frac{1}{2\gamma} + \frac{L}{2}\right)\EE\left[\|x^{k+1}-x^k\|^2\right]\notag\\
		&\overset{\eqref{eq:gamma_bound_non_cvx_finite_sums_appendix}}{\le}& \EE\left[\Phi_k\right] - \frac{\gamma}{2}\EE\left[\|\nabla f(x^k)\|^2\right],\label{eq:non_cvx_finite_sums_technical_3}
	\end{eqnarray}
	where in the last inequality, we use $\frac{\gamma(1-p)}{2pn}\left(\omega L^2 + \frac{(1+\omega)\cL^2}{b'}\right) - \frac{1}{2\gamma} + \frac{L}{2} \le 0$ following from \eqref{eq:gamma_bound_non_cvx_finite_sums_appendix}. Summing up inequalities \eqref{eq:non_cvx_finite_sums_technical_3} for $k=0,1,\ldots,K-1$ and rearranging the terms, we derive
	\begin{eqnarray}
		\frac{1}{K}\sum\limits_{k=0}^{K-1}\EE\left[\|\nabla f(x^k)\|^2\right] &\le& \frac{2}{\gamma K}\sum\limits_{k=0}^{K-1}\left(\EE[\Phi_k]-\EE[\Phi_{k+1}]\right) = \frac{2\left(\EE[\Phi_0]-\EE[\Phi_{K}]\right)}{\gamma K} = \frac{2\Delta_0}{\gamma K},\notag
	\end{eqnarray}
	since $g^0 = \nabla f(x^0)$ and $\Phi_{k+1} \ge 0$. Finally, using the tower property \eqref{eq:tower_property} and the definition of $\hat x^K$, we obtain \eqref{eq:main_res_non_cvx_finite_sums_appendix} that implies \eqref{eq:main_res_2_non_cvx_finite_sums_appendix}, \eqref{eq:main_res_3_non_cvx_finite_sums_appendix}, and \eqref{eq:main_res_4_non_cvx_finite_sums_appendix}.
\end{proof}

\begin{remark}[About batchsizes dissimilarity]
	We notice that our analysis can be easily extended to handle the version of \algname{VR-MARINA} with different batchsizes $b'_1,\ldots,b'_n$ on different workers, i.e., when $|I_{i,k}'| = b_i'$ and $\widetilde{\Delta}_i^k = \frac{1}{b_i'}\sum_{j\in I_{i,k}'}(\nabla f_{ij}(x^{k+1}) - \nabla f_{ij}(x^k))$. In this case, the statement of Theorem~\ref{thm:main_result_non_cvx_finite_sums} remains the same with the small modificiation: instead of $\frac{\cL^2}{b'}$ the complexity bounds will have $\frac{1}{n}\sum_{i=1}^n\frac{\cL_i^2}{b_i'}$.
\end{remark}

\begin{corollary}[Corollary~\ref{cor:main_result_non_cvx_finite_sums}]\label{cor:main_result_non_cvx_finite_sums_appendix}
	Let the assumptions of Theorem~\ref{thm:main_result_non_cvx_finite_sums} hold and $p = \min\left\{\frac{\zeta_{\cQ}}{d},\frac{b'}{m+b'}\right\}$, where $b' \le m$ and $\zeta_{\cQ}$ is the expected density of the quantization (see Def.~\ref{def:quantization}). If 
	\begin{equation*}
		\gamma \le \frac{1}{L + \sqrt{\frac{\max\left\{\nicefrac{d}{\zeta_{\cQ}} - 1,\nicefrac{m}{b'}\right\}}{n}\left(\omega L^2 + \frac{(1+\omega)\cL^2}{b'}\right)}},
	\end{equation*}
	then \algname{VR-MARINA} requires 
	\begin{equation*}
		\cO\left(\frac{\Delta_0}{\varepsilon^2}\left(L\left(1 + \sqrt{\frac{\omega\max\left\{\nicefrac{d}{\zeta_{\cQ}} - 1,\nicefrac{m}{b'}\right\}}{n}}\right) + \cL\sqrt{\frac{(1+\omega)\max\left\{\nicefrac{d}{\zeta_{\cQ}} - 1,\nicefrac{m}{b'}\right\}}{nb'}}\right)\right)
	\end{equation*}
	iterations/communication rounds, 
	\begin{equation*}
		\cO\left(m+\frac{\Delta_0}{\varepsilon^2}\left(L\left(b' + \sqrt{\frac{\omega\max\left\{(\nicefrac{d}{\zeta_{\cQ}} - 1)(b')^2,mb'\right\}}{n}}\right) + \cL\sqrt{\frac{(1+\omega)\max\left\{(\nicefrac{d}{\zeta_{\cQ}} - 1) b',m\right\}}{n}}\right)\right)
	\end{equation*}
	stochastic oracle calls per node in expectation in order to achieve $\EE[\|\nabla f(\hat x^K)\|^2] \le \varepsilon^2$, and the expected total communication cost per worker is
	\begin{equation*}
		\cO\left(d+\frac{\Delta_0\zeta_{\cQ}}{\varepsilon^2}\left(L\left(1 + \sqrt{\frac{\omega\max\left\{\nicefrac{d}{\zeta_{\cQ}} - 1,\nicefrac{m}{b'}\right\}}{n}}\right) + \cL\sqrt{\frac{(1+\omega)\max\left\{\nicefrac{d}{\zeta_{\cQ}} - 1,\nicefrac{m}{b'}\right\}}{nb'}}\right)\right)
	\end{equation*}
	 under an assumption that the communication cost is proportional to the number of non-zero components of transmitted vectors from workers to the server.
\end{corollary}
\begin{proof}[Proof of Corollary~\ref{cor:main_result_non_cvx_finite_sums}]
	The choice of $p = \min\left\{\frac{\zeta_{\cQ}}{d},\frac{b'}{m+b'}\right\}$ implies
	\begin{eqnarray*}
		\frac{1-p}{p} &=& \max\left\{\frac{d}{\zeta_{\cQ}}-1,\frac{m}{b'}\right\},\\
		pm + (1-p)b' &\le& \frac{2mb'}{m+b'} \le 2b',\\
		pd + (1-p)\zeta_{\cQ} &\le& \frac{\zeta_{\cQ}}{d}\cdot d + \left(1 - \frac{\zeta_{\cQ}}{d}\right)\cdot\zeta_{\cQ} \le 2\zeta_{\cQ}.
	\end{eqnarray*}	
	Plugging these relations in \eqref{eq:gamma_bound_non_cvx_finite_sums_appendix}, \eqref{eq:main_res_2_non_cvx_finite_sums_appendix}, \eqref{eq:main_res_3_non_cvx_finite_sums_appendix} and \eqref{eq:main_res_4_non_cvx_finite_sums_appendix} and using $\sqrt{a+b} \le \sqrt{a} + \sqrt{b}$, we get that if
	\begin{equation*}
		\gamma \le \frac{1}{L + \sqrt{\frac{\max\left\{\nicefrac{d}{\zeta_{\cQ}} - 1,\nicefrac{m}{b'}\right\}}{n}\left(\omega L^2 + \frac{(1+\omega)\cL^2}{b'}\right)}},
	\end{equation*}
	then \algname{VR-MARINA} requires 
	\begin{eqnarray*}
		K &=& \cO\left(\frac{\Delta_0}{\varepsilon^2}\left(L + \sqrt{\frac{1-p}{pn}\left(\omega L^2 + \frac{(1+\omega)\cL^2}{b'}\right)}\right)\right)\\
		&=& \cO\left(\frac{\Delta_0}{\varepsilon^2}\left(L + \sqrt{L^2\frac{\omega\max\left\{\nicefrac{d}{\zeta_{\cQ}} - 1,\nicefrac{m}{b'}\right\}}{n} + \cL^2\frac{(1+\omega)\max\left\{\nicefrac{d}{\zeta_{\cQ}} - 1,\nicefrac{m}{b'}\right\}}{nb'}}\right)\right)\\
		&=& \cO\left(\frac{\Delta_0}{\varepsilon^2}\left(L\left(1 + \sqrt{\frac{\omega\max\left\{\nicefrac{d}{\zeta_{\cQ}} - 1,\nicefrac{m}{b'}\right\}}{n}}\right) + \cL\sqrt{\frac{(1+\omega)\max\left\{\nicefrac{d}{\zeta_{\cQ}} - 1,\nicefrac{m}{b'}\right\}}{nb'}}\right)\right)
	\end{eqnarray*}
	iterations/communication rounds and 
	\begin{eqnarray*}
		m + K(pm + 2(1-p)b') &=& \cO\left(m + \frac{\Delta_0}{\varepsilon^2}\left(L + \sqrt{\frac{1-p}{pn}\left(\omega L^2 + \frac{(1+\omega)\cL^2}{b'}\right)}\right)(pm + (1-p)b')\right)\\	
		&=& \cO\Bigg(m+\frac{\Delta_0}{\varepsilon^2}\Bigg(L\Bigg(1 + \sqrt{\frac{\omega\max\left\{\nicefrac{d}{\zeta_{\cQ}} - 1,\nicefrac{m}{b'}\right\}}{n}}\Bigg)\\
		&&\hspace{6cm} + \cL\sqrt{\frac{(1+\omega)\max\left\{\nicefrac{d}{\zeta_{\cQ}} - 1,\nicefrac{m}{b'}\right\}}{nb'}}\Bigg)b'\Bigg)\\
		&=& \cO\Bigg(m+\frac{\Delta_0}{\varepsilon^2}\Bigg(L\Bigg(b' + \sqrt{\frac{\omega\max\left\{(\nicefrac{d}{\zeta_{\cQ}} - 1)(b')^2, mb'\right\}}{n}}\Bigg)\\
		&&\hspace{6cm} + \cL\sqrt{\frac{(1+\omega)\max\left\{(\nicefrac{d}{\zeta_{\cQ}} - 1)b',m\right\}}{n}}\Bigg)\Bigg)
	\end{eqnarray*}
	stochastic oracle calls per node in expectation in order to achieve $\EE[\|\nabla f(\hat x^K)\|^2] \le \varepsilon^2$, and the expected total communication cost per worker is
	\begin{eqnarray*}
		d + K(pd + (1-p)\zeta_{\cQ}) &=&  \cO\left(d+\frac{\Delta_0}{\varepsilon^2}\left(L + \sqrt{\frac{1-p}{pn}\left(\omega L^2 + \frac{(1+\omega)\cL^2}{b'}\right)}\right)(pd + (1-p)\zeta_{\cQ})\right)\\
		&=&\cO\Bigg(d+\frac{\Delta_0\zeta_{\cQ}}{\varepsilon^2}\Bigg(L\Bigg(1 + \sqrt{\frac{\omega\max\left\{\nicefrac{d}{\zeta_{\cQ}} - 1,\nicefrac{m}{b'}\right\}}{n}}\Bigg) \\
		&&\hspace{6cm} + \cL\sqrt{\frac{(1+\omega)\max\left\{\nicefrac{d}{\zeta_{\cQ}} - 1,\nicefrac{m}{b'}\right\}}{nb'}}\Bigg)\Bigg)
	\end{eqnarray*}
	 under an assumption that the communication cost is proportional to the number of non-zero components of transmitted vectors from workers to the server.
\end{proof}

\subsubsection{Convergence Results Under Polyak-{\L}ojasiewicz condition}\label{sec:proof_of_thm_pl_fin_sums}
In this section, we provide an analysis of \algname{VR-MARINA} under the Polyak-{\L}ojasiewicz condition in the finite sum case.
\begin{theorem}\label{thm:main_result_pl_finite_sums_appendix}
	Consider the finite sum case \eqref{eq:main_problem}+\eqref{eq:f_i_finite_sum}. Let Assumptions~\ref{as:lower_bound},~\ref{as:L_smoothness},~\ref{as:avg_smoothness}~and~\ref{as:pl_condition} be satisfied and 
	\begin{equation}
		\gamma \le \min\left\{\frac{1}{L + \sqrt{\frac{2(1-p)}{pn}\left(\omega L^2 + \frac{(1+\omega)\cL^2}{b'}\right)}},\frac{p}{2\mu}\right\},\label{eq:gamma_bound_pl_finite_sums_appendix}
	\end{equation}
	where $L^2 = \frac{1}{n}\sum_{i=1}^nL_i^2$ and $\cL^2 = \frac{1}{n}\sum_{i=1}^n\cL_i^2$. Then after $K$ iterations of \algname{VR-MARINA}, we have
	\begin{equation}
		\EE\left[f(x^K) - f(x^*)\right] \le (1-\gamma\mu)^K\Delta_0, \label{eq:main_res_pl_finite_sums_appendix}
	\end{equation}
	where $\Delta_0 = f(x^0)-f(x^*)$. That is, after
	\begin{equation}
		K = \cO\left(\max\left\{\frac{1}{p}, \frac{L + \sqrt{\frac{1-p}{pn}\left(\omega L^2 + \frac{(1+\omega)\cL^2}{b'}\right)}}{\mu}\right\}\log\frac{\Delta_0}{\varepsilon}\right) \label{eq:main_res_2_pl_finite_sums_appendix}
	\end{equation}
	iterations \algname{VR-MARINA} produces such a point $x^K$ that $\EE\left[f(x^K) - f(x^*)\right] \le \varepsilon$, and the expected total number of stochastic oracle calls per node equals
	\begin{equation}
		m + K(pm + 2(1-p)b') = \cO\left(m + \max\left\{\frac{1}{p}, \frac{L + \sqrt{\frac{1-p}{pn}\left(\omega L^2 + \frac{(1+\omega)\cL^2}{b'}\right)}}{\mu}\right\}(pm + (1-p)b')\log\frac{\Delta_0}{\varepsilon}\right). \label{eq:main_res_3_pl_finite_sums_appendix}
	\end{equation}
	Moreover, under an assumption that the communication cost is proportional to the number of non-zero components of transmitted vectors from workers to the server we have that the expected total communication cost per worker equals
	\begin{equation}
		d + K(pd + (1-p)\zeta_{\cQ}) =  \cO\left(d+\max\left\{\frac{1}{p}, \frac{L + \sqrt{\frac{1-p}{pn}\left(\omega L^2 + \frac{(1+\omega)\cL^2}{b'}\right)}}{\mu}\right\}(pd + (1-p)\zeta_{\cQ})\log\frac{\Delta_0}{\varepsilon}\right),\label{eq:main_res_4_pl_finite_sums_appendix}
	\end{equation}
	where $\zeta_{\cQ}$ is the expected density of the quantization (see Def.~\ref{def:quantization}).
\end{theorem}
\begin{proof}
	The proof is very similar to the proof of Theorem~\ref{thm:main_result_non_cvx_finite_sums}. From Lemma~\ref{lem:lemma_2_page} and P{\L} condition, we have
	\begin{eqnarray}
		\EE[f(x^{k+1}) - f(x^*)] &\le& \EE[f(x^k) - f(x^*)] - \frac{\gamma}{2}\EE\left[\|\nabla f(x^k)\|^2\right] - \left(\frac{1}{2\gamma} - \frac{L}{2}\right)\EE\left[\|x^{k+1}-x^k\|^2\right]\notag\\
		&&\quad + \frac{\gamma}{2}\EE\left[\|g^k - \nabla f(x^k)\|^2\right]\notag\\
		&\overset{\eqref{eq:pl_condition}}{\le}& (1-\gamma\mu)\EE\left[f(x^k) - f(x^*)\right] - \left(\frac{1}{2\gamma} - \frac{L}{2}\right)\EE\left[\|x^{k+1}-x^k\|^2\right] + \frac{\gamma}{2}\EE\left[\|g^k - \nabla f(x^k)\|^2\right]. \notag
	\end{eqnarray}
	Using the same arguments as in the proof of \eqref{eq:non_cvx_finite_sums_technical_2}, we obtain
	\begin{eqnarray}
		\EE\left[\|g^{k+1}-\nabla f(x^{k+1})\|^2\right] &\le&\frac{1-p}{n}\left(\omega L^2 + \frac{(1+\omega)\cL^2}{b'}\right)\EE\left[\|x^{k+1}-x^k\|^2\right] + (1-p)\EE\left[\left\|g^k - \nabla f(x^k)\right\|^2\right].\notag
	\end{eqnarray}
	Putting all together we derive that the sequence $\Phi_k = f(x^k) - f(x^*) + \frac{\gamma}{p}\|g^k - \nabla f(x^k)\|^2$ satisfies
	\begin{eqnarray}
		\EE\left[\Phi_{k+1}\right] &\le& \EE\left[(1-\gamma\mu)(f(x^k) - f(x^*)) - \left(\frac{1}{2\gamma} - \frac{L}{2}\right)\|x^{k+1}-x^k\|^2 + \frac{\gamma}{2}\|g^k - \nabla f(x^k)\|^2\right]\notag\\
		&&\quad + \frac{\gamma}{p}\EE\left[\frac{1-p}{n}\left(\omega L^2 + \frac{(1+\omega)\cL^2}{b'}\right)\|x^{k+1}-x^k\|^2 + (1-p)\left\|g^k - \nabla f(x^k)\right\|^2\right] \notag\\
		&=& \EE\left[(1-\gamma\mu)(f(x^k) - f(x^*)) + \left(\frac{\gamma}{2} + \frac{\gamma}{p}(1-p)\right)\left\|g^k - \nabla f(x^k)\right\|^2\right]\notag\\
		&&\quad + \left(\frac{\gamma(1-p)}{pn}\left(\omega L^2 + \frac{(1+\omega)\cL^2}{b'}\right) - \frac{1}{2\gamma} + \frac{L}{2}\right)\EE\left[\|x^{k+1}-x^k\|^2\right]\notag\\
		&\overset{\eqref{eq:gamma_bound_pl_finite_sums_appendix}}{\le}& (1-\gamma\mu)\EE[\Phi_k], \notag
	\end{eqnarray}
	where in the last inequality we use $\frac{\gamma(1-p)}{pn}\left(\omega L^2 + \frac{(1+\omega)\cL^2}{b'}\right) - \frac{1}{2\gamma} + \frac{L}{2} \le 0$ and $\frac{\gamma}{2} + \frac{\gamma}{p}(1-p) \le (1-\gamma\mu)\frac{\gamma}{p}$ following from \eqref{eq:gamma_bound_pl_finite_sums_appendix}. Unrolling the recurrence and using $g^0 = \nabla f(x^0)$, we obtain 
	\begin{equation*}
		\EE\left[f(x^{k+1}) - f(x^*)\right] \le \EE[\Phi_{k+1}] \le (1-\gamma\mu)^{k+1}\Phi_0 = (1-\gamma\mu)^{k+1}(f(x^0) - f(x^*))
	\end{equation*}
	that implies \eqref{eq:main_res_2_pl_finite_sums_appendix}, \eqref{eq:main_res_3_pl_finite_sums_appendix}, and \eqref{eq:main_res_4_pl_finite_sums_appendix}.
\end{proof}

\begin{corollary}\label{cor:main_result_pl_finite_sums_appendix}
	Let the assumptions of Theorem~\ref{thm:main_result_pl_finite_sums_appendix} hold and $p = \min\left\{\frac{\zeta_{\cQ}}{d},\frac{b'}{m+b'}\right\}$, where $b' \le m$ and $\zeta_{\cQ}$ is the expected density of the quantization (see Def.~\ref{def:quantization}). If 
	\begin{equation*}
		\gamma \le \min\left\{\frac{1}{L + \sqrt{\frac{2\max\left\{\nicefrac{d}{\zeta_{\cQ}} - 1,\nicefrac{m}{b'}\right\}}{n}\left(\omega L^2 + \frac{(1+\omega)\cL^2}{b'}\right)}},\frac{p}{2\mu}\right\},
	\end{equation*}
	then \algname{VR-MARINA} requires 
	\begin{equation*}
		\cO\left(\max\left\{\frac{1}{p}, \frac{L}{\mu}\left(1 + \sqrt{\frac{\omega\max\left\{\nicefrac{d}{\zeta_{\cQ}} - 1,\nicefrac{m}{b'}\right\}}{n}}\right) + \frac{\cL}{\mu}\sqrt{\frac{(1+\omega)\max\left\{\nicefrac{d}{\zeta_{\cQ}} - 1,\nicefrac{m}{b'}\right\}}{nb'}}\right\}\log\frac{\Delta_0}{\varepsilon}\right)
	\end{equation*}
	iterations/communication rounds, 
	\begin{equation*}
		\cO\left(m+\max\left\{\frac{b'}{p}, \frac{L}{\mu}\left(b' + \sqrt{\frac{\omega\max\left\{(\nicefrac{d}{\zeta_{\cQ}} - 1)(b')^2,mb'\right\}}{n}}\right) + \frac{\cL}{\mu}\sqrt{\frac{(1+\omega)\max\left\{(\nicefrac{d}{\zeta_{\cQ}} - 1) b',m\right\}}{n}}\right\}\log\frac{\Delta_0}{\varepsilon}\right)
	\end{equation*}
	stochastic oracle calls per node in expectation to achieve $\EE[f(x^K)-f(x^*)] \le \varepsilon$, and the expected total communication cost per worker is
	\begin{equation*}
		\cO\left(d+\zeta_{\cQ}\max\left\{\frac{1}{p}, \frac{L}{\mu}\left(1 + \sqrt{\frac{\omega\max\left\{\nicefrac{d}{\zeta_{\cQ}} - 1,\nicefrac{m}{b'}\right\}}{n}}\right) + \frac{\cL}{\mu}\sqrt{\frac{(1+\omega)\max\left\{\nicefrac{d}{\zeta_{\cQ}} - 1,\nicefrac{m}{b'}\right\}}{nb'}}\right\}\log\frac{\Delta_0}{\varepsilon}\right)
	\end{equation*}
	 under an assumption that the communication cost is proportional to the number of non-zero components of transmitted vectors from workers to the server.
\end{corollary}
\begin{proof}
	The choice of $p = \min\left\{\frac{\zeta_{\cQ}}{d},\frac{b'}{m+b'}\right\}$ implies
	\begin{eqnarray*}
		\frac{1-p}{p} &=& \max\left\{\frac{d}{\zeta_{\cQ}}-1,\frac{m}{b'}\right\},\\
		pm + (1-p)b' &\le& \frac{2mb'}{m+b'} \le 2b',\\
		pd + (1-p)\zeta_{\cQ} &\le& \frac{\zeta_{\cQ}}{d}\cdot d + \left(1 - \frac{\zeta_{\cQ}}{d}\right)\cdot\zeta_{\cQ} \le 2\zeta_{\cQ}.
	\end{eqnarray*}	
	Plugging these relations in \eqref{eq:gamma_bound_pl_finite_sums_appendix}, \eqref{eq:main_res_2_pl_finite_sums_appendix}, \eqref{eq:main_res_3_pl_finite_sums_appendix} and \eqref{eq:main_res_4_pl_finite_sums_appendix} and using $\sqrt{a+b} \le \sqrt{a} + \sqrt{b}$, we get that if
	\begin{equation*}
		\gamma \le \min\left\{\frac{1}{L + \sqrt{\frac{2\max\left\{\nicefrac{d}{\zeta_{\cQ}} - 1,\nicefrac{m}{b'}\right\}}{n}\left(\omega L^2 + \frac{(1+\omega)\cL^2}{b'}\right)}},\frac{p}{2\mu}\right\},
	\end{equation*}
	then \algname{VR-MARINA} requires 
	\begin{eqnarray*}
		K &=& \cO\left(\max\left\{\frac{1}{p}, \frac{L + \sqrt{\frac{1-p}{pn}\left(\omega L^2 + \frac{(1+\omega)\cL^2}{b'}\right)}}{\mu}\right\}\log\frac{\Delta_0}{\varepsilon}\right)\\
		&=& \cO\left(\max\left\{\frac{1}{p}, \frac{L + \sqrt{L^2\frac{\omega\max\left\{\nicefrac{d}{\zeta_{\cQ}} - 1,\nicefrac{m}{b'}\right\}}{n} + \cL^2\frac{(1+\omega)\max\left\{\nicefrac{d}{\zeta_{\cQ}} - 1,\nicefrac{m}{b'}\right\}}{nb'}}}{\mu}\right\}\log\frac{\Delta_0}{\varepsilon}\right)\\
		&=& \cO\left(\max\left\{\frac{1}{p}, \frac{L}{\mu}\left(1 + \sqrt{\frac{\omega\max\left\{\nicefrac{d}{\zeta_{\cQ}} - 1,\nicefrac{m}{b'}\right\}}{n}}\right) + \frac{\cL}{\mu}\sqrt{\frac{(1+\omega)\max\left\{\nicefrac{d}{\zeta_{\cQ}} - 1,\nicefrac{m}{b'}\right\}}{nb'}}\right\}\log\frac{\Delta_0}{\varepsilon}\right)
	\end{eqnarray*}
	iterations/communication rounds and 
	\begin{eqnarray*}
		m + K(pm + 2(1-p)b') &=& \cO\left(m + \max\left\{\frac{1}{p}, \frac{L + \sqrt{\frac{1-p}{pn}\left(\omega L^2 + \frac{(1+\omega)\cL^2}{b'}\right)}}{\mu}\right\}(pm + (1-p)b')\log\frac{\Delta_0}{\varepsilon}\right)\\	
		&=& \cO\Bigg(m+\max\Bigg\{\frac{1}{p}, \frac{L}{\mu}\Bigg(1 + \sqrt{\frac{\omega\max\left\{\nicefrac{d}{\zeta_{\cQ}} - 1,\nicefrac{m}{b'}\right\}}{n}}\Bigg)\\
		&&\hspace{5cm} + \frac{\cL}{\mu}\sqrt{\frac{(1+\omega)\max\left\{\nicefrac{d}{\zeta_{\cQ}} - 1,\nicefrac{m}{b'}\right\}}{nb'}}\Bigg\}b'\log\frac{\Delta_0}{\varepsilon}\Bigg)\\
		&=& \cO\Bigg(m+\max\Bigg\{\frac{b'}{p}, \frac{L}{\mu}\Bigg(b' + \sqrt{\frac{\omega\max\left\{(\nicefrac{d}{\zeta_{\cQ}} - 1)(b')^2, mb'\right\}}{n}}\Bigg)\\
		&&\hspace{5cm} + \frac{\cL}{\mu}\sqrt{\frac{(1+\omega)\max\left\{(\nicefrac{d}{\zeta_{\cQ}} - 1)b',m\right\}}{n}}\Bigg\}\log\frac{\Delta_0}{\varepsilon}\Bigg)
	\end{eqnarray*}
	stochastic oracle calls per node in expectation in order to achieve $\EE[f(x^K) - f(x^*)] \le \varepsilon$, and the expected total communication cost per worker is
	\begin{eqnarray*}
		d + K(pd + (1-p)\zeta_{\cQ}) &=&  \cO\left(d+\max\left\{\frac{1}{p}, \frac{L + \sqrt{\frac{1-p}{pn}\left(\omega L^2 + \frac{(1+\omega)\cL^2}{b'}\right)}}{\mu}\right\}(pd + (1-p)\zeta_{\cQ})\log\frac{\Delta_0}{\varepsilon}\right)\\
		&=&\cO\Bigg(d+\zeta_{\cQ}\max\Bigg\{\frac{1}{p}, \frac{L}{\mu}\Bigg(1 + \sqrt{\frac{\omega\max\left\{\nicefrac{d}{\zeta_{\cQ}} - 1,\nicefrac{m}{b'}\right\}}{n}}\Bigg) \\
		&&\hspace{5cm} + \frac{\cL}{\mu}\sqrt{\frac{(1+\omega)\max\left\{\nicefrac{d}{\zeta_{\cQ}} - 1,\nicefrac{m}{b'}\right\}}{nb'}}\Bigg\}\log\frac{\Delta_0}{\varepsilon}\Bigg)
	\end{eqnarray*}
	 under an assumption that the communication cost is proportional to the number of non-zero components of transmitted vectors from workers to the server.
\end{proof}

\subsection{Online Case}

\begin{algorithm}[h]
   \caption{\algname{VR-MARINA}: online case}\label{alg:vr_marina_online}
\begin{algorithmic}[1]
   \STATE {\bfseries Input:} starting point $x^0$, stepsize $\gamma$, minibatch sizes $b$, $b' < b$, probability $p\in(0,1]$, number of iterations $K$
   \STATE Initialize $g^0 = \frac{1}{nb}\sum_{i=1}^n\sum_{j\in I_{i,0}}\nabla f_{\xi_{ij}^0}(x^{0})$, where $I_{i,0}$ is the set of the indices in the minibatch, $|I_{i,0}| = b$, and $\xi_{ij}^0$ is independently sampled from $\cD_i$ for $i\in[n]$, $j\in[m]$
   \FOR{$k=0,1,\ldots,K-1$}
   \STATE Sample $c_k \sim \text{Be}(p)$
   \STATE Broadcast $g^k$ to all workers
   \FOR{$i = 1,\ldots,n$ in parallel} 
   \STATE $x^{k+1} = x^k - \gamma g^k$
   \STATE Set $g_{i}^{k+1} = \begin{cases}\frac{1}{b}\sum_{j\in I_{i,k}}\nabla f_{\xi_{ij}^k}(x^{k+1}),& \text{if } c_k = 1,\\ g^k + \cQ\left(\frac{1}{b'}\sum_{j\in I_{i,k}'}(\nabla f_{\xi_{ij}^k}(x^{k+1}) - \nabla f_{\xi_{ij}^k}(x^k))\right),& \text{if } c_k = 0, \end{cases}$  
    where $I_{i,k}, I_{i,k}'$ are the sets of the indices in the minibatches, $|I_{i,k}| = b$, $|I_{i,k}'| = b'$,  and $\xi_{ij}^k$ is independently sampled from $\cD_i$ for $i\in[n]$, $j\in[m]$
   \ENDFOR
   \STATE $g^{k+1} = \frac{1}{n}\sum_{i=1}^ng_i^{k+1}$
   \ENDFOR
   \STATE {\bfseries Return:} $\hat x^K$ chosen uniformly at random from $\{x^k\}_{k=0}^{K-1}$
\end{algorithmic}
\end{algorithm}

\subsubsection{Generally Non-Convex Problems}\label{sec:proof_of_thm_non_cvx_online}

In this section, we provide the full statement of Theorem~\ref{thm:main_result_non_cvx_online} together with the proof of this result.
\begin{theorem}[Theorem~\ref{thm:main_result_non_cvx_online}]\label{thm:main_result_non_cvx_online_appendix}
	Consider the online case \eqref{eq:main_problem}+\eqref{eq:f_i_expectation}. Let Assumptions~\ref{as:lower_bound},~\ref{as:L_smoothness},~\ref{as:avg_smoothness_online}~and~\ref{as:bounded_var} be satisfied and 
	\begin{equation}
		\gamma \le \frac{1}{L + \sqrt{\frac{1-p}{pn}\left(\omega L^2 + \frac{(1+\omega)\cL^2}{b'}\right)}},\label{eq:gamma_bound_non_cvx_online_appendix}
	\end{equation}
	where $L^2 = \frac{1}{n}\sum_{i=1}^nL_i^2$ and $\cL^2 = \frac{1}{n}\sum_{i=1}^n\cL_i^2$. Then after $K$ iterations of \algname{VR-MARINA}, we have
	\begin{equation}
		\EE\left[\left\|\nabla f(\hat x^K)\right\|^2\right] \le \frac{2\Delta_0}{\gamma K} + \frac{\sigma^2}{nb}\left(1 + \frac{1}{p K}\right), \label{eq:main_res_non_cvx_online_appendix}
	\end{equation}
	where $\hat{x}^K$ is chosen uniformly at random from $x^0,\ldots,x^{K-1}$ and $\Delta_0 = f(x^0)-f_*$. That is, after
	\begin{equation}
		K = \cO\left(\frac{1}{p} + \frac{\Delta_0}{\varepsilon^2}\left(L + \sqrt{\frac{1-p}{pn}\left(\omega L^2 + \frac{(1+\omega)\cL^2}{b'}\right)}\right)\right) \label{eq:main_res_2_non_cvx_online_appendix}
	\end{equation}
	iterations with $b = \Theta(\frac{\sigma^2}{n\varepsilon^2})$ \algname{VR-MARINA} produces such a point $\hat x^K$ that $\EE[\|\nabla f(\hat x^K)\|^2] \le \varepsilon^2$, and the expected total number of stochastic oracle calls per node equals
	\begin{equation}
		b + K(pb + 2(1-p)b') = \cO\left(\frac{\sigma^2}{n\varepsilon^2} + \left(\frac{1}{p} + \frac{\Delta_0}{\varepsilon^2}\left(L + \sqrt{\frac{1-p}{pn}\left(\omega L^2 + \frac{(1+\omega)\cL^2}{b'}\right)}\right)\right)\left(p\frac{\sigma^2}{n\varepsilon^2} + (1-p)b'\right)\right). \label{eq:main_res_3_non_cvx_online_appendix}
	\end{equation}
	Moreover, under an assumption that the communication cost is proportional to the number of non-zero components of transmitted vectors from workers to the server we have that the expected total communication cost per worker equals
	\begin{equation}
		d + K(pd + (1-p)\zeta_{\cQ}) =  \cO\left(d + \left(\frac{1}{p} + \frac{\Delta_0}{\varepsilon^2}\left(L + \sqrt{\frac{1-p}{pn}\left(\omega L^2 + \frac{(1+\omega)\cL^2}{b'}\right)}\right)\right)(pd + (1-p)\zeta_{\cQ})\right),\label{eq:main_res_4_non_cvx_online_appendix}
	\end{equation}
	where $\zeta_{\cQ}$ is the expected density of the quantization (see Def.~\ref{def:quantization}).
\end{theorem}
\begin{proof}[Proof of Theorem~\ref{thm:main_result_non_cvx_online}]
	The proof follows the same steps as the proof of Theorem~\ref{thm:main_result_non_cvx_finite_sums}. From Lemma~\ref{lem:lemma_2_page}, we have
	\begin{equation}
		\EE[f(x^{k+1})] \le \EE[f(x^k)] - \frac{\gamma}{2}\EE\left[\|\nabla f(x^k)\|^2\right] - \left(\frac{1}{2\gamma} - \frac{L}{2}\right)\EE\left[\|x^{k+1}-x^k\|^2\right] + \frac{\gamma}{2}\EE\left[\|g^k - \nabla f(x^k)\|^2\right]. \label{eq:non_cvx_online_technical_1}
	\end{equation}
	Next, we need to derive an upper bound for $\EE\left[\|g^{k+1}-\nabla f(x^{k+1})\|^2\right]$. Since $g^{k+1} = \frac{1}{n}\sum\limits_{i=1}^ng_i^{k+1}$, we get the following representation of $g^{k+1}$:
	\begin{equation}
		g^{k+1} = \begin{cases}\frac{1}{nb}\sum\limits_{i=1}^n\sum\limits_{j\in I_{i,k}}\nabla f_{\xi_{ij}^k}(x^{k+1})& \text{with probability } p,\\ g^k + \frac{1}{n}\sum\limits_{i=1}^n \cQ\left(\frac{1}{b'}\sum\limits_{j\in I'_{i,k}}(\nabla f_{\xi_{ij}^k}(x^{k+1}) - \nabla f_{\xi_{ij}^k}(x^k))\right)& \text{with probability } 1-p. \end{cases}\notag
	\end{equation}
	Using this, variance decomposition \eqref{eq:variance_decomposition}, tower property \eqref{eq:tower_property}, and independence of $\xi_{ij}^k$ for $i\in[n]$, $j\in I_{i,k}$, we derive:
	\begin{eqnarray}
		\EE\left[\|g^{k+1}-\nabla f(x^{k+1})\|^2\right] &\overset{\eqref{eq:tower_property}}{=}& (1-p)\EE\left[\left\|g^k + \frac{1}{n}\sum\limits_{i=1}^n \cQ\left(\frac{1}{b'}\sum\limits_{j\in I'_{i,k}}(\nabla f_{ij}(x^{k+1}) - \nabla f_{ij}(x^k))\right) - \nabla f(x^{k+1})\right\|^2\right]\notag\\
		&&\quad + \frac{p}{n^2b^2}\EE\left[\left\|\sum\limits_{i=1}^n\sum\limits_{j\in I_{i,k}}\left(\nabla f_{\xi_{ij}^k}(x^{k+1}) - \nabla f(x^{k+1})\right)\right\|^2\right] \notag\\
		&\overset{\eqref{eq:tower_property},\eqref{eq:variance_decomposition}}{=}& (1-p)\EE\left[\left\|\frac{1}{n}\sum\limits_{i=1}^n \cQ\left(\frac{1}{b'}\sum\limits_{j\in I'_{i,k}}(\nabla f_{ij}(x^{k+1}) - \nabla f_{ij}(x^k))\right) - \nabla f(x^{k+1}) + \nabla f(x^k)\right\|^2\right]\notag\\
		&&\quad + (1-p)\EE\left[\left\|g^k - \nabla f(x^k)\right\|^2\right] + \frac{p}{n^2b^2}\sum\limits_{i=1}^n\sum\limits_{j\in I_{i,k}}\EE\left[\left\|\nabla f_{\xi_{ij}^k}(x^{k+1}) - \nabla f(x^{k+1})\right\|^2\right]\notag\\
		&\overset{\eqref{eq:tower_property},\eqref{eq:bounded_var}}{=}& (1-p)\EE\left[\left\|\frac{1}{n}\sum\limits_{i=1}^n \cQ\left(\frac{1}{b'}\sum\limits_{j\in I'_{i,k}}(\nabla f_{ij}(x^{k+1}) - \nabla f_{ij}(x^k))\right) - \nabla f(x^{k+1}) + \nabla f(x^k)\right\|^2\right]\notag\\
		&&\quad + (1-p)\EE\left[\left\|g^k - \nabla f(x^k)\right\|^2\right] + \frac{p\sigma^2}{nb},\notag
	\end{eqnarray}
	where $\sigma^2 = \frac{1}{n}\sum_{i=1}^n\sigma_i^2$. Applying the same arguments as in the proof of inequality \eqref{eq:non_cvx_finite_sums_technical_2}, we obtain
	\begin{eqnarray}
		\EE\left[\|g^{k+1}-\nabla f(x^{k+1})\|^2\right]	&\le&\frac{1-p}{n}\left(\omega L^2 + \frac{(1+\omega)\cL^2}{b'}\right)\EE\left[\|x^{k+1}-x^k\|^2\right]\notag\\
		&&\quad + (1-p)\EE\left[\left\|g^k - \nabla f(x^k)\right\|^2\right] + \frac{p\sigma^2}{nb}.\label{eq:non_cvx_online_technical_2}
	\end{eqnarray}
	Next, we introduce new notation: $\Phi_k = f(x^k) - f_* + \frac{\gamma}{2p}\|g^k - \nabla f(x^k)\|^2$. Using this and inequalities \eqref{eq:non_cvx_online_technical_1} and \eqref{eq:non_cvx_online_technical_2}, we establish the following inequality:
	\begin{eqnarray}
		\EE\left[\Phi_{k+1}\right] &\le& \EE\left[f(x^k) - f_* - \frac{\gamma}{2}\|\nabla f(x^k)\|^2 - \left(\frac{1}{2\gamma} - \frac{L}{2}\right)\|x^{k+1}-x^k\|^2 + \frac{\gamma}{2}\|g^k - \nabla f(x^k)\|^2\right]\notag\\
		&&\quad + \frac{\gamma}{2p}\EE\left[\frac{1-p}{n}\left(\omega L^2 + \frac{(1+\omega)\cL^2}{b'}\right)\|x^{k+1}-x^k\|^2 + (1-p)\left\|g^k - \nabla f(x^k)\right\|^2 + \frac{p\sigma^2}{nb}\right] \notag\\
		&=& \EE\left[\Phi_k\right] - \frac{\gamma}{2}\EE\left[\|\nabla f(x^k)\|^2\right] + \left(\frac{\gamma(1-p)}{2pn}\left(\omega L^2 + \frac{(1+\omega)\cL^2}{b'}\right) - \frac{1}{2\gamma} + \frac{L}{2}\right)\EE\left[\|x^{k+1}-x^k\|^2\right] + \frac{\gamma\sigma^2}{2nb}\notag\\
		&\overset{\eqref{eq:gamma_bound_non_cvx_online_appendix}}{\le}& \EE\left[\Phi_k\right] - \frac{\gamma}{2}\EE\left[\|\nabla f(x^k)\|^2\right] + \frac{\gamma\sigma^2}{2nb},\label{eq:non_cvx_online_technical_3}
	\end{eqnarray}
	where in the last inequality, we use $\frac{\gamma(1-p)}{2pn}\left(\omega L^2 + \frac{(1+\omega)\cL^2}{b'}\right) - \frac{1}{2\gamma} + \frac{L}{2} \le 0$ following from \eqref{eq:gamma_bound_non_cvx_online_appendix}. Summing up inequalities \eqref{eq:non_cvx_online_technical_3} for $k=0,1,\ldots,K-1$ and rearranging the terms, we derive
	\begin{eqnarray}
		\frac{1}{K}\sum\limits_{k=0}^{K-1}\EE\left[\|\nabla f(x^k)\|^2\right] &\le& \frac{2}{\gamma K}\sum\limits_{k=0}^{K-1}\left(\EE[\Phi_k]-\EE[\Phi_{k+1}]\right) + \frac{\sigma^2}{nb} = \frac{2\left(\EE[\Phi_0]-\EE[\Phi_{K}]\right)}{\gamma K} + \frac{\sigma^2}{nb} \leq \frac{2\Delta_0}{\gamma K} + \frac{\sigma^2}{nb}\left(1 + \frac{1}{p K}\right),\notag
	\end{eqnarray}
	since $g^0 = \frac{1}{nb}\sum_{i=1}^n\sum_{j\in I_{i,0}}\nabla f_{\xi_{ij}^0}(x^{0})$ and $\Phi_{K} \ge 0$. Finally, using the tower property \eqref{eq:tower_property} and the definition of $\hat x^K$, we obtain \eqref{eq:main_res_non_cvx_online_appendix} that implies \eqref{eq:main_res_2_non_cvx_online_appendix}, \eqref{eq:main_res_3_non_cvx_online_appendix}, and \eqref{eq:main_res_4_non_cvx_online_appendix}.
\end{proof}

\begin{remark}[About batchsizes dissimilarity]
	Similarly to the finite sum case, our analysis can be easily extended to handle the version of \algname{VR-MARINA} with different batchsizes $b_1,\ldots,b_n$ and $b'_1,\ldots,b'_n$ on different workers, i.e., when $|I_{i,k}| = b_i$, $|I_{i,k}'| = b_i'$ for $i\in[n]$. In this case, the statement of Theorem~\ref{thm:main_result_non_cvx_online} remains the same with the small modificiation: instead of $\frac{\cL^2}{b'}$ the complexity bounds will have $\frac{1}{n}\sum_{i=1}^n\frac{\cL_i^2}{b_i'}$, and instead of the requirement $b = \Theta\left(\frac{\sigma^2}{n\varepsilon}\right)$ it will have $\frac{1}{n^2}\sum_{i=1}^n\frac{\sigma_i^2}{b_i} = \Theta(\varepsilon^2)$.
\end{remark}

\begin{corollary}[Corollary~\ref{cor:main_result_non_cvx_online}]\label{cor:main_result_non_cvx_online_appendix}
	Let the assumptions of Theorem~\ref{thm:main_result_non_cvx_online} hold and $p = \min\left\{\frac{\zeta_{\cQ}}{d},\frac{b'}{b+b'}\right\}$, where $b'\le b$, $b = \Theta\left(\nicefrac{\sigma^2}{(n\varepsilon^2)}\right)$ and $\zeta_{\cQ}$ is the expected density of the quantization (see Def.~\ref{def:quantization}). If 
	\begin{equation*}
		\gamma \le \frac{1}{L + \sqrt{\frac{\max\left\{\nicefrac{d}{\zeta_{\cQ}} - 1,\nicefrac{b}{b'}\right\}}{n}\left(\omega L^2 + \frac{(1+\omega)\cL^2}{b'}\right)}},
	\end{equation*}
	then \algname{VR-MARINA} requires 
	\begin{eqnarray*}
		\cO\left(\max\left\{\frac{d}{\zeta_{\cQ}}, \frac{\sigma^2}{nb'\varepsilon^2}\right\} + \frac{\Delta_0}{\varepsilon^2}\left(L\left(1 + \sqrt{\frac{\omega}{n}\max\left\{\frac{d}{\zeta_{\cQ}} - 1,\frac{\sigma^2}{nb'\varepsilon^2}\right\}}\right) + \cL\sqrt{\frac{(1+\omega)}{nb'}\max\left\{\frac{d}{\zeta_{\cQ}} - 1,\frac{\sigma^2}{nb'\varepsilon^2}\right\}}\right)\right)
	\end{eqnarray*}
	iterations/communication rounds and 
	\begin{eqnarray*}
		\cO\left(\max\left\{\frac{db'}{\zeta_{\cQ}}, \frac{\sigma^2}{n\varepsilon^2}\right\} +\frac{\Delta_0 L b'}{\varepsilon^2}+ \frac{\Delta_0 L}{\varepsilon^2}\sqrt{\frac{\omega b'}{n}\max\left\{\left(\frac{d}{\zeta_{\cQ}} - 1\right)b',\frac{\sigma^2}{n\varepsilon^2}\right\}}+ \frac{\Delta_0\cL}{\varepsilon^2}\sqrt{\frac{1+\omega}{n}\max\left\{\left(\frac{d}{\zeta_{\cQ}} - 1\right) b',\frac{\sigma^2}{n\varepsilon^2}\right\}}\right)
	\end{eqnarray*}
	stochastic oracle calls per node in expectation to achieve $\EE[\|\nabla f(\hat x^K)\|^2] \le \varepsilon^2$, and the expected total communication cost per worker is
	\begin{eqnarray*}
		\cO\left(\max\left\{d, \frac{\sigma^2\zeta_{\cQ}}{nb'\varepsilon^2}\right\} +\frac{\Delta_0\zeta_{\cQ}}{\varepsilon^2}\left(L\left(1 + \sqrt{\frac{\omega}{n}\max\left\{\frac{d}{\zeta_{\cQ}} - 1,\frac{\sigma^2}{nb'\varepsilon^2}\right\}}\right)+ \cL\sqrt{\frac{1+\omega}{nb'}\max\left\{\frac{d}{\zeta_{\cQ}} - 1,\frac{\sigma^2}{nb'\varepsilon^2}\right\}}\right)\right)
	\end{eqnarray*}
	 under an assumption that the communication cost is proportional to the number of non-zero components of transmitted vectors from workers to the server.
\end{corollary}
\begin{proof}[Proof of Corollary~\ref{cor:main_result_non_cvx_finite_sums}]
	The choice of $p = \min\left\{\frac{\zeta_{\cQ}}{d},\frac{b'}{b+b'}\right\}$ implies
	\begin{eqnarray*}
		\frac{1-p}{p} &=& \max\left\{\frac{d}{\zeta_{\cQ}}-1,\frac{b}{b'}\right\},\\
		pb + (1-p)b' &\le& \frac{2bb'}{b+b'} \le 2b',\\
		pd + (1-p)\zeta_{\cQ} &\le& \frac{\zeta_{\cQ}}{d}\cdot d + \left(1 - \frac{\zeta_{\cQ}}{d}\right)\cdot\zeta_{\cQ} \le 2\zeta_{\cQ}.
	\end{eqnarray*}	
	Plugging these relations in \eqref{eq:gamma_bound_non_cvx_online_appendix}, \eqref{eq:main_res_2_non_cvx_online_appendix}, \eqref{eq:main_res_3_non_cvx_online_appendix} and \eqref{eq:main_res_4_non_cvx_online_appendix} and using $\sqrt{a+b} \le \sqrt{a} + \sqrt{b}$, we get that if
	\begin{equation*}
		\gamma \le \frac{1}{L + \sqrt{\frac{\max\left\{\nicefrac{d}{\zeta_{\cQ}} - 1,\nicefrac{b}{b'}\right\}}{n}\left(\omega L^2 + \frac{(1+\omega)\cL^2}{b'}\right)}},
	\end{equation*}
	then \algname{VR-MARINA} requires 
	\begin{eqnarray*}
		K &=& \cO\left(\frac{1}{p}+ \frac{\Delta_0}{\varepsilon^2}\left(L + \sqrt{\frac{1-p}{pn}\left(\omega L^2 + \frac{(1+\omega)\cL^2}{b'}\right)}\right)\right)\\
		&=& \cO\left(\max\left\{\frac{d}{\zeta_{\cQ}}, \frac{\sigma^2}{nb'\varepsilon^2}\right\} + \frac{\Delta_0}{\varepsilon^2}\left(L + \sqrt{L^2\frac{\omega\max\left\{\nicefrac{d}{\zeta_{\cQ}} - 1,\nicefrac{b}{b'}\right\}}{n} + \cL^2\frac{(1+\omega)\max\left\{\nicefrac{d}{\zeta_{\cQ}} - 1,\nicefrac{b}{b'}\right\}}{nb'}}\right)\right)\\
		&=& \cO\Bigg(\max\left\{\frac{d}{\zeta_{\cQ}}, \frac{\sigma^2}{nb'\varepsilon^2}\right\} + \frac{\Delta_0}{\varepsilon^2}\Bigg(L\left(1 + \sqrt{\frac{\omega}{n}\max\left\{\frac{d}{\zeta_{\cQ}} - 1,\frac{\sigma^2}{nb'\varepsilon^2}\right\}}\right)\\
		&&\hspace{8cm} + \cL\sqrt{\frac{(1+\omega)}{nb'}\max\left\{\frac{d}{\zeta_{\cQ}} - 1,\frac{\sigma^2}{nb'\varepsilon^2}\right\}}\Bigg)\Bigg)
	\end{eqnarray*}
	iterations/communication rounds and 
	\begin{eqnarray*}
		b + K(pb + 2(1-p)b') &=& \cO\left(b + \frac{(1-p)b'}{p} + \frac{\Delta_0}{\varepsilon^2}\left(L + \sqrt{\frac{1-p}{pn}\left(\omega L^2 + \frac{(1+\omega)\cL^2}{b'}\right)}\right)(pb + (1-p)b')\right)\\	
		&=& \cO\Bigg(\max\left\{\frac{db'}{\zeta_{\cQ}}, b\right\} +\frac{\Delta_0}{\varepsilon^2}\Bigg(L\Bigg(1 + \sqrt{\frac{\omega\max\left\{\nicefrac{d}{\zeta_{\cQ}} - 1,\nicefrac{b}{b'}\right\}}{n}}\Bigg)\\
		&&\hspace{6cm} + \cL\sqrt{\frac{(1+\omega)\max\left\{\nicefrac{d}{\zeta_{\cQ}} - 1,\nicefrac{b}{b'}\right\}}{nb'}}\Bigg)b'\Bigg)\\
		&=& \cO\Bigg(\max\left\{\frac{db'}{\zeta_{\cQ}}, \frac{\sigma^2}{n\varepsilon^2}\right\} +\frac{\Delta_0}{\varepsilon^2}\Bigg(L\Bigg(b' + \sqrt{\frac{\omega b'}{n}\max\left\{\left(\frac{d}{\zeta_{\cQ}} - 1\right)b', \frac{\sigma^2}{n\varepsilon^2}\right\}}\Bigg) \\
		&&\hspace{6cm}+ \cL\sqrt{\frac{1+\omega}{n}\max\left\{\left(\frac{d}{\zeta_{\cQ}} - 1\right)b',\frac{\sigma^2}{n\varepsilon^2}\right\}}\Bigg)\Bigg)
	\end{eqnarray*}
	stochastic oracle calls per node in expectation to achieve $\EE[\|\nabla f(\hat x^K)\|^2] \le \varepsilon^2$, and the expected total communication cost per worker is
	\begin{eqnarray*}
		d + K(pd + (1-p)\zeta_{\cQ}) &=&  \cO\left(d + \frac{(1-p)\zeta_{\cQ}}{p} +\frac{\Delta_0}{\varepsilon^2}\left(L + \sqrt{\frac{1-p}{pn}\left(\omega L^2 + \frac{(1+\omega)\cL^2}{b'}\right)}\right)(pd + (1-p)\zeta_{\cQ})\right)\\
		&=&\cO\Bigg(\max\left\{d, \frac{\sigma^2\zeta_{\cQ}}{nb'\varepsilon^2}\right\} +\frac{\Delta_0\zeta_{\cQ}}{\varepsilon^2}\Bigg(L\Bigg(1 + \sqrt{\frac{\omega}{n}\max\left\{\frac{d}{\zeta_{\cQ}} - 1,\frac{\sigma^2}{nb'\varepsilon^2}\right\}}\Bigg)\\
		&&\hspace{6cm}+ \cL\sqrt{\frac{1+\omega}{nb'}\max\left\{\frac{d}{\zeta_{\cQ}} - 1,\frac{\sigma^2}{nb'\varepsilon^2}\right\}}\Bigg)\Bigg)
	\end{eqnarray*}
	 under an assumption that the communication cost is proportional to the number of non-zero components of transmitted vectors from workers to the server.
\end{proof}

\subsubsection{Convergence Results Under Polyak-{\L}ojasiewicz condition}\label{sec:proof_of_thm_pl_online}
In this section, we provide an analysis of \algname{VR-MARINA} under Polyak-{\L}ojasiewicz condition in the online case.
\begin{theorem}\label{thm:main_result_pl_online_appendix}
	Consider the finite sum case \eqref{eq:main_problem}+\eqref{eq:f_i_expectation}. Let Assumptions~\ref{as:lower_bound},~\ref{as:L_smoothness},~\ref{as:avg_smoothness_online},~\ref{as:pl_condition}~and~\ref{as:bounded_var} be satisfied and 
	\begin{equation}
		\gamma \le \min\left\{\frac{1}{L + \sqrt{\frac{2(1-p)}{pn}\left(\omega L^2 + \frac{(1+\omega)\cL^2}{b'}\right)}},\frac{p}{2\mu}\right\},\label{eq:gamma_bound_pl_online_appendix}
	\end{equation}
	where $L^2 = \frac{1}{n}\sum_{i=1}^nL_i^2$ and $\cL^2 = \frac{1}{n}\sum_{i=1}^n\cL_i^2$. Then after $K$ iterations of \algname{VR-MARINA}, we have
	\begin{equation}
		\EE\left[f(x^K) - f(x^*)\right] \le (1-\gamma\mu)^K\Delta_0 + \frac{\sigma^2}{nb\mu}, \label{eq:main_res_pl_online_appendix}
	\end{equation}
	where $\Delta_0 = f(x^0)-f(x^*)$. That is, after
	\begin{equation}
		K = \cO\left(\max\left\{\frac{1}{p}, \frac{L + \sqrt{\frac{1-p}{pn}\left(\omega L^2 + \frac{(1+\omega)\cL^2}{b'}\right)}}{\mu}\right\}\log\frac{\Delta_0}{\varepsilon}\right) \label{eq:main_res_2_pl_online_appendix}
	\end{equation}
	iterations with $b = \Theta\left(\frac{\sigma^2}{n\mu\varepsilon}\right)$ \algname{VR-MARINA} produces such a point $x^K$ that $\EE\left[f(x^K) - f(x^*)\right] \le \varepsilon$, and the expected total number of stochastic oracle calls per node equals
	\begin{equation}
		b + K(pb + 2(1-p)b') = \cO\left(m + \max\left\{\frac{1}{p}, \frac{L + \sqrt{\frac{1-p}{pn}\left(\omega L^2 + \frac{(1+\omega)\cL^2}{b'}\right)}}{\mu}\right\}(pb + (1-p)b')\log\frac{\Delta_0}{\varepsilon}\right). \label{eq:main_res_3_pl_online_appendix}
	\end{equation}
	Moreover, under an assumption that the communication cost is proportional to the number of non-zero components of transmitted vectors from workers to the server, we have that the expected total communication cost per worker equals
	\begin{equation}
		d + K(pd + (1-p)\zeta_{\cQ}) =  \cO\left(d+\max\left\{\frac{1}{p}, \frac{L + \sqrt{\frac{1-p}{pn}\left(\omega L^2 + \frac{(1+\omega)\cL^2}{b'}\right)}}{\mu}\right\}(pd + (1-p)\zeta_{\cQ})\log\frac{\Delta_0}{\varepsilon}\right),\label{eq:main_res_4_pl_online_appendix}
	\end{equation}
	where $\zeta_{\cQ}$ is the expected density of the quantization (see Def.~\ref{def:quantization}).
\end{theorem}
\begin{proof}
	The proof is very similar to the proof of Theorem~\ref{thm:main_result_non_cvx_online}. From Lemma~\ref{lem:lemma_2_page} and P{\L} condition, we have
	\begin{eqnarray}
		\EE[f(x^{k+1}) - f(x^*)] &\le& \EE[f(x^k) - f(x^*)] - \frac{\gamma}{2}\EE\left[\|\nabla f(x^k)\|^2\right] - \left(\frac{1}{2\gamma} - \frac{L}{2}\right)\EE\left[\|x^{k+1}-x^k\|^2\right]\notag\\
		&&\quad + \frac{\gamma}{2}\EE\left[\|g^k - \nabla f(x^k)\|^2\right]\notag\\
		&\overset{\eqref{eq:pl_condition}}{\le}& (1-\gamma\mu)\EE\left[f(x^k) - f(x^*)\right] - \left(\frac{1}{2\gamma} - \frac{L}{2}\right)\EE\left[\|x^{k+1}-x^k\|^2\right] + \frac{\gamma}{2}\EE\left[\|g^k - \nabla f(x^k)\|^2\right]. \notag
	\end{eqnarray}
	Using the same arguments as in the proof of \eqref{eq:non_cvx_online_technical_2}, we obtain
	\begin{eqnarray}
		\EE\left[\|g^{k+1}-\nabla f(x^{k+1})\|^2\right] &\le&\frac{1-p}{n}\left(\omega L^2 + \frac{(1+\omega)\cL^2}{b'}\right)\EE\left[\|x^{k+1}-x^k\|^2\right]\notag\\
		&&+ (1-p)\EE\left[\left\|g^k - \nabla f(x^k)\right\|^2\right] + \frac{p\sigma^2}{nb}.
	\end{eqnarray}
	Putting all together, we derive that the sequence $\Phi_k = f(x^k) - f(x^*) + \frac{\gamma}{p}\|g^k - \nabla f(x^k)\|^2$ satisfies
	\begin{eqnarray}
		\EE\left[\Phi_{k+1}\right] &\le& \EE\left[(1-\gamma\mu)(f(x^k) - f(x^*)) - \left(\frac{1}{2\gamma} - \frac{L}{2}\right)\|x^{k+1}-x^k\|^2 + \frac{\gamma}{2}\|g^k - \nabla f(x^k)\|^2\right]\notag\\
		&&\quad + \frac{\gamma}{p}\EE\left[\frac{1-p}{n}\left(\omega L^2 + \frac{(1+\omega)\cL^2}{b'}\right)\|x^{k+1}-x^k\|^2 + (1-p)\left\|g^k - \nabla f(x^k)\right\|^2 + \frac{p\sigma^2}{nb}\right] \notag\\
		&=& \EE\left[(1-\gamma\mu)(f(x^k) - f(x^*)) + \left(\frac{\gamma}{2} + \frac{\gamma}{p}(1-p)\right)\left\|g^k - \nabla f(x^k)\right\|^2\right] + \frac{\gamma\sigma^2}{nb}\notag\\
		&&\quad + \left(\frac{\gamma(1-p)}{pn}\left(\omega L^2 + \frac{(1+\omega)\cL^2}{b'}\right) - \frac{1}{2\gamma} + \frac{L}{2}\right)\EE\left[\|x^{k+1}-x^k\|^2\right]\notag\\
		&\overset{\eqref{eq:gamma_bound_pl_finite_sums_appendix}}{\le}& (1-\gamma\mu)\EE[\Phi_k] + \frac{\gamma\sigma^2}{nb},\notag
	\end{eqnarray}
	where in the last inequality we use $\frac{\gamma(1-p)}{pn}\left(\omega L^2 + \frac{(1+\omega)\cL^2}{b'}\right) - \frac{1}{2\gamma} + \frac{L}{2} \le 0$ and $\frac{\gamma}{2} + \frac{\gamma}{p}(1-p) \le (1-\gamma\mu)\frac{\gamma}{p}$ following from \eqref{eq:gamma_bound_pl_online_appendix}. Unrolling the recurrence and using $g^0 = \nabla f(x^0)$, we obtain 
	\begin{eqnarray*}
		\EE\left[f(x^{K}) - f(x^*)\right] \le \EE[\Phi_{K}] &\le& (1-\gamma\mu)^{K}\Phi_0 + \frac{\gamma\sigma^2}{nb}\sum\limits_{k=0}^{K-1}(1-\gamma\mu)^k \\
		&\le& (1-\gamma\mu)^{K}(f(x^0) - f(x^*)) + \frac{\gamma\sigma^2}{nb}\sum\limits_{k=0}^{\infty}(1-\gamma\mu)^k\\
		&\le& (1-\gamma\mu)^{K}(f(x^0) - f(x^*)) + \frac{\sigma^2}{nb\mu}.
	\end{eqnarray*}
	Together with $b = \Theta\left(\frac{\sigma^2}{n\mu\varepsilon}\right)$ it implies \eqref{eq:main_res_2_pl_online_appendix}, \eqref{eq:main_res_3_pl_online_appendix}, and \eqref{eq:main_res_4_pl_online_appendix}.
\end{proof}

\begin{corollary}\label{cor:main_result_pl_online_appendix}
	Let the assumptions of Theorem~\ref{thm:main_result_pl_online_appendix} hold and $p = \min\left\{\frac{\zeta_{\cQ}}{d},\frac{b'}{b+b'}\right\}$, where $b' \le b$ and $\zeta_{\cQ}$ is the expected density of the quantization (see Def.~\ref{def:quantization}). If 
	\begin{equation*}
		\gamma \le \min\left\{\frac{1}{L + \sqrt{\frac{2\max\left\{\nicefrac{d}{\zeta_{\cQ}} - 1,\nicefrac{b}{b'}\right\}}{n}\left(\omega L^2 + \frac{(1+\omega)\cL^2}{b'}\right)}},\frac{p}{2\mu}\right\}
	\end{equation*}
	and
	\begin{equation*}
		b = \Theta\left(\frac{\sigma^2}{n\mu\varepsilon}\right),\quad \sigma^2 = \frac{1}{n}\sum\limits_{i=1}^n\sigma_i^2,
	\end{equation*}
	then \algname{VR-MARINA} requires 
	\begin{equation*}
		\cO\left(\max\left\{\frac{d}{\zeta_{\cQ}}, \frac{\sigma^2}{nb'\mu\varepsilon}, \frac{L}{\mu}\left(1 + \sqrt{\frac{\omega}{n}\max\left\{\frac{d}{\zeta_{\cQ}} - 1,\frac{\sigma^2}{nb'\mu\varepsilon}\right\}}\right) + \frac{\cL}{\mu}\sqrt{\frac{1+\omega}{nb'}\max\left\{\frac{d}{\zeta_{\cQ}} - 1,\frac{\sigma^2}{nb'\mu\varepsilon}\right\}}\right\}\log\frac{\Delta_0}{\varepsilon}\right)
	\end{equation*}
	iterations/communication rounds, 
	\begin{eqnarray*}
		\cO\Bigg(\max\Bigg\{\frac{b'd}{\zeta_{\cQ}}, \frac{\sigma^2}{n\mu\varepsilon}, \frac{L}{\mu}\Bigg(b' + \sqrt{\frac{\omega b'}{n}\max\left\{\left(\frac{d}{\zeta_{\cQ}} - 1\right)b',\frac{\sigma^2}{n\mu\varepsilon}\right\}}\Bigg)&\\
		&\hspace{-2cm} + \frac{\cL}{\mu}\sqrt{\frac{1+\omega}{n}\max\left\{\left(\frac{d}{\zeta_{\cQ}} - 1\right)b',\frac{\sigma^2}{n\mu\varepsilon}\right\}}\Bigg\}\log\frac{\Delta_0}{\varepsilon}\Bigg)
	\end{eqnarray*}
	stochastic oracle calls per node in expectation to achieve $\EE[f(x^K)-f(x^*)] \le \varepsilon$, and the expected total communication cost per worker is
	\begin{equation*}
		\cO\left(\zeta_{\cQ}\max\left\{\frac{d}{\zeta_{\cQ}}, \frac{\sigma^2}{nb'\mu \varepsilon}, \frac{L}{\mu}\left(1 + \sqrt{\frac{\omega}{n}\max\left\{\frac{d}{\zeta_{\cQ}} - 1,\frac{\sigma^2}{nb'\mu}\right\}}\right) + \frac{\cL}{\mu}\sqrt{\frac{1+\omega}{nb'}\max\left\{\frac{d}{\zeta_{\cQ}} - 1,\frac{\sigma^2}{nb'\mu}\right\}}\right\}\log\frac{\Delta_0}{\varepsilon}\right)
	\end{equation*}
	 under an assumption that the communication cost is proportional to the number of non-zero components of transmitted vectors from workers to the server.
\end{corollary}
\begin{proof}
	The choice of $p = \min\left\{\frac{\zeta_{\cQ}}{d},\frac{b'}{b+b'}\right\}$ implies
	\begin{eqnarray*}
		\frac{1-p}{p} &=& \max\left\{\frac{d}{\zeta_{\cQ}}-1,\frac{b}{b'}\right\},\\
		pb + (1-p)b' &\le& \frac{2bb'}{b+b'} \le 2b',\\
		pd + (1-p)\zeta_{\cQ} &\le& \frac{\zeta_{\cQ}}{d}\cdot d + \left(1 - \frac{\zeta_{\cQ}}{d}\right)\cdot\zeta_{\cQ} \le 2\zeta_{\cQ}.
	\end{eqnarray*}	
	Plugging these relations in \eqref{eq:gamma_bound_pl_online_appendix}, \eqref{eq:main_res_2_pl_online_appendix}, \eqref{eq:main_res_3_pl_online_appendix} and \eqref{eq:main_res_4_pl_online_appendix} and using $\sqrt{a+b} \le \sqrt{a} + \sqrt{b}$, we get that if
	\begin{equation*}
		\gamma \le \min\left\{\frac{1}{L + \sqrt{\frac{2\max\left\{\nicefrac{d}{\zeta_{\cQ}} - 1,\nicefrac{b}{b'}\right\}}{n}\left(\omega L^2 + \frac{(1+\omega)\cL^2}{b'}\right)}},\frac{p}{2\mu}\right\},
	\end{equation*}
	then \algname{VR-MARINA} requires 
	\begin{eqnarray*}
		K &=& \cO\left(\max\left\{\frac{1}{p}, \frac{L + \sqrt{\frac{1-p}{pn}\left(\omega L^2 + \frac{(1+\omega)\cL^2}{b'}\right)}}{\mu}\right\}\log\frac{\Delta_0}{\varepsilon}\right)\\
		&=& \cO\left(\max\left\{\frac{d}{\zeta_{\cQ}}, \frac{b}{b'}, \frac{L + \sqrt{L^2\frac{\omega\max\left\{\nicefrac{d}{\zeta_{\cQ}} - 1,\nicefrac{b}{b'}\right\}}{n} + \cL^2\frac{(1+\omega)\max\left\{\nicefrac{d}{\zeta_{\cQ}} - 1,\nicefrac{b}{b'}\right\}}{nb'}}}{\mu}\right\}\log\frac{\Delta_0}{\varepsilon}\right)\\
		&=& \cO\left(\max\left\{\frac{d}{\zeta_{\cQ}}, \frac{\sigma^2}{nb'\mu\varepsilon}, \frac{L}{\mu}\left(1 + \sqrt{\frac{\omega}{n}\max\left\{\frac{d}{\zeta_{\cQ}} - 1,\frac{\sigma^2}{nb'\mu\varepsilon}\right\}}\right) + \frac{\cL}{\mu}\sqrt{\frac{1+\omega}{nb'}\max\left\{\frac{d}{\zeta_{\cQ}} - 1,\frac{\sigma^2}{nb'\mu\varepsilon}\right\}}\right\}\log\frac{\Delta_0}{\varepsilon}\right)
	\end{eqnarray*}
	iterations/communication rounds and 
	\begin{eqnarray*}
		b + K(pb + 2(1-p)b') &=& \cO\left(b + \max\left\{\frac{1}{p}, \frac{L + \sqrt{\frac{1-p}{pn}\left(\omega L^2 + \frac{(1+\omega)\cL^2}{b'}\right)}}{\mu}\right\}(pb + (1-p)b')\log\frac{\Delta_0}{\varepsilon}\right)\\	
		&=& \cO\Bigg(b+\max\Bigg\{\frac{d}{\zeta_{\cQ}}, \frac{b}{b'}, \frac{L}{\mu}\Bigg(1 + \sqrt{\frac{\omega\max\left\{\nicefrac{d}{\zeta_{\cQ}} - 1,\nicefrac{b}{b'}\right\}}{n}}\Bigg)\\
		&&\hspace{5cm} + \frac{\cL}{\mu}\sqrt{\frac{(1+\omega)\max\left\{\nicefrac{d}{\zeta_{\cQ}} - 1,\nicefrac{b}{b'}\right\}}{nb'}}\Bigg\}b'\log\frac{\Delta_0}{\varepsilon}\Bigg)\\
		&=& \cO\Bigg(\max\Bigg\{\frac{b'd}{\zeta_{\cQ}}, \frac{\sigma^2}{n\mu\varepsilon}, \frac{L}{\mu}\Bigg(b' + \sqrt{\frac{\omega b'}{n}\max\left\{\left(\frac{d}{\zeta_{\cQ}} - 1\right)b',\frac{\sigma^2}{n\mu\varepsilon}\right\}}\Bigg)\\
		&&\hspace{5cm} + \frac{\cL}{\mu}\sqrt{\frac{1+\omega}{n}\max\left\{\left(\frac{d}{\zeta_{\cQ}} - 1\right)b',\frac{\sigma^2}{n\mu\varepsilon}\right\}}\Bigg\}\log\frac{\Delta_0}{\varepsilon}\Bigg)
	\end{eqnarray*}
	stochastic oracle calls per node in expectation to achieve $\EE[f(x^K) - f(x^*)] \le \varepsilon$, and the expected total communication cost per worker is
	\begin{eqnarray*}
		d + K(pd + (1-p)\zeta_{\cQ}) &=&  \cO\left(d+\max\left\{\frac{1}{p}, \frac{L + \sqrt{\frac{1-p}{pn}\left(\omega L^2 + \frac{(1+\omega)\cL^2}{b'}\right)}}{\mu}\right\}(pd + (1-p)\zeta_{\cQ})\log\frac{\Delta_0}{\varepsilon}\right)\\
		&=&\cO\Bigg(\zeta_{\cQ}\max\Bigg\{\frac{d}{\zeta_{\cQ}}, \frac{\sigma^2}{nb'\mu \varepsilon}, \frac{L}{\mu}\left(1 + \sqrt{\frac{\omega}{n}\max\left\{\frac{d}{\zeta_{\cQ}} - 1,\frac{\sigma^2}{nb'\mu\varepsilon}\right\}}\right)\\
		&&\hspace{5cm} + \frac{\cL}{\mu}\sqrt{\frac{1+\omega}{nb'}\max\left\{\frac{d}{\zeta_{\cQ}} - 1,\frac{\sigma^2}{nb'\mu\varepsilon}\right\}}\Bigg\}\log\frac{\Delta_0}{\varepsilon}\Bigg)
	\end{eqnarray*}
	 under an assumption that the communication cost is proportional to the number of non-zero components of transmitted vectors from workers to the server.
\end{proof}

\clearpage

\section{Missing Proofs for \algname{PP-MARINA}}\label{sec:pp_marina_proofs}
\begin{algorithm}[h]
   \caption{\algname{PP-MARINA}}\label{alg:pp_marina}
\begin{algorithmic}[1]
   \STATE {\bfseries Input:} starting point $x^0$, stepsize $\gamma$, probability $p\in(0,1]$, number of iterations $K$, clients-batchsize $r \le n$
   \STATE Initialize $g^0 = \nabla f(x^0)$
   \FOR{$k=0,1,\ldots,K-1$}
   \STATE Sample $c_k\sim \text{Be}(p)$
   \STATE Choose $I_k' = \{1,\ldots,n\}$ if $c_k = 1$, and choose $I_k'$ as the set of $r$ i.i.d.\ samples from the uniform distribution over $\{1,\ldots,n\}$ otherwise
   \STATE Broadcast $g^k$ to all workers
   \FOR{$i = 1,\ldots,n$ in parallel} 
   \STATE $x^{k+1} = x^k - \gamma g^k$
   \STATE Set $g_i^{k+1} = \begin{cases}\nabla f_i(x^{k+1})& \text{if } c_k = 1,\\ g^k + \cQ\left(\nabla f_{i}(x^{k+1}) - \nabla f_{i}(x^k)\right)& \text{if } c_k = 0. \end{cases}$ 
   \ENDFOR
   \STATE Set $g^{k+1} = \begin{cases}\nabla f(x^{k+1})& \text{if } c_k=1,\\ g^k + \frac{1}{r}\sum\limits_{i_k\in I'_{k}}\cQ\left(\nabla f_{i_k}(x^{k+1}) - \nabla f_{i_k}(x^k)\right)& \text{if } c_k=0. \end{cases}$ 
   \ENDFOR
   \STATE {\bfseries Return:} $\hat x^K$ chosen uniformly at random from $\{x^k\}_{k=0}^{K-1}$
\end{algorithmic}
\end{algorithm}
\subsection{Generally Non-Convex Problems}\label{sec:proof_of_thm_non_cvx_pp}
In this section, we provide the full statement of Theorem~\ref{thm:main_result_non_cvx_pp} together with the proof of this result.
\begin{theorem}[Theorem~\ref{thm:main_result_non_cvx_pp}]\label{thm:main_result_non_cvx_pp_appendix}
	Let Assumptions~\ref{as:lower_bound}~and~\ref{as:L_smoothness} be satisfied and 
	\begin{equation}
		\gamma \le \frac{1}{L\left(1 + \sqrt{\frac{(1-p)(1+\omega)}{pr}}\right)},\label{eq:gamma_bound_non_cvx_pp_appendix}
	\end{equation}
	where $L^2 = \frac{1}{n}\sum_{i=1}^nL_i^2$. Then after $K$ iterations of \algname{PP-MARINA}, we have
	\begin{equation}
		\EE\left[\left\|\nabla f(\hat x^K)\right\|^2\right] \le \frac{2\Delta_0}{\gamma K}, \label{eq:main_res_non_cvx_pp_appendix}
	\end{equation}
	where $\hat{x}^K$ is chosen uniformly at random from $x^0,\ldots,x^{K-1}$ and $\Delta_0 = f(x^0)-f_*$. That is, after
	\begin{equation}
		K = \cO\left(\frac{\Delta_0 L}{\varepsilon^2}\left(1 + \sqrt{\frac{(1-p)(1+\omega)}{pr}}\right)\right) \label{eq:main_res_2_non_cvx_pp_appendix}
	\end{equation}
	iterations \algname{PP-MARINA} produces such a point $\hat x^K$ that $\EE[\|\nabla f(\hat x^K)\|^2] \le \varepsilon^2$.
	Moreover, under an assumption that the communication cost is proportional to the number of non-zero components of transmitted vectors from workers to the server, we have that the expected total communication cost (for all workers) equals
	\begin{equation}
		dn + K(pdn + (1-p)\zeta_{\cQ}r) =  \cO\left(dn+\frac{\Delta_0 L}{\varepsilon^2}\left(1 + \sqrt{\frac{(1-p)(1+\omega)}{pr}}\right)(pdn + (1-p)\zeta_{\cQ}r)\right),\label{eq:main_res_4_non_cvx_pp_appendix}
	\end{equation}
	where $\zeta_{\cQ}$ is the expected density of the quantization (see Def.~\ref{def:quantization}).
\end{theorem}
\begin{proof}[Proof of Theorem~\ref{thm:main_result_non_cvx_pp}]
	The proof is very similar to the proof of Theorem~\ref{thm:main_result_non_cvx_finite_sums}. From Lemma~\ref{lem:lemma_2_page}, we have
	\begin{equation}
		\EE[f(x^{k+1})] \le \EE[f(x^k)] - \frac{\gamma}{2}\EE\left[\|\nabla f(x^k)\|^2\right] - \left(\frac{1}{2\gamma} - \frac{L}{2}\right)\EE\left[\|x^{k+1}-x^k\|^2\right] + \frac{\gamma}{2}\EE\left[\|g^k - \nabla f(x^k)\|^2\right]. \label{eq:non_cvx_pp_technical_1}
	\end{equation}
	Next, we need to derive an upper bound for $\EE\left[\|g^{k+1}-\nabla f(x^{k+1})\|^2\right]$. By definition of $g^{k+1}$, we have
	\begin{equation}
		g^{k+1} = \begin{cases}\nabla f(x^{k+1})& \text{with probability } p,\\ g^k + \frac{1}{r}\sum\limits_{i_k\in I'_{k}}\cQ\left(\nabla f_{i_k}(x^{k+1}) - \nabla f_{i_k}(x^k)\right)& \text{with probability } 1-p. \end{cases}\notag
	\end{equation}
	Using this, variance decomposition \eqref{eq:variance_decomposition} and tower property \eqref{eq:tower_property}, we derive:
	\begin{eqnarray}
		\EE\left[\|g^{k+1}-\nabla f(x^{k+1})\|^2\right] &\overset{\eqref{eq:tower_property}}{=}& (1-p)\EE\left[\left\|g^k + \frac{1}{r}\sum\limits_{i_k\in I_k'} \cQ\left(\nabla f_{i_k}(x^{k+1}) - \nabla f_{i_k}(x^k)\right) - \nabla f(x^{k+1})\right\|^2\right]\notag\\
		&\overset{\eqref{eq:tower_property},\eqref{eq:variance_decomposition}}{=}& (1-p)\EE\left[\left\|\frac{1}{r}\sum\limits_{i_k\in I_k'} \cQ\left(\nabla f_{i_k}(x^{k+1}) - \nabla f_{i_k}(x^k)\right) - \nabla f(x^{k+1}) + \nabla f(x^k)\right\|^2\right]\notag\\
		&&\quad + (1-p)\EE\left[\left\|g^k - \nabla f(x^k)\right\|^2\right].\notag
	\end{eqnarray}
	Next, we use the notation: $\Delta_i^k = \nabla f_{i}(x^{k+1}) - \nabla f_{i}(x^k)$ for $i\in [n]$ and $\Delta^k = \nabla f(x^{k+1}) - \nabla f(x^k)$. These vectors satisfy $\EE\left[\Delta_{i_k}^k \mid x^k,x^{k+1}\right] = \Delta^k$ for all $i_k\in I_{k}'$. Moreover, $\cQ(\Delta_{i_k}^k)$ for $i_k\in I_k'$ are independent random vectors for fixed $x^k$ and $x^{k+1}$. These observations imply
	\begin{eqnarray}
		\EE\left[\|g^{k+1}-\nabla f(x^{k+1})\|^2\right] &=& (1-p)\EE\left[\left\|\frac{1}{r}\sum\limits_{i_k\in I_k'} \left(\cQ(\Delta_{i_k}^k) - \Delta^k\right)\right\|^2\right]+(1-p)\EE\left[\left\|g^k - \nabla f(x^k)\right\|^2\right]\notag\\
		&=& \frac{1-p}{r^2}\EE\left[\sum\limits_{i_k\in I_k'}\left\|\cQ(\Delta_{i_k}^k) - \Delta_{i_k}^k + \Delta_{i_k}^k - \Delta^k\right\|^2\right] + (1-p)\EE\left[\left\|g^k - \nabla f(x^k)\right\|^2\right]\notag\\
		&\overset{\eqref{eq:tower_property},\eqref{eq:variance_decomposition}}{=}& \frac{1-p}{rn}\sum\limits_{i=1}^n\left(\EE\left[\left\|\cQ(\Delta_{i}^k) - \Delta_{i}^k\right\|^2\right] + \EE\left[\left\|\Delta_{i}^k - \Delta^k\right\|^2\right]\right)\notag\\
		&&\quad + (1-p)\EE\left[\left\|g^k - \nabla f(x^k)\right\|^2\right]\notag\\
		&\overset{\eqref{eq:tower_property},\eqref{eq:quantization_def}}{=}& \frac{1-p}{rn}\sum\limits_{i=1}^n\left(\omega\EE\left[\left\|\Delta_{i}^k\right\|^2\right] + \EE\left[\left\|\Delta_{i}^k - \Delta^k\right\|^2\right]\right) + (1-p)\EE\left[\left\|g^k - \nabla f(x^k)\right\|^2\right]\notag\\
		&\overset{\eqref{eq:tower_property},\eqref{eq:variance_decomposition}}{=}& \frac{(1-p)(1+\omega)}{rn}\sum\limits_{i=1}^n\EE\left[\left\|\Delta_{i}^k\right\|^2\right]+ (1-p)\EE\left[\left\|g^k - \nabla f(x^k)\right\|^2\right].\notag
	\end{eqnarray}
	Using $L$-smoothness \eqref{eq:L_smoothness} of $f_i$ together with the tower property \eqref{eq:tower_property}, we get
	\begin{eqnarray}
		\EE\left[\|g^{k+1}-\nabla f(x^{k+1})\|^2\right] &\le& \frac{(1-p)(1+\omega)}{nr}\sum\limits_{i=1}^nL_i^2\EE\left[\|x^{k+1} - x^k\|^2\right] + (1-p)\EE\left[\left\|g^k - \nabla f(x^k)\right\|^2\right]\notag\\
		&=&\frac{(1-p)(1+\omega)L^2}{r}\EE\left[\|x^{k+1}-x^k\|^2\right] + (1-p)\EE\left[\left\|g^k - \nabla f(x^k)\right\|^2\right].\label{eq:non_cvx_pp_technical_2}
	\end{eqnarray}
	Next, we introduce new notation: $\Phi_k = f(x^k) - f_* + \frac{\gamma}{2p}\|g^k - \nabla f(x^k)\|^2$. Using this and inequalities \eqref{eq:non_cvx_pp_technical_1} and \eqref{eq:non_cvx_pp_technical_2}, we establish the following inequality:
	\begin{eqnarray}
		\EE\left[\Phi_{k+1}\right] &\le& \EE\left[f(x^k) - f_* - \frac{\gamma}{2}\|\nabla f(x^k)\|^2 - \left(\frac{1}{2\gamma} - \frac{L}{2}\right)\|x^{k+1}-x^k\|^2 + \frac{\gamma}{2}\|g^k - \nabla f(x^k)\|^2\right]\notag\\
		&&\quad + \frac{\gamma}{2p}\EE\left[\frac{(1-p)(1+\omega)L^2}{r}\|x^{k+1}-x^k\|^2 + (1-p)\left\|g^k - \nabla f(x^k)\right\|^2\right] \notag\\
		&=& \EE\left[\Phi_k\right] - \frac{\gamma}{2}\EE\left[\|\nabla f(x^k)\|^2\right] + \left(\frac{\gamma(1-p)(1+\omega)L^2}{2pr} - \frac{1}{2\gamma} + \frac{L}{2}\right)\EE\left[\|x^{k+1}-x^k\|^2\right]\notag\\
		&\overset{\eqref{eq:gamma_bound_non_cvx_pp_appendix}}{\le}& \EE\left[\Phi_k\right] - \frac{\gamma}{2}\EE\left[\|\nabla f(x^k)\|^2\right],\label{eq:non_cvx_pp_technical_3}
	\end{eqnarray}
	where in the last inequality we use $\frac{\gamma(1-p)(1+\omega)L^2}{2pr} - \frac{1}{2\gamma} + \frac{L}{2} \le 0$ following from \eqref{eq:gamma_bound_non_cvx_pp_appendix}. Summing up inequalities \eqref{eq:non_cvx_finite_sums_technical_3} for $k=0,1,\ldots,K-1$ and rearranging the terms, we derive
	\begin{eqnarray}
		\frac{1}{K}\sum\limits_{k=0}^{K-1}\EE\left[\|\nabla f(x^k)\|^2\right] &\le& \frac{2}{\gamma K}\sum\limits_{k=0}^{K-1}\left(\EE[\Phi_k]-\EE[\Phi_{k+1}]\right) = \frac{2\left(\EE[\Phi_0]-\EE[\Phi_{K}]\right)}{\gamma K} = \frac{2\Delta_0}{\gamma K},\notag
	\end{eqnarray}
	since $g^0 = \nabla f(x^0)$ and $\Phi_{k+1} \ge 0$. Finally, using the tower property \eqref{eq:tower_property} and the definition of $\hat x^K$, we obtain \eqref{eq:main_res_non_cvx_pp_appendix} that implies \eqref{eq:main_res_2_non_cvx_pp_appendix} and \eqref{eq:main_res_4_non_cvx_pp_appendix}.
\end{proof}

\begin{corollary}[Corollary~\ref{cor:main_result_non_cvx_pp}]\label{cor:main_result_non_cvx_pp_appendix}
	Let the assumptions of Theorem~\ref{thm:main_result_non_cvx_pp} hold and $p = \frac{\zeta_{\cQ}r}{dn}$, where $r \le n$ and $\zeta_{\cQ}$ is the expected density of the quantization (see Def.~\ref{def:quantization}). If 
	\begin{equation*}
		\gamma \le \frac{1}{L\left(1 + \sqrt{\frac{1+\omega}{r}\left(\frac{dn}{\zeta_{\cQ}r}-1\right)}\right)},
	\end{equation*}
	then \algname{PP-MARINA} requires 
	\begin{equation*}
		K = \cO\left(\frac{\Delta_0 L}{\varepsilon^2}\left(1 + \sqrt{\frac{1+\omega}{r}\left(\frac{dn}{\zeta_{\cQ}r}-1\right)}\right)\right)
	\end{equation*}
	iterations/communication rounds to achieve $\EE[\|\nabla f(\hat x^K)\|^2] \le \varepsilon^2$, and the expected total communication cost is
	\begin{equation*}
		\cO\left(dn+\frac{\Delta_0 L}{\varepsilon^2}\left(\zeta_{\cQ}r + \sqrt{(1+\omega)\zeta_{\cQ}\left(dn-\zeta_{\cQ}r\right)}\right)\right)
	\end{equation*}
	 under an assumption that the communication cost is proportional to the number of non-zero components of transmitted vectors from workers to the server.
\end{corollary}
\begin{proof}[Proof of Corollary~\ref{cor:main_result_non_cvx_pp}]
	The choice of $p = \frac{\zeta_{\cQ}r}{dn}$ implies
	\begin{eqnarray*}
		\frac{1-p}{p} &=& \frac{dn}{\zeta_{\cQ}r}-1,\\
		pdn + (1-p)\zeta_{\cQ}r &\le& \zeta_{\cQ}r + \left(1 - \frac{\zeta_{\cQ}r}{dn}\right)\cdot\zeta_{\cQ}r \le 2\zeta_{\cQ}r.
	\end{eqnarray*}	
	Plugging these relations in \eqref{eq:gamma_bound_non_cvx_pp_appendix}, \eqref{eq:main_res_2_non_cvx_pp_appendix}, and \eqref{eq:main_res_4_non_cvx_pp_appendix}, we get that if
	\begin{equation*}
		\gamma \le \frac{1}{L\left(1 + \sqrt{\frac{1+\omega}{r}\left(\frac{dn}{\zeta_{\cQ}r}-1\right)}\right)},
	\end{equation*}
	then \algname{PP-MARINA} requires 
	\begin{eqnarray*}
		K &=& \cO\left(\frac{\Delta_0 L}{\varepsilon^2}\left(1 + \sqrt{\frac{(1-p)(1+\omega)}{pr}}\right)\right)\\
		&=& \cO\left(\frac{\Delta_0 L}{\varepsilon^2}\left(1 + \sqrt{\frac{1+\omega}{r}\left(\frac{dn}{\zeta_{\cQ}r}-1\right)}\right)\right)
	\end{eqnarray*}
	iterations/communication rounds in order to achieve $\EE[\|\nabla f(\hat x^K)\|^2] \le \varepsilon^2$, and the expected total communication cost is
	\begin{eqnarray*}
		dn + K(pdn + (1-p)\zeta_{\cQ}r) &=&  \cO\left(dn+\frac{\Delta_0 L}{\varepsilon^2}\left(1 + \sqrt{\frac{(1-p)(1+\omega)}{pr}}\right)(pdn + (1-p)\zeta_{\cQ}r)\right)\\
		&=&\cO\left(dn+\frac{\Delta_0 L}{\varepsilon^2}\left(\zeta_{\cQ}r + \sqrt{(1+\omega)\zeta_{\cQ}\left(dn-\zeta_{\cQ}r\right)}\right)\right)
	\end{eqnarray*}
	 under an assumption that the communication cost is proportional to the number of non-zero components of transmitted vectors from workers to the server.
\end{proof}

\subsection{Convergence Results Under Polyak-{\L}ojasiewicz condition}\label{sec:proof_of_thm_pl_pp}
In this section, we provide an analysis of \algname{PP-MARINA} under Polyak-{\L}ojasiewicz condition.
\begin{theorem}\label{thm:main_result_pl_pp_appendix}
	Let Assumptions~\ref{as:lower_bound},~\ref{as:L_smoothness}~and~\ref{as:pl_condition} be satisfied and 
	\begin{equation}
		\gamma \le \min\left\{\frac{1}{L\left(1 + \sqrt{\frac{2(1-p)(1+\omega)}{pr}}\right)}, \frac{p}{2\mu}\right\},\label{eq:gamma_bound_pl_pp_appendix}
	\end{equation}
	where $L^2 = \frac{1}{n}\sum_{i=1}^nL_i^2$. Then after $K$ iterations of \algname{PP-MARINA}, we have
	\begin{equation}
		\EE\left[f(x^K) - f(x^*)\right] \le (1-\gamma\mu)^K\Delta_0, \label{eq:main_res_pl_pp_appendix}
	\end{equation}
	where $\Delta_0 = f(x^0)-f(x^*)$. That is, after
	\begin{equation}
		K = \cO\left(\max\left\{\frac{1}{p},\frac{L}{\mu}\left(1 + \sqrt{\frac{(1-p)(1+\omega)}{pr}}\right)\right\}\log\frac{\Delta_0}{\varepsilon}\right) \label{eq:main_res_2_pl_pp_appendix}
	\end{equation}
	iterations \algname{PP-MARINA} produces such a point $x^K$ that $\EE[f(x^K) - f(x^*)] \le \varepsilon$.
	Moreover, under an assumption that the communication cost is proportional to the number of non-zero components of transmitted vectors from workers to the server, we have that the expected total communication cost (for all workers) equals
	\begin{equation}
		dn + K(pdn + (1-p)\zeta_{\cQ}r) =  \cO\left(dn+\max\left\{\frac{1}{p},\frac{L}{\mu}\left(1 + \sqrt{\frac{(1-p)(1+\omega)}{pr}}\right)\right\}(pdn + (1-p)\zeta_{\cQ}r)\log\frac{\Delta_0}{\varepsilon}\right),\label{eq:main_res_4_pl_pp_appendix}
	\end{equation}
	where $\zeta_{\cQ}$ is the expected density of the quantization (see Def.~\ref{def:quantization}).
\end{theorem}
\begin{proof}
	The proof is very similar to the proof of Theorem~\ref{thm:main_result_non_cvx_pp}. From Lemma~\ref{lem:lemma_2_page} and P{\L} condition we have
	\begin{eqnarray}
		\EE[f(x^{k+1}) - f(x^*)] &\le& \EE[f(x^k) - f(x^*)] - \frac{\gamma}{2}\EE\left[\|\nabla f(x^k)\|^2\right] - \left(\frac{1}{2\gamma} - \frac{L}{2}\right)\EE\left[\|x^{k+1}-x^k\|^2\right]\notag\\
		&&\quad + \frac{\gamma}{2}\EE\left[\|g^k - \nabla f(x^k)\|^2\right]\notag\\
		&\overset{\eqref{eq:pl_condition}}{\le}& (1-\gamma\mu)\EE\left[f(x^k) - f(x^*)\right] - \left(\frac{1}{2\gamma} - \frac{L}{2}\right)\EE\left[\|x^{k+1}-x^k\|^2\right] + \frac{\gamma}{2}\EE\left[\|g^k - \nabla f(x^k)\|^2\right]. \notag
	\end{eqnarray}
	Using the same arguments as in the proof of \eqref{eq:non_cvx_pp_technical_2}, we obtain
	\begin{eqnarray}
		\EE\left[\|g^{k+1}-\nabla f(x^{k+1})\|^2\right] &\le& \frac{(1-p)(1+\omega)L^2}{r}\EE\left[\|x^{k+1}-x^k\|^2\right] + (1-p)\EE\left[\left\|g^k - \nabla f(x^k)\right\|^2\right].\notag
	\end{eqnarray}
	Putting all together, we derive that the sequence $\Phi_k = f(x^k) - f(x^*) + \frac{\gamma}{p}\|g^k - \nabla f(x^k)\|^2$ satisfies
	\begin{eqnarray}
		\EE\left[\Phi_{k+1}\right] &\le& \EE\left[(1-\gamma\mu)(f(x^k) - f(x^*)) - \left(\frac{1}{2\gamma} - \frac{L}{2}\right)\|x^{k+1}-x^k\|^2 + \frac{\gamma}{2}\|g^k - \nabla f(x^k)\|^2\right]\notag\\
		&&\quad + \frac{\gamma}{p}\EE\left[\frac{(1-p)(1+\omega)L^2}{r}\|x^{k+1}-x^k\|^2 + (1-p)\left\|g^k - \nabla f(x^k)\right\|^2 \right] \notag\\
		&=& \EE\left[(1-\gamma\mu)(f(x^k) - f(x^*)) + \left(\frac{\gamma}{2} + \frac{\gamma}{p}(1-p)\right)\left\|g^k - \nabla f(x^k)\right\|^2\right]\notag\\
		&&\quad + \left(\frac{\gamma(1-p)(1+\omega)L^2}{pr} - \frac{1}{2\gamma} + \frac{L}{2}\right)\EE\left[\|x^{k+1}-x^k\|^2\right]\notag\\
		&\overset{\eqref{eq:gamma_bound_pl_pp_appendix}}{\le}& (1-\gamma\mu)\EE[\Phi_k],\notag
	\end{eqnarray}
	where in the last inequality we use $\frac{\gamma(1-p)(1+\omega)L^2}{pr} - \frac{1}{2\gamma} + \frac{L}{2} \le 0$ and $\frac{\gamma}{2} + \frac{\gamma}{p}(1-p) \le (1-\gamma\mu)\frac{\gamma}{p}$ following from \eqref{eq:gamma_bound_pl_pp_appendix}. Unrolling the recurrence and using $g^0 = \nabla f(x^0)$, we obtain 
	\begin{eqnarray*}
		\EE\left[f(x^{K}) - f(x^*)\right] \le \EE[\Phi_{K}] &\le& (1-\gamma\mu)^{K}\Phi_0 = (1-\gamma\mu)^{K}(f(x^0) - f(x^*))
	\end{eqnarray*}
	that implies \eqref{eq:main_res_2_pl_pp_appendix} and \eqref{eq:main_res_4_pl_pp_appendix}.
\end{proof}

\begin{corollary}\label{cor:main_result_pl_pp_appendix}
	Let the assumptions of Theorem~\ref{thm:main_result_pl_pp_appendix} hold and $p = \frac{\zeta_{\cQ}r}{dn}$, where $r \le n$ and $\zeta_{\cQ}$ is the expected density of the quantization (see Def.~\ref{def:quantization}). If 
	\begin{equation*}
		\gamma \le \min\left\{\frac{1}{L\left(1 + \sqrt{\frac{2(1+\omega)}{r}\left(\frac{dn}{\zeta_{\cQ}r}-1\right)}\right)}, \frac{p}{2\mu}\right\},
	\end{equation*}
	then \algname{PP-MARINA} requires 
	\begin{equation*}
		K = \cO\left(\max\left\{\frac{dn}{\zeta_{\cQ}r}\frac{L}{\mu}\left(1 + \sqrt{\frac{1+\omega}{r}\left(\frac{dn}{\zeta_{\cQ}r}-1\right)}\right)\right\}\log\frac{\Delta_0}{\varepsilon}\right)
	\end{equation*}
	iterations/communication rounds to achieve $\EE[f(x^K) - f(x^*)] \le \varepsilon$, and the expected total communication cost is
	\begin{equation*}
		\cO\left(dn+\max\left\{dn,\frac{L}{\mu}\left(\zeta_{\cQ}r + \sqrt{(1+\omega)\zeta_{\cQ}\left(dn-\zeta_{\cQ}r\right)}\right)\right\}\log\frac{\Delta_0}{\varepsilon}\right)
	\end{equation*}
	 under an assumption that the communication cost is proportional to the number of non-zero components of transmitted vectors from workers to the server.
\end{corollary}
\begin{proof}
	The choice of $p = \frac{\zeta_{\cQ}r}{dn}$ implies
	\begin{eqnarray*}
		\frac{1-p}{p} &=& \frac{dn}{\zeta_{\cQ}r}-1,\\
		pdn + (1-p)\zeta_{\cQ}r &\le& \zeta_{\cQ}r + \left(1 - \frac{\zeta_{\cQ}r}{dn}\right)\cdot\zeta_{\cQ}r \le 2\zeta_{\cQ}r.
	\end{eqnarray*}	
	Plugging these relations in \eqref{eq:gamma_bound_pl_pp_appendix}, \eqref{eq:main_res_2_pl_pp_appendix}, and \eqref{eq:main_res_4_pl_pp_appendix}, we get that if
	\begin{equation*}
		\gamma \le \min\left\{\frac{1}{L\left(1 + \sqrt{\frac{2(1+\omega)}{r}\left(\frac{dn}{\zeta_{\cQ}r}-1\right)}\right)}, \frac{p}{2\mu}\right\},
	\end{equation*}
	then \algname{PP-MARINA} requires 
	\begin{eqnarray*}
		K &=& \cO\left(\max\left\{\frac{1}{p},\frac{L}{\mu}\left(1 + \sqrt{\frac{(1-p)(1+\omega)}{pr}}\right)\right\}\log\frac{\Delta_0}{\varepsilon}\right)\\
		&=& \cO\left(\max\left\{\frac{dn}{\zeta_{\cQ}r}\frac{L}{\mu}\left(1 + \sqrt{\frac{1+\omega}{r}\left(\frac{dn}{\zeta_{\cQ}r}-1\right)}\right)\right\}\log\frac{\Delta_0}{\varepsilon}\right)
	\end{eqnarray*}
	iterations/communication rounds to achieve $\EE[f(x^K)-f(x^*)] \le \varepsilon$, and the expected total communication cost is
	\begin{eqnarray*}
		dn + K(pdn + (1-p)\zeta_{\cQ}r) &=&  \cO\left(dn+\max\left\{\frac{1}{p},\frac{L}{\mu}\left(1 + \sqrt{\frac{(1-p)(1+\omega)}{pr}}\right)\right\}(pdn + (1-p)\zeta_{\cQ}r)\log\frac{\Delta_0}{\varepsilon}\right)\\
		&=&\cO\left(dn+\max\left\{dn,\frac{L}{\mu}\left(\zeta_{\cQ}r + \sqrt{(1+\omega)\zeta_{\cQ}\left(dn-\zeta_{\cQ}r\right)}\right)\right\}\log\frac{\Delta_0}{\varepsilon}\right)
	\end{eqnarray*}
	 under an assumption that the communication cost is proportional to the number of non-zero components of transmitted vectors from workers to the server.
\end{proof}

\end{document}